\documentclass{article}

\usepackage{arxiv}

\usepackage{amsmath, amsthm, amssymb, graphicx, hyperref, etoolbox, longtable, colonequals}

\makeatletter
\let\@nodottedtocline\@dottedtocline
\patchcmd{\@nodottedtocline}{\hbox{.}}{\hbox{}}{}{}
\patchcmd{\@nodottedtocline}{\normalcolor #5}{\normalcolor}{}{}
\newcommand*\l@sectionsubtitle{\@nodottedtocline{1}{0em}{1.5em}}
\makeatother

\usepackage{imakeidx}
\usepackage[all]{xy}

\makeindex[title=Index,columns=3]
\indexsetup{noclearpage}

\usepackage{amsmath}

\DeclareFontFamily{U}{dmjhira}{}
\DeclareFontShape{U}{dmjhira}{m}{n}{ <-> dmjhira }{}

\DeclareRobustCommand{\yo}{\text{\usefont{U}{dmjhira}{m}{n}\symbol{"48}}}

\usepackage{tikz}
\usetikzlibrary{calc,arrows.meta}

\newtheorem{assumption}{Theorem}

\theoremstyle{definition}

\newcommand\cofib\rightarrowtail

\newcommand\mdel[1]{}

\renewcommand\geq\geqslant
\renewcommand\leq\leqslant

\usepackage[utf8]{inputenc} 
\usepackage[T1]{fontenc}    
\usepackage{hyperref}       
\usepackage{url}            
\usepackage{booktabs}       
\usepackage{amsfonts}       
\usepackage{nicefrac}       
\usepackage{microtype}      
\usepackage{lipsum}		
\usepackage{graphicx}
\usepackage{natbib}
\usepackage{doi}
\usepackage{tikz-cd}
\usepackage{mathtools}

\usepackage{amsmath}
\usepackage{amssymb}
\usepackage{amsthm}

\newtheorem{theorem}{Theorem}
\newtheorem{definition}{Definition}
\newtheorem{lemma}{Lemma}
\newtheorem{example}{Example}

\usepackage[boxruled]{algorithm2e}
\usepackage{algorithmic}

\usepackage{multirow}
\usepackage{booktabs}
\usepackage{balance}
\usepackage{subfig}

\newcommand{\bigCI}{\mathrel{\text{\scalebox{1.07}{$\perp\mkern-10mu\perp$}}}}

\newcommand{\CI}{\mathrel{\perp\mspace{-10mu}\perp}}

\title{Universal Imitation Games\thanks{This paper is a condensed draft of a forthcoming book. } }

\author{ Sridhar Mahadevan \\
	Adobe Research and University of Massachusetts, Amherst\\
	\texttt{smahadev@adobe.com, mahadeva@umass.edu}
}

\hypersetup{
pdftitle={A template for the arxiv style},
pdfsubject={q-bio.NC, q-bio.QM},
pdfauthor={David S.~Hippocampus, Elias D.~Striatum},
pdfkeywords={First keyword, Second keyword, More},
}

\begin{document}
\maketitle

\begin{abstract}
In $1950$, Alan Turing proposed a framework called an {\em imitation game} in which the participants are to be classified {\tt Human or Machine} solely from  natural language interactions. Using mathematics largely developed since Turing -- category theory -- we investigate a broader class of {\em universal imitation games} (UIGs). Choosing a category means defining a collection of objects and a collection of composable arrows between each pair of objects that represent ``measurement probes"  for solving UIGs. The theoretical foundation of our paper rests on two celebrated results by Yoneda.  The first, called the Yoneda Lemma, discovered in 1954 -- the year of Turing's death -- shows that objects in categories can be identified up to isomorphism solely with measurement probes defined by composable arrows. Yoneda embeddings are universal representers of objects in categories. A simple yet general solution to the static UIG problem, where the participants are not changing during the interactions,  is to determine if the Yoneda embeddings are (weakly) isomorphic. A {\em universal property} in category theory is defined by an {\em initial} or {\em final} object. A second foundational result of Yoneda from 1960 defines initial objects called {\em coends} and final objects called {\em ends}, which yields a categorical ``integral calculus" that unifies probabilistic generative models, distance-based kernel, metric and optimal transport models, as well as topological manifold representations. When participants adapt during interactions, we study two special cases: in {\em dynamic UIGs}, ``learners" imitate ``teachers".  We contrast the initial object  framework of {\em passive learning from observation} over well-founded sets using inductive inference -- extensively studied by Gold, Solomonoff, Valiant, and Vapnik -- with the final object framework of {\em coinductive inference} over non-well-founded sets and universal coalgebras, which formalizes  learning from {\em active experimentation} using causal inference or reinforcement learning. We define a category-theoretic notion of minimum description length or Occam's razor based on final objects in coalgebras.   Finally, we explore {\em evolutionary UIGs}, where a society of participants is playing a large-scale imitation game. Participants in evolutionary UIGs can go extinct from ``birth-death" evolutionary processes that model how novel technologies or ``mutants" disrupt previously stable equilibria.  Given the rapidly rising energy costs of playing imitation games on classical computers, it seems  likely that tomorrow's imitation games may have to played on non-classical computers. We end with a brief discussion of how our categorical framework extends to imitation games on quantum computers. 
\end{abstract}


\bigskip 

\begin{quote}
``I propose to consider the question, `Can machines think'?" -- {\em Alan Turing, Mind, Volume LIX, Issue 236, October 1950, Pages 433–460.}
\end{quote}

\newpage 

\tableofcontents

\newpage

\section{Overview of the Paper} 

\begin{figure}[h]
\centering
\includegraphics[scale=.4]{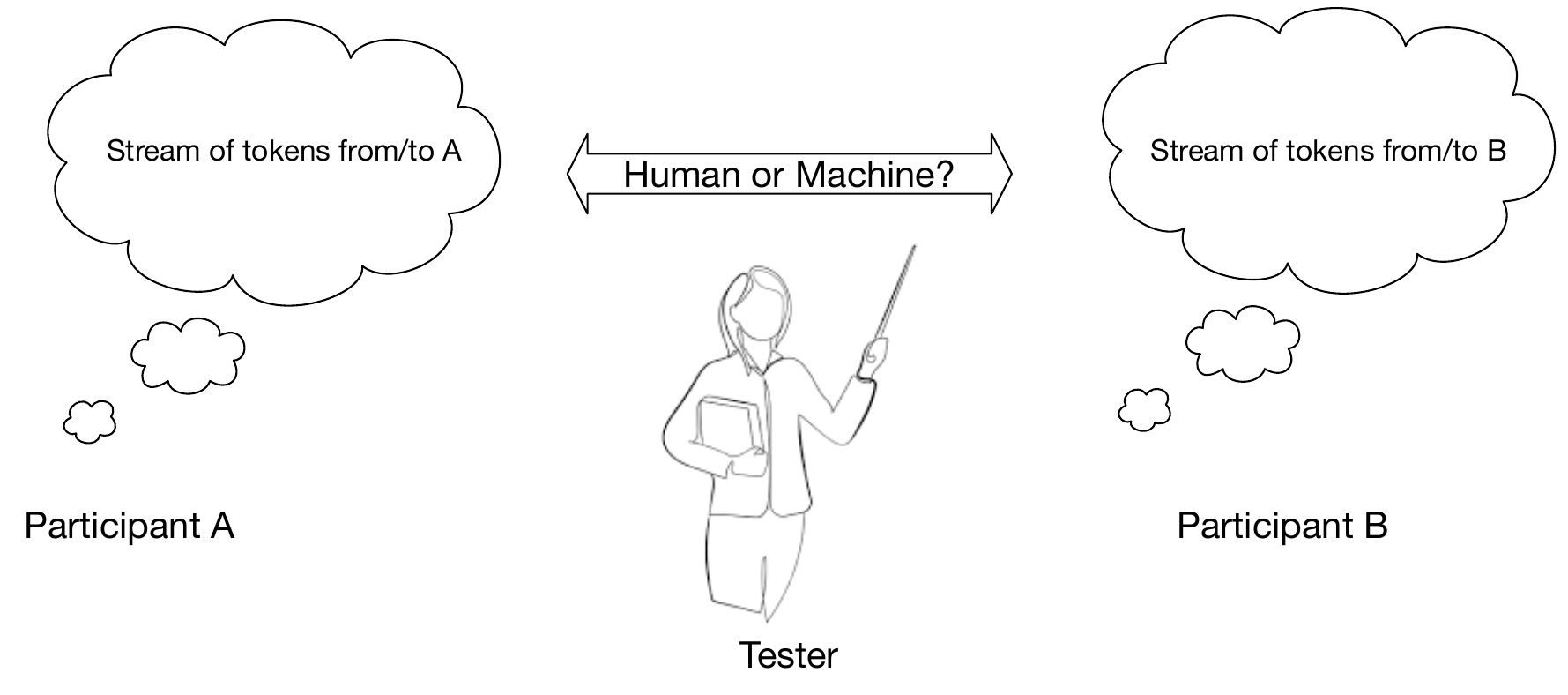}
\caption{Alan Turing proposed imitation games as a way to pose the problem of whether machines could think.} 
\label{img} 
\end{figure}

\subsection{Turing's imitation game}

The field of artificial intelligence (AI) has existed for about 7 decades, if we define its origin as beginning with Turing's famous paper \cite{turing}.  To answer the question ``Can machines think?", Turing suggested it was necessary first to define the words ``machine" and ``think". He proposed an {\em imitation game}, as shown in Figure~\ref{img}, as a concrete way to frame the problem.  Quoting from Turing's paper: 

\begin{quote}
    The new form of the problem can be described in terms of a game which we call the ‘imitation game’. 
    
    It is played with three people, a man (A), a woman (B), and an interrogator (C) who may be of either sex. The interrogator stays in a room apart from the other two.
    
    The object of the game for the interrogator is to determine which of the other two is the man and which is the woman. He knows them by labels X and Y, and at the end of the game he says either ‘X is A and Y is B’ or ‘X is B and Y is A’. 
\end{quote}

Turing proposes replacing one of the human participants with a machine, thereby framing the original question of whether machines can think by the more concrete version of being able to tell from interactions whether one is conversing with a human or a machine: 

\begin{quote}
    We now ask the question, ‘What will happen when a machine takes the part of A in this game?’ Will the interrogator decide wrongly as often when the game is played like this as he does when the game is played between a man and a woman? These questions replace our original, ‘Can machines think?’
\end{quote} 

The imitation game, or what is now more popularly called the {\em Turing Test}, has become the mainstay of the field of AI for the ensuing seven decades since Turing wrote his original paper, prompting an immense amount of attention that would be impossible to survey within this paper. The study of imitation games is no longer an academic pursuit: the recent success of generative AI  is already being felt in terms of its economic  impact around the world. Millions of humans are now playing imitation games with generative AI systems, and there is an alarming increase in the use of ``deepfakes" for nefarious purposes.  In response, a growing number of software tools are being released that attempt to solve limited classes of imitation games, such as {\tt gptzero}. \footnote{{\tt gptzero} can be accessed at {\tt https://gptzero.me}.} We seek to develop a framework for UIGs that applies both to the current generation of generative AI systems being developed on classical computers that are ``Turing's machine" models, as well as the next generation of quantum computers that build on more exotic forms of interaction in braided categories \cite{Coecke_Kissinger_2017}.

Two recent articles are noteworthy. The first article by Bernardo Goncalves \cite{historyturingtest} presents a fascinating historical reconstruction of Turing's test from archival sources. The second article by Terry Sejnowski \cite{DBLP:journals/neco/Sejnowski23} proposes an interesting alternative called a ``Reverse Turing Test" based on the recent successes in building large language models (LLMs), which has resulted in possibly the largest ever ``natural" experiment conducted on imitation games with literally millions of humans conversing with chatbots like Open AI's chatGPT. The literature on applications of LLMs is already beyond comprehension, and a review would be impossible here. \cite{DBLP:journals/aim/Adesso23} explores how LLMs can shape scientific discovery in the coming decades. \cite{DBLP:journals/inroads/Jacques23} examines how the teaching of basic computer science will have to be revised in light of the new capabilities of LLMs. 

\subsection{Universal Imitation Games}

\begin{figure}[h]
\centering
\includegraphics[scale=.4]{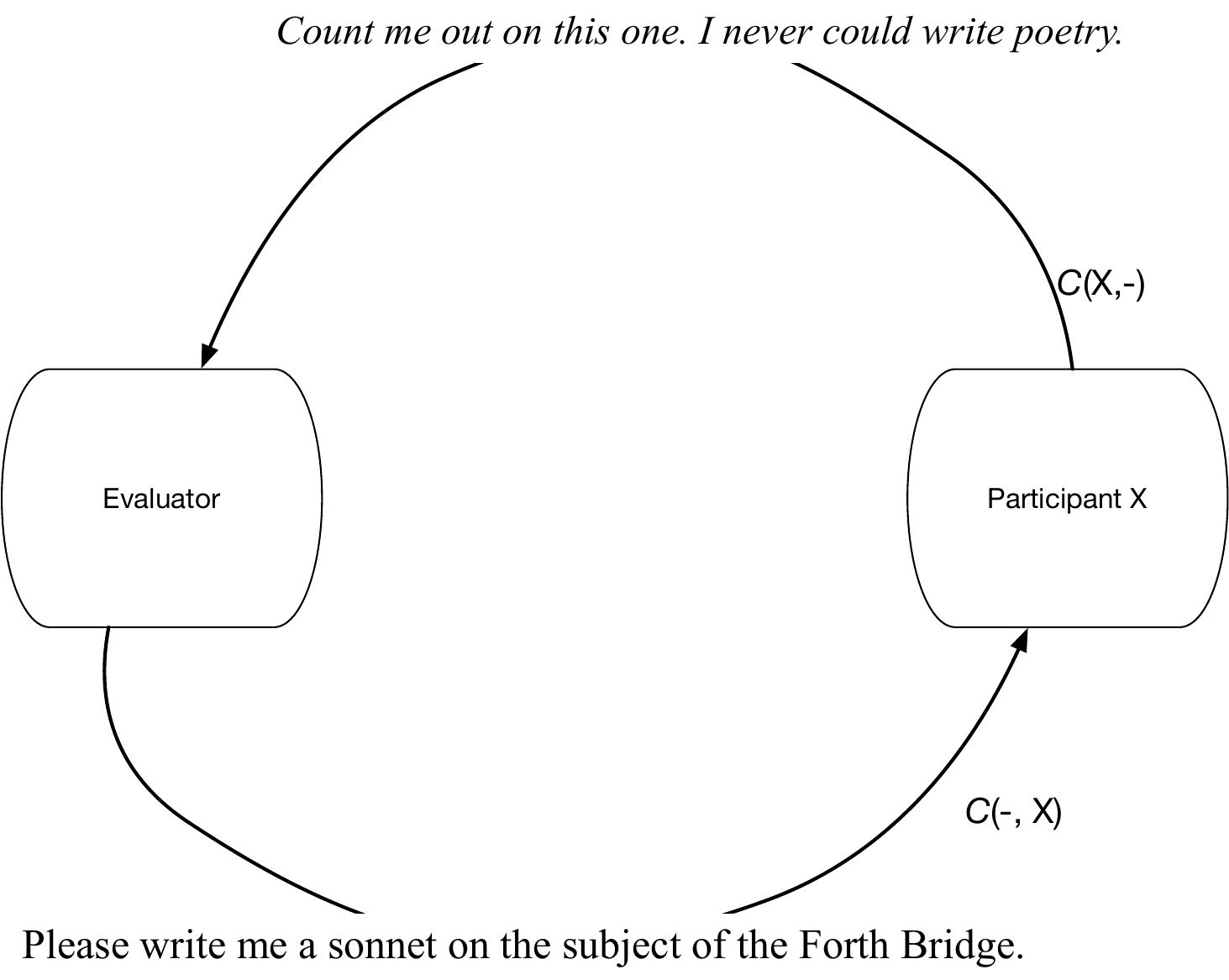}
\caption{In Turing's imitation game, the identity of a participant is determined by a series of questions and answers (the specific questions above are from Turing's original paper \cite{turing}.  We model interactions as ${\cal C}(-, x): {\cal C}^{op} \rightarrow {\cal D}$ and ${\cal C}(x, -): {\cal C} \rightarrow {\cal D}$ as {\em contravariant} and {\em covariant} functors, respectively,  from a category ${\cal C}$ into another category ${\cal D}$ (such as {\bf Sets} or any other enriched category). Two key results by Yoneda provide the theoretical foundation for our paper. The Yoneda Lemma shows objects can be defined purely in terms of these contravariant and covariant functors. In addition, Yoneda investigated bivalent functors $F: {\cal C}^{op} \times {\cal C} \rightarrow {\cal D}$ that combine both contravariant and covariant actions, and proposed a categorical ``integral calculus" over coends and ends, which reveal deep similarities between generative probabilistic models, distance-based models, kernel methods, optimal transport and topological representations.} 
\label{ttquestions} 
\end{figure}

Based largely on mathematics developed since Turing -- in particular, category theory \cite{maclane:71} -- we investigate a broader class of {\em universal imitation games} (UIGs) that captures many other forms of interactions. Category theory is exceptionally well-suited to the study of imitation games as it constitutes a ``universal" language  for defining interactions: choosing a category means defining a collection of objects and a collection of composable arrows between each pair of objects that represent ``measurement probes"  for solving UIGs.
Figure~\ref{ttquestions} illustrates our approach to modeling (universal) imitation games, where interactions define {\em contravariant} functors ${\cal C}(-,x)$ or {\em covariant} functors ${\cal C}(x, -)$, and the solution to an imitation game corresponds to (weak) isomorphism in a category ${\cal C}$. The Python package DisCoPy \cite{de_Felice_2021} is a software implementation for computing with string diagrams in category theory, and could be used to implement some of the ideas in this paper. \footnote{Categories for language are described in {\tt https://docs.discopy.org/en/main/notebooks/21-05-03-tallcat.html}.} \cite{DBLP:phd/ethos/Toumi22} has a detailed description of natural language processing on quantum computers.  We will discuss UIGs over quantum computers at the end of the paper in Section~\ref{uigquantum}.  

The theoretical foundation of our paper rests on two fundamental insights developed by Yoneda (see Figure~\ref{yonedathms}).   The celebrated Yoneda Lemma \cite{maclane:71} asserts that objects in a category ${\cal C}$ can be defined purely in terms of their interactions with other objects. This interaction is modeled by {\em contravariant} or {\em covariant} functors: 

\[ {\cal C}(-, x): {\cal C}^{op} \rightarrow {\bf Sets},  \ \ \  {\cal C}(x, -): {\cal C} \rightarrow {\bf Sets} \]

The {\em Yoneda embedding} $x \rightarrow {\cal C}(-, x)$ is sometimes denoted as $\yo(x)$ for the Japanese Hiragana symbol for {\tt yo}, serves as a {\em universal representer}, and generalizes many other similar ideas in machine learning, such as representers $K(-, x)$ in kernel methods \cite{kernelbook} and representers of causal information \cite{DBLP:journals/entropy/Mahadevan23}. There are many variants of the Yoneda Lemma, including versions that map the functors ${\cal C}(-, x)$ and ${\cal C}(x, -)$ into an {\em enriched} category. In particular, \cite{bradley:enriched-yoneda-llms} contains an extended discussion of the use of an enriched Yoneda Lemma to model natural language interactions that result from using a large language model. 

The second major insight from Yoneda \cite{yoneda-end} is based on a powerful concept of the {\em coend} and {\em end} of a {\em bifunctor} $F: {\cal C}^{op} \times {\cal C} \rightarrow {\cal D}$ that combines both a {\em contravariant} and a {\em covariant} action.  It can be shown that many approaches to solving imitation games, such as building a probabilistic generative model of participants, or using distances in some metric space, correspond to initial or terminal objects in a category of wedges, defined by objects that represent bifunctors, and the arrows are dinatural transformations. These initial or terminal objects correspond to coends and ends. For example, dimensionality reduction methods such as UMAP \cite{umap} are based mapping a dataset (say of interactions with a participant) into a topological space, which can be shown to define a coend object. Dually, building a probabilistic generative model of a participant essentially defines an end object, as probabilities can be shown to be codensity monads \cite{Avery_2016} that are in fact defined by ends. In summary, coends represent geometric ways to solve UIGs, whereas ends represent probabilistic approaches. Bifunctors $F: {\cal C}^{op} \times {\cal C} \rightarrow {\cal D} $ can be used to construct universal representers of distance functions in generalized metric spaces leading to a ``metric Yoneda Lemma" \cite{BONSANGUE19981}. 

\begin{figure}[h]
\centering
\includegraphics[scale=.4]{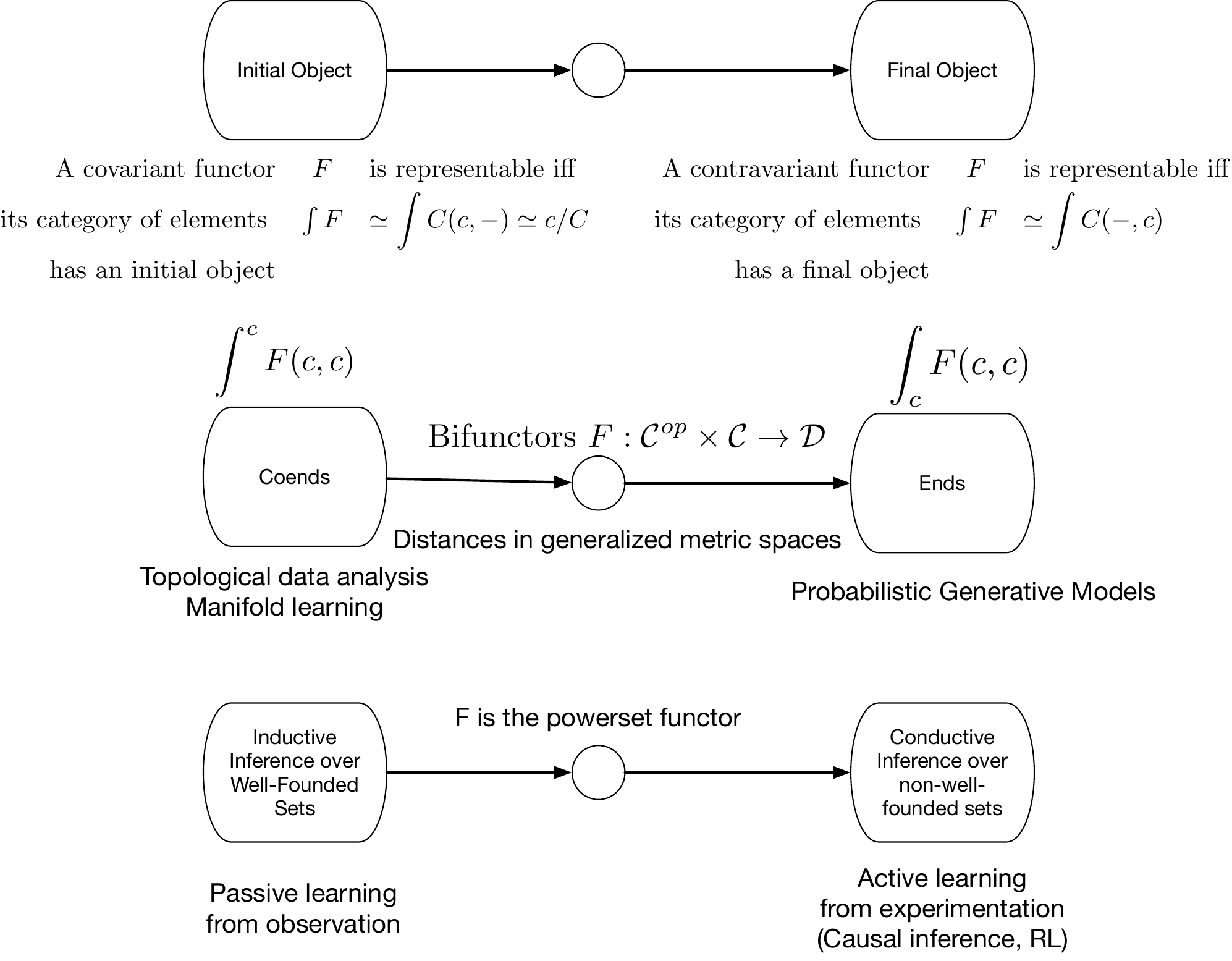}
\caption{The theoretical foundation of our approach to universal imitation games is based on two celebrated results of Yoneda. The first (top row) shows that Yoneda embeddings $\yo(x) = C(-, x)$ are universal representers of objects in a category. The second (middle row) is based on Yoneda's categorical ``integral calculus" using coends and ends that unify diverse approaches to solving imitation games studied in AI over the past six decades. The bottom row shows that universal properties defined by initial and final objects provide a unified way to characterize approaches to dynamic UIGs using passive learning (inductive inference) as well as active learning (causal inference and reinforcement learning).} 
\label{yonedathms} 
\end{figure}

\begin{figure}[h]
\centering
\includegraphics[scale=.4]{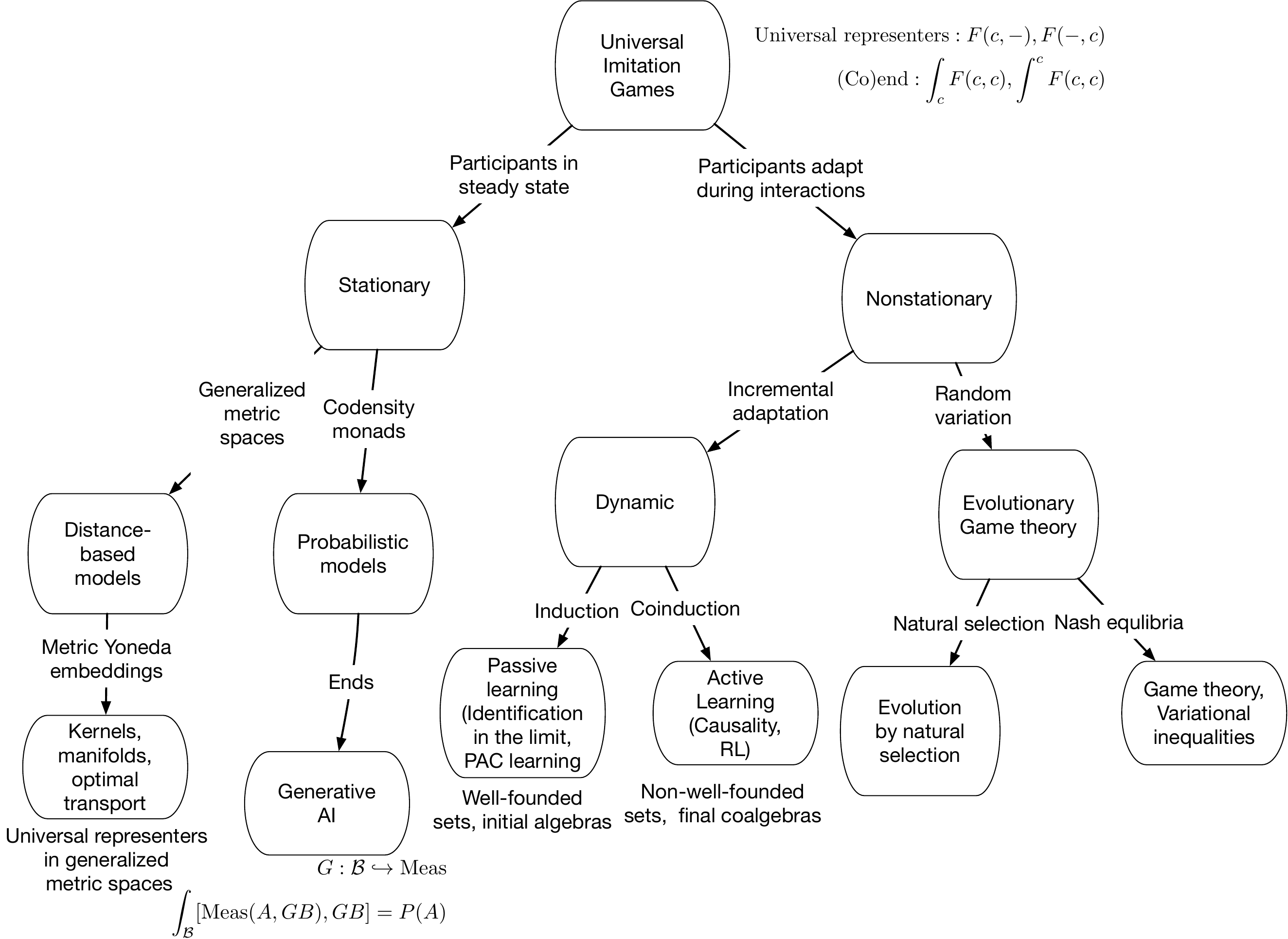}
\caption{A classification of universal imitation games based on the characteristics of the participants and the framework employed to discriminate among them.} 
\label{uigclasses} 
\end{figure}

Figure~\ref{uigclasses} depicts a classification of universal imitation games (UIGs) that will be the primary focus of our paper. In Turing's original model, discussed above, the participants were assumed to be {\em static}, that is, they were not adapting during the course of the interactions. In the case that participants are static, we can try to distinguish between them using a large variety of methods, ranging from distance-based methods like kernel methods \cite{kernelbook} or  {\em optimal transport} \cite{ot}, or {\em density-based } methods such as constructing a generative model. The fundamental ideas developed by Yoneda, in particular Yoneda embeddings $\yo(x) = {\cal C}(-, x)$ and (co)ends play a major role in our paper. For example, \cite{Avery_2016} showed that the set of probability distributions on a set can be defined as a codensity monad, which is mathematically equivalent to an end of a bifunctor defined over a functor $G: {\cal B} \hookrightarrow \mbox{Meas}$, where ${\cal B}$ is a category  consisting of all convergent sequences $[0,1]$ with affine maps, and $\mbox{Meas}$ is the category of all measurable spaces. Probabilities are essentially final objects, and in contrast, topological embeddings of the type produced by dimensionality reduction methods like UMAP \cite{umap} correspond to coends. Another theme of our paper is that we model generative models categorically as universal coalgebras \cite{jacobs:book,rutten2000universal,SOKOLOVA20115095}. 

When participants adapt during interactions, we study two special cases: in {\em dynamic UIGs}, ``learners" imitate ``teachers".  We contrast the initial object  framework of {\em passive learning from observation} over well-founded sets using inductive inference -- extensively studied by Gold, Solomonoff, Valiant, and Vapnik -- with the final object framework of {\em coinductive inference} over non-well-founded sets and universal coalgebras, which formalizes  learning from {\em active experimentation} using causal inference or reinforcement learning. We define a category-theoretic notion of minimum description length or Occam's razor based on final objects in coalgebras.  

\begin{figure}[h]
\centering
\includegraphics[scale=.4]{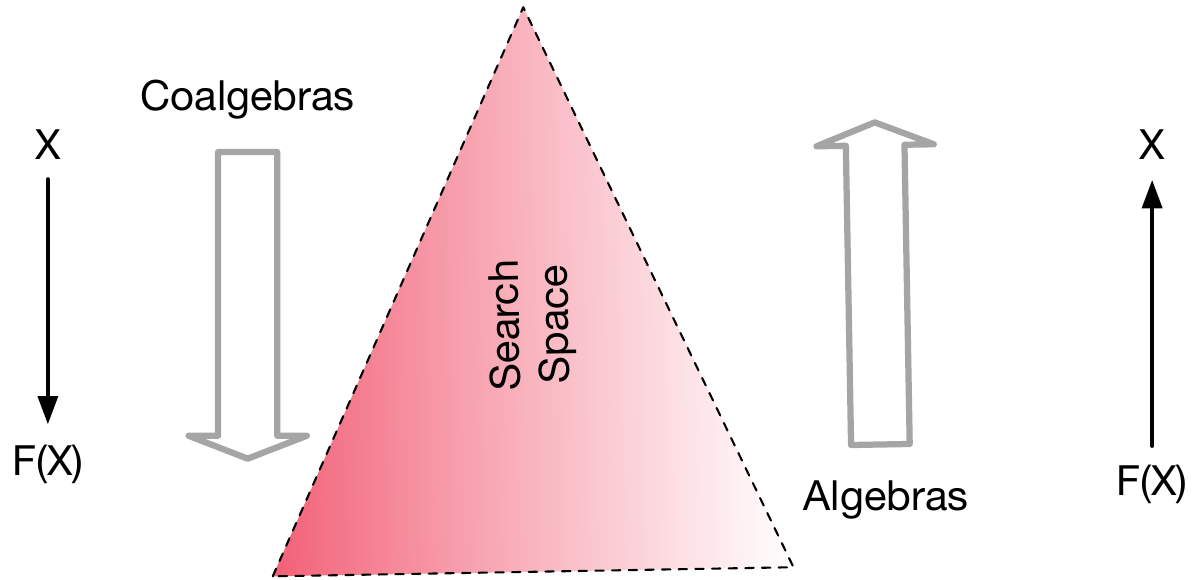}
\caption{In the study of dynamic UIGs, we contrast the approach of inductive inference, which we associate with initial algebras over well-founded sets, with the framework of coinductive inference, which relates to final coalgebras over non-well-founded sets. Algebras can be defined as mappings $F(X) \rightarrow X$, for some functor $F$ on a category ${\cal C}$ (e.g., of {\bf Sets}), where $X$ is some object in ${\cal C}$, whereas coalgebras go in the opposite direction $X \rightarrow F(X)$.  } 
\label{algebracoalgebra} 
\end{figure}

Figure~\ref{algebracoalgebra} illustrates our study of dynamic UIGs, where ``learner" participants are trying to imitate the behavior of ``teacher" participants. We contrast two approaches, representing initial and final objects in a category of coalgebras. Induction corresponds to initial algebras over a category defined by the ${\cal P}$ powerset functor over sets. In contrast, coinduction corresponds to final coalgebras. Formally,  {\em inductive inference} studied by Gold \cite{GOLD1967447}, Solomonoff \cite{SOLOMONOFF19641}, Valiant \cite{DBLP:journals/cacm/Valiant84}, and Vapnik \cite{DBLP:journals/tnn/Vapnik99} corresponds to mathematical induction defined by initial objects over algebras.  We contrast inductive inference with coinductive inference in universal coalgebras \cite{Aczel1988-ACZNS,rutten2000universal}, which formally corresponds to final objects. Reinforcement learning algorithms, such as TD-learning \cite{DBLP:books/lib/SuttonB98}, are typically analyzed as stochastic approximation methods \cite{rm,borkar}. In our paper, we model RL in terms of the metric coinduction framework \cite{kozen} over a coalgebra. The Markov decision process (MDP) model, and its innumerable variants, are all easily modeled as probabilistic coalgebras \cite{SOKOLOVA20115095}, and the process of TD-learning is essentially a stochastic coinduction inference step towards determining the final coalgebra. This perspective differs considerably from the standard stochastic approximation perspective used to analyze RL algorithms thus far \cite{bertsekas:alphazero,bertsekas:rlbook}. The fundamental idea underlying dynamic programming (DP) and reinforcement learning (RL) methds is the Bellman optimality principle, which essentially states that the restriction of any shortest path to a segment must again produce a shortest path, otherwise we can produce a shorter overall path by switching to a shorter segment. This principle is just a special case of the general principle of {\em sheaves}, which we discuss in Section~\ref{sheavestopoi}. To the best of our knowledge, the relationship between the Bellman optimality equation and the theory of sheaves \cite{maclane:sheaves} has not been made before, and represents one of many such examples of the deep insights that come from  category-theoretic perspective. Also, by understanding that RL is essentially based on the discovery of final coalgebras in a category of coalgebras, it is possible to explore the vast repertoire of probabilistic coalgebras that have been studied by \cite{SOKOLOVA20115095} and others, which may yield many new RL approaches.

Finally, we explore {\em evolutionary UIGs}, where a society of participants is playing a large-scale imitation game. Participants in evolutionary UIGs can go extinct from ``birth-death" evolutionary processes that model how novel technologies or ``mutants" disrupt previously stable equilibria.  In the case when the participants are {\em evolving}, we introduce the framework of game theory as developed by von Neumann and Morgenstern \cite{vonneumann1947} and Nash \cite{nash}, combined with the use ideas from Darwin's model of evolution by natural selection. We propose category-theoretic formulations of {\em evolutionary game theory} \cite{nowak}. 

Figure~\ref{ttquestions} illustrates a sample of interactions between an evaluator and a participant, with the questions and replies taken from Turing's original paper \cite{turing}. The essence of our paper is to model such interactions by morphisms into and out of an object in a category \cite{maclane:71}. A number of previous studies have explored how natural language can be modeled in terms of objects and morphisms of a category \cite{asudeh,bradley2021enriched,coecke2020mathematics}. At the outset, it is important to clarify that we are proposing a significantly broader vision of imitation games, which we refer to as {\em universal imitation games}. More specifically, we are not constraining ourselves to interactions that are purely focused on text, but we include other forms of interaction, including {\em causal experimentation} \cite{rubin-book,pearl-book}, use of machine learning methods more broadly, including distance-based kernel methods \cite{kernelbook}, statistical methods \cite{tibshirani2010regularization}, information theoretic methods \cite{cover}, inductive inference \cite{GOLD1967447,SOLOMONOFF19641}, and universal coalgebraic methods that rely on coinduction \cite{jacobs:book,rutten2000universal}. In short, the aim of our framework is to explore the full gamut of approaches available to determine the identity of abstract objects in a category, with a primary goal being its application in AI. To clarify this point further, we view the problem of deciding if a particular image was created by a human or a diffusion based generative AI system \cite{DBLP:conf/nips/SongE19}  as an example of an imitation game,  or deciding if a particular computer program was written by a human or an AI copilot. 

One significant strength of category theory is that it can reveal surprising similarities between algebraic structures that superficially look very different, such as metric spaces and partially ordered sets (see Table~\ref{homtable}). A key principle that is often exploited is to explicitly represent the structure in the collection of morphisms between two objects. That is, for some category {\cal C}, the {\bf Hom}$_{\cal C}(a,b)$ between objects $a$ and $b$ might itself have some additional structure beyond that of merely being a collection or a set. For example, in the category of vector spaces, the set of morphisms (linear transformations) between two vector spaces $U$ and $V$ is itself a vector space. So-called {\cal V}-enriched categories signify cases when the {\bf Hom} values are specified in some structure {\cal V}. Examples include metric spaces, where the {\bf Hom} values are non-negative real numbers representing distances, and partially ordered sets (posets) where the {\bf Hom} values are Boolean. 

\begin{table}[t] 
\caption{UIGs can be ``played" in many different categories.}
\vskip 0.1in
\begin{minipage}{0.7\textwidth}
 \begin{small}\hfill 
  \begin{tabular}{|c|c|c|c|c|} \hline 
{\cal C}  & {\bf Hom}$_{\cal C}$ values   & Composition  and & Domain  & Domain for \\ 
& & and identity law & for composition & for identity laws \\ \hline 
General category & Sets & Functions & Cartesian product & One element set \\ \hline 
Metric spaces & Non-negative numbers  & $\geq$  & sum & zero \\ \hline  
Posets & Truth values & Entailment & Conjunction & true \\ \hline
{\cal V}-enriched  & objects  & morphisms  & tensor product  & unit object for   \\ 
category & in {\cal V} & in {\cal V} & in {\cal V} & tensor product in {\cal V} \\ \hline 
$F:{\cal C}^{op} \times C \rightarrow D$ & Bivalent functors & Dinatural transformations & Probabilities, distances & Unit object \\
Coends, ends & & & topological embeddings & \\ \hline 
\end{tabular}
\end{small}
\end{minipage}
\label{homtable}
\end{table}

We view imitation games as fundamental to much of the research in AI in many areas. For example, an autonomous car that is traversing the streets of a busy city like San Francisco is engaged in playing an imitation game at many levels. It has to decide the identity of unknown objects that appear in its perceptual field: is the perceived object a pedestrian or another car?  It has to determine its current ``state", which could be a complex function of its past observations. Many approaches to modeling state have been studied in the literature, including diversity based representations in automata based on tests \cite{DBLP:journals/jacm/RivestS94}, extensions of diversity automata to stochastic sequential models such as partially observable Markov decision processes \cite{singh-uai04}, coalgebraic notions of state \cite{jacobs:book}, as well as causal experimentation methods \cite{rubin-book}. One can view the problem of ``black box identification" as containing within it much of the complexity of AI. For example, conversational models involve determining what someone else knows or is thinking about, which involves formal models of knowledge such as modal logic \cite{halpern:ac}. 

We build on the extensive literature on object isomorphisms in category theory, the bulk of which was developed after Turing's death in 1954, to formulate {\em universal properties} of imitation games, building on the rich mathematical literature on {\em category theory} \cite{maclane:71,riehl2017category,richter2020categories,Lurie:higher-topos-theory}. As we do not assume the prospective reader to have a background in category theory, we have sought to provide a condensed but detailed overview of those parts of category theory that are relevant to formulating UIGs. The aim of category theory is to build a ``unified field theory" of mathematics based on a simple model of {\em objects} that {\em interact} with each other, analogous to directed graph representations. In graphs, vertices represent arbitrary entities, and the edges denote some form of (directional) interaction. In categories, there is no restriction on how many edges exist between any two objects. In a {\em locally small} category, there is assumed to be a ``set's worth" of edges, meaning that it could still be infinite! In addition, small categories are assumed to contain a set's worth of objects (again, which might not be finite).  The framework is {\em compositional}, in that categories can be formed out of objects, {\em arrows} that define the interaction between objects, or {\em functors} that define the interactions between categories. This compositionality gives us a rich {\em generative} space of models that will be invaluable in modeling UIGs. 

We expand on Turing's concept and analyze a broader class of  {\em universal imitation games} (UIGs) using mathematics -- principally {\em category theory} \cite{maclane:71,riehl2017category,richter2020categories,Lurie:higher-topos-theory} -- that was largely developed after Turing. Category theory gives an exceptional set of  ``measuring tools" for imitation games. Choosing a category means selecting a collection of objects and a collection of composable arrows by which each pair of objects interact. This choice of objects and arrows defines the measurement apparatus that is used in formulating and solving an imitation game. A key result called the Yoneda Lemma shows that {\em objects can be identified up to isomorphism solely by their interactions with other objects}. Category theory also embodies the principle of  {\em universality}: a property is universal if it defines an {\em initial} or {\em final} object in a category. Many approaches for solving imitation games, such as probabilistic generative models or distance metrics, can be abstractly characterized as initial or final objects in a category of {\em wedges}, where the objects are bifunctors and the arrows are dinatural transformations.  \cite{loregian_2021} has an excellent treatment of the calculus of coends, which we will discuss in detail later in the paper. 

We classify UIGs into three types: {\em static UIGs}, which is closest to Turing's original formulation, but based  on (weak) isomorphism and homotopy; {\em dynamic UIGs}, where one participant -- the ``learner" -- is seeking to ``imitate" the other participant -- a ``teacher" -- where we contrast the initial object framework of theoretical machine learning defined by  {\em inductive inference} first studied by Gold and Solomonoff in the 1960s with the dynamical systems final object framework of {\em coinductive inference} developed by Aczel and Rutten since the 1980s; and finally {\em evolutionary UIGs}, which models a network economy of participants and we analyze the possibility that a new ``mutant" can cause the entire economy to imitate it by being better fit. Evolutionary imitation games model both the spread of diseases through a population or the spread of new ideas or technologies that promise investors a better return. We give a category-theoretic model of evolutionary game theory, combining Darwin's model of evolution by natural selection with the framework of games and economic behavior pioneered by von Neumann, Morgenstern, and Nash. 

Abstractly, the problem of solving imitation games can be viewed as comparing two potentially infinitely long strings of tokens that represent the communication from and to the participants. At a high level, our approach builds on the notion of object isomorphism in category theory. We formulate the question of imitation games as deciding if two objects are isomorphic:  

\begin{definition}
Two objects $X$ and $Y$ in a category ${\cal C}$ are deemed {\bf {isomorphic}}, or $X \cong Y$ if and only if there is an invertible morphism $f: X \rightarrow Y$, namely $f$ is both {\em left invertible} using a morphism $g: Y \rightarrow X$ so that $g \circ f = $ {{\bf id}}$_X$, and $f$ is {\em right invertible} using a morphism $h$ where $f \circ h = $ {{\bf id}}$_Y$. 
\end{definition}

Category theory provides a rich language to describe how objects interact, including notions like {\em braiding} that plays a key role in quantum computing \cite{Coecke_2016}. The notion of isomorphism can be significantly weakened to include notions like homotopy.  This notion of homotopy generalizes the notion of homotopy in topology, which defines why an object like a coffee cup is topologically homotopic to a doughnut (they have the same number of ``holes''). 

In the category {\bf {Sets}}, two finite sets are considered isomorphic if they have the same number of elements, as it is then trivial to define an invertible pair of morphisms between them. In the category {\bf {Vect}}$_k$ of vector spaces over some field $k$, two objects (vector spaces) are isomorphic if there is a set of invertible linear transformations between them. As we will see below, the passage from a set to the ``free'' vector space generated by elements of the set is another manifestation of the universal arrow property. In the category of topological spaces {\bf Top}, two objects are isomorphic if there is a pair of continuous functions that makes them {\em homeomorphic} \cite{may2012more}. A more refined category is {\em hTop}, the category defined by topological spaces where the arrows are now given by homotopy classes of continuous functions. 
     
 \begin{definition}
 Let $C$ and $C'$ be a pair of objects in a category ${\cal C}$. We say $C$ is {\bf {a retract}} of $C'$ if there exists maps $i: C \rightarrow C'$ and $r: C' \rightarrow C$ such that $r \circ i = \mbox{id}_{\cal C}$. 
 \end{definition}
 
 \begin{definition}
 Let ${\cal C}$ be a category. We say a morphism $f: C \rightarrow D$ is a {\bf {retract of another morphism}} $f': C \rightarrow D$ if it is a retract of $f'$ when viewed as an object of the functor category {\bf {Hom}}$([1], {\cal C})$. A collection of morphisms $T$ of ${\cal C}$ is {\bf {closed under retracts}} if for every pair of morphisms $f, f'$ of ${\cal C}$, if $f$ is a retract of $f'$, and $f'$  is in $T$, then $f$ is also in $T$. 
 \end{definition}

 The point of these examples is to illustrate that choosing a category, which means choosing a collection of objects and arrows, is like defining a measurement system for deciding if two objects are isomorphic. Thus, in the framework of universal imitation games, as we will show in detail later, there are many choices of categories and each one provides some set of measurement tools. Application areas, like natural language processing or computer vision or robotics, may dictate which set of measurement tools is most appropriate. It is, however, possible to state results that hold regardless of what  measurement tools are used (such as the Yoneda Lemma), and we will review the most important of these results for imitation games. 

To motivate why homotopy might be useful, a number of recent generative AI systems based on neural or state-space sequence models \cite{DBLP:conf/iclr/GuJTRR23,DBLP:conf/nips/VaswaniSPUJGKP17} enable the summarization of long documents, which can be viewed as compression of a sequence of tokens. In what sense can we determine if the compressed sequence and the original sequence denote the same object? Homotopy provides an answer. 

A richer model of interaction is provided by {\em simplicial sets} \cite{may1992simplicial}, which is a {\em graded set} $S_n, \ n \geq 0$, where $S_0$ represents a set of non-interacting objects, $S_1$ represents a set of pairwise interactions, $S_2$ represents a set of three-way interactions, and so on. We can map any category into a simplicial set by constructing sequences of length $n$ of composable morphisms. For example, we can model sequences of words in a language as composable morphisms, thereby constructing a simplicial set representation of language-based interactions in an imitation game. Then, the corresponding notion of homotopy between simplicial sets is defined as \cite{richter2020categories}: 

 \begin{definition}
  Let X and Y be simplicial sets, and suppose we are given a pair of morphisms $f_0, f_1: X \rightarrow Y$. A {\bf {homotopy}} from $f_0$ to $f_1$ is a morphism $h: \Delta^1 \times X \rightarrow Y$ satisfying $f_0 = h |_{{0} \times X}$ and $f_1 = h_{ 1 \times X}$. 
 \end{definition}

 These definitions illustrate how category theory can be useful in defining general ways to formulate the problem of imitation games. We now illustrate how category theory is useful in defining generative models that capture much of the work in generative AI, including sequence models. 

We expand significantly on Turing's model of imitation games, which limited itself to static games, to include {\em dynamic} imitation games where one participant is a ``teacher" and the other is a ``learner", as well as {\em evolutionary} imitation games, where there is an entire economy of participants that are competing to maximize individual fitness values.  

Our core theoretical framework is based defining participants in a UIG as {\em functors} that act on {\em categories},  collections of {\em objects} that interact through {\em arrows} or {\em morphisms}. To solve Turing's imitation game, we need to decide if the infinite stream of tokens (e.g., words in a natural language) from two participants in an imitation game (possibly one or both being humans or machines) are ``indistinguishable" from each other. Category theory asks whether two objects are {\em isomorphic}, not whether they are equal. Accordingly, our formulation of UIGs is based on determining whether two UIG objects are isomorphic, not whether they are equal as Turing originally phrased it. In addition, we can bring to bear the notion of {\em homotopy} in category theory \cite{richter2020categories,Quillen:1967} and ask if the two participants in a static UIG are homotopic. 

\begin{definition}
    {\bf Static UIGs}: Given a potentially infinite stream of (multimedia) tokens from and to two participants in an imitation game (see Figure~\ref{img}),  modeled abstractly as objects $A$ and $B$ in some category ${\cal C}$,  are these two objects {\em (weakly) isomorphic}?
\end{definition}

In the second type of UIG that we study, it is known that the two participants are in fact different {\em initially} from each other. The goal of {\em dynamic UIGs} is to determine if asymptotically, the two participants are indistinguishable. This second model captures the process by which generative AI systems are being developed today to pass imitation games, using large language models (LLMs) \cite{DBLP:conf/nips/VaswaniSPUJGKP17} or through building generative models by diffusion \cite{DBLP:conf/nips/SongE19}. An LLM at the beginning outputs random tokens, but over a long process of training on potentially trillions of tokens from datasets like {\tt Common Crawl}, they appear capable in principle of passing a Turing test. 

\begin{definition}
    {\bf Dynamic UIGs}: This formulation differs from  Turing's original conception, in that the two participants are known to be different initially, but the question is to determine if asymptotically, the two participants become indistinguishable. That is, given a potentially infinite stream of tokens from the two participants in a dynamic UIG, modeled again as objects $A$ and $B$ in some category ${\cal C}$, do these objects become isomorphic {\em in the limit}? We can intuitively think of participant $A$ as a ``teacher" and participant $B$ as the ``learner" in a dynamic UIG. 
\end{definition}

Our formulation of dynamic UIGs relates closely to the original theoretical model of {\em inductive inference} or {\em language identification in the limit} studied by Gold \cite{GOLD1967447} and Solomonoff \cite{SOLOMONOFF19641}, and subsequently refined by many others, incuding Valiant \cite{DBLP:journals/cacm/Valiant84} and Vapnik \cite{DBLP:journals/tnn/Vapnik99}. We use the mathematics of {\em non-well-founded} sets \cite{Aczel1988-ACZNS} and {\em universal coalgebras} \cite{jacobs:book, rutten2000universal} to define a relation called {\em bisimulation} to address this problem. We then illustrate how our formulation of dynamic UIGs gives rise to a new framework called {\em coinductive inference} for machine learning (ML), and contrast it to the 60-year-old framework of {\em inductive inference} proposed by Gold and Solomonoff.

Finally, we explore a third formulation of UIGs termed {\em evolutionary UIGs}, combining the principle of evolution by natural selection pioneered by Darwin with the framework of games and economic behavior pioneered by von Neumann and Morgenstern, Nash, and subsequently developed by many others over the past 70 years \cite{Maschler_Solan_Zamir_2013}. Fundamentally, humans and all other biological organisms on Earth evolved by natural selection, and as the recent Covid-19 pandemic has illustrated, the evolutionary game is one that is indeed a game of the {\em survival of the fittest}. 

\begin{definition}
    {\bf Evolutionary UIGs}: The third formulation differs from the previous two, in that there is a network economy of participants, each of which has a certain {\em fitness} to an environment, and the goal is to understand if a particular {\em mutant} can cause the entire economy to ``imitate" it.  Evolutionary UIGs model the spread of diseases through societies, or the spread of new ideas and technologies that promise investors a better return.  The problem of evolutionary UIGs is to determine if the economy find an {\em equilibrium} point. 
\end{definition}

Evolutionary UIGs integrates  the framework of evolutionary dynamics \cite{nowak}, and non-cooperative games pioneered by Nash \cite{nash}. In particular, we focus on identifying fundamental differences between adaptation by incremental processing of experience, as in dynamic UIGs, with the process of random mutation by natural selection, where global changes occur to disrupt an equilibrium. We investigate several models of UIGs, from Valiant's evolvability \cite{evolvability}, to birth-death stochastic processes \cite{novak:book}. We give a detailed introduction to variational inequalities (VIs) \cite{facchinei-pang:vi}, which generalize game theory and optimization and illustrate it with an example of playing a generative AI network economy game. We describe a metric coinduction stochastic approximation method for solving VIs \cite{iusem,DBLP:journals/mp/WangB15}, and outline how to extend such methods to evolutionary UIGs. 

The length of this paper is largely due to its tutorial nature, beginning in the next section with a detailed overview of category theory for the reader who may be new to this topic. In Section~\ref{staticuig}, besides giving a detailed introduction to category theory, we also define several examples of static UIGs. We turn to describe dynamic UIGs in Section~\ref{dynamicuig}. Finally, in Section~\ref{evolutionaryuig}, we define evolutionary UIGs.

\section{Static Universal Imitation Games: An Introduction to Categories and Functors}

\label{staticuig} 

In this  section, we introduce category theory more formally, and illustrate how it is useful in modeling universal imitation games (UIGs). As noted above, UIGs fundamentally involve determining answering the question {\tt Human or Machine?}, in the original model that Turing proposed \cite{turing}. We propose to generalize this notion in the remainder of this paper, fundamentally starting with our use of category theory, a field of mathematics that has largely developed since Turing. We turn the question posed by Turing into a question of deciding if two objects in a category are (weakly) isomorphic. Thus, in reading this somewhat long section summarizing a lot of work in category theory over the past seven decades, it is useful to keep in mind that at the heart of what we are trying to do is to decide if two participants in a UIG are in some abstract sense ``similar" using the various tools provided by category theory. Different categories allow us to ``play" the imitation game in various ways, and this chapter is intended to whet the reader's appetite for the enormous variety of possible ways in which this question can be posed. 

Categories are comprised of objects, which can be arbitrarily complex, which gives it an unmatched power of abstraction in virtually all of mathematics \cite{maclane:71,riehl2017category,richter2020categories}.  Categories are compositional structures, and can be built out of smaller category-like objects. Objects in a category interact with each other through arrows or morphisms. One of the most remarkable results in category theory -- the Yoneda Lemma -- shows that objects in a category can be defined by their interactions. This result forms the underlying motif throughout the rest of this paper, as it relates closely to the problem posed by UIGs. Another central principle in category theory is the notion of a {\em universal property}: a property is defined to be universal if it represents an {\em initial} or {\em final} object in a category. There is a unique arrow from the initial object in a category to every other object. Analogously, there is a unique morphism from every object in the category into the final object. These two special cases illustrate the theme of the Yoneda Lemma of characterizing every object from its interactions. We introduce universal constructions through colimits and limits, pullbacks, equalizers and co-equalizers, which give a rich and flexible way of integrating over objects to find common properties. Functors map one category to another, mapping not just objects, but also the morphisms in a category. Many of the generative models we define in this paper will correspond to functors over some category. 

Category theory \cite{maclane:71,riehl2017category,maclane:sheaves} differs substantially from earlier formulations of AI in terms of causal inference, logic, probability theory, neural networks, statistics, or optimization, all of which essentially can be shown to define special types of categories (e.g., causal models are endofunctors in the category of preorders, probabilities are defined as endofunctors over the category of measurable spaces, optimization methods represent endofunctors over the category of vector spaces, etc). As we will see in the remainder of the paper, many of the theoretical models used in AI can be described in the language of category theory, which provides an unmatched power of abstraction in all of mathematics. It is also singularly suited to modeling computation, and has been extensively explored in other areas of computer science \cite{jacobs:book}. 

\begin{figure}[h]
\centering
\includegraphics[scale=.4]{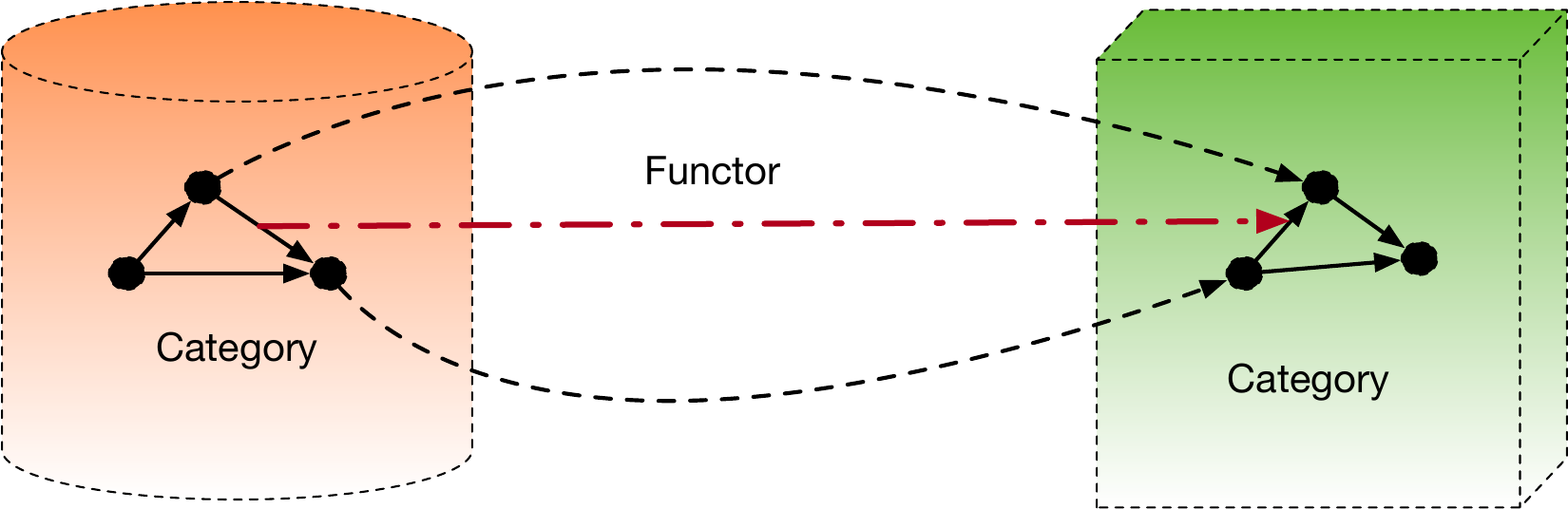}
\caption{Categories are defined by collection of arbitrary objects that interact through morphisms (also called arrows). Functors map objects from one category into another, but also map the arrows of the domain category into corresponding arrows in the co-domain category. An underlying motif that runs through this entire paper is to view AI and ML systems as functors acting on categories.}
\label{functors} 
\end{figure}

\subsection{Categories and Functors} 

Table~\ref{UIGcategories} gives a few examples from a myriad of ways in which categories can be defined for formulating UIGs, including approaches based on information theory \cite{baez}, statistical methods based on a decomposable probability distribution \cite{pearl-book}, causal methods based on {\em experimentation} to test a difference between the participants \cite{pearl-book,rubin-book}, reinforcement learning methods that are based on modeling the participants as a stochastic dynamical system \cite{DBLP:books/lib/SuttonB98},  formal models of knowledge \cite{fagin}, universal coalgebraic methods of modeling dynamical systems \cite{jacobs:book,rutten2000universal}, and last but not least, generative AI models based on large language models (LLMs) \cite{DBLP:conf/nips/VaswaniSPUJGKP17} and diffusion models \cite{DBLP:conf/nips/SongE19}. 

\begin{table}[t]
    \centering
    \begin{tabular}{|c|c|c|} \hline 
      {\bf Category}   &  {\bf Characterization} & {\bf UIG Application} \\ \hline 
     Topological spaces  &  Continuous functions & Analyze document summarization\\ \hline 
      Measurable spaces     &  Probability monads  & Statistical model of participants\\ \hline 
      Coalgebras    &  $X \rightarrow F(X)$, $F$ is a functor over sets &  Generative AI models \\ \hline
      Simplicial sets & $[n] = (0, 1, \ldots, n)$ as objects & Evolutionary UIGs \\ \hline 
      Wedges & Dinatural transformations between $F: {\cal C}^{op} \times {\cal C} \rightarrow {\cal D}$ & Generalized metric spaces \\ \hline 
      Coends & $\int^D {\cal C}(D, D)$ models topological embeddings & Manifold representations \\ \hline
       Ends & $\int_D {\cal C}(D, D)$ models integration over measures & Categorical probability \\ \hline
    \end{tabular} \vskip 0.2in
    \caption{Categories can be defined in a myriad of ways for formulating UIGs. We illustrate a few possibilities here.}
    \label{UIGcategories}
\end{table}

Category theory is based fundamentally on defining {\em universal properties} \cite{riehl2017category}, which can be defined as the {\em initial} or {\em final} object in some category. To take a simple example, the Cartesian product of two sets can be defined as the set of ordered pairs, which tells us what it is, but not what it is good for, or why it is special in some way. Alternatively, we can define the Cartesian product of two sets as an object in the category {\bf Sets} that has the unique property that every function onto those sets must decompose uniquely as a composition of a function into the Cartesian product object, and then a projection component onto each component set. Furthermore, among all such objects that share this property, the Cartesian product is the final object. In sum, we seek such universal properties that define solutions to imitation games.

We propose a {\em categorical calculus} for formulating {\em universal imitation games} (UIGs),  building on a $70$-year rich legacy of research in AI, mathematics, and computer science. Principally, our framework builds on ideas in category theory that were principally developed since Turing tragically died in 1954 \cite{maclane:71,maclane:sheaves,riehl2017category,richter2020categories}. That very year, Saunders Maclane, a distinguished American mathematician, met with a young Japanese mathematician, Nobuo Yoneda, in the {\em Gare du Nord} train station in Paris. That possibly chance encounter led to the development of a profound series of mathematical ideas that revolve around a central core question: {\em can objects can be identified by their interactions with other objects}?. The relevance of this question to UIGs is clear, if we abstractly model the participants in a UIG as objects in some category. In a nutshell, our categorical calculus for UIGs posits that two participants in a UIG are indistinguishable if they are {\em isomorphic}. One cardinal principle of category theory is to never ask if two objects are {\em equal}, only if they are equivalent under a pair of invertible mappings. This categorical perspective succinctly summarizes our formulation of UIGs. 

Maclane published a classic textbook on category theory \cite{maclane:71}, where he introduced the mathematical world to the {\em Yoneda Lemma}.  Our categorical calculus for UIGs builds on the Yoneda Lemma, as well as many ideas that have been developed since, both in category theory \cite{loregian_2021,richter2020categories,Lurie:higher-topos-theory}, as well as in theoretical computer science on {\em universal coalgebras} \cite{jacobs:book,rutten2000universal}, and the axiomatic foundations of non-well-founded sets \cite{Aczel1988-ACZNS}. In 1958, Daniel Kan introduced the concept of {\em adjoint functors}, denoted $F \vdash G$, where the functors $F$ and $G$ map between two categories ${\cal C}$ and ${\cal D}$ in opposite directions \cite{kan}. The fundamental principle here was to note that the collection of interactions between two objects $A$ and $B$ in a category ${\cal C}$, denoted variously as ${\bf Hom}_{{\cal C}}(A, B)$, or more simply as ${\cal C}(A, B)$ was {\em contravariant} in the first argument, but {\em covariant} in the second argument. It appeared to differ from a {\em tensor product} $\otimes$ functor that produced a new object $A \otimes B$, and was covariant in both arguments. As it turned out, both the ${\bf Hom}$ functor and the tensor product functor $\otimes$ were not independent, but formally adjoints of each other. Adjoint functors gives us a further refinement of our categorical calculus for UIGs, since we can now construct functors that map between two apparently dissimilar categories, which may indeed be comparable, such as the category of conditional independence models, modeled as {\em separoids} \cite{DBLP:journals/amai/Dawid01}, and the category of causal models, modeled by an algebraic structure, such as a directed acyclic graph (DAG) \cite{pearl-book}. This relationship was explored previously in \cite{DBLP:journals/entropy/Mahadevan23} In the context of UIGs, two objects in the category of separoids are only {\em statistically distinguishable}, whereas two objects in the causal category of DAG models are {\em causally distinct}. By suitably {\em enriching} the {\bf Hom} functor, we can model a rich variety of structures that capture properties of the interactions between two objects $A$ and $B$. 

Kan also introduced {\em Kan extensions} in his paper on adjoint functors, which provide us deeper insight into the study of dynamic UIGs, where initially one participant is a ``teacher" and the other is a ``learner" that are distinguishable initially, but may become indistinguishable in the limit. \citep{GOLD1967447} proposed the theoretical framework of {\em identification in the limit}, motivated by the remarkable ability of human children in acquiring natural language from hearing spoken text.  \citep{SOLOMONOFF19641} proposed {\em inductive inference} as a model for our ability to generalize from examples. Gold's framework was subsequently elaborated by \citep{DBLP:journals/cacm/Valiant84} and  \citep{DBLP:journals/tnn/Vapnik99}, based on introducing a probabilistic perspective, as well as introducing the need for efficient estimation with a polynomial number of samples or time.  Abstractly, the problem of inductive inference or identification in the limit can be viewed as extrapolating a function from samples of the function. Given a function $f: C \rightarrow D$ that is defined on some subset $C \subset M$ that we want to {\em extend} to the entire set $M$, there are no obvious canonical solutions to this problem as it is inherently ill-defined. Much of the work in machine learning over the past six decades has explored various formulations of this function induction problem, usually introducing some {\em regularization} principle such as {\em Occam's Razor}, or prefer the simplest hypothesis \cite{manifoldregularization}, or using some intrinsic measure of the complexity of an object, such as {\em algorithmic information theory} \cite{chaitin} based on the shortest program that can generate an object (also referred to as {\em Kolmogorov complexity}). In contrast, if instead, we assume that we are extending a {\em functor} $F: C \rightarrow D$ from some subcategory $C$ of a larger category $M$, there are two canonical and ``universal" solutions to the functor extension problem given by Kan extensions. Thus, the Kan extension gives us a powerful set of tools to formulate dynamic UIGs. 

Yoneda \cite{yoneda-end} developed another remarkable idea in 1960 based on defining the {\bf end} of a bifunctor $F: C^{op} \times C \rightarrow D$ that, like the {\bf Hom} functor above, is contravariant in the first argument, but covariant in the second argument. The end of a bifunctor $F: C^{op} \times C \rightarrow D$ refers to an object $\mbox{end}(P) \in {\cal D}$ that represents the terminal object in a category of {\em wedges} defined by {\em  dinatural transformations} between bifunctors $F,G: C^{op} \times C \rightarrow {\cal D}$. A dual notion is the {\bf coend}, which represents the initial object. The significance of the end and the coend of a bifunctor to this paper is that they capture universal properties of many AI and ML approaches. In particular, the end can be shown to define  the universal property of probability distributions on a space \cite{Avery_2016} in terms of the {\em codensity monad}, which is defined as the right Kan extension of a functor against itself. Analogously, the coend can be shown to define the topological realization functor of a (fuzzy) simplicial set, which has been the basis for a widely used manifold learning method called UMAP \cite{umap}. A detailed discussion of the calculus of coends is given in \cite{loregian_2021}. 

Finally, we build on the framework of {\em non-well-founded} sets \cite{Aczel1988-ACZNS} and {\em universal coalgebras} \cite{jacobs:book,rutten2000universal} in modeling UIGs. Here, we have a collection of participants that are neither static, nor are they dynamic, but rather they {\em evolve} by interacting with each other in a {\em non-cooperative game} \cite{Maschler_Solan_Zamir_2013}. von Neumann and Morgenstern \cite{vonneumann1947}, and subsequently Nash \cite{nash},  developed non-cooperative games as a model of economic behavior, which has had a persuasive influence on the development of AI, computer science, and of course, economics. It has also had profound significance in biology in the study of evolutionary games. Our framework builds on the theory of {\em network economics} \cite{nagurney:vibook}, where the participants in an evolutionary UIG are interacting on a graph. The economic activity that occurs in a network economy classifies the participants as {\em producers}, such as vendors of generative AI products, {\em transporters} who control the network infrastructure over which data is transmitted, and finally {\em consumers} who constitute the demand market that must choose between a producer and a transporter. Each participant in an evolutionary UIG is seeking to maximize a selfish goal, and the formulation of evolutionary UIGs is based on determining if the network economy, formulated as a category of coalgebras, has a {\em final coalgebra} \cite{rutten2000universal}.

\begin{figure}[t]
\begin{center}
\begin{tabular}{|c |c | } \hline 
{\bf Set theory } & {\bf Category theory} \\ \hline 
 set & object \\ 
 subset & subobject \\
 truth values $\{0, 1 \} $ & subobject classifier $\Omega$ \\
power set $P(A) = 2^A$ & power object $P(A) = \Omega^A$ \\ \hline
bijection & isomorphims \\ 
injection & monic arrow \\
surjection & epic arrow \\ \hline
singleton set $\{ * \}$ & terminal object ${\bf 1}$ \\ 
empty set $\emptyset$ & initial object ${\bf 0}$ \\
elements of a set $X$ & morphism $f: {\bf 1} \rightarrow X$ \\
- & functors, natural  transformations \\ 
- & limits, colimits, adjunctions \\ \hline
\end{tabular}
\end{center}
\caption{Comparison of notions from set theory and category theory.} 
\label{setvscategories}
\end{figure} 

\begin{figure}[h]
\centering
\includegraphics[scale=.5]{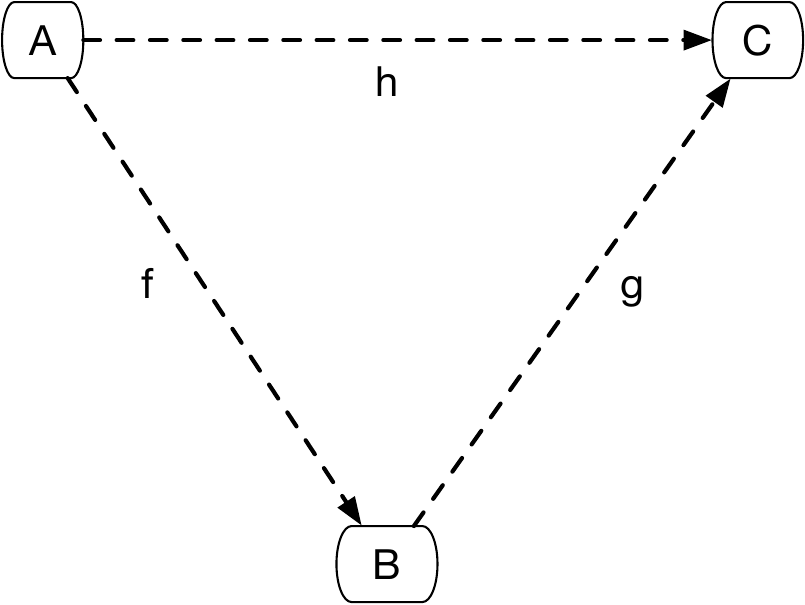}
\caption{Category theory is a compositional model of a system in terms of objects and their interactions.}
\label{morphisms} 
\end{figure}

Figure~\ref{setvscategories} compares the basic notions in set theory vs. category theory. Figure~\ref{morphisms} illustrates a simple category of 3 objects: $A$, $B$, and $C$ that interact through the morphisms $f: A \rightarrow B$, $g: B \rightarrow C$, and $h: A \rightarrow C$. Categories involve a fundamental notion of {\em composition}: the morphism $h: A \rightarrow C$ can be defined as the composition $g \circ f$ of the morphisms from $f$ and $g$. What the objects and morphisms represent is arbitrary, and like the canonical directed graph model, this abstractness gives category theory -- like graph theory -- a universal quality in terms of applicability to a wide range of problems. While categories and graphs and intimately related, in a category, there is no assumption of finiteness in terms of the cardinality of objects or morphisms. A category is defined to be {\em small} or {\em locally small} if there is a set's worth of objects and between any two objects, a set's worth of morphisms, but of course, a set need not be finite. As a simple example, the set of integers $\mathbb{Z}$ defines a category, where each integer $z$ is an object and there is a morphism $f: a \rightarrow b$ between integers $a$ and $b$ if $a \leq b$. This example serves to immediately clarify an important point: a category is only defined if both the objects and morphisms are defined. The category of integers $\mathbb{Z}$ may be defined in many ways, depending on what the morphisms represent. 

Briefly, a category is a collection of objects, and a collection of morphisms between pairs of objects, which are closed under composition, satisfy associativity, and include an identity morphism for every object. For example, sets form a category under the standard morphism of functions. Groups, modules, topological spaces and vector spaces all form categories in their own right, with suitable morphisms (e.g, for groups, we use group homomorphisms, and for vector spaces, we use linear transformations). 

A simple way to understand the definition of a category is to view it as a ``generalized" graph, where there is no limitation on the number of vertices, or the number of edges between any given pair of vertices. Each vertex defines an object in a category, and each edge is associated with a morphism. The underlying graph induces a ``free'' category where we consider all possible paths between pairs of vertices (including self-loops) as the set of morphisms between them. In the reverse direction, given a category, we can define a ``forgetful'' functor that extracts the underlying graph from the category, forgetting the composition rule. 

\begin{definition}
\label{cat-defn}
A {\bf {graph}} ${\cal G}$ (sometimes referred to as a quiver) is a labeled directed multi-graph defined by a set $O$ of {\em objects}, a set $A$ of {\em arrows},  along with two morphisms $s: A \rightarrow O$ and $t: A \rightarrow O$ that specify the domain and co-domain of each arrow.  In this graph, we define the set of composable pairs of arrows by the set 
\[ A \times_O A = \{\langle g, f  \rangle | \ g, f \in A, \ \ s(g) = t(f) \} \]

A {\bf {category}} ${\cal C}$ is a graph ${\cal G}$ with two additional functions: {${\bf id}:$} $O \rightarrow A$, mapping each object $c \in C$ to an arrow {${\bf id}_c$} and $\circ: A \times_O A \rightarrow A$, mapping each pair of composable morphisms $\langle f, g \rangle$ to their composition $g \circ f$. 
\end{definition}

It is worth emphasizing that no assumption is made here of the finiteness of a graph, either in terms of its associated objects (vertices) or arrows (edges). Indeed, it is entirely reasonable to define categories whose graphs contain an infinite number of edges. A simple example is the group $\mathbb{Z}$ of integers under addition, which can be represented as a single object, denoted $\{ \bullet \}$ and an infinite number of morphisms $f: \bullet \rightarrow \bullet$, each of which represents an integer, where composition of morphisms is defined by addition. In this example, all morphisms are invertible. In a general category with more than one object, a {\em groupoid} defines a category all of whose morphisms are invertible. A central principle in category theory is to avoid the use of equality, which is pervasive in mathematics, in favor of a more general notion of {\em isomorphism} or weaker versions of it. Many examples of categories can be given that are relevant to specific problems in AI and ML.  Some examples of categories of common mathematical structures are illustrated below. 

\begin{itemize} 

\item {\bf Set}: The canonical example of a category is {\bf Set}, which has as its objects, sets, and morphisms are functions from one set to another. The {\bf Set} category will play a central role in our framework, as it is fundamental to the universal representation constructed by Yoneda embeddings. 

\item {\bf Top:} The category {\bf Top} has topological spaces as its objects, and continuous functions as its morphisms. Recall that a topological space $(X, \Xi)$ consists of a set $X$, and a collection of subsets $\Xi$ of $X$ closed under finite intersection and arbitrary unions. 

\item {\bf Group:} The category {\bf Group} has groups as its objects, and group homomorphisms as its morphisms. 

\item {\bf Graph:} The category {\bf Graph} has graphs (undirected) as its objects, and graph morphisms (mapping vertices to vertices, preserving adjacency properties) as its morphisms. The category {\bf DirGraph} has directed graphs as its objects, and the morphisms must now preserve adjacency as defined by a directed edge. 

\item {\bf Poset:} The category {\bf Poset} has partially ordered sets as its objects and order-preserving functions as its morphisms. 

\item {\bf Meas:} The category {\bf Meas} has measurable spaces as its objects and measurable functions as its morphisms. Recall that a measurable space $(\Omega, {\cal B})$ is defined by a set $\Omega$ and an associated $\sigma$-field of subsets {\cal B} that is closed under complementation, and arbitrary unions and intersections, where the empty set $\emptyset \in {\cal B}$. 

\end{itemize}

Functors can be viewed as a generalization of the notion of morphisms across algebraic structures, such as groups, vector spaces, and graphs. Functors do more than functions: they not only map objects to objects, but like graph homomorphisms, they need to also map each morphism in the domain category to a corresponding morphism in the co-domain category. Functors come in two varieties, as defined below. 

 \begin{definition} 
A {\bf {covariant functor}} $F: {\cal C} \rightarrow {\cal D}$ from category ${\cal C}$ to category ${\cal D}$, and defined as \mbox{the following: }

\begin{itemize} 
    \item An object $F X$ (also written as $F(X)$) of the category ${\cal D}$ for each object $X$ in category ${\cal C}$.
    \item An  arrow  $F(f): F X \rightarrow F Y$ in category ${\cal D}$ for every arrow  $f: X \rightarrow Y$ in category ${\cal C}$. 
   \item The preservation of identity and composition: $F \ id_X = id_{F X}$ and $(F f) (F g) = F(g \circ f)$ for any composable arrows $f: X \rightarrow Y, g: Y \rightarrow Z$. 
\end{itemize}
\end{definition} 

\begin{definition} 
A {\bf {contravariant functor}} $F: {\cal C} \rightarrow {\cal D}$ from category ${\cal C}$ to category ${\cal D}$ is defined exactly like the covariant functor, except all the arrows are reversed. 
\end{definition} 

The {\em functoriality} axioms dictate how functors have to be behave: 

\begin{itemize} 

\item For any composable pair $f, g$ in category $C$, $Fg \cdot Ff = F(g \cdot f) $.

\item For each object $c$ in $C$, $F (1_c) = 1_{Fc}$.

\end{itemize} 

One of the motifs that underlies much of the application of category theory to AI and ML in this paper is to model problems in terms of functors. Figure~\ref{clustering} illustrates the use of functoriality in designing an algorithm for clustering points in a finite metric space based on pairwise distances, one of the most common ways to preprocess data in ML and statistics. Treating clustering as a functor implies the resulting algorithm should behave appropriately under suitable modifications of the input space. Here, we can define clustering formally as a functor $F$ that maps the input category of finite metric spaces {\bf FinMet} defined by $(X, d)$, where $X$ is a finite set of points in $\mathbb{R}^n$ and $d:X \times X \rightarrow [0, \infty]$ is a (generalized) finite metric space, and the output category {\bf Part} is the set of all partitions $X$ into subsets $X_i$ such that $\cup_i X_i = X$. 

\begin{figure}[h]
\centering
\includegraphics[scale=.5]{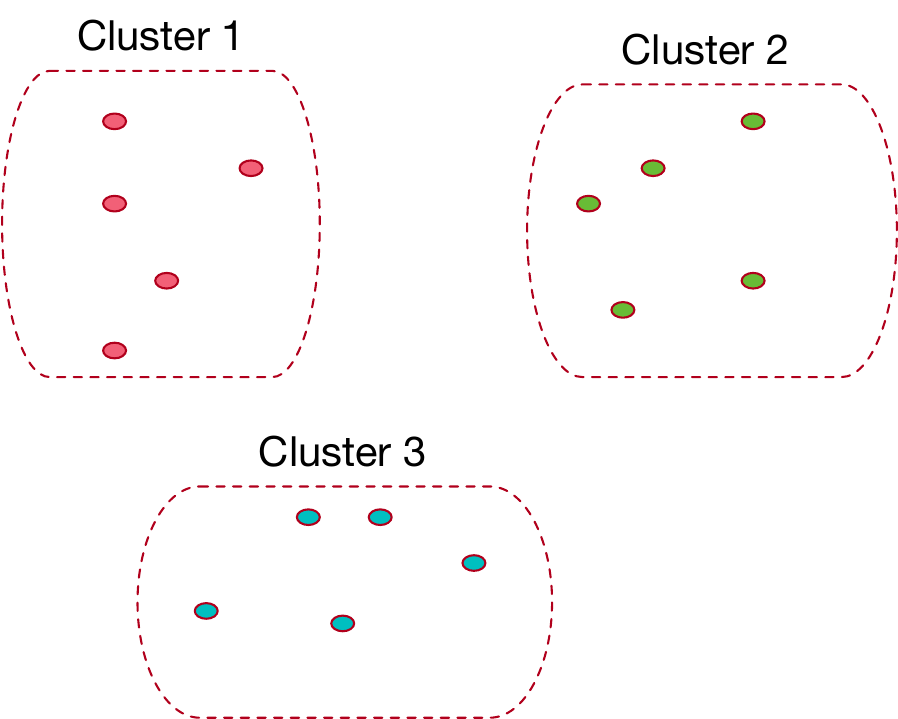}
\caption{One of the most traditional problems in statistics and ML is {\em clustering} by constructing a partition of a finite metric space by grouping points together based on their pairwise distances. Treating clustering as a functor implies designing an algorithm that behaves functorially: if the distances were scaled by some factor, the clustering should not change.}
\label{clustering} 
\end{figure}

One can impose three criteria on a clustering algorithm, which seem entirely natural, and yet, as Kleinberg \cite{kleinbergimpossibility} showed, no standard clustering algorithm satisfies all these conditions: 

\begin{itemize} 

\item {\em Scale invariance:} If the distance metric $d$ is increased or decreased by $c \cdot d$, where $c$ is a scalar real number, the output clustering should not change. In terms of Figure~\ref{clustering}, if the points in each cluster became closer together or further apart proportionally, the clustering should remain the same. 

\item {\em Completeness:} For any given partition of the space $X$, there should exist some distance function $d$ such that the clustering algorithm when given that distance function should return the desired partition. 

\item {\em Monotonicity:} If the  distance between points within each cluster in Figure~\ref{clustering} were decreased, and the distances between points in different clusters were increased, the clustering should not change either. 

\end{itemize} 

Remarkably, it turns out that no clustering algorithm exists that satisfies all these three basic conditions. Yet, as Carlsson  and Memoli showed \cite{Carlsson2010}, treating clustering as a functor makes it possible to define a modified clustering problem in terms of creating {\em persistent clusters} that overcomes Kleinberg's impossibility result. This simple but deep example reveals the power of functorial design, and gives a concrete illustration of the importance of categorical thinking in AI and ML.

\subsection{Natural Transformations and Universal Arrows}

Given any two functors $F: C \rightarrow D$ and $G: C \rightarrow D$ between the same pair of categories, we can define a mapping between $F$ and $G$ that is referred to as a natural transformation. These are defined through a collection of mappings, one for each object $c$ of $C$, thereby defining a morphism in $D$ for each object in $C$. 

\begin{definition}
    Given categories $C$ and $D$, and functors $F, G: C \rightarrow D$, a {\bf natural transformation} $\alpha: F \Rightarrow G$ is defined by the following data: 

    \begin{itemize}
        \item an arrow $\alpha_c: Fc \rightarrow Gc$ in $D$ for each object $c \in C$, which together define the components of the natural transformation. 
        \item For each morphism $f: c \rightarrow c'$, the following commutative diagram holds true: 

\[\begin{tikzcd}
	Fc &&& Gc \\
	\\
	{Fc'} &&& {Gc'}
	\arrow["{\alpha_c}", from=1-1, to=1-4]
	\arrow["Ff"', from=1-1, to=3-1]
	\arrow["{\alpha_{c'}}"', from=3-1, to=3-4]
	\arrow["Gf", from=1-4, to=3-4]
\end{tikzcd}\]

    \end{itemize}
    A {\bf natural isomorphism} is a natural transformation $\alpha: F \Rightarrow G$ in which every component $\alpha_c$ is an isomorphism. 
\end{definition}

We will see many examples of natural transformations in this paper, but we will use one example here to illustrate its versatility. 
Figure~\ref{causalfunctors} shows that we can model causal models as functors mapping from an input category of algebraic structures, whose objects represent structural descriptions of the causal model, into an output category of probabilistic representations, whose objects specify the semantics of the causal model. 

\begin{figure}[h]
\centering
\includegraphics[scale=.4]{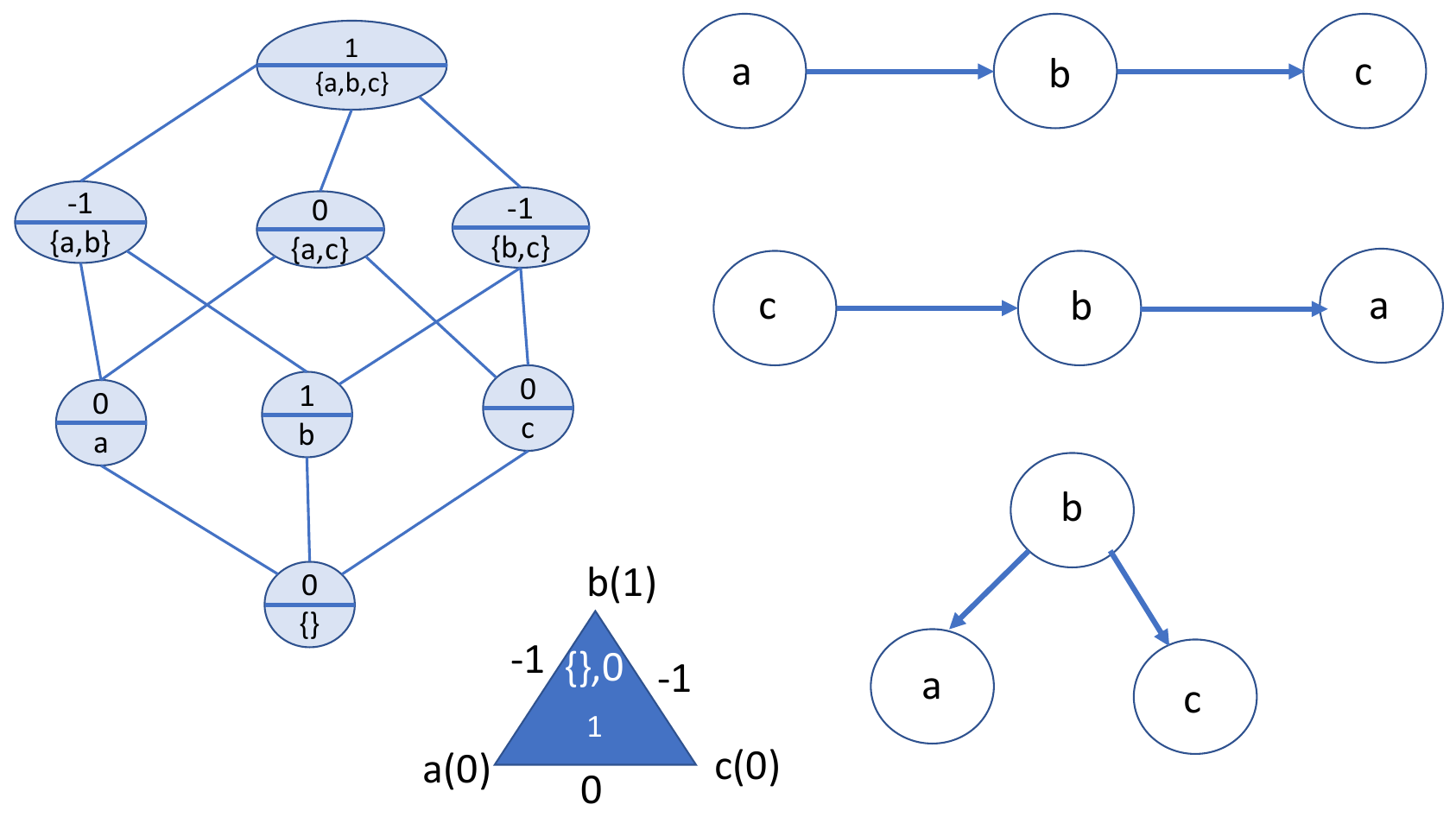}
\caption{Causal models capture directional relationships between entities. A causal model can be viewed as a functor that maps a category of algebraic structures, whose objects are directed acyclic graph (DAG) models as shown here, or integer-valued multisets shown here as a lattice, into a category of probabilistic representations, each of whose objects can be viewed as a structured probability distribution.}
\label{causalfunctors} 
\end{figure}

This process of going from a category to its underlying directed graph  embodies a fundamental universal construction in category theory, called the {\em {universal arrow}}. It lies at the heart of many useful results, principally the Yoneda lemma that shows how object identity itself emerges from the structure of morphisms that lead into (or out of) it. 

\begin{definition}
Given a functor $S: D \rightarrow C$ between two categories, and an object $c$ of category $C$, a {\bf {universal arrow}} from $c$ to $S$ is a pair $\langle r, u \rangle$, where $r$ is an object of $D$ and $u: c \rightarrow Sr$ is an arrow of $C$, such that the following universal property holds true: 

\begin{itemize} 
\item For every pair $\langle d, f \rangle$ with $d$ an object of $D$ and $f: c \rightarrow Sd$ an arrow of $C$, there is a unique arrow $f': r \rightarrow d$ of $D$ with $S f' \circ u = f$. 
\end{itemize}
\end{definition}

\begin{definition}
If $D$ is a category and $H: D \rightarrow$ {\bf {Set}} is a set-valued functor, a {\bf {universal element}} associated with the functor $H$ is a pair $\langle r, e \rangle$ consisting of an object $r \in D$ and an element $e \in H r$ such that for every pair $\langle d, x \rangle$ with $x \in H d$, there is a unique arrow $f: r \rightarrow d$ of $D$ such \mbox{that $(H f) e = x$. }
\end{definition}

\begin{example}
Let $E$ be an equivalence relation on a set $S$, and consider the quotient set $S/E$ of equivalence classes, where $p: S \rightarrow S/E$ sends each element $s \in S$ into its corresponding equivalence class. The set of equivalence classes $S/E$ has the property that any function $f: S \rightarrow X$ that respects the equivalence relation can be written as $f s = f s'$ whenever $s \sim_E s'$, that is, $f = f' \circ p$, where the unique function $f': S/E \rightarrow X$. Thus, $\langle S/E, p \rangle$ is a universal element for the functor $H$. 
\end{example}

Figure~\ref{universal-arrow} illustrates the concept of universal arrows through the connection between categories and graphs. For every (directed) graph $G$, there is a universal arrow from $G$ to the ``forgetful'' functor $U$ mapping the category {\bf {Cat}} of all categories to {\bf {Graph}}, the category of all (directed) graphs, where for any category $C$, its associated graph is defined by $U(C)$. Intuitively, this forgetful functor ``throws'' away all categorical information, obliterating for example the distinction between the primitive morphisms $f$ and $g$ vs. their compositions $g \circ f$, both of which are simply viewed as edges in the graph $U(C)$.  To understand this functor,  consider a directed graph $U(C)$ defined from a category $C$, forgetting the rule for composition. That is, from the category $C$, which associates to each pair of composable arrows $f$ and $g$, the composed arrow $g \circ f$, we derive the underlying graph $U(C)$ simply by forgetting which edges correspond to elementary arrows, such as $f$ or $g$, and which are composites. For example, consider a partial order as the category ${\cal C}$, and then define  $U({\cal C} )$ as the directed graph that results from the transitive closure of the partial ordering.

 The universal arrow from a graph $G$ to the forgetful functor $U$  is defined as a pair $\langle C, u: G \rightarrow U(C) \rangle$, where $u$ is a ``universal''  graph homomorphism. This arrow possesses the following {\em {universal property}}: for every other pair $\langle D, v: G \rightarrow H \rangle$, where $D$ is a category, and $v$ is an arbitrary graph homomorphism, there is a functor  $f': C \rightarrow D$, which is an arrow in the category {\bf {Cat}} of all categories, such that {\em {every}} graph homomorphism $\phi: G \rightarrow H$ uniquely factors through the universal graph homomorphism $u: G \rightarrow U(C)$  as the solution to the equation $\phi = U(f') \circ u$, where $U(f'): U(C) \rightarrow H$ (that is, $H = U(D)$).  Namely, the dotted arrow defines a graph homomorphism $U(f')$ that makes the triangle diagram ``commute'', and the associated ``extension'' problem of finding this new graph homomorphism $U(f')$ is solved by ``lifting'' the associated category arrow $f': C \rightarrow D$. This property of universal arrows, as we show in the paper, provide the conceptual underpinnings of universal properties in many applications in AI and ML, as we will see throughout this paper.  

 \begin{figure}[t] 
\centering
\vskip 0.1in
\begin{minipage}{0.9\textwidth}
\includegraphics[scale=0.4]{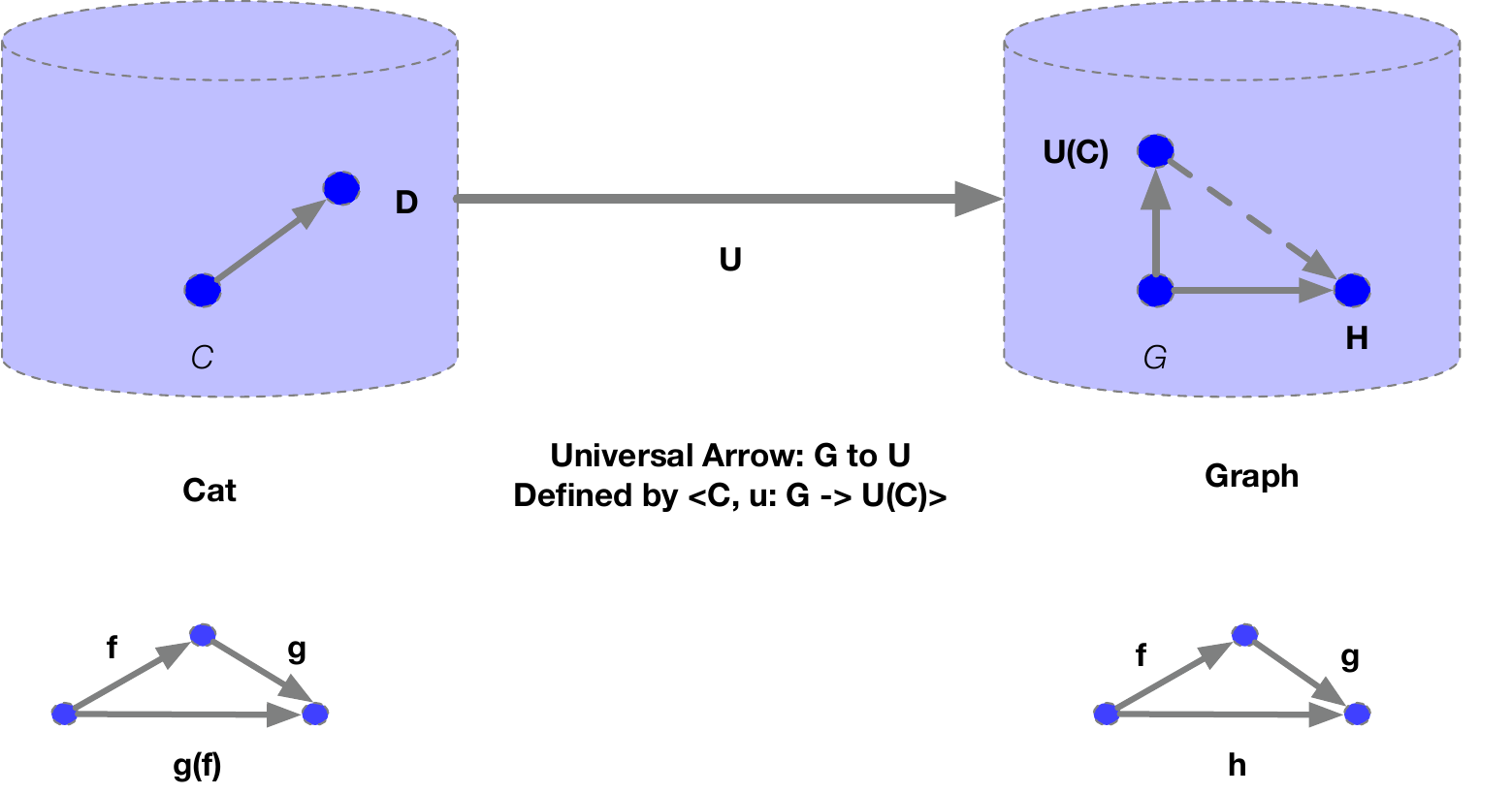}
\end{minipage}
\caption{The concept of universal arrows is illustrated through the connection between directed graphs, and their associated ``free" categories. In this example, the forgetful functor $U$ between {\bf {Cat}}, the category of all categories, and {\bf {Graph}}, the category of all (directed) graphs, maps any category into its underlying graph, forgetting which arrows are primitive and which are compositional.  The universal arrow from a graph $G$ to the forgetful functor $U$  is defined as a pair $\langle C, u: G \rightarrow U(C) \rangle$, where $u$ is a ``universal''  graph homomorphism. The universal arrow property asserts that every graph homomorphism $\phi: G \rightarrow H$ uniquely factors through the universal graph homomorphism $u: G \rightarrow U(C)$, where $U(C)$ is the graph induced  by category $C$ defining the universal arrow property. In other words, the associated {\em {extension}} problem of ``completing'' the triangle of graph homomorphisms in the category of {\bf {Graph}} can be uniquely solved by ``lifting'' the associated category arrow $h: C \rightarrow D$.}
\label{universal-arrow}
 \end{figure} 

\subsection{Yoneda lemma and the Universality of Diagrams} 

The Yoneda Lemma  is one of the most celebrated results in category theory, and it provides a concrete example of the power of categorical thinking. Stated in simple terms, it states the mathematical objects are determined (up to isomorphism) by the interactions they make with other objects in a category. We will show the surprising results of applying this lemma to problems involving computing distances between objects in a metric space, reasoning about causal inference, and many other problems of importance in AI and ML. An analogy from particle physics proposed by Theo Johnson-Freyd might help give insight into this remarkable result: ``You work at a particle accelerator. You want to understand some particle. All you can do is throw other particles at it and see what happens. If you understand how your mystery particle responds to all possible test particles at all possible test energies, then you know everything there is to know about your mystery particle". The Yoneda Lemma states that the set of all morphisms into an object $d$ in a category ${\cal C}$, denoted as {\bf Hom}$_{\cal C}(-,d)$ and called the {\em contravariant functor} (or presheaf),  is sufficient to define $d$ up to isomorphism. The category of all presheaves forms a {\em category of functors}, and is denoted $\hat{{\cal C}} = $ {\bf Set}$^{{\cal C}^{op}}$.We will briefly describe two concrete applications of this lemma to two important areas in AI and ML in this section: reasoning about causality and reasoning about distances. The Yoneda lemma plays a crucial role in this paper because it defines the concept of a {\em universal representation} in category theory. We first show that associated with universal arrows is the corresponding induced isomorphisms between {\bf {Hom}} sets of morphisms in categories. This universal property then leads to the Yoneda lemma. 

\begin{theorem}
Given any functor $S: D \rightarrow C$, the universal arrow $\langle r, u: c \rightarrow Sr \rangle$ implies a bijection exists between the {\bf {Hom}} sets 
\[ \mbox{{\bf Hom}}_{D}(r, d) \simeq \mbox{{\bf Hom}}_{C}(c, Sd) \]
\end{theorem}

A special case of this natural transformation that transforms the identity morphism {\bf {1}}$_r$ leads us to the Yoneda lemma.
\begin{center}
 \begin{tikzcd}
  D(r,r) \arrow{d}{D(r, f')} \arrow{r}{\phi_r}
    & C(c, Sr) \arrow[]{d}{C(c, S f')} \\
  D(r,d)  \arrow[]{r}[]{\phi_d}
&C(c, Sd)\end{tikzcd}
 \end{center} 

 As the two paths shown here must be equal in a commutative diagram, we get the property that a bijection between the {\bf {Hom}} sets holds precisely when $\langle r, u: c \rightarrow Sr \rangle$ is a universal arrow from $c$ to $S$. Note that for the case when the categories $C$ and $D$ are small, meaning their {\bf Hom} collection of arrows forms a set, the induced functor {\bf {Hom}}$_C(c, S - )$ to {\bf Set} is isomorphic to the functor {\bf {Hom}}$_D(r, -)$. This type of isomorphism defines a universal representation, and is at the heart of the causal reproducing property (CRP) defined below. 

\begin{lemma}
{\bf {Yoneda lemma}}: For any functor $F: C \rightarrow {\bf Set}$, whose domain category $C$ is ``locally small" (meaning that the collection of morphisms between each pair of objects forms a set), any object $c$ in $C$, there is a bijection 

\[ \mbox{Hom}(C(c, -), F) \simeq Fc \]

that defines a natural transformation $\alpha: C(c, -) \Rightarrow F$ to the element $\alpha_c(1_c) \in Fc$. This correspondence is natural in both $c$ and $F$. 
\end{lemma}

There is of course a dual form of the Yoneda Lemma in terms of the contravariant functor $C(-, c)$ as well using the natural transformation $C(-, c) \Rightarrow F$. A very useful way to interpret the Yoneda Lemma is through the notion of universal representability through a covariant or contravariant functor.

\begin{definition}
    A {\bf universal representation} of an object $c \in C$ in a category $C$ is defined as a contravariant functor $F$ together with a functorial representation $C(-, c) \simeq F$ or by a covariant functor $F$ together with a representation $C(c, -) \simeq F$. The collection of morphisms $C(-, c)$ into an object $c$ is called the {\bf presheaf}, and from the Yoneda Lemma, forms a universal representation of the object. 
\end{definition}

Later in this Section, we will see how the Yoneda Lemma gives us a novel perspective on causal inference, as well as distance metrics that are at the core of many applications in AI and ML. For causal inference, any causal influence on a variable $c$ in a causal model defined as an object in a category must be transmitted through its presheaf. The collection of all such causal influences in effect defines the object itself, giving a more rigorous way to substantiate an observation by the philosopher Plato made more than two millennia ago that 

\begin{quote}
    Everything that becomes or changes must do so owing to some cause; for nothing can come to be without a cause \cite{plato}
\end{quote}

Another useful concept was introduced by the mathematician Grothendieck, who made many important contributions to category theory. 

\begin{definition}
    The {\bf category of elements} $\int F$ of a covariant functor $F: C \rightarrow \mbox{{\bf Set}}$ is defined as

    \begin{itemize}
        \item a collection of objects  $(c, x)$ where $c \in C$ and $x \in Fc$

        \item a collection of morphisms $(c, x) \rightarrow (c', x')$ for every morphism $f: c \rightarrow c'$ such that $F f(x) = x'$. 
    \end{itemize}
\end{definition}

\begin{definition}
    The {\bf category of elements} $\int F$ of a contravariant functor $F: C^{op} \rightarrow \mbox{{\bf Set}}$ is defined as

    \begin{itemize}
        \item a collection of objects  $(c, x)$ where $c \in C$ and $x \in Fc$

        \item a collection of morphisms $(c, x) \rightarrow (c', x')$ for every morphism $f: c \rightarrow c'$ such that $F f(x') = x$. 
    \end{itemize}
\end{definition}

There is a natural ``forgetful" functor $\pi: \int F \rightarrow C$ that maps the pairs of objects $(c, x) \in \int F$ to $c \in C$ and maps morphisms $(c, x) \rightarrow (c', x') \in \int F$ to $f: c \rightarrow c' \in C$. Below we will show that the category of elements $\int F$ can be defined through a universal construction as the pullback in the diagram of categories.

A key distinguishing feature of category theory is the use of diagrammatic reasoning. However, diagrams are also viewed more abstractly as functors mapping from some indexing category to the actual category. Diagrams are useful in understanding universal constructions, such as limits and colimits of diagrams. To make this somewhat abstract definition concrete, let us look at some simpler examples of universal properties, including co-products and quotients (which in set theory correspond to disjoint unions). Coproducts refer to the universal property of abstracting a group of elements into a larger one.

 Before we formally the concept of limit and colimits, we consider some examples.  These notions generalize the more familiar notions of Cartesian products and disjoint unions in the category of {\bf {Sets}}, the notion of meets and joins in the category {\bf {Preord}} of preorders, as well as the  least upper bounds and greatest lower bounds in lattices, and many other concrete examples from mathematics. 

\begin{example} 
If  we consider a small  ``discrete'' category ${\cal D}$ whose only morphisms are identity arrows, then the colimit of a functor ${\cal F}: {\cal D} \rightarrow {\cal C}$ is the {\em categorical coproduct} of ${\cal F}(D)$ for $D$, an object of category {\cal D}, is denoted as 
\[ \mbox{Colimit}_{\cal D} F = \bigsqcup_D {\cal F}(D) \]

In the special case when the category {\cal C} is the category {\bf {Sets}}, then the colimit of this functor is simply the disjoint union of all the sets $F(D)$ that are mapped from objects $D \in {\cal D}$. 
\end{example} 

\begin{example} 
Dual to the notion of colimit of a functor is the notion of {\em limit}. Once again, if we consider a small  ``discrete'' category ${\cal D}$ whose only morphisms are identity arrows, then the limit of a functor ${\cal F}: {\cal D} \rightarrow {\cal C}$ is the {\em categorical product} of ${\cal F}(D)$ for $D$, an object of category {\cal D}, is denoted as 
\[ \mbox{limit}_{\cal D} F = \prod_D {\cal F}(D) \]

In the special case when the category {\cal C} is the category {\bf {Sets}}, then the limit of this functor is simply the Cartesian product of all the sets $F(D)$ that are mapped from objects $D \in {\cal D}$. 
\end{example} 

Category theory relies extensively on {\em universal constructions}, which satisfy a universal property. One of the central building blocks is the identification of universal properties through formal diagrams.  Before introducing these definitions in their most abstract form, it greatly helps to see some simple examples. 

 We can illustrate the limits and colimits in diagrams using pullback and pushforward mappings.

\begin{tikzcd}
    & Z\arrow[r, "p"] \arrow[d, "q"]
      & X \arrow[d, "f"] \arrow[ddr, bend left, "h"]\\
& Y \arrow[r, "g"] \arrow[drr, bend right, "i"] &X \sqcup Y \arrow[dr, "r"]  \\ 
& & & R 
\end{tikzcd}

An example of a universal construction is given by the above commutative diagram, where the coproduct object $X \sqcup Y$ uniquely factorizes any mapping $h: X \rightarrow R$, such that any mapping $i: Y \rightarrow R$, so that $h = r \circ f$, and furthermore $i = r \circ g$. Co-products are themselves special cases of the more general notion of co-limits. Figure~\ref{univpr}  illustrates the fundamental property of a {\em {pullback}}, which along with {\em pushforward}, is one of the core ideas in category theory. The pullback square with the objects $U,X, Y$ and $Z$ implies that the composite mappings $g \circ f'$ must equal $g' \circ f$. In this example, the morphisms $f$ and $g$ represent a {\em {pullback}} pair, as they share a common co-domain $Z$. The pair of morphisms $f', g'$ emanating from $U$ define a {\em {cone}}, because the pullback square ``commutes'' appropriately. Thus, the pullback of the pair of morphisms $f, g$ with the common co-domain $Z$ is the pair of morphisms $f', g'$ with common domain $U$. Furthermore, to satisfy the universal property, given another pair of morphisms $x, y$ with common domain $T$, there must exist another morphism $k: T \rightarrow U$ that ``factorizes'' $x, y$ appropriately, so that the composite morphisms $f' \ k = y$ and $g' \ k = x$. Here, $T$ and $U$ are referred to as {\em cones}, where $U$ is the limit of the set of all cones ``above'' $Z$. If we reverse arrow directions appropriately, we get the corresponding notion of pushforward. So, in this example, the pair of morphisms $f', g'$ that share a common domain represent a pushforward pair. 
As Figure~\ref{univpr}, for any set-valued functor $\delta: S \rightarrow$ {\bf {Sets}}, the Grothendieck category of elements $\int \delta$ can be shown to be a pullback in the diagram of categories. Here, {${\bf Set}_*$} is the category of pointed sets, and $\pi$ is a projection that sends a pointed set $(X, x \in X)$ to its \mbox{underlying set $X$.}

\begin{figure}[h]
\hspace{21pt}


\centering
\begin{tikzcd}
  T
  \arrow[drr, bend left, "x"]
  \arrow[ddr, bend right, "y"]
  \arrow[dr, dotted, "k" description] & & \\
    & 
    U\arrow[r, "g'"] \arrow[d, "f'"]
      & X \arrow[d, "f"] \\
& Y \arrow[r, "g"] &Z
\end{tikzcd}
\begin{tikzcd}
  T
  \arrow[drr, bend left, "x"]
  \arrow[ddr, bend right, "y"]
  \arrow[dr, dotted, "k" description] & & \\
    & 
    \int \delta \arrow[r, "\delta'"] \arrow[d, "\pi_\delta"]
      & {\bf Set}_* \arrow[d, "\pi"] \\
& S \arrow[r, "\delta"] & {\bf Set}
\end{tikzcd}
\caption{(\textbf{Left})
 Universal Property of pullback mappings. (\textbf{Right}) The Grothendieck category of elements $\int \delta$ of any set-valued functor $\delta: S \rightarrow$ {\bf {Set}} can be described as a pullback in the diagram of categories. Here, {\bf Set}$_*$ is the category of pointed sets $(X, x \in X)$, and $\pi$ is the ``forgetful" functor that sends a pointed set $(X, x \in X)$ into the underlying set $X$.  } 
\label{univpr}
\end{figure}
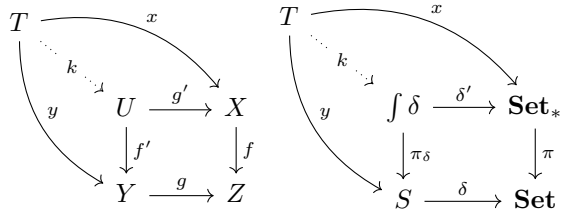

We can now proceed to define limits and colimits more generally. We define a {\em diagram} $F$ of {\em shape} $J$ in a category $C$ formally as a functor $F: J \rightarrow C$. We want to define the somewhat abstract concepts of {\em limits} and {\em colimits}, which will play a central role in this paper in identifying properties of AI and ML techniques.  A convenient way to introduce these concepts is through the use of {\em universal cones} that are {\em over} and {\em under} a diagram.

For any object $c \in C$ and any category $J$, the {\em constant functor} $c: J \rightarrow C$ maps every object $j$ of $J$ to $c$ and every morphism $f$ in $J$ to the identity morphisms $1_c$. We can define a constant functor embedding as the collection of constant functors $\Delta: C \rightarrow C^J$ that send each object $c$ in $C$ to the constant functor at $c$ and each morphism $f: c \rightarrow c'$ to the constant natural transformation, that is, the natural transformation whose every component is defined to be the morphism $f$. 

\begin{definition}
    A {\bf cone over} a diagram $F: J \rightarrow C$ with the {\bf summit} or {\bf apex} $c \in C$ is a natural transformation $\lambda: c \Rightarrow F$ whose domain is the constant functor at $c$. The components $(\lambda_j: c \rightarrow Fj)_{j \in J}$ of the natural transformation can be viewed as its {\bf legs}. Dually, a {\bf cone under} $F$ with {\bf nadir} $c$ is a natural transformation $\lambda: F \Rightarrow c$ whose legs are the components $(\lambda_j: F_j \rightarrow c)_{j \in J}$.

\[\begin{tikzcd}
	&& c &&&& Fj &&&& Fk \\
	\\
	{F j} &&&& Fk &&&& c
	\arrow["{\lambda_j}", from=1-3, to=3-1]
	\arrow["{\lambda_k}"', from=1-3, to=3-5]
	\arrow["{F f}", from=3-1, to=3-5]
	\arrow["Ff", from=1-7, to=1-11]
	\arrow["{\lambda_j}"', from=1-7, to=3-9]
	\arrow["{\lambda_k}", from=1-11, to=3-9]
\end{tikzcd}\]
    
\end{definition}

Cones under a diagram are referred to usually as {\em cocones}. Using the concept of cones and cocones, we can now formally define the concept of limits and colimits more precisely. 

\begin{definition}
    For any diagram $F: J \rightarrow C$, there is a functor 

    \[ \mbox{Cone}(-, F): C^{op} \rightarrow \mbox{{\bf Set}} \]

    which sends $c \in C$ to the set of cones over $F$ with apex $c$. Using the Yoneda Lemma, a {\bf limit} of $F$ is defined as an object $\lim F \in C$ together with a natural transformation $\lambda: \lim F \rightarrow F$, which can be called the {\bf universal cone} defining the natural isomorphism 

    \[ C(-, \lim F) \simeq \mbox{Cone}(-, F) \]

    Dually, for colimits, we can define a functor 

    \[ \mbox{Cone}(F, -): C \rightarrow \mbox{{\bf Set}} \]

    that maps object $c \in C$ to the set of cones under $F$ with nadir $c$. A {\bf colimit} of $F$ is a representation for $\mbox{Cone}(F, -)$. Once again, using the Yoneda Lemma, a colimit is defined by an object $\mbox{Colim} F \in C$ together with a natural transformation $\lambda: F \rightarrow \mbox{colim} F$, which defines the {\bf colimit cone} as the natural isomorphism 

    \[ C(\mbox{colim} F, -) \simeq \mbox{Cone}(F, -) \]
\end{definition}

Limit and colimits of diagrams over arbitrary categories can often be reduced to the case of their corresponding diagram properties over sets. One important stepping stone is to understand how functors interact with limits and colimits. 

\begin{definition}
    For any class of diagrams $K: J \rightarrow C$, a functor $F: C \rightarrow D$ 

    \begin{itemize}
        \item {\bf preserves} limits if for any diagram $K: J \rightarrow C$ and limit cone over $K$, the image of the cone defines a limit cone over the composite diagram $F K: J \rightarrow D$. 

        \item {\bf reflects} limits if for any cone over a diagram $K: J \rightarrow C$ whose image upon applying $F$ is a limit cone for the diagram $F K: J \rightarrow D$ is a limit cone over $K$

        \item {\bf creates} limits if whenever $FK : J \rightarrow D$ has a limit in $D$, there is some limit cone over $F K$ that can be lifted to a limit cone over $K$ and moreoever $F$ reflects the limits in the class of diagrams. 
    \end{itemize}
\end{definition}

To interpret these abstract definitions, it helps to concretize them in terms of a specific universal construction, like the pullback defined above $c' \rightarrow c \leftarrow c''$ in $C$. Specifically, for pullbacks: 

\begin{itemize} 

\item A functor $F$ {\bf preserves pullbacks} if whenever $p$ is the pullback of  $c' \rightarrow c \leftarrow c''$ in $C$, it follows that $Fp$ is the pullback of  $Fc' \rightarrow Fc \leftarrow Fc''$ in $D$.

\item A functor $F$ {\bf reflects  pullbacks}  if  $p$ is the pullback of  $c' \rightarrow c \leftarrow c''$ in $C$ whenever $Fp$ is the pullback of  $Fc' \rightarrow Fc \leftarrow Fc''$ in $D$.

\item A functor $F$ {\bf creates pullbacks} if there exists some $p$ that is the pullback of  $c' \rightarrow c \leftarrow c''$ in $C$ whenever there exists a $d$ such  that $d$ is the pullback of  $Fc' \rightarrow Fc \leftarrow Fc''$ in $F$.

\end{itemize} 

\subsection*{Universality of Diagrams}

In the category {\bf {Sets}}, we know that every object (i.e., a set) $X$ can be expressed as a coproduct (i.e., disjoint union)  of its elements $X \simeq \sqcup_{x \in X} \{ x \}$, where $x \in X$. Note that we can view each element $x \in X$ as a morphism $x: \{ * \} \rightarrow X$ from the one-point set to $X$. The categorical generalization of this result is called the {\em {density theorem}} in the theory of sheaves. First, we define the key concept of a {\em comma category}. 

\begin{definition}
Let $F: {\cal D} \rightarrow {\cal C}$ be a functor from category ${\cal D}$ to ${\cal C}$. The {\bf {comma category}} $F \downarrow {\cal C}$ is one whose objects are pairs $(D, f)$, where $D \in {\cal D}$ is an object of ${\cal D}$ and $f \in$ {\bf {Hom}}$_{\cal C}(F(D), C)$, where $C$ is an object of ${\cal C}$. Morphisms in the comma category $F \downarrow {\cal C}$ from $(D, f)$ to $(D', f')$, where $g: D \rightarrow D'$, such that $f' \circ F(g) = f$. We can depict this structure through the following commutative diagram: 
\begin{center} 
\begin{tikzcd}[column sep=small]
& F(D) \arrow{dl}[near start]{F(g)} \arrow{dr}{f} & \\
  F(D')\arrow{rr}{f'}&                         & C
\end{tikzcd}
\end{center} 
\end{definition} 

We first introduce the concept of a {\em {dense}} functor: 

\begin{definition}
Let {\cal D} be a small category, {\cal C} be an arbitrary category, and $F: {\cal D} \rightarrow {\cal D}$ be a functor. The functor $F$ is {\bf {dense}} if for all objects $C$ of ${\cal C}$, the natural transformation 
\[ \psi^C_F: F \circ U \rightarrow \Delta_C, \ \ (\psi^C_F)_{({\cal D}, f)} = f\]
is universal in the sense that it induces an isomorphism $\mbox{Colimit}_{F \downarrow C} F \circ U \simeq C$. Here, $U: F \downarrow C \rightarrow {\cal D}$ is the projection functor from the comma category $F \downarrow {\cal C}$, defined by $U(D, f) = D$. 

\end{definition} 

A fundamental consequence of the category of elements is that every object in the functor category of presheaves, namely contravariant functors from a category into the category of sets, is the colimit of a diagram of representable objects, via the Yoneda lemma. Notice this is a generalized form of the density notion from the category {\bf {Sets}}.

\begin{theorem}
\label{presheaf-theorem}
{\bf {Universality of Diagrams}}: In the functor category of presheaves {\bf {Set}}$^{{\cal C}^{op}}$, every object $P$ is the colimit of a diagram of representable objects, in a canonical way. 
\end{theorem}

\subsection*{Lifting Problems}
\label{lift} 

Lifting problems provide elegant ways to define solutions to computational problems in category theory regarding the existence of mappings. We will use these lifting diagrams later in this paper. For example, the notion of injective and surjective functions, the notion of separation in topology, and many other basic constructs can be formulated as solutions to lifting problems. Lifting problems define ways of decomposing structures into simpler pieces, and putting them back together again. 
 
 \begin{definition}
 Let ${\cal C}$ be a category. A {\bf {lifting problem}} in ${\cal C}$ is a commutative diagram $\sigma$ in ${\cal C}$. 
 \begin{center}
 \begin{tikzcd}
  A \arrow{d}{f} \arrow{r}{\mu}
    & X \arrow[]{d}{p} \\
  B  \arrow[]{r}[]{\nu}
&Y \end{tikzcd}
 \end{center} 
 \end{definition}
 
 \begin{definition}
 Let ${\cal C}$ be a category. A {\bf {solution to a lifting problem}} in ${\cal C}$ is a morphism $h: B \rightarrow X$ in ${\cal C}$ satisfying $p \circ h = \nu$ and $h \circ f = \mu$ as indicated in the diagram below. 
 \begin{center}
 \begin{tikzcd}
  A \arrow{d}{f} \arrow{r}{\mu}
    & X \arrow[]{d}{p} \\
  B \arrow[ur,dashed, "h"] \arrow[]{r}[]{\nu}
&Y \end{tikzcd}
 \end{center} 
 \end{definition}

 \begin{definition}
 Let ${\cal C}$ be a category. If we are given two morphisms $f: A \rightarrow B$ and $p: X \rightarrow Y$ in ${\cal C}$, we say that $f$ has the {\bf {left lifting property}} with respect to $p$, or that p has the {\bf {right lifting property}} with respect to f if for every pair of morphisms $\mu: A \rightarrow X$ and $\nu: B \rightarrow Y$ satisfying the equations $p \circ \mu = \nu \circ f$, the associated lifting problem indicated in the diagram below. 
 \begin{center}
 \begin{tikzcd}
  A \arrow{d}{f} \arrow{r}{\mu}
    & X \arrow[]{d}{p} \\
  B \arrow[ur,dashed, "h"] \arrow[]{r}[]{\nu}
&Y \end{tikzcd}
 \end{center} 
admits a solution given by the map $h: B \rightarrow X$ satisfying $p \circ h = \nu$ and $h \circ f = \mu$. 
 \end{definition}
 
 \begin{example}
 Given the paradigmatic non-surjective morphism $f: \emptyset \rightarrow \{ \bullet \}$, any morphism p that has the right lifting property with respect to f is a {\bf {surjective mapping}}. 
. 
 \begin{center}
 \begin{tikzcd}
  \emptyset \arrow{d}{f} \arrow{r}{\mu}
    & X \arrow[]{d}{p} \\
  \{ \bullet \} \arrow[ur,dashed, "h"] \arrow[]{r}[]{\nu}
&Y \end{tikzcd}
 \end{center} 

 \end{example}

  \begin{example}
 Given the paradigmatic non-injective morphism $f: \{ \bullet, \bullet \} \rightarrow \{ \bullet \}$, any morphism p that has the right lifting property with respect to f is an {\bf {injective mapping}}. 
.
 \begin{center}
 \begin{tikzcd}
  \{\bullet, \bullet \} \arrow{d}{f} \arrow{r}{\mu}
    & X \arrow[]{d}{p} \\
  \{ \bullet \} \arrow[ur,dashed, "h"]  \arrow[]{r}[]{\nu}
&Y \end{tikzcd}
 \end{center} 

 \end{example}

\subsection{Adjoint Functors} 

Adjoint functors naturally arise in a number of contexts, among the most important being between ``free" and ``forgetful" functors. Let us consider a canonical example that  is of prime significance in many applications in AI and ML. 

\begin{figure}[h] 
\centering
\caption{Adjoint functors provide an elegant characterization of the relationship between the category of statistical models and that of causal models. Statistical models can be viewed as the result of applying a ``forgetful" functor to a causal model that drops the directional structure in a causal model, whereas causal models can be viewed as ``words" in a ``free" algebra that results from the left adjoint functor to the forgetful functor.  \label{causalstatistical}}
\vskip 0.1in
\begin{minipage}{0.7\textwidth}
\vskip 0.1in
\includegraphics[scale=0.45]{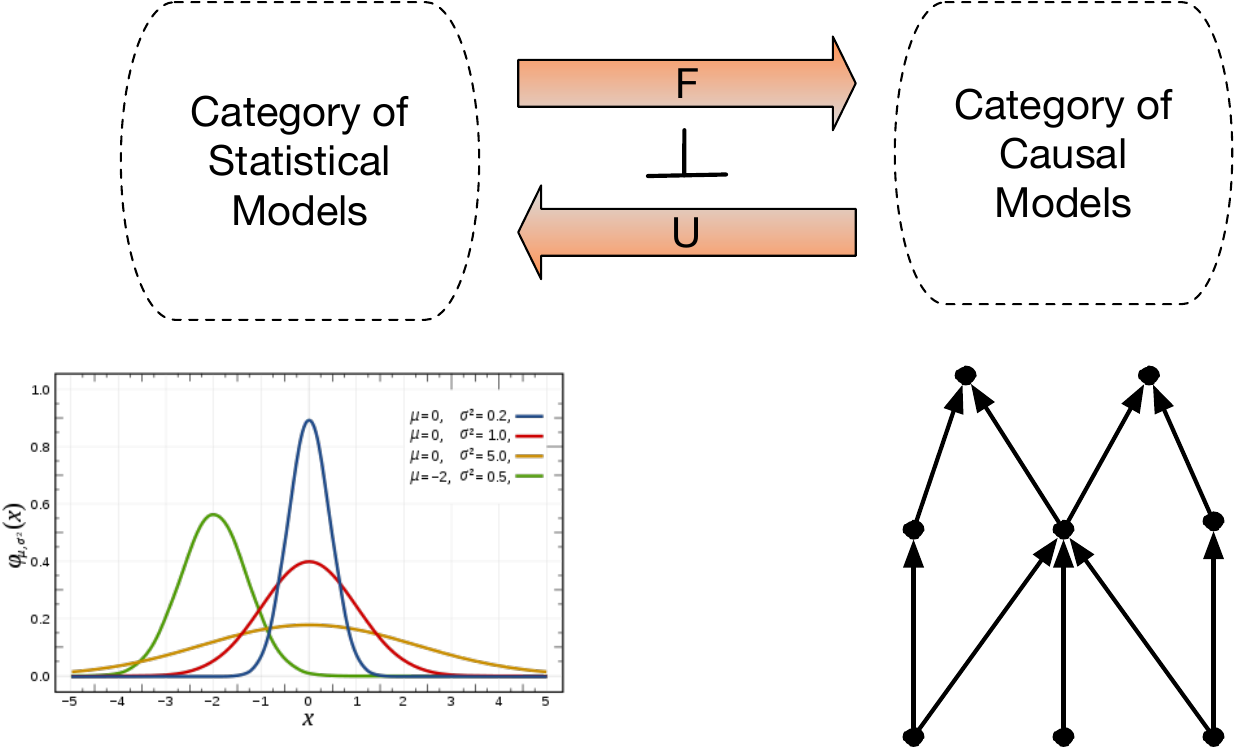}
\end{minipage}
\end{figure}

Figure~\ref{causalstatistical} provides a high level overview of the relationship between a category of statistical models and a category of causal models that can be seen as being related by a pair of adjoint ``forgetful-free" functors. A statistical model can be abstractly viewed in terms of its conditional independence properties. More concretely, the category of {\em separoids}, defined in Section 2, consists of objects called separoids $(S, \leq)$, which are semilattices with a preordering $\leq$ where the elements $x, y, z \in S$ denote entities  in a statistical model. We define a ternary relation $(\bullet \perp \bullet | \bullet) \subseteq S \times S \times S$, where $(x \perp y | z)$ is interpreted as the statement $x$ is conditionally independent of $y$ given $z$ to denote a relationship between triples that captures abstractly the property that occurs in many applications in AI and ML. For example, in statistical ML, a sufficient statistic $T(X)$ of some dataset $X$, treated as a random variable, is defined to be any function for which the conditional independence relationship $(X \perp  \theta| T(X))$, where $\theta \in \mathbb{R}^k$ denotes the parameter vector of some statistical model $P(X)$ that defines the true distribution of the data. Similarly, in causal inference, $(x \perp y | z)  \Rightarrow p(x, y, z) = p(x | z) p(y | z)$ denotes a statement about the probabilistic conditional independence of $x$ and $y$ given $z$. In causal inference, the goal is to recover a partial order defined as a directed acyclic graph (DAG) that ascribes causality among a set of random variables from a dataset specifying a sample of their joint distribution. It is well known that without non-random interventions, causality cannot be inferred uniquely, since because of Bayes rule, there is no way to distinguish causal models such as $x \rightarrow y \rightarrow z$ from the reverse relationship $z \rightarrow y \rightarrow x$. In both these models, $x \perp z | y$ and because of Bayes inversion, one model can be recovered from the other. We can define a ``free-forgetful" pair of adjoint functors between the category of conditional independence relationships, as defined by separoid objects, and the category of causal models parameterized by DAG models.

We first review some basic material relating to adjunctions defined by adjoint functors, before proceeding to describe the theory of monads, as the two are intimately related. Our presentation of adjunctions and monads is based on Riehl's excellent textbook on category theory \cite{riehl2017category} to which the reader is referred to for a more detailed explanation.
Adjunctions are defined by an opposing pair of functors $F: C \leftrightarrow D: G$ that can be defined more precisely as follows. 

\begin{definition}
    An {\bf adjunction} consists of a pair of functors $F: C \rightarrow D$ and $G: D \rightarrow C$, where $F$ is often referred to {\em left adjoint} and  $G$ is referred to as the {\em right adjoint}, that result in the following isomorphism relationship holding between their following sets of homomorphisms in categories $C$ and $D$: 

    \[ D(Fc, d) \simeq C(c, Gd) \]
\end{definition}

We can express the isomorphism condition more explicitly in the form of the following commutative diagram: 

\begin{center}
\begin{tikzcd}
  D(Fc, d) \arrow[r, "\simeq"] \arrow[d, "k_*"]
    & C(c, Gd) \arrow[d, "Gk_*" ] \\
  D(Fc, d') \arrow[r,  "\simeq"]
& C(c, Gd')
\end{tikzcd}
\end{center}

Here, $k: d \rightarrow d'$ is any morphism in $D$, and $k_*$ denotes the ``pullback" of $k$ with the mapping $f: Fc \rightarrow d$ to yield the composite mapping $k \circ f$. The adjunction condition holds that the transpose of this composite mapping is equal to the composite mapping $g: c \rightarrow Gd$ with $G k: Gd \rightarrow G d'$. We can express this dually as well, as follows:

\begin{center}
\begin{tikzcd}
  D(Fc, d) \arrow[r, "\simeq"] \arrow[d, "Fh^*"]
    & C(c, Gd) \arrow[d, "h^*" ] \\
  D(Fc', d) \arrow[r,  "\simeq"]
& C(c', Gd')
\end{tikzcd}
\end{center}

where now $h: c' \rightarrow c$ is a morphism in $C$, and $h^*$ denote the ``pushforward" of $h$. Once again, the adjunction condition is a statement that the transpose of the composite mapping $f \circ Fh: F c' \rightarrow d$ is identical to the composite of the mappings $h: c \rightarrow c'$ with $f: c \rightarrow Gd$.

It is common to denote adjoint functors in this turnstile notation, indicating that $F: C \rightarrow D$ is left adjoint to $G: D \rightarrow C$, or more simply as $F \vdash G$. 

\[
        \begin{tikzcd}
            \mathcal{D}\arrow[r, shift left=.75ex, "G"{name=G}] & \mathcal{C}\arrow[l, shift left=.75ex, "F"{name=F}] 
            \arrow[phantom, from=F, to=G, "\dashv" rotate=90].      
        \end{tikzcd}
    \]

We can use the concept of universal arrows introduced in Section 2 to give more insight into adjoint functors. The adjunction condition for a pair of adjoint functors $F \vdash G$ 

\[ D(Fc, d) \simeq C(c, Gd) \]

implies that for any object $c \in C$, the object $Fc \in D$ represents the functor $C(c, G -): D \rightarrow {\bf Set}$. Recall from the Yoneda Lemma that the natural isomorphism $D(Fc, -) \simeq C(c, G-)$ is determined by an element of $C(c,GFc)$, which can be viewed as the transpose of $1_{Fc}$. Denoting such elements as $\eta_c$, they can be assembled jointly into the natural transformation $\eta: 1_C \rightarrow GF$. Below we will see that this forms one of the conditions for an endofunctor to define a monad. 

\begin{theorem}
    The {\bf unit} $\eta: 1_C \rightarrow GF$ is a natural transformation defined by an adjunction $F \vdash G$, whose component $\eta_c: c \rightarrow GF c$ is defined to be the transpose of the identity morphism $1_{Fc}$. 
\end{theorem}

{\bf Proof:} We need to show that for every $f: c \rightarrow c'$, the following diagram commutes, which follows from the definition of adjunction and the isomorphism condition that it imposes, as well as the obvious commutativity of the second transposed diagram below the first one. 

\begin{center}
\begin{tikzcd}
  c \arrow[r, "\eta_c"] \arrow[d, "f"]
    & GF c\arrow[d, "GF f" ] \\
  c' \arrow[r,  "\eta_{c'}"]
& GF c'
\end{tikzcd}
\end{center}

\begin{center}
\begin{tikzcd}
  Fc \arrow[r, "1_{Fc}"] \arrow[d, "Ff"]
    & Fc \arrow[d, "F f" ] \\
  Fc' \arrow[r,  "1_{Fc'}"]
& Fc'
\end{tikzcd}
\end{center}

The dual of the above theorem leads to the second major component  of an adjunction. 

\begin{theorem}
    The {\bf counit} $\epsilon: FG \Rightarrow 1_D$ is a natural transformation defined by an adjunction $F \vdash G$, whose components $\epsilon_c: F G d \rightarrow d$ at $d$ is  defined to be the transpose of the identity morphism $1_{Gd}$. 
\end{theorem}

Adjoint functors interact with universal constructions, such as limits and colimits, in ways that turn out to be important for a variety of applications in AI and ML. We state the main results here, but refer the reader to \cite{riehl2017category} for detailed proofs. Before getting to the general case, it is illustrative to see the interaction of limits and colimits with adjoint functors for preorders. Recall from above that separoids are defined by a preorder $(S, \leq)$ on a join lattice of elements from a set $S$. Given two separoids $(S, \leq_S)$ and $(T, \leq_T)$, we can define the functors $F: S \rightarrow T$ and $G: T \rightarrow S$ to be order-preserving functions such that 

\[ Fa \leq_T b \ \ \ \mbox{if and only if} \ \ \ a \leq_S Gb \]

Such an adjunction between preorders is often called a {\em Galois connection}.  For preorders, the limit is defined by the {\em meet} of the preorder, and the colimit is defined by the {\em join} of the preorder. We can now state a useful result. For a fuller discussion of preorders and their applications from a category theory perspective, see \cite{fong2018seven}. 

\begin{theorem}
    {\bf Right adjoints preserve meets in a preorder}: Let $f: P \rightarrow Q$ be left adjoint to $g: Q \rightarrow P$, where $P, Q$ are both preorders, and $f$ and $g$ are monotone order-preserving functions. For any subset $A \subseteq Q$, let $g(A) = \{ g(a) | a \in Q \}$. If $A$ has a meet $\bigwedge A \in Q$, then $g(A)$ has a meet $\wedge g(A) \in P$, and we can see that $g(\wedge A) \simeq \bigwedge g(A) $, that is, right adjoints preserve meets. Similarly, left adjoints preserve meets, so that if $A \subset P$ such that $\bigvee A \in P$ then $f(A)$ has a join $\vee f(A) \in Q$ and we can set $f(\vee A) \simeq \bigvee f(A)$, so that left adjoints preserve joins. 
\end{theorem}

{\bf Proof:} The proof is not difficult in this special case of the category being defined as a preorder. If $f: P \rightarrow Q$ and $g: Q \rightarrow P$ are monotone adjoint maps on preorders $P, Q$, and $A \subset Q$ is any subset such that its meet is $m = \wedge A$. Since $g$ is monotone, $g(m) \leq g(a), \ \forall a \in A$, hence it follows that $g(m) \leq g(A)$. To show that $g(m$ is the greatest lower bound, if we take any other lower bound $b \leq g(a), \ \forall a \in A$, then we want to show that $b \leq g(m)$. Since $f$ and $g$ are adjoint, for every $p \in  P, q \in Q$, we have

\[ p \leq g(f(p)) \ \ \ \mbox{and} \ \ \ f(g(q)) \leq q \]

Hence, $f(b) \leq a$ for all $a \in A$, which implies $f(b)$ is  a lower bound for $A$ on $Q$. Since the meet $m$ is the greatest lower bound, we have $f(b) \leq m$. Using the Galois connection, we see that $b \leq g(m)$, and hence showing that $g(m)$ is the greatest lower bound as required. An analogous proof follows to show that left adjoints preserve joins. $\qed$

We can now state the more general cases for any pair of adjoint functors, as follows. 

\begin{theorem}
    A category ${\cal C}$ admits all limits of diagrams indexed by a small category ${\cal J}$ if and only if the constant functor $\Delta: {\cal C} \rightarrow {\cal C}^{{\cal J}}$ admits a right adjoint, and admits all colimits of ${\cal J}$-indexed diagrams if and only if $\Delta$ admits a left adjoint. 
\end{theorem}

By way of explanation, the constant functor $c: J \rightarrow C$ sends every object of $J$ to $c$ and every morphism of $J$ to the identity morphism  $1_c$. Here, the constant functor $\Delta$ sends every object $c$ of $C$ to the constant diagram $\Delta c$, namely the functor that maps each object $i$ of $J$ to the object $c$ and each morphism of $J$ to the identity $1_c$. The theorem follows from the definition of the universal properties of colimits and limits. Given any object $c \in C$, and any diagram (functor) $F \in {\cal C}^{{\cal J}}$, the set of morphisms ${\cal C}^{{\cal J}}(\Delta c, F)$ corresponds to the set of natural transformations from the constant ${\cal J}$-diagram at $c$ to the diagram $F$. These natural transformations precisely correspond to the cones over $F$ with summit $c$ in the definition given earlier in Section 2. It follows that there is an object $\lim F \in {\cal C}$ together with an isomorphism 

\[ {\cal C}^{{\cal J}}(\Delta c, F) \simeq {\cal C}(c, \lim F) \]

We can now state the more general result that we showed above for the special case of adjoint functors on preorders. 

\begin{theorem}
    Right adjoints preserve limits, whereas left adjoints preserve colimits. 
\end{theorem}

\subsection{Ends, Coends, and Kan Extensions} 

Perhaps the most canonical category for formulating UIGs is the category of {\em wedges}, which are defined by a collection of objects comprised of bifunctors $F: {\cal C}^{op} \times C \rightarrow {\cal D}$, and a collection of arrows between each pair of bifunctors $F, G$ called a {\em dinatural transformation} (as an abbreviation for diagonal natural transformation). We will see below that the initial and terminal objects in the category of wedges correspond to a beatiful idea first articulated by Yoneda called the {\em coend} or {\em end} \cite{yoneda-end}. \cite{loregian_2021} has an excellent treatment of coend calculus, which we will use below. 

\begin{definition}
    Given a pair of bifunctors $F, G: {\cal C}^{op} \times {\cal C} \rightarrow {\cal D}$, a {\bf dinatural transformation} is defined as follows: 

\[\begin{tikzcd}
	&& {F(c',c)} \\
	{F(c,c)} &&&& {F(c',c')} \\
	\\
	{G(c,c)} &&&& {G(c',c')} \\
	&& {G(c,c')}
	\arrow["{F(f,c)}", from=1-3, to=2-1]
	\arrow["{F(c',f)}"', from=1-3, to=2-5]
	\arrow[dashed, from=2-1, to=4-1]
	\arrow[dashed, from=2-5, to=4-5]
	\arrow["{G(c,f)}", from=4-1, to=5-3]
	\arrow["{G(f,c)}"', from=4-5, to=5-3]
\end{tikzcd}\]

\end{definition}

As \cite{loregian_2021} observes, just as a natural transformation interpolates between two regular functors $F$ and $G$ by filling in the gap between their action on a morphism $Ff$ and $Fg$ on the codomain category, a dinatural transformation ``fills in the gap" between the top of the hexagon above and the bottom of the hexagon. 

We can define a {\em constant bifunctor} $\Delta_d: {\cal C}^{op} \times {\cal C} \rightarrow {\cal D}$ by the object it maps everything to, namely the input pair of objects $(c, c') \rightarrow d$ are both mapped to the object $d \in {\cal D}$, and the two input morphisms $(f, f') \rightarrow {\bf 1}_d$ are both mapped to the identity morphism on $d$. We can now define {\em wedges} and {\em cowedges}. 

\begin{definition}
    A {\bf wedge} for a bifunctor $F: {\cal C}^{op} \times {\cal C} \Rightarrow {\cal D}$ is a dinatural transformation $\Delta_d \rightarrow F$ from the constant functor on the object $d \in {\cal D}$ to $F$. Dually, we can define a {\bf cowedge} for a bifunctor $F$ by the dinatural transformation $P \Rightarrow \Delta_d$. 
\end{definition}

We can now define a {\em category of wedges}, each of whose objects are wedges, and for arrows, we choose arrows in the co-domain category that makes the diagram below commute. 

\begin{definition}
    Given a fixed bifunctor $F: {\cal C}^{op} \times {\cal C} \rightarrow {\cal D}$, we define the {\bf category of wedges} ${\cal W}(F)$ where each object is a wedge $\Delta_d \Rightarrow F$ and given a pair of wedges $\Delta_d \Rightarrow F$ and $\Delta_d' \Rightarrow F$, we choose an arrow $f: d \rightarrow d'$ that makes the following diagram commute: 

\[\begin{tikzcd}
	d &&&& {d'} \\
	\\
	&& {F(c,c)}
	\arrow["f", from=1-1, to=1-5]
	\arrow["{\alpha_{cc}}"', from=1-1, to=3-3]
	\arrow["{\alpha'_{cc}}", from=1-5, to=3-3]
\end{tikzcd}\]
Analogously, we can define a {\bf category of cowedges} where each object is defined as a cowedge $F \Rightarrow \Delta_d$. 
\end{definition}

With these definitions in place, we can once again define the universal property in terms of initial and terminal objects. In the category of wedges and cowedges, these have special significance for formulating and solving UIGs, as we will see in the next section. 

\begin{definition}
    Given a bifunctor $F: {\cal C}^{op} \times {\cal C} \rightarrow {\cal D}$, the {\bf end} of $F$ consists of a terminal wedge $\omega: \underline{{\bf end}}(F) \Rightarrow F$. The object $\underline{{\bf end}}(F) \in D$ is itself called the end. Dually, the {\bf coend} of $F$ is the initial object in the category of cowedges $F \Rightarrow \underline{{\bf coend}}(F)$, where the object $\underline{{\bf coend}}(F) \in {\cal D}$ is itself called the coend of $F$.  
\end{definition}

Remarkably, probabilities can be formally shown to define ends of a category \cite{Avery_2016}, and topological embeddings of datasets, as implemented in popular dimensionality reduction methods like UMAP \cite{umap}, correspond to coends \cite{maclane:71}.  These connections suggest the canonical importance of the category of wedges and cowedges in formulating and solving UIGs. First, we introduce another universal construction, the Kan extension, which turns out to be the basis of every other concept in category theory. 

\subsection*{Extending Functors rather than Functions: Kan Extensions}

Often, in machine learning, we are given samples of a function defined on some subset $f: A \rightarrow B$ and we want to extend the function over a larger set $A \subset M$, but there is no obvious or canonical extension of functions. This ill-defined nature of machine learning has prompted a large variety of solutions, such as regularization or Occam's razor (prefer the simplest function). In contrast, if we are given a functor $F: A \rightarrow B$, and we want to extend the functor to a larger category $M$, there are only two canonical solutions that present themselves. These are referred to as the Kan extensions. There is a close connection between the notions of ends and coends and Kan extensions \cite{loregian_2021}, as well as to the categorical foundation of probability \cite{Avery_2016}. 

\begin{definition}
A {\bf {left Kan extension}} of a functor ${\cal F}: {\cal C} \rightarrow {\cal E}$ along another functor ${\cal K}: {\cal C} \rightarrow {\cal D}$, is a functor $\mbox{Lan}_{\cal K} {\cal F}: {\cal D} \rightarrow {\cal E}$ with a natural transformation $\eta: F \Rightarrow \mbox{Lan}_F \circ K$ such that for any other such pair $(G: {\cal D} \rightarrow {\cal E}, \gamma: F \Rightarrow G K)$, $\gamma$ factors uniquely through $\eta$. In other words, there is a unique natural transformation $\alpha: \mbox{Lan}_F \Rightarrow G$. \\

\begin{center}
\begin{tikzcd}[row sep=huge, column sep=huge] 
 \mathcal{C} \arrow[dr, "\mathcal{K}"'{name=F}] 
 \arrow[rr, "\mathcal{F}", ""{name=H, below}] && \mathcal{E} \\ 
 & |[alias=D]| \mathcal{D} \arrow[ur, swap, dashed,
 "\operatorname{Lan}_{\mathcal{K}}\mathcal{F}"] 
 \arrow[Rightarrow, from=H, to=D, "\eta",shorten >=1em,shorten <=1em] 
\end{tikzcd} 
\end{center} 

\end{definition}

\begin{definition}

A {\bf {right Kan extension}} of a functor ${\cal F}: {\cal C} \rightarrow {\cal E}$ along another functor ${\cal K}: {\cal C} \rightarrow {\cal D}$, is a functor $\eta: \mbox{Ran}_F \circ K \rightarrow F$  with a natural transformation $\eta: \mbox{Lan}_F \circ K \Rightarrow {\cal F}$ such that for any other such pair $(G: {\cal D} \rightarrow {\cal E}, \gamma: G K \Rightarrow F)$, $\gamma$ factors uniquely through $\eta$. In other words, there is a unique natural transformation $\alpha: G \Rightarrow \mbox{Ran}_F$.

\begin{center}
\begin{tikzcd}[row sep=huge, column sep=huge] 
 \mathcal{C} \arrow[dr, "\mathcal{K}"'{name=F}] 
 \arrow[rr, "\mathcal{F}", ""{name=H, below}] && \mathcal{E} \\ 
 & |[alias=D]| \mathcal{D} \arrow[ur, swap, dashed,
 "\operatorname{Ran}_{\mathcal{K}}\mathcal{F}"] 
 \arrow[Rightarrow, to=H, from=D, "\eta",shorten >=1em,shorten <=1em] 
\end{tikzcd} 
\end{center} 

\end{definition} 

\subsection{Monads and Categorical Probability} 

Now, we turn to defining monads more formally, and relate them to adjoint functors. Categorically speaking, probabilities are essentially  monads \cite{Avery_2016,giry1982}. Like the case with coalgebras, which we discussed extensively in previous Sections, monads also are defined by an endofunctor on a category, but one that has some special properties. These additional properties make monads possess algebraic structure, which leads to many interesting properties. Monads provide a categorical foundation for probability, based on the property that the set of all distributions on a measurable space is itself a measurable space \cite{Avery_2016}.  The well-known {\em Giry} monad \cite{giry1982} been also shown to arise as the {\em codensity monad}  of a forgetful functor from the category of convex sets with affine maps to the category of measurable spaces \cite{Avery_2016}. Our goal in this paper is to apply monads to shed light into causal inference.   We first review the basic definitions of monads, and then discuss monad algebras, which provide ways of characterizing categories. 

Consider the pair of adjoint free and forgetful functors between graphs and categories. Here, the domain category is {\bf Cat}, the category of all categories whose objects are categories and whose morphisms are functors. The co-domain category is the category {\bf Graph} of all graphs, whose objects are directed graphs, and whose morphisms are graph homomorphisms. Here, a monad $T = U \circ F$ is induced by composing the ``free" functor $F$ that maps a graph into its associated ``free" category, and the ``forgetful" functor $U$ that maps a category into its associated graph. The monad $T$  in effect takes a directed graph $G$ and computes its transitive closure $G_{tc}$. More precisely,  for every (directed) graph $G$, there is a universal arrow from $G$ to the ``forgetful" functor $U$ mapping the category {\bf Cat} of all categories to {\bf Graph}, the category of all (directed) graphs, where for any category $C$, its associated graph is defined by $U(C)$. 

To understand this functor, simply consider a directed graph $U(C)$ as a category $C$ forgetting the rule for composition. That is, from the category $C$, which associates to each pair of composable arrows $f$ and $g$, the composed arrow $g \circ f$, we derive the underlying graph $U(G)$ simply by forgetting which edges correspond to elementary functions, such as $f$ or $g$, and which are composites. The universal arrow from a graph $G$ to the forgetful functor $U$  is defined as a pair $\langle G, u: G \rightarrow U(C) \rangle$, where $u$ is a a graph homomorphism. This arrow possesses the following {\em universal property}: for every other pair $\langle D, v: G \rightarrow H \rangle$, where $D$ is a category, and $v$ is an arbitrary graph homomorphism, there is a functor  $f': C \rightarrow D$, which is an arrow in the category {\bf Cat} of all categories, such that {\em every} graph homomorphism $\phi: G \rightarrow H$ uniquely factors through the universal graph homomorphism $u: G \rightarrow U(C)$  as the solution to the equation $\phi = U(f') \circ u$, where $U(f'): U(C) \rightarrow H$ (that is, $H = U(D)$).  Namely, the dotted arrow defines a graph homomorphism $U(f')$ that makes the triangle diagram ``commute", and the associated ``extension" problem of finding this new graph homomorphism $U(f')$ is solved by ``lifting" the associated category arrow $f': C \rightarrow D$. In causal inference using graph-based models, the transitive closure graph is quite important in a number of situations. It can be the initial target of a causal discovery algorithm that uses conditional independence oracles. It is also common in graph-based causal inference \cite{pearl-book} to model causal effects through a directed acyclic graph (DAG) $G$, which specifies its algebraic structure, and through a set of probability distributions on $G$ that specifies its semantics $P(G)$. Often, reasoning about causality in a DAG requires examining paths that lead from some vertex $x$, representing a causal variable, to some other vertex $y$. The process of constructing the transitive closure of a DAG provides a simple example of a causal monad. 

\begin{definition}
    A {\bf monad} on a category $C$ consists of

    \begin{itemize} 
    \item An endofunctor $T: C \rightarrow C$
    \item A {\bf unit} natural transformation $\eta: 1_C \Rightarrow T$ 
    \item A {\bf multiplication} natural transformation $\mu: T^2 \rightarrow T$
    \end{itemize} 
    such that the following commutative diagram in the category $C^C$ commutes (notice the arrows in this diagram are natural transformations as each object in the diagram is a functor). 
\end{definition}

\begin{center}

\[\begin{tikzcd}
	{T^3} &&&& {T^2} \\
	\\
	\\
	{T^2} &&&& T
	\arrow["{T \mu}", from=1-1, to=1-5]
	\arrow["{\mu T}"', from=1-1, to=4-1]
	\arrow["\mu"', from=4-1, to=4-5]
	\arrow["\mu", from=1-5, to=4-5]
\end{tikzcd}\]

\[\begin{tikzcd}
	T &&& {T^2} &&& T \\
	\\
	\\
	&&& T
	\arrow["{T \eta}"', from=1-7, to=1-4]
	\arrow["\mu", from=1-4, to=4-4]
	\arrow["{\eta T}", from=1-1, to=1-4]
	\arrow["{1_T}"', from=1-1, to=4-4]
	\arrow["{1_T}"', from=1-7, to=4-4]
\end{tikzcd}\]

\end{center}

It is useful to think of monads as the ``shadow" cast by an adjunction on the category corresponding to the co-domain of the right adjoint $G$. Consider the following pair of adjoint functors $F \vdash G$. 

\[
        \begin{tikzcd}
            \mathcal{C}\arrow[r, shift left=.75ex, "F"{name=F}] & \mathcal{D}\arrow[l, shift left=.75ex, "G"{name=G}] 
            \arrow[phantom, from=F, to=G, "\dashv" rotate=270].      
        \end{tikzcd}  \ \ \ \ \eta: 1_C \Rightarrow UF, \ \ \ \epsilon: FU \Rightarrow 1_D
    \]

In the language of ML, if we treat category $C$ as representing ``labeled training data" where we have full information, and category $D$ as representing a new domain for which we have no labels, what can we conclude about category $D$ from the information we have from the adjunction? The endofunctor $UF$ on $C$ is of course available to us, as is the natural transformation $\eta: 1_C \Rightarrow UF$. The map $\epsilon_A: FG A \rightarrow A$ for any object $A \in D$ is an endofunctor on $D$, about which we have no information. However, the augmented natural transformation $U \epsilon F A: G F G F A \rightarrow G F A$ can be studied in category $C$. From this data, what can we conclude about the objects in category $D$? In response to the natural question of whether every monad can be defined by a pair of adjoint functors, two solutions arose that came about from two different pairs of adjoint functors. These are referred to as the {\em Eilenberg-Moore} category and the {\em Kleisli} category \cite{maclane:71}. 

\subsection*{Codensity Monads and Probability} 

A striking recent finding is that categorical probability structures, such as Giry monads, are in essence {\em codensity monads} that result from extending a certain functor along itself \cite{Avery_2016}. 

\begin{definition}

A {\bf {codensity monad}} $T^{\cal F}$ of a functor ${\cal F}$ is the right Kan extension of ${\cal F}$ along itself (if it exists). The codensity monad inherits the university property from the Kan extension.  

\begin{center}
\begin{tikzcd}[row sep=huge, column sep=huge] 
 \mathcal{C} \arrow[dr, "\mathcal{F}"'{name=F}] 
 \arrow[rr, "\mathcal{F}", ""{name=H, below}] && \mathcal{E} \\ 
 & |[alias=D]| \mathcal{E} \arrow[ur, swap, dashed,
 "\operatorname{T}^{\mathcal{F}}"] 
 \arrow[Rightarrow, to=H, from=D, "\eta",shorten >=1em,shorten <=1em] 
\end{tikzcd} 
\end{center} 

\end{definition} 

Codensity monads can also be written using Yoneda's abstract integral calculus as ends: 

\[ T^{\cal F} e = \int_{c \in C} [{\cal E}(e, {\cal F}c), {\cal F}c]\]

Here, the notation $[A, m]$, where $A$ is any set, and $m$ is any object of a category ${\cal M}$, denotes the product in ${\cal M}$ of $A$ copies of $m$. 

\begin{definition}
    A {\bf convex set} $c$ is a convex subset of a real vector space, where for all $x, y \in c$, and for all $r \in [0,1]$, the convex combination $r x + (1 - r) y \in c$. An {\bf affine} map $h: c \rightarrow c'$ is a function such that $h(x +_r y) = h(x) +_r h(y)$ where $x +_r y = r x + (1 - r) y, r \in [0,1]$. 
\end{definition}

To define categorical probability as codensity monads, we need to  define forgetful functors from the category ${\cal C}'$ of compact convex subsets of $\mathbb{R}{^n}$ with affine maps to the category ${\bf Meas}$ of measurable spaces and measurable functions. In addition, let ${\cal D}'$ be the category ${\cal C}'$ with the object $d_0$ adjoined, where $d_0$ is the convex set of convergent sequences in the unit interval $I = [0,1]$. 

\begin{theorem}\cite{Avery_2016}
    The ${\cal C}'$ be the category of compact convex subsets of $\mathbb{R}^n$ for varying $n$ with affine maps between them, and let ${\cal D}'$ be the same with the object $d_0$ adjoined. Then, the codensity monads of the forgetful functors $U': {\cal C}' \rightarrow {\bf Meas}$ and $V': {\cal D}' \rightarrow {\bf Meas}$ are the finitely additive Giry monad and the Giry monad respectively. 
\end{theorem}

The well-known Giry monad \cite{giry1982} defines probabilities in both the discrete case and the continuous case (over Polish spaces) in terms of endofunctor on the category of measurable spaces ${\bf Meas}$. 

\subsection{Sheaves and Topoi}

\label{sheavestopoi} 

In this section, we define an important categorical structure defined by sheaves and topoi \cite{maclane:sheaves}.  Yoneda embeddings $\yo(x): {\cal C}^{op} \rightarrow {\bf Sets}$ define (pre)sheaves, which satisfy a number of crucial properties that make it remarkably similar to the category of {\bf Sets}. The sheaf condition plays an important role in many applications of machine learning, from dimensionality reduction \cite{umap} to causal inference \cite{DBLP:journals/entropy/Mahadevan23}. \cite{maclane:sheaves} provides an excellent overview of sheaves and topoi, and how remarkably they unify much of mathematics, from geometry to logic and topology. We will give only the briefest of overviews here, and apply in the main ideas to the study of UIGs. 

\begin{figure}[h] 
\centering
\caption{Two applications of sheaf theory in AI: (top) minimizing travel costs in weighted graphs satisfies the sheaf principle, one example of which is the Bellman optimality principle in dynamic programming \cite{DBLP:books/lib/Bertsekas05} and reinforcement learning \cite{bertsekas:rlbook,DBLP:books/lib/SuttonB98} (bottom): Approximating a function over a topological space must satisfy the sheaf condition. \label{sheaves}}
\vskip 0.1in
\begin{minipage}{0.7\textwidth}
\vskip 0.1in
\includegraphics[scale=0.35]{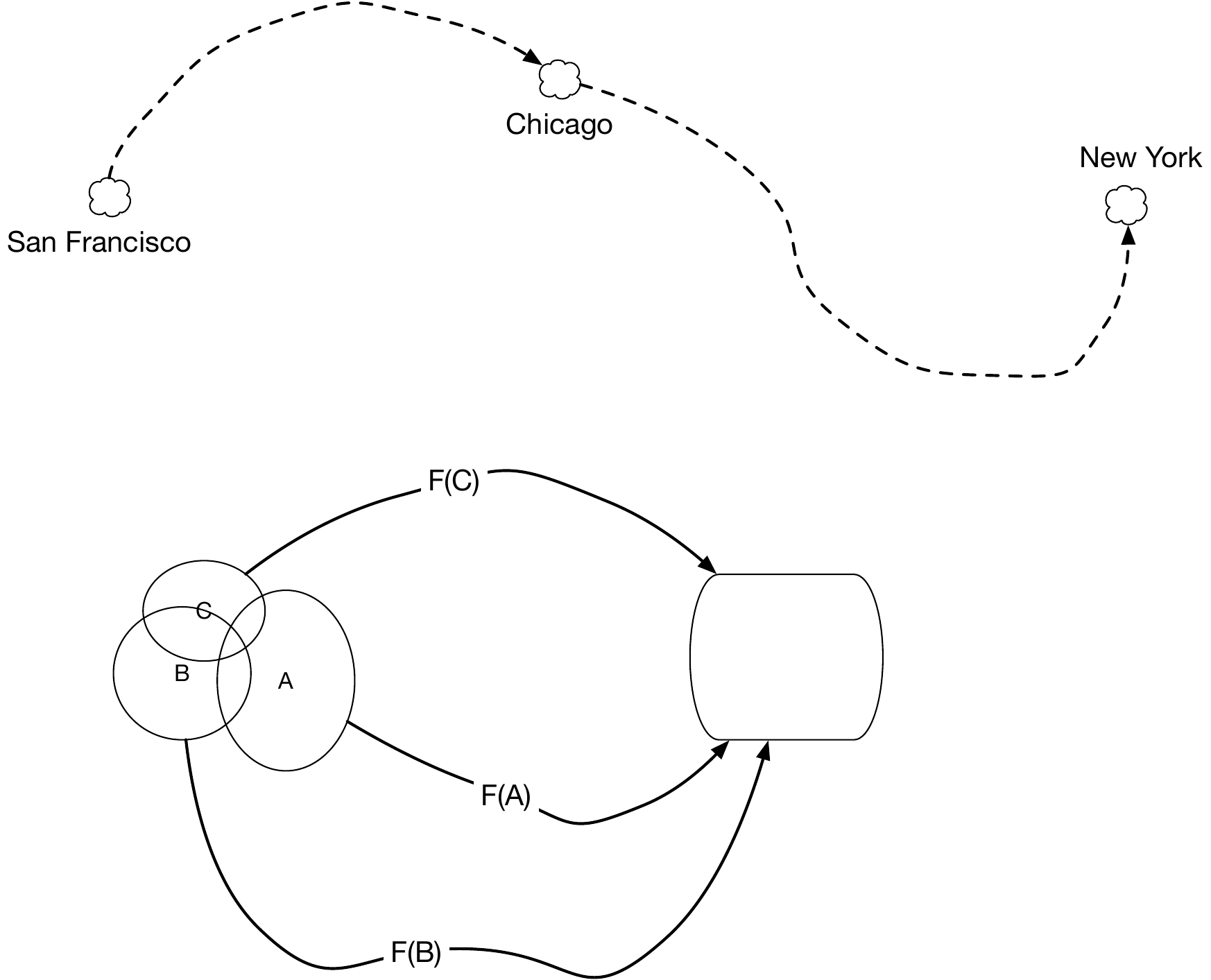}
\end{minipage}
\end{figure}

Figure~\ref{sheaves} gives two concrete examples of sheaves. In a minimum cost transportation problem, say using optimal transport \cite{ot} or reinforcement learning \cite{DBLP:books/lib/SuttonB98}, any optimal solution has the property that any restriction of the solution must also be optimal. In RL, this sheaf principle is codified by the Bellman equation, and leads to the fundamental principle of dynamic programming \cite{DBLP:books/lib/Bertsekas05}. Consider routing candy bars from San Francisco to New York city. If the cheapest way to route candy bars is through Chicago, then the restriction of the overall route to the (sub) route from Chicago to New York City must also be optimal, otherwise it is possible to find a shortest overall route by switching to a lower cost route. Similarly, in function approximation with real-valued functions $F: {\cal C} \rightarrow \mathbb{R}$, where ${\cal C}$ is the category of topological spaces, the (sub)functions $F(A), F(B)$ and $F(C)$ restricted to the open sets $A$, $B$ and $C$ must agree on the values they map the elements in the intersections $A \cap B$, $A \cap C$, $A \cap B \cap C$ and so on. Similarly, in causal inference, any probability distribution that is defined over a causal model must satisfy the sheaf condition in that any restriction of the causal model to a submodel must be consistent, so that two causal submodels that overlap in their domains must agree on the common elements. 

Sheaves can be defined over arbitrary categories, and we introduce the main idea by focusing on the category of sheaves over {\bf Sets}. 

\begin{definition}\cite{maclane:sheaves}
    A {\bf sheaf} of sets $F$ on a topological space $X$ is a functor $F: {\cal O}^{op} \rightarrow {\bf Sets} $ such that each open covering $U = \bigcup_i U_i, i \in I$ of an open set $O$ of $X$ yields an equalizer diagram
\[
\xymatrix{
FU\ar@{-->}[r]^e&} 
\begin{tikzcd}
 \prod_i FU_i \ar[r,shift left=.75ex,"p"]
  \ar[r,shift right=.75ex,swap,"q"]
&
\prod_{i,} F(U_i \cap U_j)
\end{tikzcd}
\]

The above definition succinctly captures what Figure~\ref{sheaves} shows for the example of approximating functions: the value of each subfunction must be consistent over the shared elements in the intersection of each open set. 

\begin{definition}
    The category $\mbox{Sh}(X)$ of sheaves over a space $X$ is a full subcategory of the functor category ${\bf Sets}^{{\cal O}(X)^{op}}$.
\end{definition}

\subsection*{Grothendieck Topologies}

We can generalize the notion of sheaves to arbitrary categories using the Yoneda embedding $\yo(x) = {\cal C}(-, x)$. We explain this generalization in the context of a more abstract topology on categories called the {\em Grothendieck topology} defined by {\em sieves}. A sieve can be viewed as a {\em subobject} $S \subseteq \yo(x)$ in the presheaf ${\bf Sets}^{{\cal C}^{op}}$, but we can define it more elegantly as a family of morphisms in ${\cal C}$, all with codomain $x$ such that

\[ f \in S \Longrightarrow f \circ g \in S \]

Figure~\ref{sieves} illustrates the idea of sieves. A simple way to think of a sieve is as a {\em right ideal}. We can define that more formally as follows: 

\begin{definition}
    If $S$ is a sieve on $x$, and $h: D \rightarrow x$ is any arrow in category ${\cal C}$, then 

    \[ h^* = \{g \ | \ \mbox{cod}(g) = D, hg \in S \}\]
\end{definition}

\begin{figure}[t] 
\centering
\caption{Sieves are subobjects of of $\yo(x)$ Yoneda embeddings of a category ${\cal C}$, which generalizes the concept of sheaves over sets in Figure~\ref{sheaves}. \label{sieves}}
\vskip 0.1in
\begin{minipage}{0.7\textwidth}
\vskip 0.1in
\includegraphics[scale=0.35]{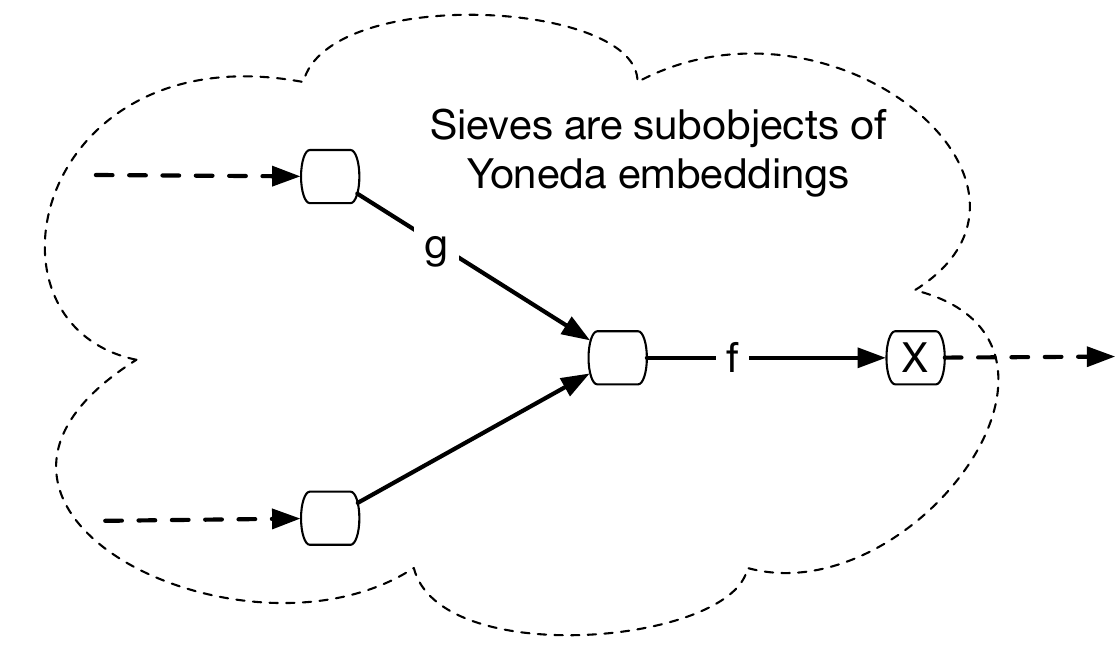}
\end{minipage}
\end{figure}

\begin{definition}\cite{maclane:sheaves}
A {\bf Grothendieck topology} on a category ${\cal C}$ is a function $J$ which assigns to each object $x$  of ${\cal C}$ a collection $J(x)$ of sieves on $x$ such that
\begin{enumerate}
    \item the maximum sieve $t_x = \{ f | \mbox{cod}(f) = x \}$ is in $J(x) $. 
    \item If $S \in J(x)$  then $h^*(S) \in J(D)$ for any arrow $h: D \rightarrow x$. 
    \item If $S \in J(x)$ and $R$ is any sieve on $x$, such that $h^*(R) \in J(D)$ for all $h: D \rightarrow x$, then $R \in J(C)$. 
\end{enumerate}
\end{definition}

We can now define categories with a given Grothendieck topology as {\em sites}. 

\begin{definition}
    A {\bf site} is defined as a pair $({\cal C}, J)$ consisting of a small category ${\cal C}$ and a Grothendieck topology $J$ on ${\cal C}$. 
\end{definition}

An intuitive way to interpret a site is as a generalization of the notion of a topology on a space $X$, which is defined as a set $X$ together with a collection of open sets ${\cal O}(X)$. The sieves on a category play the role of ``open sets". 

\end{definition}

\subsection*{Exponential Objects and Cartesian Closed Categories}

To define a topos, we need to understand the category of {\bf Sets} a bit more. Clearly, the single point set $\{ \bullet \}$ is a terminal object for {\bf Sets}, and the binary product of two sets $A \times B$ can always be defined. Furthermore, given two sets $A$ and $B$, we can define $B^A$ as the exponential object representing the set of all functions $f: A \rightarrow B$. We can define exponential objects in any category more generally as follows. 

\begin{definition}
    Given any category ${\cal C}$ with products, for a fixed object $x$ in ${\cal C}$, we can define the functor 

    \[ x \times - :  \rightarrow {\cal C}\]

    If this functor has a right adjoint, which can be denoted as 

    \[ (-)^x: {\cal C} \rightarrow {\cal C} \]

    then we say $x$ is an {\bf exponentiable} object of ${\cal C}$. 
\end{definition}

\begin{definition}
    A category ${\cal C}$ is {\bf Cartesian closed} if it has finite products (which is equivalent to saying it has a terminal object and binary products) and if all objects in ${\cal C}$ are {\em exponentiable}. 
\end{definition}

A result that is of foundational importance to this paper is that the category defined by Yoneda embeddings is Cartesian closed. 

\begin{theorem}\cite{maclane:sheaves}
    For any small category ${\cal C}$, the functor category ${\bf Sets}^{{\cal C}^{op}}$ is Cartesian closed
\end{theorem}

For a detailed proof, the reader is referred to \cite{maclane:sheaves}. A further result of significance is the {\em density theorem}, which can be seen as the generalization of the simple result that any set $S$ can be defined as the union of single point sets $\bigcup_{x \in S} \{ x \}$. 

\begin{theorem}\cite{maclane:sheaves}
    In a functor category ${\bf Sets}^{{\cal C}^{op}}$, any object $x$ is the colimit of a diagram of representable objects in a canonical way. 
\end{theorem}

Recall that an object is representable if it is isomorphic to a Yoneda embedding $\yo(x)$. This result has numerous applications to AI and ML, among them to causal inference \cite{DBLP:journals/entropy/Mahadevan23} and universal decision models \cite{sm:udm}. 

\subsection*{Subobject Classifiers} 

A topos builds on the property of subobject classifiers in {\bf Sets}. Given any subset $S \subset X$, we can define $S$ as the monic arrow $S \hookrightarrow X$ defined by the inclusion of $S$ in $X$, or as the characteristic function $\phi_S$ that is equal to $1$ for all elements $x \in X$ that belong to $S$, and takes the value $0$ otherwise. We can define the set ${\bf 2} = \{0, 1 \}$ and treat {\bf true} as the inclusion $\{1 \}$ in ${\bf 2}$. The characteristic function $\phi_S$ can then be defined as the pullback of {\bf true} along $\phi_S$. 

\[\begin{tikzcd}
	S &&& {{\bf 1}} \\
	\\
	X &&& {{\bf 2}}
	\arrow["m", tail, from=1-1, to=3-1]
	\arrow[from=1-1, to=1-4]
	\arrow["{{\bf true}}"{description}, tail, from=1-4, to=3-4]
	\arrow["{\phi_S}"{description}, dashed, from=3-1, to=3-4]
\end{tikzcd}\]

We can now define subobject classifiers in a category ${\cal C}$ as follows. 

\begin{definition}
    In a category ${\cal C}$ with finite limits, a {\bf subobject classifier} is a {\em monic} arrow ${\bf true}: {\bf 1} \rightarrow \Omega$, such that to every other monic arrow $S \hookrightarrow X$ in ${\cal C}$, there is a unique arrow $\phi$ that forms the following pullback square: 

\[\begin{tikzcd}
	S &&& {{\bf 1}} \\
	\\
	X &&& \Omega
	\arrow["m", tail, from=1-1, to=3-1]
	\arrow[from=1-1, to=1-4]
	\arrow["{{\bf true}}"{description}, tail, from=1-4, to=3-4]
	\arrow["{\phi}"{description}, dashed, from=3-1, to=3-4]
\end{tikzcd}\]
    
\end{definition}

This definition can be rephrased as saying that the subobject functor is representable. In other words, a subobject of an object $x$ in a category ${\cal C}$ is an equivalence class of monic arrows $m: S \hookrightarrow  x$. 

\cite{maclane:sheaves} provide many examples of subobject classifiers.  \cite{vigna2003guided} gives a detailed description of the topos of graphs.  

\subsection*{Heyting Algebras} 

A truly remarkable finding is that the logic of topoi is not classical Boolean logic, but intuitionistic logic defined by {\em Heyting algebras}. 

\begin{definition}
    A {\bf Heyting algebra} is a poset with all finite products and coproducts, which is Cartesian closed. That is, a Heyting algebra is a lattice with ${\bf 0}$ and ${\bf 1}$ which has to each pair of elements $x$ and $y$ an exponential $y^x$. The exponential is written $x \Rightarrow y$, and defined as the adjunction 

    \[ z \leq (x \Rightarrow y) \ \ \mbox{if and only if} \ \ z \wedge x \leq y\]
\end{definition}

Alternatively, $x \Rightarrow y$ is a least upper bound for all those elements $z$ with $z \wedge x \leq y$. Therefore, for the particular case of $y$, we get that $y \leq (x \Rightarrow y)$. In the figure below, the arrows show the partial ordering relationship. As a concrete example, for a topological space $X$ the set of open setx ${\cal O}(X)$ is a Heyting algebra. The binary intersections and unions of open sets yield open sets. The empty set $\emptyset$ represents ${\bf 0}$ and the complete set $X$ represents ${\bf 1}$. Given any two open sets $U$ and $V$, the exponential object $U \Rightarrow W$ is defined as the union $\bigcup_i W_i$ of all open sets $W_i$ for which $W \cap U \subset V$. 

\[\begin{tikzcd}
	&&& {x \Rightarrow y} \\
	x && y \\
	& {x \wedge y}
	\arrow[from=3-2, to=2-3]
	\arrow[from=2-3, to=1-4]
	\arrow[from=3-2, to=2-1]
\end{tikzcd}\]

Note that in a Boolean algebra, we define implication as the relationship 

\[ (x \Rightarrow y) \equiv \neg x \vee y \]

This property, which is sometimes referred to as the ``law of the excluded middle" (because if $x = y$, then this translates to $\neg x \vee x = {\bf true}$), does not hold in a Heyting algebra. For example, on a real line $\mathbb{R}$, if we define the open sets by the open intervals $(a, b), a, b \in \mathbb{R}$, the complement of an open set need not be open. 

We can now state what is a truly remarkable result about the subobjects of a (pre)sheaf. 

\begin{theorem}\cite{maclane:sheaves}
For any functor category $\hat{C} = {\bf Sets}^{{\cal C}^{op}}$ of a small category ${\cal C}$, the partially ordered set $\mbox{Sub}_{\hat{C}}(x)$ of subobjects of $x$, for any object $x$ of $\hat{C}$ is a Heyting algebra. 
\end{theorem}

This result has deep implications for a lot of applications in AI and ML that are based modeling presheaves, including causal inference and decision making. It implies that the proper logic to employ in these settings is intuitionistic logic, not classical logic as is often used in AI \cite{pearl-book,fagin,halpern:ac}. 

Finally, we can now define the category of topoi. 

\begin{definition}
    A {\bf topos} is a category ${\cal E}$ with 
    \begin{enumerate}
        \item A pullback for every diagram $X \rightarrow B \leftarrow Y$. 

        \item A terminal object ${\bf 1}$. 

        \item An object $\Omega$ and a monic arrow ${\bf true}: 1 \rightarrow \Omega$ such that any monic $m: S \hookrightarrow B$, there is a unique arrow $\phi: B \rightarrow \Omega$ in ${\cal E}$ for which the following square is a pullback: 
        
        \[\begin{tikzcd}
	S &&& {{\bf 1}} \\
	\\
	X &&& \Omega
	\arrow["m", tail, from=1-1, to=3-1]
	\arrow[from=1-1, to=1-4]
	\arrow["{{\bf true}}"{description}, tail, from=1-4, to=3-4]
	\arrow["{\phi}"{description}, dashed, from=3-1, to=3-4]
\end{tikzcd}\]

\item To each object $x$ an object $P x$ and an arrow $\epsilon_x: x \times P x \rightarrow \Omega$ such that for every arrow $f: x \times y \rightarrow \Omega$, there is a unique arrow $g: y \rightarrow P x$ for which the following diagrams commute: 

\[\begin{tikzcd}
	y && {x \times y} &&& \Omega \\
	\\
	Px && {x \times P x} &&& \Omega
	\arrow["g", dashed, from=1-1, to=3-1]
	\arrow["f", from=1-3, to=1-6]
	\arrow["{\epsilon_x}", from=3-3, to=3-6]
	\arrow["{1 \times g}"{description}, dashed, from=1-3, to=3-3]
\end{tikzcd}\]

    \end{enumerate}
\end{definition}

\subsection{Higher-Order Categories} 

Simplicial sets are higher-dimensional generalizations of directed graphs, partially ordered sets, as well as regular categories themselves.  Importantly, simplicial sets and simplicial objects form a foundation for higher-order category theory. Simplicial objects have long been a foundation for algebraic topology, and  more recently in  higher-order category theory. The category $\Delta$ has non-empty ordinals $[n] = \{0, 1, \ldots, n]$ as objects, and order-preserving maps $[m] \rightarrow [n]$ as arrows. An important property in $\Delta$ is that any many-to-many mapping is decomposable as a composition of an injective and a surjective mapping,  each of which is decomposable into a sequence of elementary injections $\delta_i: [n] \rightarrow [n+1]$, called {\em {coface}} mappings, which omits $i \in [n]$, and a sequence of elementary surjections $\sigma_i: [n] \rightarrow [n-1]$, called {\em {co-degeneracy}} mappings, which repeats $i \in [n]$. The fundamental simplex $\Delta([n])$ is the presheaf of all morphisms into $[n]$, that is, the representable functor $\Delta(-, [n])$.  The Yoneda Lemma  assures us that an $n$-simplex $x \in X_n$ can be identified with the corresponding map $\Delta[n] \rightarrow X$. Every morphism $f: [n] \rightarrow [m]$ in $\Delta$ is functorially mapped to the map $\Delta[m] \rightarrow \Delta[n]$ in ${\cal S}$. 

Any morphism in the category $\Delta$ can be defined as a sequence of {\em co-degeneracy} and {\em co-face} operators, where the co-face operator $\delta_i: [n-1] \rightarrow [n], 0 \leq i \leq n$ is defined as: 
\[ 
\delta_i (j)  =
\left\{
	\begin{array}{ll}
		j,  & \mbox{for } \ 0 \leq j \leq i-1 \\
		j+1 & \mbox{for } \  i \leq j \leq n-1 
	\end{array}
\right. \] 

Analogously, the co-degeneracy operator $\sigma_j: [n+1] \rightarrow [n]$ is defined as 
\[ 
\sigma_j (k)  =
\left\{
	\begin{array}{ll}
		j,  & \mbox{for } \ 0 \leq k \leq j \\
		k-1 & \mbox{for } \  j < k \leq n+1 
	\end{array}
\right. \] 

Note that under the contravariant mappings, co-face mappings turn into face mappings, and co-degeneracy mappings turn into degeneracy mappings. That is, for any simplicial object (or set) $X_n$, we have $X(\delta_i) \coloneqq d_i: X_n \rightarrow X_{n-1}$, and likewise, $X(\sigma_j) \coloneqq s_j: X_{n-1} \rightarrow X_n$. 

The compositions of these arrows define certain well-known properties \cite{may1992simplicial,richter2020categories}: 
\begin{eqnarray*}
    \delta_j \circ \delta_i &=& \delta_i \circ \delta_{j-1}, \ \ i < j \\
    \sigma_j \circ \sigma_i &=& \sigma_i \circ \sigma_{j+1}, \ \ i \leq j \\ 
    \sigma_j \circ \delta_i (j)  &=&
\left\{
	\begin{array}{ll}
		\sigma_i \circ \sigma_{j+1},  & \mbox{for } \ i < j \\
		1_{[n]} & \mbox{for } \  i = j, j+1 \\ 
		\sigma_{i-1} \circ \sigma_j, \mbox{for} \ i > j + 1
	\end{array}
\right.
\end{eqnarray*}

\begin{example}
The ``vertices'' of a simplicial object ${\cal C}_n$ are the objects in  ${\cal C}$, and the ``edges'' of ${\cal C}$ are its arrows $f: X \rightarrow Y$, where $X$ and $Y$ are objects in ${\cal C}$. Given any such arrow, the degeneracy operators $d_0 f = Y$ and $d_1 f = X$ recover the source and target of each arrow. Also, given an object $X$ of category ${\cal C}$, we can regard the face operator $s_0 X$ as its identity morphism ${\bf 1}_X: X \rightarrow X$. 
\end{example}

\begin{example} 
Given a category ${\cal C}$, we can identify an $n$-simplex $\sigma$ of a simplicial set ${\cal C}_n$ with \mbox{the sequence: }
\[ \sigma = C_o \xrightarrow[]{f_1} C_1 \xrightarrow[]{f_2} \ldots \xrightarrow[]{f_n} C_n \] 
the face operator $d_0$ applied to $\sigma$ yields the sequence 
\[ d_0 \sigma = C_1 \xrightarrow[]{f_2} C_2 \xrightarrow[]{f_3} \ldots \xrightarrow[]{f_n} C_n \] 
where the object $C_0$ is ``deleted'' along with the morphism $f_0$ leaving it.  

\end{example} 

\begin{example} 
Given a category ${\cal C}$, and an $n$-simplex $\sigma$ of the simplicial set ${\cal C}_n$, the face operator $d_n$ applied to $\sigma$ yields the sequence 
\[ d_n \sigma = C_0 \xrightarrow[]{f_1} C_1 \xrightarrow[]{f_2} \ldots \xrightarrow[]{f_{n-1}} C_{n-1} \] 
where the object $C_n$ is ``deleted'' along with the morphism $f_n$ entering it.  Note this face operator can be viewed as analogous to interventions on leaf nodes in a causal DAG model. 

\end{example} 

\begin{example} 
Given a category  ${\cal C}$, and an $n$-simplex $\sigma$ of the simplicial set ${\cal C}_n$
the face operator $d_i, 0 < i < n$ applied to $\sigma$ yields the sequence 
\[ d_i \sigma = C_0 \xrightarrow[]{f_1} C_1 \xrightarrow[]{f_2} \ldots C_{i-1} \xrightarrow[]{f_{i+1} \circ f_i} C_{i+1} \ldots \xrightarrow[]{f_{n}} C_{n} \] 
where the object $C_i$ is ``deleted'' and the morphisms $f_i$ is composed with morphism $f_{i+1}$.  Note that this process can be abstractly viewed as intervening on object $C_i$ by choosing a specific value for it (which essentially ``freezes'' the morphism $f_i$ entering object $C_i$ to a constant value). 

\end{example} 

\begin{example} 
Given a category ${\cal C}$, and an $n$-simplex $\sigma$ of the simplicial set ${\cal C}_n$, 
the degeneracy operator $s_i, 0 \leq i \leq n$ applied to $\sigma$ yields the sequence 
\[ s_i \sigma = C_0 \xrightarrow[]{f_1} C_1 \xrightarrow[]{f_2} \ldots C_{i} \xrightarrow[]{{\bf 1}_{C_i}} C_{i} \xrightarrow[]{f_{i+1}} C_{i+1}\ldots \xrightarrow[]{f_{n}} C_{n} \] 
where the object $C_i$ is ``repeated'' by inserting its identity morphism ${\bf 1}_{C_i}$. 

\end{example} 

\begin{definition} 
Given a category ${\cal C}$, and an $n$-simplex $\sigma$ of the simplicial set ${\cal C}_n$, 
$\sigma$ is a {\bf {degenerate}} simplex if some $f_i$ in $\sigma$  is an identity morphism, in which case $C_i$ and $C_{i+1}$ are equal. 
\end{definition} 

\subsection*{Simplicial Subsets and Horns}

We now describe more complex ways of extracting parts of categorical structures using simplicial subsets and horns. These structures will play a key role in defining suitable lifting problems. 
 
 \begin{definition}
 The {\bf {standard simplex}} $\Delta^n$ is the simplicial set defined by the construction 
 \[ ([m] \in \Delta) \mapsto {\bf Hom}_\Delta([m], [n]) \] 
 
 By convention, $\Delta^{-1} \coloneqq \emptyset$. The standard $0$-simplex $\Delta^0$ maps each $[n] \in \Delta^{op}$ to the single element set $\{ \bullet \}$. 
 \end{definition}
 
 \begin{definition}
 Let $S_\bullet$ denote a simplicial set. If for every integer $n \geq 0$, we are given a subset $T_n \subseteq S_n$, such that the face and degeneracy maps 
 \[ d_i: S_n \rightarrow S_{n-1} \ \ \ \ s_i: S_n \rightarrow S_{n+1} \] 
 applied to $T_n$ result in 
 \[ d_i: T_n \rightarrow T_{n-1} \ \ \ \ s_i: T_n \rightarrow T_{n+1} \] 
 then the collection $\{ T_n \}_{n \geq 0}$ defines a {\bf {simplicial subset}} $T_\bullet \subseteq S_\bullet$
 \end{definition}
 
 \begin{definition}
 The {\bf {boundary}} is a simplicial set $(\partial \Delta^n): \Delta^{op} \rightarrow$ {\bf {Set}} defined as
 \[ (\partial \Delta^n)([m]) = \{ \alpha \in {\bf Hom}_\Delta([m], [n]): \alpha \ \mbox{is not surjective} \} \]
 \end{definition}
 
 Note that the boundary $\partial \Delta^n$ is a simplicial subset of the standard $n$-simplex $\Delta^n$. 
 
 \begin{definition}
 The {\bf {Horn}} $\Lambda^n_i: \Delta^{op} \rightarrow$ {\bf {Set}} is defined as
 \[ (\Lambda^n_i)([m]) = \{ \alpha \in {\bf Hom}_\Delta([m],[n]): [n] \not \subseteq \alpha([m]) \cup \{i \} \} \] 
 \end{definition}
 
 Intuitively, the Horn $\Lambda^n_i$ can be viewed as the simplicial subset that results from removing the interior of the $n$-simplex $\Delta^n$ together with the face opposite its $i$th vertex.   
 
Consider the problem of composing $1$-dimensional simplices  to form a $2$-dimensional simplicial object. Each simplicial subset of an $n$-simplex induces a  a {\em horn} $\Lambda^n_k$, where  $ 0 \leq k \leq n$. Intuitively, a horn is a subset of a simplicial object that results from removing the interior of the $n$-simplex and the face opposite the $i$th vertex. Consider the three horns defined below. The dashed arrow  $\dashrightarrow$ indicates edges of the $2$-simplex $\Delta^2$ not contained in the horns. 
\begin{center}
 \begin{tikzcd}[column sep=small]
& \{0\}  \arrow[dl] \arrow[dr] & \\
  \{1 \} \arrow[rr, dashed] &                         & \{ 2 \} 
\end{tikzcd} \hskip 0.5 in 
 \begin{tikzcd}[column sep=small]
& \{0\}  \arrow[dl] \arrow[dr, dashed] & \\
  \{1 \} \arrow{rr} &                         & \{ 2 \} 
\end{tikzcd} \hskip 0.5in 
 \begin{tikzcd}[column sep=small]
& \{0\}  \arrow[dl, dashed] \arrow[dr] & \\
  \{1 \} \arrow{rr} &                         & \{ 2 \} 
\end{tikzcd}
\end{center}

The inner horn $\Lambda^2_1$ is the middle diagram above, and admits an easy solution to the ``horn filling'' problem of composing the simplicial subsets. The two outer horns on either end pose a more difficult challenge. For example, filling the outer horn $\Lambda^2_0$ when the morphism between $\{0 \}$ and $\{1 \}$ is $f$ and that between $\{0 \}$ and $\{2 \}$ is the identity ${\bf 1}$ is tantamount to finding the left inverse of $f$ up to homotopy. Dually, in this case, filling the outer horn $\Lambda^2_2$ is tantamount to finding the right inverse of $f$ up to homotopy. A considerable elaboration of the theoretical machinery in category theory is required to describe the various solutions proposed, which led to different ways of defining higher-order category theory \cite{weakkan,quasicats,Lurie:higher-topos-theory}. 

\subsection*{Higher-Order Categories} 

We now formally introduce higher-order categories, building on the framework proposed in a number of formalisms. We briefly summarize various approaches to the horn filling problem in higher-order category theory.
 
 \begin{definition}
 Let $f: X \rightarrow S$ be a morphism of simplicial sets. We say $f$ is a {\bf {Kan fibration}} if, for each $n > 0$, and each $0 \leq i \leq n$, every lifting problem. 
 \begin{center}
 \begin{tikzcd}
  \Lambda^n_i \arrow{d}{} \arrow{r}{\sigma_0}
    & X \arrow[]{d}{f} \\
  \Delta^n \arrow[ur,dashed, "\sigma"] \arrow[]{r}[]{\bar{\sigma}}
&S \end{tikzcd}
 \end{center}  
 admits a solution. More precisely, for every map of simplicial sets $\sigma_0: \Lambda^n_i \rightarrow X$ and every $n$-simplex $\bar{\sigma}: \Delta^n \rightarrow S$ extending $f \circ \sigma_0$, we can extend $\sigma_0$ to an $n$-simplex $\sigma: \Delta^n \rightarrow X$ satisfying $f \circ \sigma = \bar{\sigma}$. 
 \end{definition}
 
 \begin{example}
Given a simplicial set $X$, then a projection map $X \rightarrow \Delta^0$ that is a Kan fibration is called a {\bf {Kan complex}}. 
\end{example} 

\begin{example}
Any isomorphism between simplicial sets is a Kan fibration. 
\end{example}

\begin{example}
The collection of Kan fibrations is closed under retracts. 
\end{example}

\begin{definition}\cite{Lurie:higher-topos-theory}
\label{ic} 
An $\infty$-category is a simplicial object $S_\bullet$ which satisfies the following condition: 

\begin{itemize} 
\item For $0 < i < n$, every map of simplicial sets $\sigma_0: \Lambda^n_i \rightarrow S_\bullet$ can be extended to a map $\sigma: \Delta^n \rightarrow S_i$. 
\end{itemize} 
\end{definition}

This definition emerges out of a common generalization of two other conditions on a simplicial set $S_i$: 

\begin{enumerate} 
\item {\bf {Property K}}: For $n > 0$ and $0 \leq i \leq n$, every map of simplicial sets $\sigma_0: \Lambda^n_i \rightarrow S_\bullet$ can be extended to a map $\sigma: \Delta^n \rightarrow S_i$. 

\item {\bf {Property C}}:  for $0 < 1 < n$, every map of simplicial sets $\sigma_0: \Lambda^n_i \rightarrow S_i$ can be extended uniquely to a map $\sigma: \Delta^n \rightarrow S_i$. 
\end{enumerate} 

Simplicial objects that satisfy property K were defined above to be Kan complexes. Simplicial objects that satisfy property C above can be identified with the nerve of a category, which yields a full and faithful embedding of a category in the category of sets.  definition~\ref{ic} generalizes both of these definitions, and was called a {\em {quasicategory}} in \cite{quasicats} and {\em {weak Kan complexes}} in \cite{weakkan} when ${\cal C}$ is a category. We will use the nerve of a category below in defining homotopy colimits as a way of characterizing a causal model. 

 \subsection*{Example: Simplicial Objects over Causal String Diagrams}  

We now illustrate the above formalism of simplicial objects by illustrating how it applies to the special case where causal models are defined over symmetric monoidal categories \cite{fong:ms,string-diagram-surgery,fritz:jmlr}. For a detailed overview of symmetric monoidal categories, we recommend the book-length treatment by \citet{fong2018seven}. Symmetric monoidal categories (SMCs) are useful in modeling processes where objects can be combined together to give rise to new objects, or where objects disappear. For example, \citet{Coecke_2016} propose a mathematical framework for resources based on SMCs. We focus on the work of \citet{string-diagram-surgery}.  It is important to point out that monoidal categories can be defined as a special type of Grothendieck fibration \cite{richter2020categories}. We discuss one specific case of the general  Grothendieck construction in the next section  construction, and refer the reader to \cite{richter2020categories} for how the structure of monoidal categories itself emerges from \mbox{this construction. }

Our goal in this section is to illustrate how we can define simplicial objects over the SMC category CDU category {\bf {Syn}}$_G$ constructed by \citet{string-diagram-surgery} to mimic the process of working with an actual Bayesian network DAG $G$ For the purposes of our illustration, it is not important to discuss the intricacies involved in this model, for which we refer the reader to the original paper. In particular, we can solve an associated lifting problem that is defined by the functor mapping the simplicial category $\Delta$ to their SMC category. They use the category of stochastic matrices to capture the process of working with the joint distribution as shown in the figure. Instead, one can use some other category, such as the category of {\bf {Sets}}, or {\bf {Top}} (the category of topological spaces), or indeed, the category {\bf {Meas}} of measurable spaces. 

Recall that Bayesian networks \cite{pearl-book} define a joint probability distribution 
 \[ P(X_1, \ldots, X_n) = \prod_{i=1}^n P(X_i | \mbox{Pa}(X_i)], \]  
 where $\mbox{Pa}(X_i) \subset \{X_1, \ldots, X_n \} \setminus {X_i}$ represents a subset of variables (not including the variable itself). \citet{string-diagram-surgery} show Bayesian network models can be constructed using symmetric monoidal categories, where the tensor product operation is used to combine multiple variables into a ``tensored'' variable that then probabilistically maps into an output variable. In particular, the monoidal category {\bf {Stoch}} has as objects finite sets, and morphisms $f: A \rightarrow B$ are $|B| \times |A|$ dimensional stochastic matrices. Composition of stochastic matrices corresponds to matrix multiplication. The monoidal product $\otimes$ in {\bf {Stoch}} is the cartesian product of objects, and the Kronecker product of matrices $f \otimes g$. \mbox{\citet{string-diagram-surgery}} define three additional operations, the copy map, the discarding map, and the uniform state. 

 \begin{definition}
 A CDU category (for copy, discard, and uniform) is a SMC category ({{\bf C}}, $\otimes$, $I$), where each object $A$ has a copy map $C_A: A \rightarrow A \otimes A$, and discarding map $D_A: A \rightarrow I$, and a uniform state map $U_A: I \rightarrow A$, satisfying a set of equations detailed in \citet{string-diagram-surgery}. CDU functors are symmetric monoidal functors between CDU categories, preserving the CDU maps. 
 \end{definition}

 The key theorem we are interested in is the following from the original paper \cite{string-diagram-surgery}: 

 \begin{theorem}
There is an isomorphism (1-1 correspondence) between Bayesian networks based on a DAG $G$ and CDU functors $F:$ {\bf {Syn}}$_G \rightarrow$ {\bf {Stoch}}. 
 \end{theorem}
 
 \subsection*{Nerve of a Category}

The nerve of a category ${\cal C}$ enables embedding ${\cal C}$ into the category of simplicial objects, which is a fully faithful embedding \cite{Lurie:higher-topos-theory,richter2020categories}. 

\begin{definition} 
\label{fully-faithful} 
Let ${\cal F}: {\cal C} \rightarrow {\cal D}$ be a functor from category ${\cal C}$ to category ${\cal D}$. If for all arrows $f$ the mapping $f \rightarrow F f$

\begin{itemize}
    \item injective, then the functor ${\cal F}$ is defined to be {\bf {faithful}}. 
    \item surjective, then the functor ${\cal F}$ is defined to be {\bf {full}}.  
    \item bijective, then the functor ${\cal F}$ is defined to be {\bf {fully faithful}}. 
\end{itemize}
\end{definition}

\begin{definition}
The {\bf {nerve}} of a category ${\cal C}$ is the set of composable morphisms of length $n$, for $n \geq 1$.  Let $N_n({\cal C})$ denote the set of sequences of composable morphisms of length $n$.  
\[ \{ C_o \xrightarrow[]{f_1} C_1 \xrightarrow[]{f_2} \ldots \xrightarrow[]{f_n} C_n \ | \ C_i \ \mbox{is an object in} \ {\cal C}, f_i \ \mbox{is a morphism in} \ {\cal C} \} \] 
\end{definition}

The set of $n$-tuples of composable arrows in {\cal C}, denoted by $N_n({\cal C})$,  can be viewed as a functor from the simplicial object $[n]$ to ${\cal C}$.  Note that any nondecreasing map $\alpha: [m] \rightarrow [n]$ determines a map of sets $N_m({\cal C}) \rightarrow N_n({\cal C})$.  The nerve of a category {\cal C} is the simplicial set $N_\bullet: \Delta \rightarrow N_n({\cal C})$, which maps the ordinal number object $[n]$ to the set $N_n({\cal C})$.

The importance of the nerve of a category comes from a key result \cite{richter2020categories}, showing it defines a full and faithful embedding of a category: 

\begin{theorem}
The {\bf {nerve functor}} $N_\bullet:$ {\bf {Cat}} $\rightarrow$ {\bf {Set}} is fully faithful. More specifically, there is a bijection $\theta$ defined as: 
\[ \theta: {\bf Cat}({\cal C}, {\cal C'}) \rightarrow {\bf Set}_\Delta (N_\bullet({\cal C}), N_\bullet({\cal C'}) \] 
\end{theorem}

Using this concept of a nerve of a category, we can now state a theorem that shows it is possible to easily embed the CDU symmetric monoidal category defined above that represents Bayesian Networks and their associated ``string diagram surgery'' operations for causal inference as a simplicial set. 

\begin{theorem}
  Define the {\bf {nerve}} of the CDU symmetric monoidal  category  ({\bf {C}}, $\otimes$, $I$), where each object $A$ has a copy map $C_A: A \rightarrow A \otimes A$, and discarding map $D_A: A \rightarrow I$, and a uniform state map $U_A: I \rightarrow A$ as the set of composable morphisms of length $n$, for $n \geq 1$.  Let $N_n({\cal C})$ denote the set of sequences of composable morphisms of length $n$.  
\[ \{ C_o \xrightarrow[]{f_1} C_1 \xrightarrow[]{f_2} \ldots \xrightarrow[]{f_n} C_n \ | \ C_i \ \mbox{is an object in} \ {\cal C}, f_i \ \mbox{is a morphism in} \ {\cal C} \} \]   

The associated {\bf {nerve functor}} $N_\bullet:$ {\bf {Cat}} $\rightarrow$ {\bf  {Set}} from the CDU category is fully faithful. More specifically, there is a bijection $\theta$ defined as: 
\[ \theta: {\bf Cat}({\cal C}, {\cal C'}) \rightarrow {\bf Set}_\Delta (N_\bullet({\cal C}), N_\bullet({\cal C'}) \] 
\end{theorem}

This theorem is just a special case of the above theorem attesting to the full and faithful embedding of any category using its nerve, which then makes it a simplicial set. We can then use the theoretical machinery at the top layer of the UCLA architecture to manipulate causal interventions in this category using face and degeneracy operators as defined above. 

Note that the functor $G$ from a simplicial object $X$ to a category ${\cal C}$ can be lossy. For example, we can define the objects of ${\cal C}$ to be the elements of $X_0$, and the morphisms of ${\cal C}$ as the elements $f \in X_1$, where $f: a \rightarrow b$, and $d_0 f = a$, and $d_1 f = b$, and $s_0 a, a \in X$ as defining the identity morphisms {${\bf 1}_a$}. Composition in this case can be defined as the free algebra defined over elements of $X_1$, subject to the constraints given by elements of $X_2$. For example, if $x \in X_2$, we can impose the requirement that $d_1 x = d_0 x \circ d_2 x$. Such a definition of the left adjoint would be quite lossy because it only preserves the structure of the simplicial object $X$ up to the $2$-simplices. The right adjoint from a category to its associated simplicial object, in contrast, constructs a full and faithful embedding of a category into a simplicial set.  In particular, the  nerve of a category is such a right adjoint. 

 \subsection*{Topological Embedding of Simplicial Sets}

Simplicial sets can be embedded in a topological space using coends \cite{maclane:71}, which is the basis for a popular machine learning method for reducing the dimensionality of data called UMAP (Uniform Manifold Approximation and Projection) \cite{umap}. 

\begin{definition}
 The {\bf geometric realization} $|X|$ of a simplicial set $X$ is defined as the topological space 

 \[ |X| = \bigsqcup_{n \geq 0} X_n \times \Delta^n / ~\sim \]

 where the $n$-simplex $X_n$ is assumed to have a {\em discrete} topology (i.e., all subsets of $X_n$ are open sets), and $\Delta^n$ denotes the {\em topological} $n$-simplex 

 \[ \Delta^n = \{(p_0, \ldots, p_n) \in \mathbb{R}^{n+1} \ | \ 0 \leq p_i \leq 1, \sum_i p_i = 1 \]

 The spaces $\Delta^n, n \geq 0$ can be viewed as {\em cosimplicial} topological spaces with the following degeneracy and face maps: 

 \[ \delta_i(t_0, \ldots, t_n) = (t_0, \ldots, t_{i-1}, 0, t_i, \ldots, t_n)  \ \mbox{for} \ 0 \leq i \leq n\]

 \[ \sigma_j(t_0, \ldots, t_n) = (t_0, \ldots, t_{j} + t_{j+1}, \ldots, t_n) \ \mbox{for} \ 0 \leq i \leq n\]

 Note that $\delta_i: \mathbb{R}^n \rightarrow \mathbb{R}^{n+1}$, whereas $\sigma_j: \mathbb{R}^n \rightarrow \mathbb{R}^{n-1}$. 

 The equivalence relation $\sim$ above that defines the quotient space is  given as: 

 \[ (d_i(x), (t_0, \ldots, t_n)) \sim (x, \delta_i(t_0, \ldots, t_n) )\]

 \[ (s_j(x), (t_0, \ldots, t_n)) \sim (x, \sigma_j (t_0, \ldots, t_n)) \]
\end{definition}

\subsection*{Topological Embeddings as Coends}

We now bring in the perspective that topological embeddings can be interpreted as coends as well. Consider the functor 

\[ F: \Delta^o \times \Delta \rightarrow \mbox{Top} \] 

where 

\[ F([n], [m]) = X_n \times \Delta^m \]

where $F$ acts {\em contravariantly} as a functor from $\Delta$ to ${\bf Sets}$ mapping $[n] \mapsto X_n$,  and {\em covariantly} mapping $[m] \mapsto \Delta^m$ as a functor from $\Delta$ to the category $\mbox{Top}$ of topological spaces. 

\subsection{Homotopy in Categories} 

To motivate the need to consider {\em homotopical equivalence}, we consider the following problem: a generative AI system can be used to construct summaries of documents, which raises the question of how to decide if a document summary reflects the actual document. If we view a document as an object in a category, then the question becomes one of deciding object equivalence in a looser sense of homotopy, namely is there an invertible transformation between the original document and its summary?  We discuss how to construct the topological embedding of an arbitrary category by embedding it into a simplicial set by constructing its nerve, and then finding the topological embedding of the nerve using the {\em homotopy colimit} \cite{richter2020categories}. First, we discuss the topological embedding of a simplicial set, and formulate it in terms of computing a coend.  As another example, causal models can only be determined up to some equivalence class from data, and while many causal discovery algorithms assume arbitrary interventions can be carried out to  discover the unique structure, such interventions are generally impossible to do in practical applications. The concept of {\em {essential graph}} \cite{anderson-annals} is based on defining a ``quotient space'' of graphs, but similar issues arise more generally for non-graph based models as well. Thus, it is useful to understand how to formulate the notion of equivalent classes of causal models in an arbitrary category. For example, given the conditional independence structure $A \CI B | C$, there are at least three different symmetric monoidal categorical representations that all satisfy this conditional independence \cite{fong:ms,string-diagram-surgery,fritz:jmlr}, and we need to define the quotient space over all such equivalent categories. 

In our previous work on causal homotopy \cite{sm:homotopy}, we exploited the connection between causal DAG graphical models and finite topological spaces. In particular, for a DAG model $G = (V, E)$, it is possible to define a finite space topology ${\cal T} = (V, {\cal O})$, whose open sets ${\cal O}$ are subsets of the vertices $V$ such that each vertex $x$ is associated with an open set $U_x$ defined as the intersection of all open sets that contain $x$. This structure is referred to an {\em Alexandroff} topology, which can be shown to emerge from universal representers defined by Yoneda embeddings in generalized metric spaces.  An intuitive way to construct an Alexandroff topology is to define the open set for each variable $x$ by the set of its ancestors $A_x$, or by the set of its descendants $D_x$. This approach transcribes a DAG graph into a finite topological space, upon which the mathematical tools of algebraic topology can be applied to construct homotopies among equivalent causal models. Our approach below generalizes this construction to simplicial objects, as well as general categories. 

\subsection*{The Category of Fractions: Localizing Invertible Morphisms in a Category}

One way to pose the question of homotopy is to ask whether a category can be reduced in some way such that all invertible morphisms can be ``localized" in some way. The problem of defining a category with a given subclass of invertible morphisms, called the category of fractions \citep{gabriel1967calculus}, is another concrete illustration of the close relationships between categories and graphs. \citet{borceux_1994} has a detailed discussion of the ``calculus of fractions'', namely how to define a category where a subclass of morphisms are to be treated as isomorphisms. The formal definition is as follows: 

\begin{definition}
Consider a category ${\cal C}$ and a class $\Sigma$ of arrows of ${\cal C}$. The {\bf {category of fractions}} ${\cal C}(\Sigma^{-1})$ is said to exist when a category ${\cal C}(\Sigma^{-1})$ and a functor $\phi: {\cal C} \rightarrow {\cal C}(\Sigma^{-1})$ can be found with the following properties: 

\begin{enumerate}
    \item $\forall f, \phi(f)$ is an isomorphism. 
    \item If ${\cal D}$ is a category, and $F: {\cal C} \rightarrow {\cal D}$ is a functor such that for all morphisms $f \in \Sigma$, $F(f)$ is an isomorphism, then there exists a unique functor $G: {\cal C}(\Sigma^{-1}) \rightarrow {\cal D}$ such that $G \circ \phi = F$. 
\end{enumerate}
\end{definition}

A detailed construction of the category of fractions is given in \cite{borceux_1994}, which uses the underlying directed graph skeleton associated with the category.

\subsection*{Homotopy of Simplicial Objects}

We will discuss homotopy in categories more generally now.  This notion of homotopy generalizes the notion of homotopy in topology, which defines why an object like a coffee cup is topologically homotopic to a doughnut (they have the same number of ``holes''). 

 \begin{definition}
 Let $C$ and $C'$ be a pair of objects in a category ${\cal C}$. We say $C$ is {\bf {a retract}} of $C'$ if there exists maps $i: C \rightarrow C'$ and $r: C' \rightarrow C$ such that $r \circ i = \mbox{id}_{\cal C}$. 
 \end{definition}
 
 \begin{definition}
 Let ${\cal C}$ be a category. We say a morphism $f: C \rightarrow D$ is a {\bf {retract of another morphism}} $f': C \rightarrow D$ if it is a retract of $f'$ when viewed as an object of the functor category {\bf {Hom}}$([1], {\cal C})$. A collection of morphisms $T$ of ${\cal C}$ is {\bf {closed under retracts}} if for every pair of morphisms $f, f'$ of ${\cal C}$, if $f$ is a retract of $f'$, and $f'$  is in $T$, then $f$ is also in $T$. 
 \end{definition}

 \begin{definition}
  Let X and Y be simplicial sets, and suppose we are given a pair of morphisms $f_0, f_1: X \rightarrow Y$. A {\bf {homotopy}} from $f_0$ to $f_1$ is a morphism $h: \Delta^1 \times X \rightarrow Y$ satisfying $f_0 = h |_{{0} \times X}$ and $f_1 = h_{ 1 \times X}$. 
 \end{definition}

 \subsection*{Classifying Spaces and Homotopy Colimits}

Building on the intuition proposed above, we now introduce a construction of a topological space associated with the nerve of a category. As we saw above, the nerve of a category is a full and faithful embedding of a category as a simplicial object. 

\begin{definition}
The {\bf {classifying space}} of a category ${\cal C}$ is the topological space associated with the nerve of the category $|N_\bullet {\cal C}|$
\end{definition}

To understand the classifying space $|N_\bullet {\cal C}|$ of a category ${\cal C}$, let us go over some simple examples to gain some insight. 

\begin{example}
For any set $X$, which can be defined as a discrete category ${\cal C}_X$ with no non-trivial morphisms, the classifying space $|N_\bullet {\cal C}_X|$ is just the discrete topology over $X$ (where the open sets are all possible subsets of $X$). 
\end{example}

\begin{example} 
If we take  a partially ordered set $[n]$, with its usual order-preserving morphisms, then the nerve of $[n]$ is isomorphic to the representable functor $\delta(-, [n])$, as shown by the Yoneda Lemma, and in that case, the classifying space is just the topological space $\Delta_n$ defined above. 
\end{example}

\begin{definition}
The {\bf {homotopy colimit}} of the nerve of the category of elements associated with the set-valued functor $\delta: {\cal C} \rightarrow$ {\bf {Set}} mapping the  category ${\cal C}$ into the category of {\bf Sets}, namely $N_\bullet \left(\int \delta \right)$. 
\end{definition}

 We can extend the above definition straightforwardly to these cases using an appropriate functor ${\cal T}$: {\bf {Set}} $\rightarrow$ {\bf {Top}}, or alternatively ${\cal M}$: {\bf {Set}} $\rightarrow$ {\bf {Meas}}. These augmented constructions can then be defined with respect to a more general notion called the {\em {homotopy colimit}} \cite{richter2020categories} of a causal model. 

\begin{definition}
The {\bf  {topological homotopy colimit}} $\mbox{hocolim}_{{\cal T} \circ \delta}$ of a category ${\cal C}$, along with its associated category of elements associated with  a set-valued functor $\delta: {\cal C} \rightarrow$ {\bf {Set}}, and a topological functor ${\cal T}$: {\bf {Set}} $\rightarrow$ {\bf {Top}} is isomorphic to topological space associated with the nerve of the category of elements, that is  $\mbox{hocolim}_{{\cal T} \circ \delta} \simeq|N_\bullet \left(\int \delta \right) |$. 
\end{definition}

\subsection*{Singular Homology} 

Our goal is to define an abstract notion of an object in terms of its underlying classifying space as a category, and show how it can be useful in defining homotopy. We will also clarify how it relates to determining equivalences among objects, namely homotopical invariance, and also how it sheds light on UIGs. We build on the topological realization of $n$-simplices defined above.  Define the set of all morphisms $\mbox{Sing}_n(X) = {\bf Hom}_{\bf Top}(\Delta_n, |{\cal N}_\bullet({\cal C})|)$ as the set of singular $n$-simplices of $|{\cal N}_\bullet({\cal C})|$. 

\begin{definition}
For any topological space defined  by  $|{\cal N}_\bullet({\cal C})|$,  the {\bf {singular homology groups}} {$H_*(|{\cal N}_\bullet({\cal C})|; {\bf Z})$} are defined as the homology groups of a chain complex 
\[ \ldots \xrightarrow[]{\partial} {\bf Z}(\mbox{Sing}_2(|{\cal N}_\bullet({\cal C})|)) \xrightarrow[]{\partial} {\bf Z}(\mbox{Sing}_1(|{\cal N}_\bullet({\cal C})|)) \xrightarrow[]{\partial} {\bf Z}(\mbox{Sing}_0(|{\cal N}_\bullet({\cal C})|)) \] 
where {${\bf Z}(\mbox{Sing}_n(|{\cal N}_\bullet({\cal C})|))$} denotes the free Abelian group generated by the set $\mbox{Sing}_n(|{\cal N}_\bullet({\cal C})|)$ and the differential $\partial$ is defined on the generators by the formula 
\[ \partial (\sigma) = \sum_{i=0}^n (-1)^i d_i \sigma \] 
\end{definition}

Intuitively, a chain complex builds a sequence of vector spaces that can be used to construct an algebraic invariant of a causal model from its classifying space  by choosing the left {\bf {k}} module {${\bf Z}$} to be a vector space. Each differential $\partial$ then becomes a linear transformation whose representation is constructed by modeling its effect on the basis elements in \mbox{each {${\bf Z}(\mbox{Sing}_n(X))$.}} 

\begin{example}
Let us illustrate the singular homology groups defined by an integer-valued multiset~\cite{studeny2010probabilistic} used to model conditional independence. Imsets over a DAG of three variables $N = \{a, b, c \} $ can be viewed as a finite discrete topological space. For this topological space $X$, the singular homology groups $H_*(X; {\bf Z})$ are defined as the homology groups of a \mbox{chain complex} 
\[  {\bf Z}(\mbox{Sing}_3(X)) \xrightarrow[]{\partial}  {\bf Z}(\mbox{Sing}_2(X)) \xrightarrow[]{\partial} {\bf Z}(\mbox{Sing}_1(X)) \xrightarrow[]{\partial} {\bf Z}(\mbox{Sing}_0(X)) \] 
where {${\bf Z}(\mbox{Sing}_i(X))$} denotes the free Abelian group generated by the set $\mbox{Sing}_i(X)$ and the differential $\partial$ is defined on the generators by the formula 
\[ \partial (\sigma) = \sum_{i=0}^4 (-1)^i d_i \sigma \] 

The set $\mbox{Sing}_n(X)$ is the set of all morphisms {${\bf Hom}_{Top}(|\Delta_n|, X)$}. For an imset over the three variables $N = \{a, b, c \}$, we can define the singular $n$-simplex $\sigma$ as: 
\[ \sigma: |\Delta^4| \rightarrow X \ \ \mbox{where} \ \ |\Delta^n | = \{t_0, t_1, t_2, t_3 \in [0,1]^4 : t_0 + t_1 + t_2 + t_3 = 1 \} \] 

The $n$-simplex $\sigma$ has a collection of faces denoted as $d_0 \sigma, d_1 \sigma, d_2 \sigma$ and $ d_3 \sigma$.  If we pick the $k$-left module {${\bf Z}$} as the vector space over real numbers $\mathbb{R}$, then the above chain complex represents a sequence of vector spaces that can be used to construct an algebraic invariant of a topological space defined by the integer-valued multiset.  Each differential $\partial$ then becomes a linear transformation whose representation is constructed by modeling its effect on the basis elements in each {${\bf Z}(\mbox{Sing}_n(X))$}. An alternate approach to constructing a chain homology for an integer-valued multiset is to use M\"obius inversion to define the chain complex in terms of the nerve of a category (see our recent work on categoroids \citep{categoroids} for details). 
\end{example}

Below,  we illustrate the application of concepts from category theory summarized above to {\em static} universal imitation games. Here, the two (or more) participants interact with an evaluator, and the goal is to determine if the two participants are {\em isomorphic} or {\em homotopic}. As the name ``static" suggests, we are mainly interested in characterizing participants who are in a ``steady state" and not adapting or evolving their behavior based on the interactions. The latter cases will be addressed in the next two sections of the paper. Given two static participants, whom we model abstractly as objects in some category, we want to determine if these objects are isomorphic or some weaker notion of isomorphisms, like homotopy \cite{richter2020categories}. These notions depend on the ``measuring instruments" provided by the category chosen to model an imitation game. For example, if the category chosen is the category of {\bf Sets}, the answer is relatively straightforward: two sets are isomorphic if they contain the same number of elements. Even in such a simple case, there can be subtleties: if the sets include non-well-founded sets \cite{Aczel1988-ACZNS}, where the well-foundedness axiom does not hold, then the comparison of the two sets must be made using the concept of {\em bisimulation} or the {\em anti-foundation-axiom} (AFA) pioneered in \cite{Aczel1988-ACZNS}. To clarify the distinction between static and dynamic UIGs to be discussed in the next section, in static UIGs, we are not concerned with the process by which the participants are adapting based on the interactions that they are engaged in. In dynamic UIGs, the participants are adapting (in particular, one participant is viewed as a ``teacher" and the other is viewed as a ``learner"). Although some ideas will be useful in both, such as causal models or reinforcement learning, here we remain focused on a static characterization of objects, and their ``steady state" behavior, rather than their convergence to some steady state behavior.

\subsection{Static Imitation Games over Non-Well-Founded Sets and Universal Coalgebras} 

\begin{figure}[h]
\centering
\includegraphics[scale=.35]{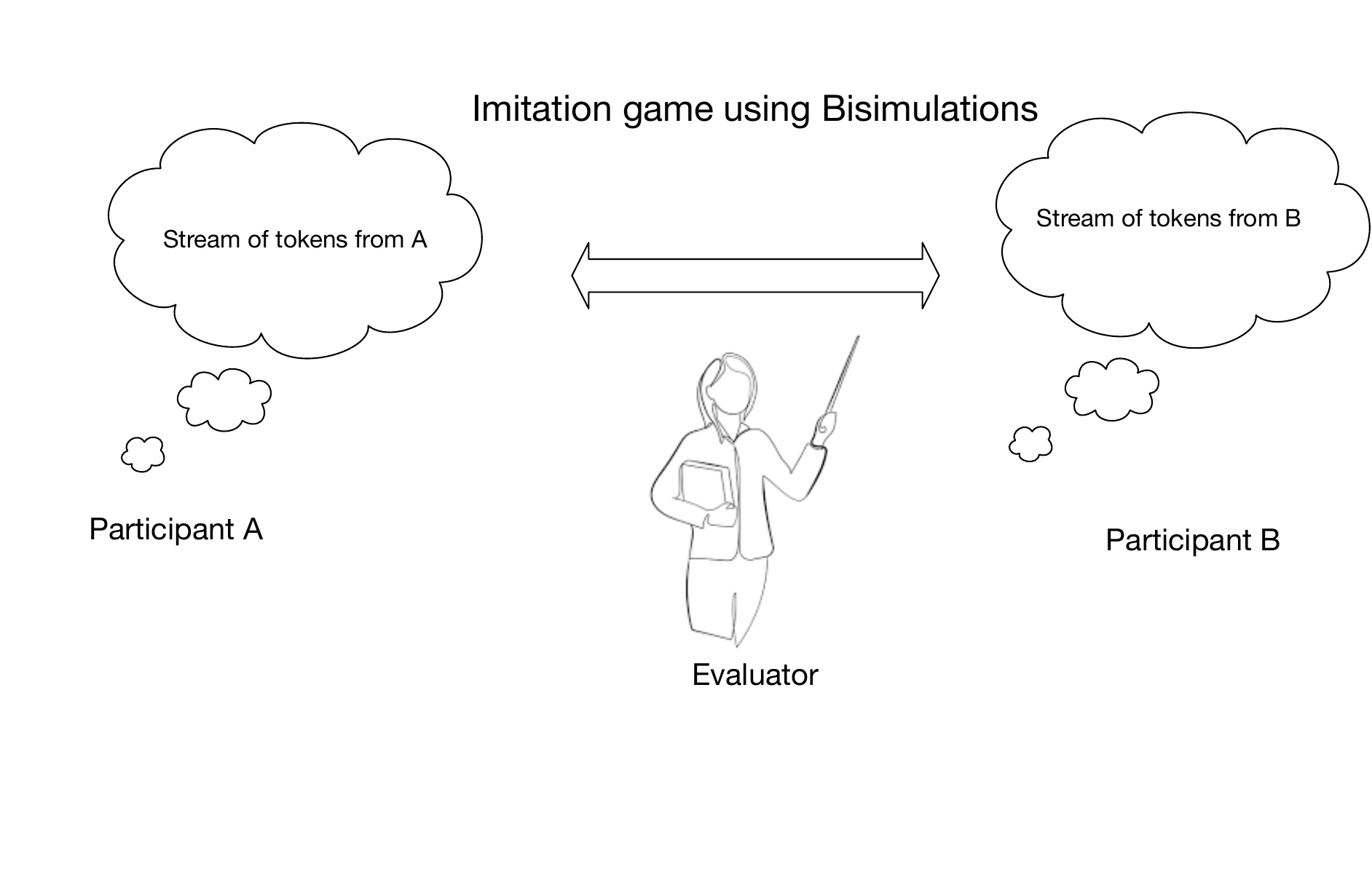}
\caption{Our approach to solving Turing's imitation game is through the use of {\em bisimulation} \cite{Aczel1988-ACZNS,rutten2000universal} to compare two infinite stream of tokens.} 
\label{imitgamebisim} 
\end{figure}

To begin with, we present an elegant formalism for solving static UIGs that is based on the concept of universal coalgebras \cite{rutten2000universal}, and non-well-founded sets \cite{Aczel1988-ACZNS}. Figure~\ref{imitgamebisim} illustrates the main idea. We define the two participants in an imitation game as universal coalgebras (or non-well-founded sets) and ask if there is a bisimulation relationship between them. This characterization covers a wide range of probabilistic models, including Markov chains and Markov decision processes \cite{DBLP:books/lib/SuttonB98}, and automata-theoretic models, as well as generative AI models \cite{DBLP:conf/iclr/GuJTRR23}.

Generative AI has become  popular recently due to the successes of neural and structured-state space sequence models  \cite{DBLP:conf/iclr/GuGR22,DBLP:conf/nips/VaswaniSPUJGKP17} and text-to-image diffusion models \cite{DBLP:conf/nips/SongE19}. The underlying paradigm of building generative models has a long history in computer science and AI, and it is useful to begin with the simplest models that have been studied for several decades, such as deterministic finite state machines, Markov chains, and context-free grammars.  We use category theory to build generative AI models and analyze them, which is one of the unique and novel aspects of this paper. To explain briefly, we represent a generative model in terms of {\em universal coalgebras} \cite{rutten2000universal} generated by an {\em endofunctor} $F$ acting on a category $C$. Coalgebras provide an elegant way to model dynamical systems, and capture the notion of state \cite{jacobs:book} in ways that provide new insight into the design of AI and ML systems. Perhaps the simplest and in some ways, the most general, type of generative AI model that is representable as a coalgebra is the {\em powerset functor} 

\[ F: S \Rightarrow {\cal P}(S)\]

where $S$ is any set (finite or not), and ${\cal P}(S)$ is the set of all subsets of $S$, that is: 

\[ {\cal P}(S) = \{ A | A \subseteq S \} \]

Notice in the specification of the powerset functor coalgebra, the same term $S$ appears on both sides of the equation. That is a hallmark of coalgebras, and it is what distinguishes coalgebras from algebras. Coalgebras generate search spaces, whereas algebras compact search spaces and summarize them.  This admittedly simple structure nonetheless is extremely versatile and enables modeling a remarkably rich and diverse set of generative AI models, including the ones listed in Figure~\ref{genaimodel}. To explain briefly, we can model a context-free grammar as a mapping from a set $S$ that includes all the vertices in the context-free grammar graph shown in Figure~\ref{genaimodel} to the powerset of the set $S$. More specifically, if $S = N \cup T$ is defined as the non-terminals $N$ as well as the terminal symbols (the actual words) $T$, any context-free grammar rule can be represented in terms of a power set functor. We will explain how this approach can be refined later in this Section, and in much more detail in later sections of the paper. To motivate further why category theory provides an elegant way to model generative AI systems, we look at some actual examples of generative AI systems to see why they can be modeled as functors. 

\begin{figure}[h]
\centering
\includegraphics[scale=.3]{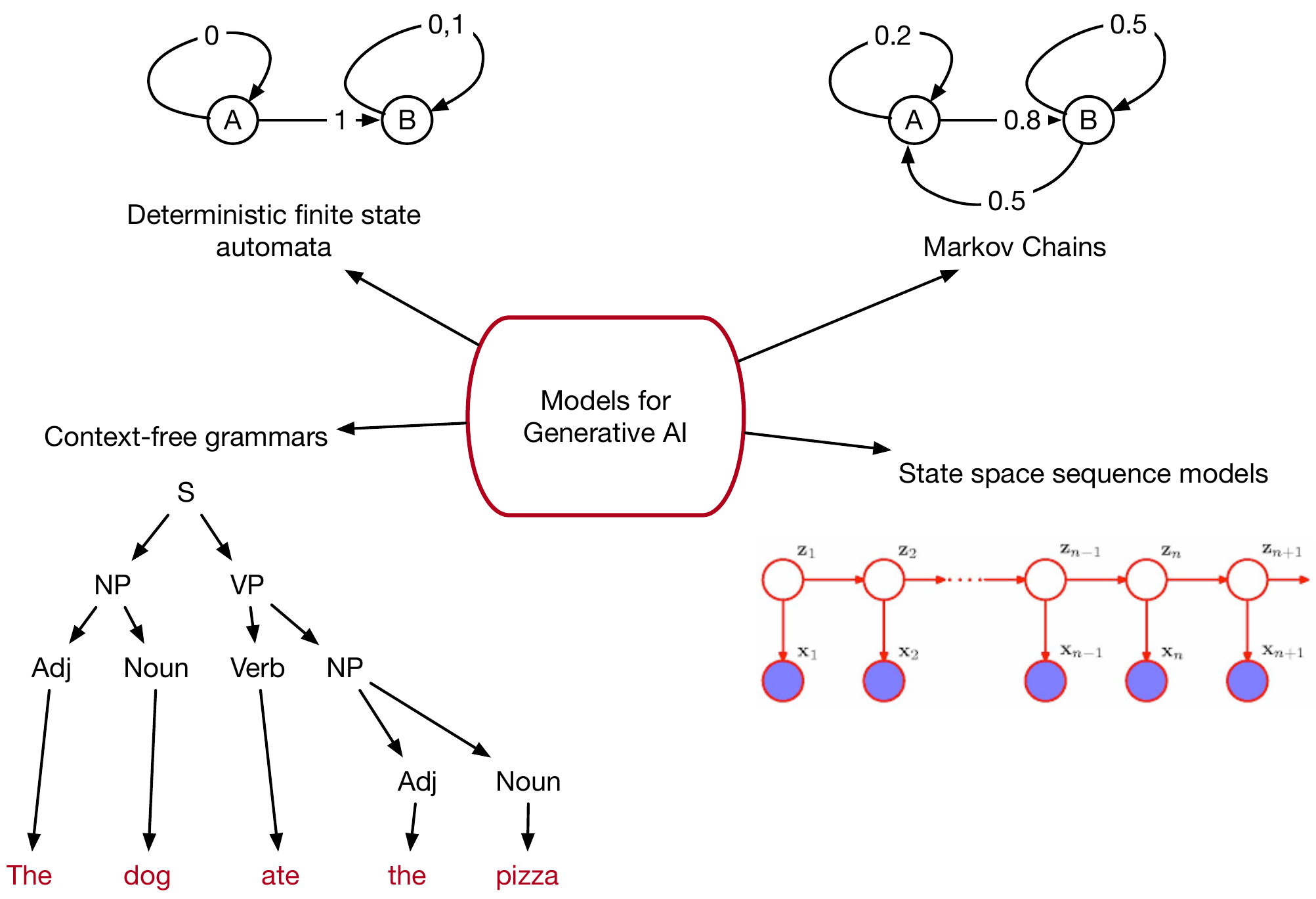}
\caption{In this paper, generative AI models, from the earliest models studied in computer science such as deterministic finite state automata and context-free grammars, to models in statistics and information theory like Markov chains, and lastly, sequence models can be represented as universal coalgebras. }
\label{genaimodel} 
\end{figure}

To solve Turing's imitation game, we need to compare two potentially infinite data streams of tokens (e.g., words, or in general, other forms of communication represented digitally by bits). Many problems in AI and ML involve reasoning about circular phenomena. These include reasoning about {\em common knowledge} \cite{barwise,fagin} such as social conventions, dealing with infinite data structures such as lists or trees in computer science, and causal inference in systems with feedback where part of the input comes from the output of the system. In all these situations, there is an intrinsic problem of having to deal with infinite sets that are recursive and violate a standard axiom called well-foundedness in set theory. First, we explain some of the motivations for including non-well-founded sets in AI and ML, and then proceed to define the standard ZFC axiomatization of set theory and how to modify it to allow circular sets. We build on the pioneering work of Peter Aczel on the anti-foundation-axiom in modeling non-well-founded sets \cite{Aczel1988-ACZNS}, which has elaborated previously in other books as well \cite{barwise,jacobs:book}, although we believe this paper is perhaps one of the first to focus on the application of non-well-founded sets and universal coalgebras to problems in AI and ML at a broad level. 

\begin{figure}[t]
\centering
\includegraphics[scale=.5]{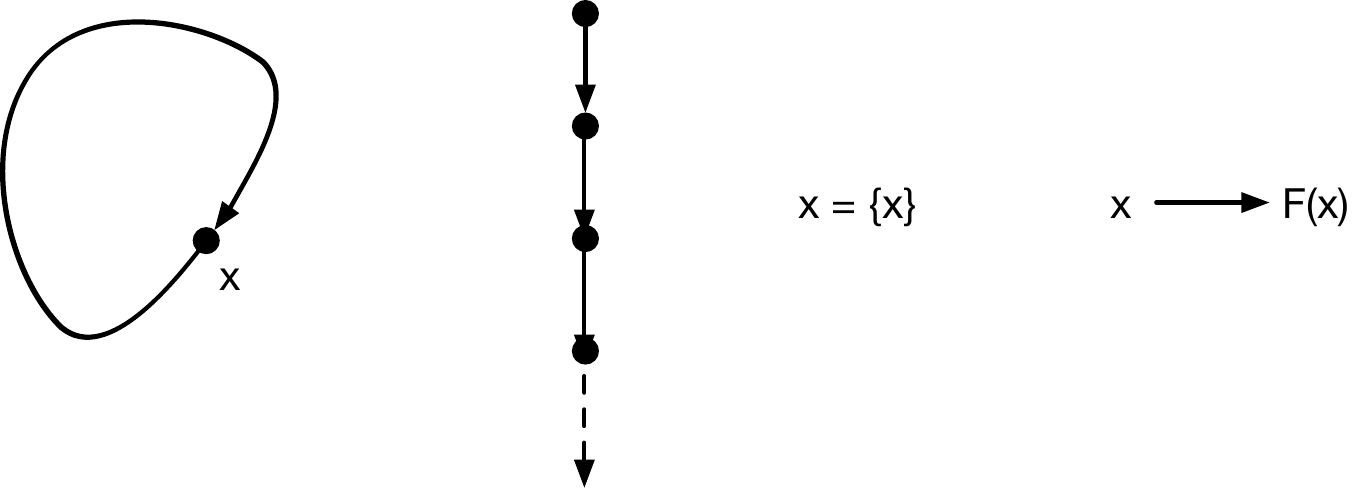}
\caption{Three representations of infinite data streams to solve Turing's imitation game: non-well-founded set $x = \{ x \}$: accessible pointed graphs (AGPs), non-well-founded sets specified by systems of equations, and universal coalgebras. We can view these as generative models of the recursive set $\{ \{ \{ \ldots \} \} \} $.}
\label{threereps} 
\end{figure}

Figure~\ref{threereps} illustrates three ways to represent an infinite object, as a directed graph, a (non-well-founded) set or as a system of equations.  We begin with perhaps the simplest approach introduced by Peter Aczel called accessible pointed graphs (APGs) (see Figure~\ref{threereps}), but also include the category-theoretic approach of using {\em universal coalgebras} \cite{rutten2000universal}, as well as systems of equations \cite{barwise}. 

We now turn to describing coalgebras, a much less familiar construct that will play a central role in the proposed ML framework of coinductive inference. Coalgebras capture hidden state, and enable modeling infinite data streams. Recall that in the previous Section, we explored non-well-founded sets, such as the set $\Omega = \{ \Omega \}$, which gives rise to a circularly defined object. As another example, consider the infinite data stream comprised of a sequence of objects, indexed by the natural numbers: 

\[ X = (X_0, X_1, \ldots, X_n, \ldots ) \]

We can define this infinite data stream as a coalgebra, comprised of an accessor function {\bf head} that returns the head of the list, and a destructor function that gives the {\bf tail}  of the list, as we will show in detail below. 

To take another example, consider a deterministic finite state machine model defined as the tuple $M = (X, A, \delta)$, where $X$ is the set of possible states that the machine might be in, $A$ is a set of input symbols that cause the machine to transition from one state to another, and $\delta: X \times A \rightarrow X$ specifies the transition function. To give a coalgebraic definition of a finite state machine, we note that we can define a functor $F: X \rightarrow {\cal P}(A \times X)$ that maps any given state $x \in X$ to the subset of possible future states $y$ that the machine might transition to for any given input symbol $a \in A$. 

We can now formally define $F$-coalgebras analogous to the definition of $F$-algebras given above. 

\begin{definition}
Let $F: {\cal C} \rightarrow {\cal C}$ be an endofunctor on the category ${\cal C}$. An {\bf $F$-coalgebra} is defined as a pair $(A, \alpha)$ comprised of an object $A$ and an arrow $\alpha: A \rightarrow F(A)$. 
\end{definition}

The fundamental difference between an algebra and a coalgebra is that the structure map is reversed! This might seem to be a minor distinction, but it makes a tremendous difference in the power of coalgebras to model state and capture dynamical systems. Let us use this definition to capture infinite data streams, as follows. 

\[ {\bf Str}: {\bf Set} \rightarrow {\bf Set}, \ \ \ \ \ \ {\bf Str}(X) = \mathbb{N} \times X\]

Here, {\bf Str} is defined as a functor on the category {\bf Set}, which generates a sequence of elements. Let $N^\omega$ denote the set of all infinite data streams comprised of natural numbers:

\[ N^\omega = \{ \sigma | \sigma: \mathbb{N} \rightarrow \mathbb{N} \} \]

To define the accessor function {\bf head} and destructor function {\bf tail} alluded to above, we proceed as follows: 

\begin{eqnarray}
{\bf head}&:& \mathbb{N}^\omega \rightarrow \mathbb{N} \ \ \  \ \ \ \ {\bf tail}: \mathbb{N}^\omega \rightarrow\mathbb{N}^\omega \\
{\bf head}(\sigma) &=& \sigma(0) \ \ \ \ \ \ \ {\bf tail}(\sigma) = (\sigma(1), \sigma(2), \ldots )
\end{eqnarray}

Another standard example that is often used to illustrate coalgebras, and provides a foundation for many AI and  ML applications, is that of a {\em labelled transition system}. 

\begin{definition}
    A {\bf labelled transition system} (LTS) $(S, \rightarrow_S, A)$ is defined by a set $S$ of states, a transition relation $\rightarrow_S \subseteq S \times A \times S$, and a set $A$ of labels (or equivalently, ``inputs" or ``actions"). We can define the transition from state $s$ to $s'$ under input $a$ by the transition diagram $s \xrightarrow[]{a} s'$, which is equivalent to writing $\langle s, a,  s' \rangle \in \rightarrow_S$. The ${\cal F}$-coalgebra for an LTS is defined by the functor 

    \[ {\cal F}(X) = {\cal P}(A \times X) = \{V | V \subseteq A \times X\} \]
\end{definition}

Just as before, we can also define a category of $F$-coalgebras over any category ${\cal C}$, where each object is a coalgebra, and the morphism between two coalgebras is defined as follows, where $f: A \rightarrow B$ is any morphism in the category ${\cal C}$. 

\begin{definition}
Let $F: {\cal C} \rightarrow {\cal C}$ be an endofunctor. A {\em homomorphism} of $F$-coalgebras $(A, \alpha)$ and $(B, \beta)$ is an arrow $f: A \rightarrow B$ in the category ${\cal C}$ such that the following diagram commutes:

\begin{center}
\begin{tikzcd}
  A \arrow[r, "f"] \arrow[d, "\alpha"]
    & B \arrow[d, "\beta" ] \\
  F(A) \arrow[r,  "F(f)"]
& F(B)
\end{tikzcd}
\end{center}
\end{definition}

For example, consider two labelled transition systems $(S, A, \rightarrow_S)$ and $(T, A, \rightarrow_T)$ over the same input set $A$, which are defined by the coalgebras $(S, \alpha_S)$ and $(T, \alpha_T)$, respectively. An $F$-homomorphism $f: (S, \alpha_S) \rightarrow (T, \alpha_T)$ is a function $f: S \rightarrow T$ such that $F(f) \circ \alpha_S  = \alpha_T \circ f$. Intuitively, the meaning of a homomorphism between two labeled transition systems means that: 

\begin{itemize}
    \item For all $s' \in S$, for any transition $s \xrightarrow[]{a}_S s'$ in the first system $(S, \alpha_S)$, there must be a corresponding transition in the second system $f(s) \xrightarrow[]{a}_T f(s;)$ in the second system. 

    \item Conversely, for all $t \in T$, for any transition $t \xrightarrow[]{a}_T t'$ in the second system, there exists two states $s, s' \in S$ such that $f(s) = t, f(t) = t'$ such that $s \xrightarrow[]{a}_S s'$ in the first system. 
\end{itemize}

If we have an $F$-homomorphism $f: S \rightarrow T$ with an inverse $f^{-1}: T \rightarrow S$ that is also a $F$-homomorphism, then the two systems $S \simeq T$ are isomorphic. If the mapping $f$ is {\em injective}, we have a  {\em monomorphism}. Finally, if the mapping $f$ is a surjection, we have an {\em epimorphism}. 

The analog of congruence in universal algebras is {\em bisimulation} in universal coalgebras. Intuitively, bisimulation allows us to construct a more ``abstract" representation of a dynamical system that is still faithful to the original system. We will explore many applications of the concept of bisimulation to AI and ML systems in this paper. We introduce the concept in its general setting first, and then in the next section, we will delve into concrete examples of bisimulations. 

\begin{definition}
Let $(S, \alpha_S)$ and $(T, \alpha_T)$ be two systems specified as coalgebras acting on the same category ${\cal C}$. Formally, a $F$-{\bf bisimulation} for coalgebras defined on a set-valued functor $F: {\bf Set} \rightarrow {\bf Set}$ is a relation $R \subset S \times T$ of the Cartesian product of $S$ and $T$ is a mapping $\alpha_R: R \rightarrow F(R)$ such that the projections of $R$ to $S$ and $T$ form valid $F$-homomorphisms.

\begin{center}
\begin{tikzcd}
  R \arrow[r, "\pi_1"] \arrow[d, "\alpha_R"]
    & S \arrow[d, "\alpha_S" ] \\
  F(R) \arrow[r,  "F(\pi_1)"]
& F(S)
\end{tikzcd}
\end{center}

\begin{center}
\begin{tikzcd}
  R \arrow[r, "\pi_2"] \arrow[d, "\alpha_R"]
    & T \arrow[d, "\alpha_T" ] \\
  F(R) \arrow[r,  "F(\pi_2)"]
& F(T)
\end{tikzcd}
\end{center}

Here, $\pi_1$ and $\pi_2$ are projections of the relation $R$ onto $S$ and $T$, respectively. Note the relationships in the two commutative diagrams should hold simultaneously, so that we get 

\begin{eqnarray*}
    F(\pi_1) \circ \alpha_R &=& \alpha_S \circ \pi_1 \\
    F(\pi_2) \circ \alpha_R &=& \alpha_T \circ \pi_2 
\end{eqnarray*}

Intuitively, these properties imply that we can ``run" the joint system $R$ for one step, and then project onto the component systems, which gives us the same effect as if we first project the joint system onto each component system, and then run the component systems. More concretely, for two labeled transition systems that were considered above as an example of an $F$-homomorphism, an $F$-bisimulation between $(S, \alpha_S)$ and $(T, \alpha_T)$ means that  there exists a relation $R \subset S \times T$ that satisfies for all $\langle s, t \rangle \in R$

\begin{itemize}
    \item For all $s' \in S$, for any transition $s \xrightarrow[]{a}_S s'$ in the first system $(S, \alpha_S)$, there must be a corresponding transition in the second system $f(s) \xrightarrow[]{a}_T f(s;)$ in the second system, so that $\langle s', t' \rangle \in R$

    \item Conversely, for all $t \in T$, for any transition $t \xrightarrow[]{a}_T t'$ in the second system, there exists two states $s, s' \in S$ such that $f(s) = t, f(t) = t'$ such that $s \xrightarrow[]{a}_S s'$ in the first system, and $\langle s', t' \rangle \in R$.
\end{itemize}

\end{definition} 

A simple example of a bisimulation of two coalgebras is shown in Figure~\ref{bisim}. 

\begin{figure}[t]
\centering
\includegraphics[scale=.45]{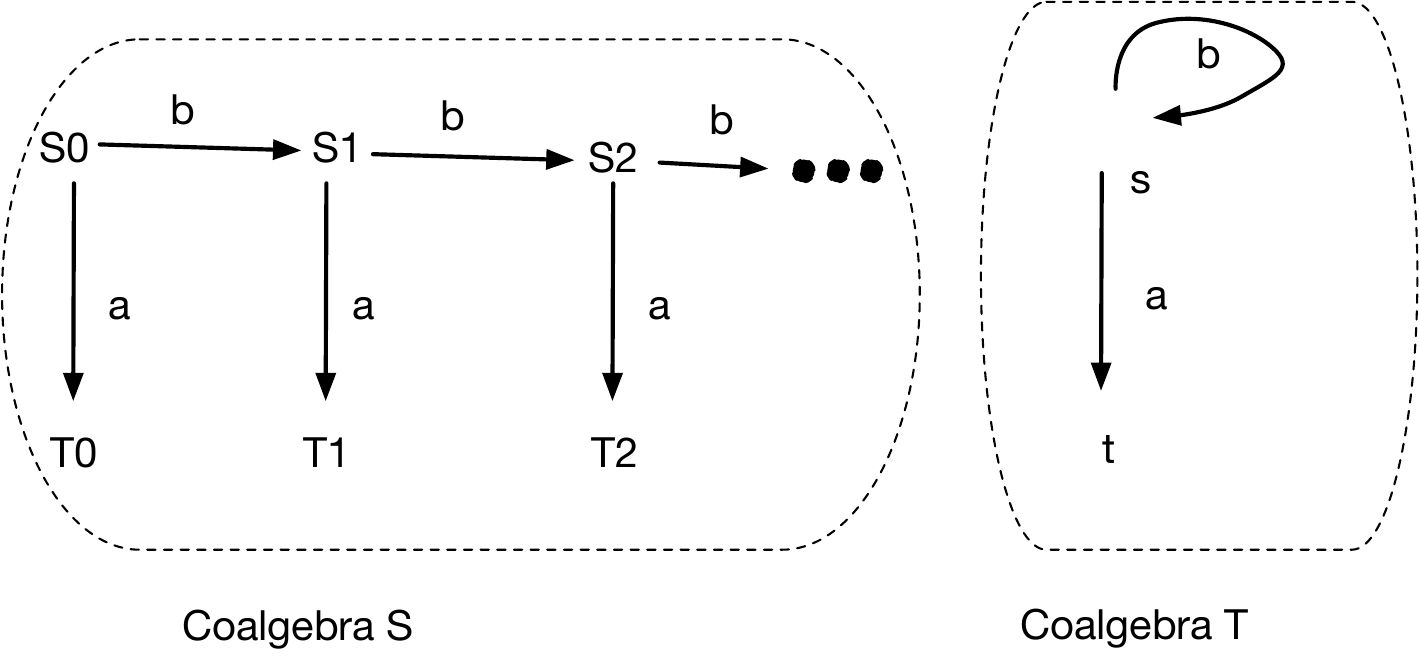}
\caption{A bisimulation among two coalgebras.}
\label{bisim} 
\end{figure}

There are a number of basic properties about bisimulations, which we will not prove, but are useful to summarize here: 

\begin{itemize}
    \item If $(R, \alpha_R)$ is a bisimulation between systems $S$ and $T$, the inverse $R^{-1}$ of $R$ is a bisimulation between systems $T$ and $S$. 

    \item Two homomorphisms $f: T \rightarrow S$ and $g:T \rightarrow U$ with a common domain $T$ define a {\em span}. The {\em image} of the span $\langle f, g \rangle(T) = \{ \langle f(t), g(t) \rangle | t \in T \}$ of $f$ and $g$ is also a bisimulation between $S$ and $U$. 

    \item The composition $R \circ Q$ of two bisimulations $R \subseteq S \times T$ and $Q \subseteq T \times U$ is a bisimulation between $S$ and $U$. 

    \item The union $\cup_k R_k$ of a family of bisimulations between $S$ and $T$ is also a bisimulation. 

    \item The set of all bisimulations between systems $S$ and $T$ is a complete lattice, with least upper bounds and greatest lower bounds given by: 

    \[ \bigvee_k R_k = \bigcup_k R_k \]

    \[ \bigwedge_K R_k = \bigcup \{ R | R \ \mbox{is a bisimulation between} S \ \mbox{and} \ T \ \mbox{and} \ R \subseteq \cap_k R_k \} \]

    \item The kernel $K(f) = \{ \langle s, s' \rangle | f(s) = f(s') \}$ of a homomorphism $f: S \rightarrow T$ is a bisimulation equivalence.

    \end{itemize}

  We can design a generic {\em coinductive inference} framework for systems by exploiting these properties of bisimulations. This framework will give us a broad paradigm for designing AI and ML systems that is analogous to inductive inference, which we will discuss in the next Section.

\subsection{Static Imitation Games over Generalized Metric Spaces}

A  general principle in machine learning (see Figure~\ref{metric-space}) to discriminate two objects (e.g., probability distributions, images, text documents etc.) is to compare them in a suitable metric space. We now describe a category of generalized metric spaces, where a metric form of the Yoneda Lemma gives us surprising insight and yields a metric version of the Causal Reproducing Property stated above.  Often, in category theory, we want to work in an enriched category \cite{kelly-book}.  One of the most interesting ways to design categories for applications in AI and ML is to look to augment the basic structure of a category with additional properties. For example, the collection of morphisms from an object $x$ to an object $y$ in a category ${\cal C}$ often has additional structure, besides just being a set. Often, it satisfies additional properties, such as forming a space of some kind such as a vector space or a topological space. We can think of such categories as {\em enriched} categories that exploit some desirable properties. We will illustrate one such example of primary importance to applications in AI and ML that involve measuring the distance between two objects. A distance function is assumed to return some non-negative value between $0$ and $\infty$, and we will view distances as defining enriched $[0, \infty]$ categories. We summarize some results here from \cite{BONSANGUE19981}. 

\begin{figure}[h]
\centering
\begin{minipage}{0.9\textwidth}
\includegraphics[scale=0.8]{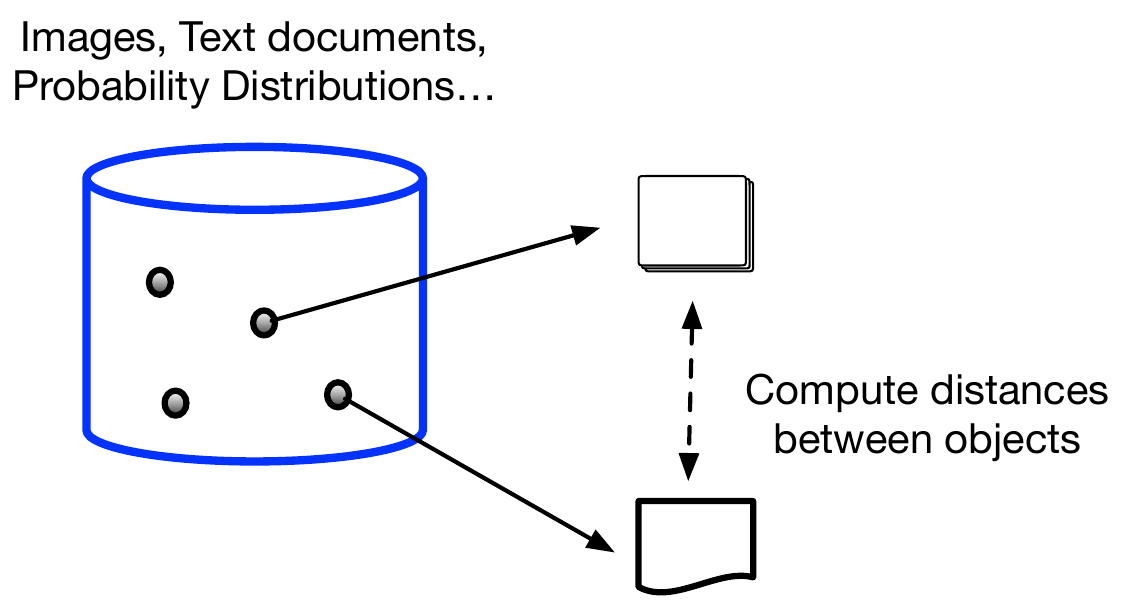}
\end{minipage}
\caption{Many algorithms in AI and ML involve computing distances between objects in a {\em metric space}. Interpreting distances categorically leads to powerful ways to reason about generalized metric spaces.}   
\label{metric-space}
\end{figure}

Figure~\ref{metric-space} illustrates a common motif among many AI and ML algorithms: define a problem in terms of computing distances between a group of objects. Examples of objects include points in $n$-dimensional Euclidean space, probability distributions, text documents represented as strings of tokens, and images represented as matrices. More abstractly, a {\em generalized metric space} $(X, d)$ is a set $X$ of objects, and a non-negative function $X(-,-): X \times X \rightarrow [0, \infty]$ that satisfies the following properties: 

\begin{enumerate}
    \item $X(x,x) = 0$: distance between the same object and itself is $0$. 

    \item $X(x,z) \leq X(x,y) + X(y,z)$: the famous {\em triangle inequality} posits that the distance between two objects cannot exceed the sum of distances between each of them and some other intermediate third object.
\end{enumerate}

In particular, generalized metric spaces are not required to be {\em symmetric}, or satisfy the property that if the distance between two objects $x$ and $y$ is $0$ implies $x$ must be identical to $y$, or finally that distances must be finite. These additional three properties listed below are what defines the usual notion of a {\em metric} space: 

\begin{enumerate}
    \item If $X(x, y) = 0$ and $X(y, x) = 0$ then $x = y$. 

\item $X(x, y) = X(y, x)$. 

\item $X(x, y) < \infty$. 
\end{enumerate}

In fact, we can subsume the previous discussion of causal inference under the notion of generalized metric spaces by defining a category around {\em preorders} $(P, \leq)$, which are relations that are reflexive and transitive, but not symmetric. Causal inference fundamentally involves constructing a preorder over the set of variables in a domain. Here are some examples of generalized metric spaces: 

\begin{enumerate}
    \item Any preorder $(P, \leq)$ such that all $p, q, r \in P$, if $p \leq q$ and $q \leq r$, then, $p \leq r$, and $p \leq p$, where 

   \[ P(p,q) =     \left\{ \begin{array}{rcl}
         0 & \mbox{if}
         & p \leq q \\ \infty  & \mbox{if} & p \not \leq q
                \end{array}\right\} \] 

                 \item The set of strings $\Sigma^*$ over some alphabet defined as the set $\Sigma$ where the distance between two strings $u$ and $v$ is defined as 

   \[ \Sigma^*(u,v) =     \left\{ \begin{array}{rcl}
         0 & \mbox{if}
         & u \ \mbox{is a prefix of} \ v \\ 2^{-n} & \mbox{otherwise} & \mbox{where} \ n \ \mbox{is the longest common prefix of } \ u \ \mbox{and} \ v
                \end{array}\right\} \] 

                 \item The set of non-negative distances $[0,\infty]$ where the distance between two objects $u$ and $v$ is defined as 

   \[ [0,\infty](u,v) =     \left\{ \begin{array}{rcl}
         0 & \mbox{if}
         & u \geq  v \\ v - u & \mbox{otherwise} & \mbox{where} \ r < s 
                \end{array}\right\} \] 

                 \item The powerset ${\cal P}(X)$ of all subsets of a standard metric space, where the distance between two subsets $V, W \subseteq X$ is defined as

   \[  {\cal P}(X)(V, W) = \inf \{ \epsilon > 0 | \forall v \in V, \exists w \in W, X(v, w) \leq \epsilon \} \] 

which is often referred to as the {\em non-symmetric Hausdorff distance}. 
                
\end{enumerate}

Generalized metric spaces can be shown to be $[0, \infty]$-enriched categories as the collection of all morphisms between any two objects itself defines a category.  In particular, the category $[0,\infty]$ is a complete and co-complete symmetric monoidal category. It is a category because objects are the non-negative real numbers, including $\infty$, and for two objects $r$ and $s$ in $[0,\infty]$, there is an arrow from $r$ to $s$ if and only if $r \leq s$. It is complete and co-complete because all equalizers and co-equalizers exist as there is at most one arrow between any two objects. The categorical product $r \sqcap s$ of two objects $r$ and $s$ is simply $\max\{r,s\}$, and the categorical coproduct $r \sqcup s$  is simply $\min\{r,s\}$. More generally, products are defined by supremums, and coproducts are defined by infimums. Finally, the {\em monoidal } structure is induced by defining the tensoring of two objects through ``addition":   

\[ +: [0, \infty] \times [0, \infty] \rightarrow [0,\infty]\]

where $r + s$ is simply their sum, and where as usual $r + \infty = \infty + r = \infty$. 

The category $[0,\infty]$ is also a {\em compact closed} category, which turns out to be a fundamentally important property, and can be simply explained in this case as follows. We can define an ``internal hom functor" $[0,\infty](-, -)$ between any two objects $r$ and $s$ in $[0, \infty]$ the distance $[0,\infty]$ as defined above, and the {\em co pre-sheaf} $[0,\infty](t,-)$ is {\em right adjoint} to $t + -$ for any $t \in [0, \infty]$. 

\begin{theorem}
For all $r, s$ and $t \in [0,\infty]$, 

\[ t + s \geq r \ \ \ \mbox{if and only if} \  \  \ s \geq [0,\infty](t,r)\]
\end{theorem}

We will explain the significance of compact closed categories for reasoning about AI and ML systems in more detail later, but in particular, we note that reasoning about feedback requires using compact closed categories to represent ``dual" objects that are diagrammatically represented by arrows that run in the ``reverse" direction from right to left (in addition to the usual convention of information flowing from left to right from inputs to outputs in any process model). 

We can also define a category of generalized metric spaces, where each generalized metric space itself as an object, and for the morphism between generalized metric spaces $X$ and $Y$, we can choose a {\em non-expansive function} $f: X \rightarrow Y$ which has the {\em contraction property}, namely 

\[ Y(f(x), f(y)) \leq c \cdot X(x,y) \]

where $0 < c < 1$ is assumed to be some real number that lies in the unit interval. The category of generalized metric spaces will turn out to be of crucial importance in this paper as we will use a central result in category theory -- the Yoneda Lemma -- to give a new interpretation to distances. 

Finally, let us state a ``metric" version of the Yoneda Lemma specifically for the case of $[0,\infty]$-enriched categories in generalized metric spaces: 

\begin{theorem}\cite{BONSANGUE19981}
({\bf Yoneda Lemma for generalized metric spaces}): Let $X$ be a generalized metric space. For any $x \in X$, let 

\[ X(-, x): X^{\mbox{op}} \rightarrow [0, \infty], \ \ y \longmapsto X(y, x)\]
\end{theorem}

Intuitively, what the generalized metric version of the Yoneda Lemma is stating is that it is possible to represent an element of a generalized metric space by its co-presheaf, exactly analogous to what we will see below in the next section for causal inference! If we use the notation

\[ \hat{X} = [0, \infty]^{X^{\mbox{op}}}\]

to indicate the set of all non-expansive functions from $X^{\mbox{op}}$ to $[0, \infty]$, then the Yoneda embedding defined by $y \longmapsto X(y, x)$ is in fact a non-expansive function, and itself an element of $\hat{X}$! Thus, it follows from the general Yoneda Lemma that for any other element $\phi$ in $\hat{X}$, 

\[ \hat{X}(X(-, x), \phi) = \phi(x) \]

Another fundamental result is that the Yoneda embedding for generalized metric spaces is an {\em isometry}. Again, this is exactly analogous to what we see below for causal inference, which we will denote as the causal reproducing property. 

\begin{theorem}
The Yoneda embedding $y: X \rightarrow \hat{X}$, defined for $x \in X$ by $y(x) = X(-, x)$ is {\em isometric}, that is, for all $x, x' \in X$, we have: 

\[ X(x, x') = \hat{X}(y(x), y(x')) = \hat{X}(X(-, x), X(-, x'))\]
\end{theorem}

Once again, we will see a remarkable resemblance of this result to the Causal Representer Theorem below.  With the metric Yoneda Lemma in hand, we can now define a framework for solving static UIGs in generalized metric spaces. 

 \begin{definition}
Two objects $c$ and $d$ are isomorphic in a generalized metric space category $X$ if they are isometrically mapped into the category $\hat{X}$ by the Yoneda embedding $c \rightarrow X(-, c)$ and $d \rightarrow X(-, d)$ such that $X(c,d) = \hat{X}(X(-, c), X(-, d))$, where they can be defined isomorphically by a suitable pair of suitable natural transformations. 
 \end{definition}

 \subsection*{Imitation Games using Coalgebraic Behavioral Distances} 

 We can build on the theory of generalized metric spaces to define {\em behavioral distance metrics} on coalgebras \cite{baldan2014behavioral}. These distance metrics can be then used to define a way to formulate UIGs based on comparing two participants based on a coalgebraic model of them. The basic idea of behavioral distance metrics is {\em lifting of functors}. A coalgebra defined by a functor $F$ on the category of {\bf Sets} is lifted to the category of (generalized or pseudo) metric spaces. Then, behavioral distances can be defined with respect to the lifted functor $\bar{F}$. These can be shown to correspond to well-known distance metrics in optimal transport \cite{ot}, such as Wasserstein distances or Kantorovich distances. 

\subsection*{Imitation Games in Generalized Metric spaces over the Category of Wedges}

Note we can also generalize  previous work on behavioral metrics for coalgebras by extending them to generalized metric spaces using the metric Yoneda embedding \cite{BONSANGUE19981}. We can also define imitation games over the category of {\em wedges} defined as a collection of objects $F: {\cal C}^{op} \times {\cal C} \rightarrow {\cal D}$, and the arrows are defined as dinatural transformations. We leave these extensions to a future paper. 

\section{Dynamic Universal Imitation Games: From Inductive to Coinductive Inference} 

\label{dynamicuig}

In the case when the participants are changing during the course of the interactions, we  need to consider how a participant changes over the course of interactions. In {\em dynamic UIGs}, ``learner" participants imitate ``teacher" participants.  We show here that initial objects correspond to the framework of {\em passive learning from observation} over well-founded sets using inductive inference -- expensively studied by Gold \cite{GOLD1967447} Solomonoff \cite{SOLOMONOFF19641} Valiant \cite{DBLP:journals/cacm/Valiant84} and Vapnik \cite{DBLP:journals/tnn/Vapnik99}.  In contrast, final objects correspond {\em coinductive inference} over non-well-founded sets and universal coalgebras \cite{jacobs:book,rutten2000universal},  which includes  learning from {\em active experimentation} corresponding to causal inference \cite{rubin-book} or reinforcement learning \cite{DBLP:books/lib/SuttonB98}.  We define a category-theoretic notion of minimum description length or Occam's razor based on final objects in coalgebras.  

We now introduce coinductive inference, a novel theoretical foundation for ML that is based on the conceptualization of non-well-founded sets in terms of APGs \cite{Aczel1988-ACZNS}, graphs not sets, or in terms of universal coalgebras \cite{rutten2000universal}.  Put succinctly, the difference between coinductive inference vs. inductive inference lies in that in the former, APGs (or equivalently universal coalgebras) are enumerated to find one consistent with the teacher's presentation. In contrast, as we saw above, inductive inference is based on enumerating recursively enumerable sets, which do not allow for non-well-founded sets. 

The term coinductive inference, which we introduce in this paper, is based on {\em coinduction},  a proof principle, developed initially by Peter Aczel in his ground-breaking work on non-well-founded circular (or recursively defined) sets \cite{Aczel1988-ACZNS}. Figure~\ref{coindfig} illustrates the general framework of coinductive inference that is based on the theory of universal coalgebras. Comparing the framework for inductive inference presented above, the major difference is the use of AGPs or universal coalgebras as the formalism for describing the language used by the teacher to specify a particular model from which data is generated, and by the learner to produce a solution. Universal coalgebras provide a universal language for AI and ML in being able to describe a wide class of problems, ranging from causal inference, (non)deterministic and probabilistic finite state machines, game theory and reinforcement learning. 

\begin{figure}[t]
\includegraphics[scale=.35]{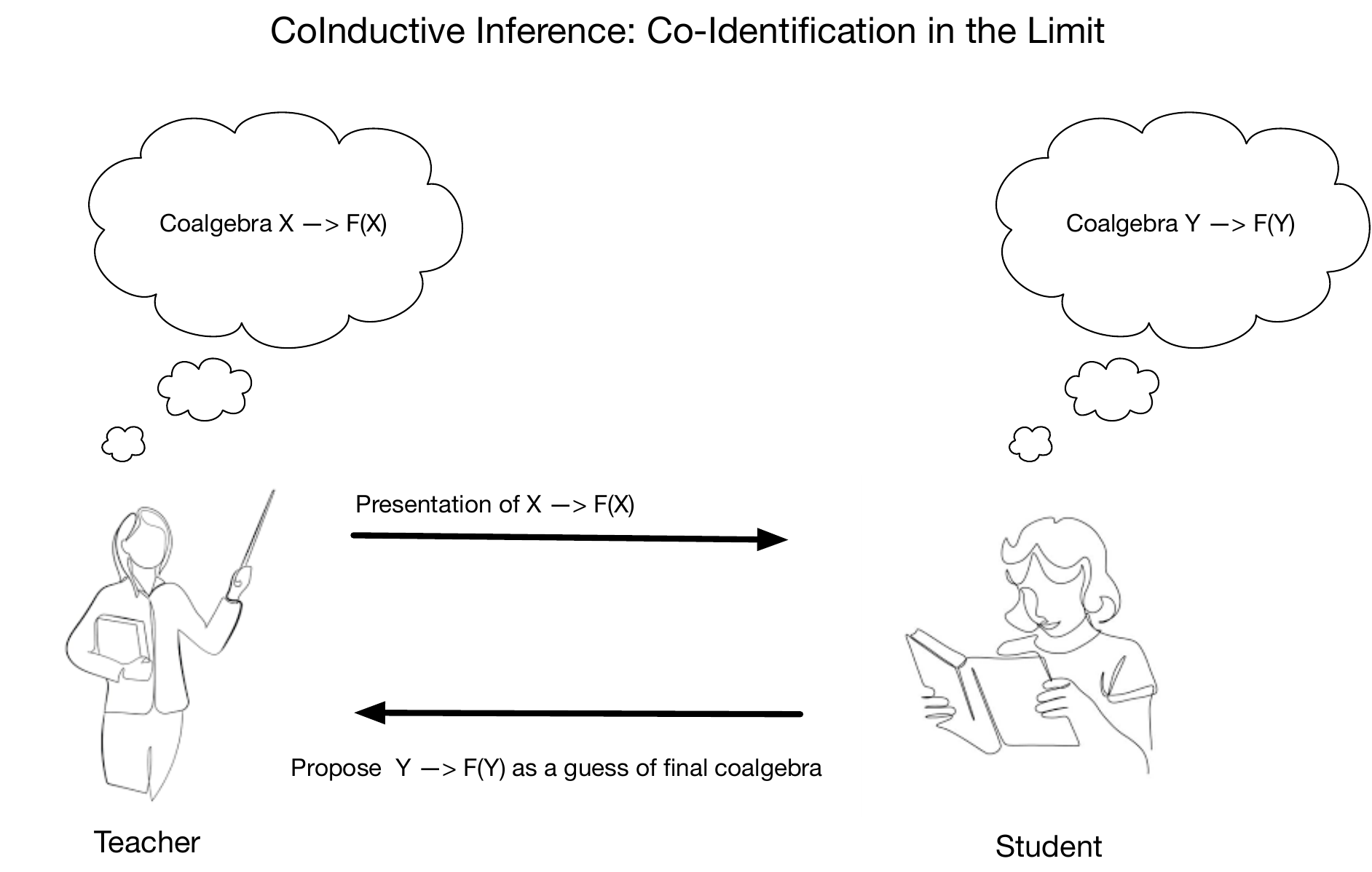}
%
\caption{The coinductive inference framework proposed in this paper models the process of machine learning as discovering a final coalgebra in the category of coalgebras. Rather than enumerating sets, as in inductive inference, representations of non-well-founded sets as accessible pointed graphs (APGs) or universal coalgebras are enumerated. Coidentification refers to the use of {\em bisimulation}, a fundamental relation used to compare two non-well-founded sets or universal coalgebras.}
\label{coindfig} 
\end{figure}

To take some concrete examples, let us begin with the standard problem of learning a deterministic finite state machine model. We can specify a finite state machine model as a {\em labeled transition system} $(S, \rightarrow_S, A)$ consisting of a set $S$ of states, a transition relation $\rightarrow_S \subseteq S \times A \times S$, where $A$ is a set of input labels. Let us define a functor ${\cal B}(X) = {\cal P}(A \times X) = \{V | V \subseteq A \times X$ for any set $X$. $B$ is then an endofunctor on the category of sets, which specifies the transition dynamics of the finite state machine. We can then represent a labelled transition system as a universal coalgebra of the following type: 

\[ \alpha_S: S \rightarrow {\cal B}(S), \ \ \ \ s \rightarrow \{(a, s') | s \xrightarrow[]{a} s' \} \]

for some fixed set of input symbols $A$. The coinductive inference framework assumes the teacher selects a particular coalgebra specifying a finite state machine and generates a presentation of it for the learner. Upon receiving a sequence of examples, the learner produces a hypothesis coalgebra describing the finite state machine. The learner's ultimate goal can be succinctly summarized as discovering the final coalgebra that represents the minimal finite state machine model that is isomorphic to the coalgebra selected by the teacher. 

In this second class of UIGs, we turn from Turing's original question ``Can machines think?" to ``Can machines learn?". We , we turn to discuss the fundamental question ``Can machines learn?". As with the previous case, we have to define what ``learn" means. We introduce a categorical perspective that neatly spans the initial object framework of {\em inductive inference} over well-founded sets  proposed by Gold \cite{GOLD1967447} and Solomonoff \cite{SOLOMONOFF19641} and later refined by Valiant \cite{DBLP:journals/cacm/Valiant84} and Vapnik \cite{DBLP:journals/tnn/Vapnik99} to the probabilistic case, with the  final object {\em coinductive inference}, a categorical framework over non-well-founded sets based on universal coalgebras \cite{jacobs:book,rutten2000universal}  framework for ML based on {\em coidentification in the limit}. Inductive inference is formulated in the traditional realm of ZFC set theory,  which does not permit non-well-founded sets. Aczel \cite{Aczel1988-ACZNS} introduced a formalism for non-well-founded sets using graphs, and developed a coinduction principle for reasoning about circular sets. This framework was later extended by Rutten for the study of dynamical systems.  We show a wide range of generative AI and ML systems can be modeled as {\em endofunctors} that act on specific categories. Algebras compress and summarize information. Coalgebras, in contrast act as generators and enable rigorously modeling non-well-founded infinite data streams produced by generative models, including automata, grammars, probabilistic transition models among others. Coidentification in the limit is defined as finding a {\em final coalgebra} in a category of coalgebras. Final coalgebras generalizes the use of (greatest) fixed points. We introduce the theory of monads, which gives an elegant framework to integrate algebraic and coalgebraic reasoning, and provides a categorical foundation for modeling probabilities, Markov chains and stochastic processes.

Before introducing our novel formulation of ML as coinductive inference over non-well-founded sets, we want to give a brief historical review of ML, focusing on the theoretical framework of inductive inference that has been the cornerstone of the field for many decades, originally introduced by Gold \cite{GOLD1967447}, Solomonoff \cite{SOLOMONOFF19641} and others based on the principle of mathematical induction over well-founded sets. Induction is a natural way to model learning as a process of generalizing rules from examples. A long tradition in philosophy, dating back to Hume \cite{hume} and even before to Aristotle \cite{aristotle} have thought about learning as induction.  We first briefly review the framework of Gold on inductive inference or identification in the limit, and then describe probably approximately correct (PAC) learning introduced by Valant \cite{DBLP:journals/cacm/Valiant84} that addressed some limitations of the Gold paradigm. We finally discuss some limitations of the PAC model as well.  

Humans appear singularly adept at inducing general rules from even a relatively small number of  well-chosen examples, an ability that provides the foundation for many of our intellectual capabilities, from learning language to motor behavior and social skills. Figure~\ref{bongard} illustrates a problem of visual induction introduced by Bongard \cite{bongard} in his book on pattern recognition.  Bongard's book featured a collection of several hundred such problems of varying difficulty, which can be found online as well. \footnote{See {\tt https://www.foundalis.com/res/bps/bpidx.htm} for a listing of Bongard problems.} 

\begin{figure}[t]
\includegraphics[scale=.65]{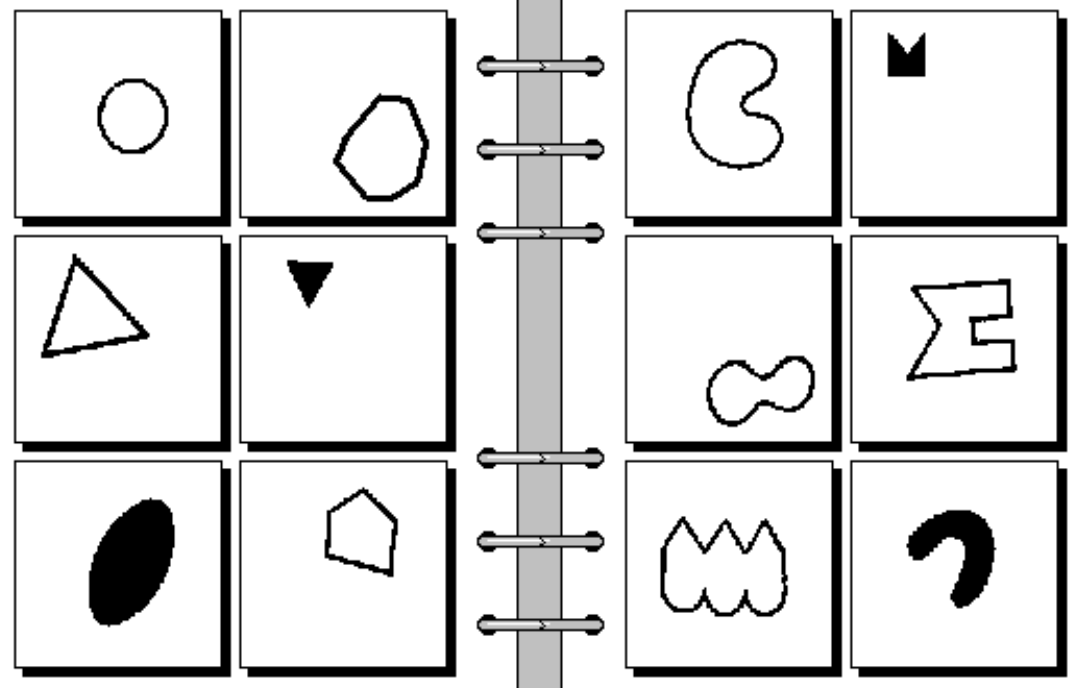}
%
\caption{An example of visual pattern recognition from a collection of several hundred such problems originally from a book by Bongard \cite{bongard}. The learner is given a series of figures grouped into two classes, and is required to generate a rule that distinguishes the six figures on the left from a similar number on the right. .}
\label{bongard} 
\end{figure}

Our ability at inducing patterns from examples extends across modalities. Consider, for example, the problem of inducing a function over the natural numbers, from a sequence of examples such as the one below: 

\begin{verbatim}
1, 2, 3, 5, 7, 11, 13, ...
\end{verbatim}

This sequence immediately brings our attention to the possibility that the sequence in question represents the {\em prime numbers}, numbers that are not divisible by any other number except $1$ and themselves. We can immediately see that the problem of induction seems {\em ill-posed}, as there may be many hypotheses consistent with a given sequence of examples, and it may not be possible to determine the right hypothesis from the given data. For example, consider the inductive inference problem posed by the initial part of the above sequence: 

\begin{verbatim}
1, 2, 3, ...
\end{verbatim}

In this case, the number of possible hypotheses could be enormous. Certainly, it would include the previous hypothesis about the sequence being the prime numbers, but it could also be all natural numbers, or many other possibilities. It is relatively easy to design inductive inference problems that can defeat any learner based on some initial presentation. Consider the following somewhat contrived example: 

\begin{verbatim}
0, 0, 0, 0, ...
\end{verbatim}

In this case, the learner might conjecture that this sequence is simply the set of all $0$'s, but it could have been generated from the following polynomial: 

\[ f(n) = (n - 1) (n - 2) (n -3) (n - 4) \]

in which case the correct answer for the fifth element in the sequence is: 

\begin{verbatim}
0, 0, 0, 0, 24, ...
\end{verbatim}

Finally, let us consider an {\em adversarial} setting, where a teacher simply generates a sequence that keeps changing depending on a learner's guesses. Such considerations suggest that to define a useful model of ML using inductive inference, additional structure must be imposed to make the problem well-defined. We now review the most influential models of ML next, beginning with the model of inductive inference: identification in the limit. 

\subsection{Passive Learning: Inductive Inference over Well-Founded Sets}
\label{ch1:sec:3}

\begin{figure}[t]
\includegraphics[scale=.35]{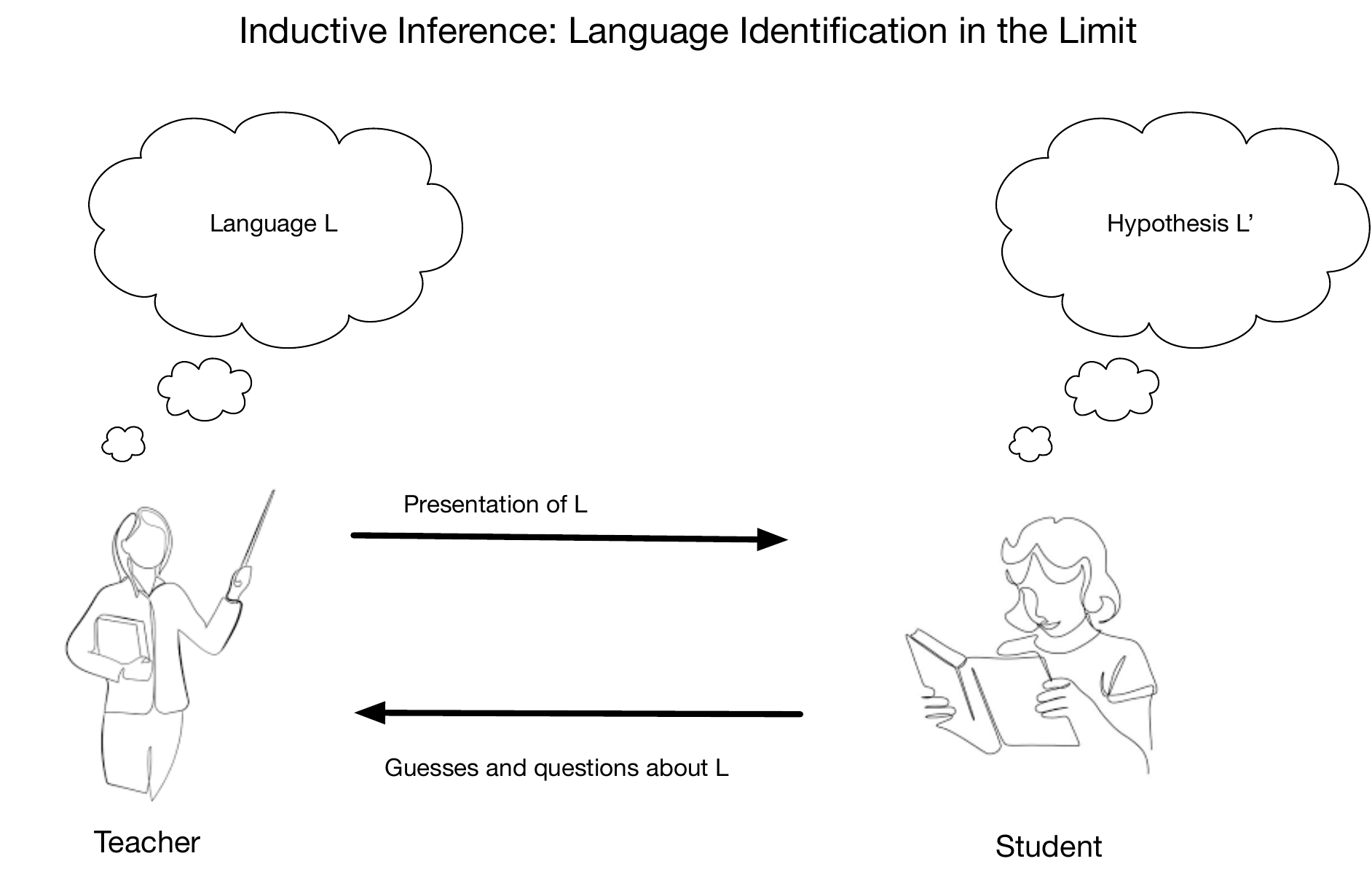}
%
\caption{The model of inductive inference.}
\label{gold-paradigm} 
\end{figure}

We briefly describe a formal theory of inductive inference proposed by Solomonoff in 1964, and further studied by Gold in a classic paper in 1967. Figure~\ref{gold-paradigm} describes the setting of language identification in the limit, a model of inductive inference proposed by Gold in 1967.  In this framework, a {\em teacher} selects a candidate language $L$ from some hypothesis space of possible languages, and constructs a {\em presentation} of $L$ to teacher a student learner. The student's task is to identify the unknown language from the teacher's presentation. The term {\em identification in the limit} is intended to reflect the continuing nature of this inductive inference task. This original setting of inductive inference is intended similar to the theoretical model of computability: are there intrinsic limitations to induction as a model of ML, independent of computational complexity criteria or any finiteness considerations?  Table~\ref{gold-table} summarizes Gold's principal results. To understand the table, we need to define the main components of this model more carefully, as this structure will also inform the design of our coinductive inference framework later in this paper. 

\begin{table}[!t]
\caption{Gold's results on language identification in the limit.}
\label{gold-table}       
%
%
		\begin{tabular}{l l l}
			\toprule
			\textbf{Model} & \textbf{Languages} & \textbf{Learnable?}\\
			\midrule
			Anomalous text &Recursively enumerable & Yes \\
   	 &Recursive  & Yes \\ \hline 
			Informant & Primitive Recursive & Yes (not above) \\
    & Context-sensitive & Yes \\
      & Context-free & Yes \\
          & Regular & Yes \\
              & Super-finite & Yes \\ \hline 
              	Text  & Finite & Yes (not above) \\
			\bottomrule
		\end{tabular}
		\caption{Gold, Information and Control, 1967}
	\end{table}

A {\em language learnability model}, as defined by Gold, consisted of the following components:

\begin{enumerate}
\item A definition of learnability: Gold focused on {\em identifiability in the limit}. 

\item A method of information presentation: Table~\ref{gold-table} includes anomalous text, informant, and text. 

\item A naming relation: Gold focused on principally two naming relations, including a {\em tester}, and a {\em generator}. 
\end{enumerate}

This conceptualization is fundamentally set-theoretic: identification in the limit implies naming a well-founded set through some method, such as {\em enumeration} of all recursively enumerable sets. In practice, more efficient methods have been developed over the past 60 years, and we focus on one particularly instructive framework next. 

\subsection*{Inductive Inference using Version Spaces}
\label{ch1:sec:4}

As a concrete example of an inductive inference algorithm, in this section we describe Tom Mitchell's {\em version spaces} framework \cite{DBLP:conf/ijcai/Mitchell77} (see Figure~\ref{versionspaces}). Version spaces is particularly interesting as it specifies not one specific algorithm, but rather a broad framework that can be specialized in many ways to design more specific methods. The $S$ set represents the set of all {\em maximally specific} hypotheses consistent with the data, whereas the $G$ set represents the set of all {\em maximally general} hypotheses consistent with the data. The notion of specificity or generality is with respect to the partial ordering implicit in the lattice structure shown in Figure~\ref{gold-paradigm}. Bringing in the category-theoretic perspective, the lattice structure can be defined as a {\em preorder} category $C = ({\cal H}, \leq)$, where $H$ is the set of all hypotheses under consideration (e.g., all context-free languages), and $\leq$ is a reflexive transitive relationship on hypotheses. A morphism $h_1 \rightarrow h_2$ exists between hypotheses $h_1$ and $h_2$ when $h_1 \leq h_2$. As we will show later, computing the maximally specific and maximally general hypotheses in a version space corresponds to computing the meets and joins in a partial order, respectively, or more generally, finding the {\em limits} and {\em colimits} of universal diagrams $F: J \rightarrow H$. 

\begin{figure}[t]
\includegraphics[scale=.3]{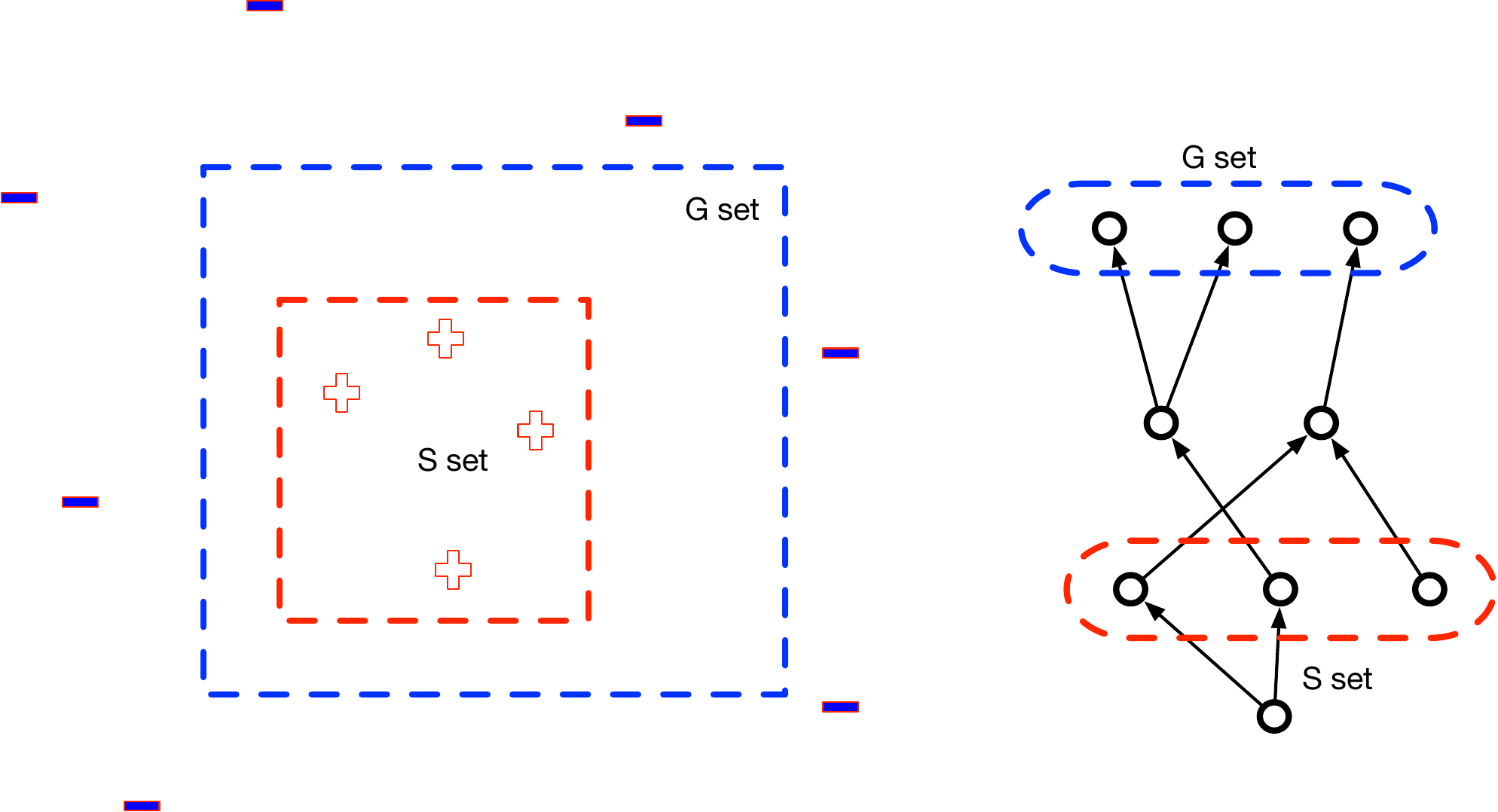}
%
\caption{The version spaces algorithm is a canonical example of an inductive inference method. In terms of category theory, the lattice defines a preorder category. The set of maximally specific hypotheses $S$ and maximally general hypotheses $G$ consistent with the training data set can be computed as the {\em meets} and {\em joins} in a preorder, a special case of computing the {\em limits} and {\em colimits} of a universal diagram $F: J \rightarrow H$, a functor from some indexing category $J$ into the category $H$ of hypotheses. In Gold's paradigm, the learner has to guess the right axis-parallel rectangle that the teacher has selected, such that every positive example lies within the rectangle, and every negative example lies outside it. In Valiant's PAC learning framework, it is sufficient to guess a rectangle whose error as measured by the probability of misclassification is $\leq \epsilon$, and do so reliably with probability $\geq 1 -  \delta$. In Vapnik and Chervonenkis formulation, the VC-dimension of all axis-parallel rectangles on $\mathbb{R}^2$ is $4$, and hence this hypothesis space ${\cal H}$ is PAC-learnable from a polynomial number of samples. In our coalgebraic framework of coinduction, the hypothesis space is modeled as a category of coalgebras that define dynamical systems that generate positive or negative examples. }
\label{versionspaces} 
\end{figure}

\subsection*{Probably Approximately Correct Learning}
\label{ch1:sec:5}

A significant weakness of the original inductive inference model proposed by Gold is that it required the learner to exactly identify a language, and it did not impose any computational requirement in terms of the number of examples required or the amount of time needed for identification. In an influential paper published in the {\em Communications of the ACM}, Valiant proposed a significantly revised model of inductive inference in 1984 that has come to be known as {\em probably approximately correct} (PAC) learning \cite{DBLP:journals/cacm/Valiant84}. In PAC learning, the data is assumed to be generated by some unknown (to the learner) but fixed probability distribution $P$ over a hypothesis space $X$ (e.g., strings generated by a regular language associated with a finite state automata). The goal for the learner is to discover an approximation $h(x)$ of the unknown function $f(x)$ that is accurate as measured by its differences with the true function on data that is sampled from the distribution $P$. In this setting, accuracy is defined as the probability that a randomly chosen point $x \in X$ will result in the outcome $h(x) \neq f(x)$. The goal is to find an approximate hypothesis such that the probability of disagreement can be bounded by some scalar value $0 < \epsilon < 1$: 

\[ P(h(x) \neq f(x)) \leq \epsilon \]

However, since the teacher uses the distribution $P$ to sample instances, an unlucky run of truly ``bad" examples cannot be ruled out, but the goal is to ensure that such an unrepresentative set of training examples occurs rarely. In other words, the learner is allowed to fail, but with probability bounded by a second parameter $0 < \delta < 1$. The aim of PAC learning is to determine for any given class of hypotheses ${\cal H}$ whether every function $f \in {\cal H}$ can be identified in the PAC sense using only polynomially many examples in the parameters $n$ (size of a hypothesis), $\frac{1}{\epsilon}$ and $\frac{1}{\epsilon}$. The intuition here is that $\epsilon$ and $\delta$ are reduced closer to $0$, requiring the learner to produce a more accurate hypothesis more reliably, the number of examples will grow as well. Valiant's PAC framework was further generalized by combining it with the notion of dimensionality proposed by Vapnik and Chervonenkis, which came to be known as the VC-dimension \cite{DBLP:journals/tnn/Vapnik99}. A fundamental result in this setting was that a hypothesis class is polynomially sample PAC learnable if its VC-dimension was finite (in the continuous setting) or polynomial in the relevant parameters ($n$, $\frac{1}{\epsilon}$, $\frac{1}{\delta}$). In the particular case where the hypothesis space is all axis-parallel rectangles on $\mathbb{R}^2$, the VC-dimension is $4$ as no training set of size larger than $4$ can be ``shattered". A training data set ${\cal D}$ can be shattered if there exists a hypothesis $h \in {\cal H}$ that is consistent with every possible labeling of the points in ${\cal D}$. In the case of axis-parallel rectangles, there exists a set of $5$ points on the plane where the interior point can be labeled $-$ and the remaining points labeled $+$, and no hypothesis in ${\cal H}$ is consistent with this labeling. In more concrete terms, using the example hypothesis space of all axis-parallel rectangles in Figure~\ref{versionspaces}, it suffices to maintain one hypothesis rectangle consistent with all the training data seen so far, and it can be shown that the version space of all consistent hypotheses (i.e., axis-parallel rectangles) can be $\epsilon$-exhausted given polynomially many samples in the VC-dimension (which is $4$ in this case) and accuracy parameter $\epsilon$ and reliability parameter $\delta$ \cite{DBLP:journals/ai/Haussler88,HAUSSLER1987324}. 

\subsection*{High-Dimensional Statistics: Limitations of the PAC Model}

In practice, there are a number of significant limitations of the PAC model, mostly having to do with its inability to exploit the actual structure that exists in many real-world machine learning problems (see Figure~\ref{highdim}). In high-dimensional statistics, data tends to exhibit a {\em concentration} effect. In low dimensions, a uniform distribution of data points in a a hypersphere appear scattered throughout, but as the dimensionality increases, much of the data concentrates on the boundary. The set of all covariance matrices that might summarize statistical properties of datasets that lie in vector spaces forms a nonlinear manifold. It is possible to exploit such statistical and topological regularities to design efficient inductive inference algorithms that behave far better in high-dimensional problems that the idealized PAC bounds would  predict. In addition, there seems to be a growing realization that the phenomenon of {\em overfitting}, where a model has a very large number of parameters, such as the billions of weights in large language models \cite{DBLP:conf/nips/VaswaniSPUJGKP17}, is not as problematic as some of the simplified statistical models would predict \cite{overfitting}.  

\begin{figure}[t]
\centering
\includegraphics[scale=.35]{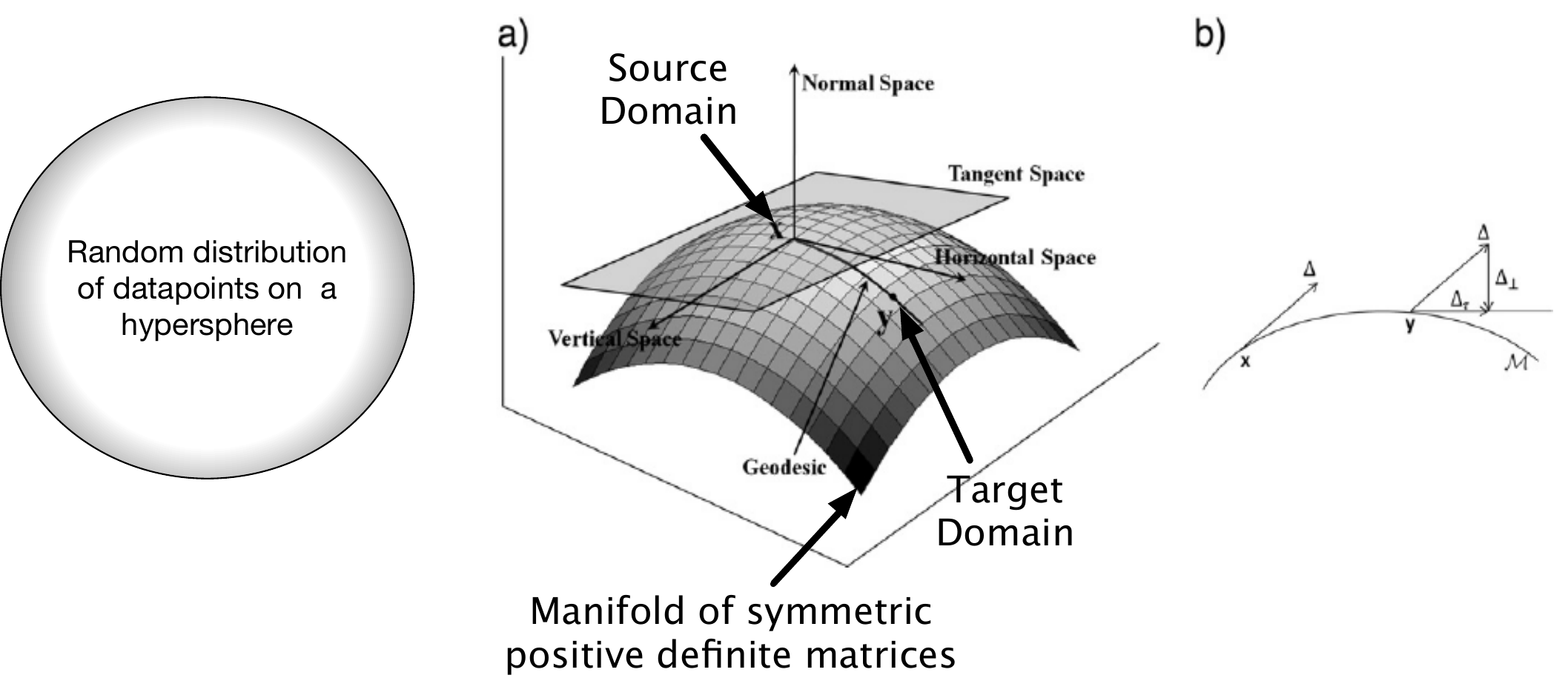}
\caption{(Left): High-dimensional datasets behave counter intuitively: a random collection of points distributed uniformly on a hypersphere concentrates near the boundary as the dimensionality increases. (Right) The set of symmetric positive definite matrices that model covariances across variables in datasets data defines a nonlinear Riemannian manifold surface.}
\label{highdim} 
\end{figure}

\subsection*{Algorithmic Information Theory and Kolmogorov Complexity}
\label{ch1:sec:6}

The mathematician Kolmogorov proposed in 1965 a model of the {\em intrinsic complexity} of an object, which can be viewed as a rigorous definition of inductive inference based on the heuristic of Occam's razor: {\em the simplest explanation is the best}. Kolmogorov complexity can be viewed as an intellectual predecessor to information theory, as it does not depend on a probability distribution, and has been proposed as a {\em universal probability model} \cite{DBLP:journals/tcs/Vitanyi13}. Chaitin proposed a similar framework called {\em algorithmic information theory} \cite{chaitin}. 

As we discuss below, we use coinductive inference to propose a new perspective on characterizing minimum description length type formulations in terms of the final coalgebras. As we will see later, in categories of coalgebras generated by a functor that admits a final coalgebra, a representation is minimal if the unique morphism mapping it to the final coalgebra is injective. This category-theoretic reformulation of the MDL principle is one of the novel aspects of our formulation of ML, which we turn to describe in more detail now.

\subsection*{Categorical Version Spaces using Ends and Coends}

We can generalize the classic version space algorithm \cite{DBLP:conf/ijcai/Mitchell77} to arbitrary categories, by modeling the process of finding the $S$ and $G$ sets in terms of the limits and colimits of a set of data points. We can use the simplicial set representation introduced in the previous section, and used in dimensionality reduction methods such as UMAP \cite{umap}.  We leave this extension to a future paper.

\subsection{Passive vs. Active Learning in Dynamic UIGs: Statistical vs. Causal Inference}   

To make the transition from inductive inference over well-founded sets to coinductive inference over non-well-founded sets and universal coalgebras, we work through a series of more concrete examples, starting with the distinction between passive statistical learning from data \cite{manifoldregularization} vs. learning from active experimentation. We use as examples the difference between building a statistical model based on conditional independences, vs. doing causal inference, as well as learning a finite state machine model from passive observation vs. learning a diversity automata from experimentation \cite{DBLP:journals/jacm/RivestS94}. 

To define UIGs over statistical models, we define a category of objects, where each object abstractly represents a statistical model in terms of a set of conditional independences \cite{DBLP:journals/amai/Dawid01,pearl-book}.
{\em Conditional independence} is a foundational concept in statistics, which finds many applications to problems in AI and ML, including structured representations of probabilistic graphical models, causal inference, and estimation of statistical models from data. One can define a generalized algebraic notion of conditional independence in terms of {\em separoids}. 

A separoid $({\cal S}, \leq, \mathrel{\perp\mspace{-10mu}\perp})$ is defined as a semi-lattice ${\cal S}$, where the join $\vee$ operator over the semi-lattice ${\cal S}$ defines a preorder $\leq$, and the ternary relation $\CI$ is defined over triples of the form $(x \CI y | z)$ (which are interpreted to mean $x$ is conditionally independent of $y$ given $z$).  We show briefly how to define a category for universal conditional independence, where each object is a separoid, and the morphisms are homomorphisms from one separoid to another. It is possible to define ``lattice'' objects in any category by interpreting an arrow $f: x \rightarrow y$ as defining the partial ordering.

\begin{definition}
\label{separoid}
A {\bf {separoid}} defines a category over a preordered set $({\cal S}, \leq)$, namely $\leq$ is reflexive and transitive, equipped with a {\em ternary} relation $\CI$ on triples $(x,y,z)$, where $x, y, z \in {\cal S}$ satisfy the following properties: 

\begin{itemize}
    \item {\bf {S1:}} $({\cal S}, \leq)$ is a join semi-lattice. 
    \item {\bf {P1:}} $x \mathrel{\perp\mspace{-10mu}\perp} y \ | \ x$
    \item {\bf {P2:}} $x \mathrel{\perp\mspace{-10mu}\perp} y \ | \ z \ \ \ \Rightarrow \ \ \ y \CI x \ | z$ 
    \item {\bf {P3:}} $x \mathrel{\perp\mspace{-10mu}\perp}  y \ | \ z \ \ \ \mbox{and} \ \ \ w \leq y \ \ \ \Rightarrow \ \ \ x \mathrel{\perp\mspace{-10mu}\perp} w \ | z$ 
    \item {\bf {P4:}} $x \mathrel{\perp\mspace{-10mu}\perp} y \ | \ z \ \ \  \mbox{and} \ \ \ w \leq y \ \ \ \Rightarrow \ \ \ x \CI y \ | \ (z \vee w)$
    \item {\bf {P5:}} $x \CI y \ | \ z \ \ \ \mbox{and} \ \ \ x \CI w \ | \ (y \vee z) \ \ \ \Rightarrow \ \ \ x \CI (y \vee w) \ | \ z$
\end{itemize}

A {\bf {strong separoid}} also defines a categoroid. A strong separoid is defined over a lattice ${\cal S}$ has in addition to a join $\vee$, a meet $\wedge$ operation, and satisfies an additional axiom: 

\begin{itemize} 

\item {\bf {P6}}: If $z \leq y$ and $w \leq y$, then $x \CI y \ | \ z \ \ \ \mbox{and} \ \ \ x \CI y \ | \ w \ \ \ \Rightarrow \ \ \ x \CI y  \ | \ z \ \wedge  \ w$
\end{itemize}
\end{definition} 

To define a category of separoids, we have to define the notion of a homomorphism between separoids: 

\begin{definition}
Let $\langle {\cal S}, \leq, \CI \rangle$ and $\langle {\cal S'}, \leq', \CI' \rangle$ be two separoids. A map $f: {\cal S} \rightarrow {\cal S}'$ is a {\bf {separoid homomorphism}} if:

\begin{enumerate}
    \item It is a join-lattice homomorphism, namely $f(x \vee y) = f(x) \vee' f(y)$, which implies that $x \leq y \rightarrow f(x) \leq' f(y)$. 
    \item $x \CI y \ | z \rightarrow f(x) \CI' f(y) \ | f(z) $. 
    \item In case both {\cal S} and {\cal S}' are strong separoids, we can define the notion of a strong separoid homomorphism to additionally include the condition: $f(x \wedge y) \rightarrow f(x) \wedge' f(y)$.
\end{enumerate}
\end{definition}

With this definition, we can now define the category of separoids and a representation-independent characterization of universal conditional independence as follows: 

\begin{theorem}
 The category of separoids is defined as one where each object in the category is defined as a separoid $\langle {\cal S}, \leq, \CI \rangle$, and the arrows are defined as (strong) separoid homomorphisms. The category of separoids provides an axiomatization of universal conditional independence, namely that it enables a universal representation through the use of universal arrows and Yoneda lemma. 
\end{theorem}
\begin{proof}
First, we note that the category of separoids indeed forms a category as it straightforwardly satisfies all the basic properties. The (strong) separoid homomorphisms compose, so that $g \circ f$ as a composition of two (strong) separoid homomorphisms produces another (strong) separoid homomorphism. The universal property derives from the use of the Yoneda lemma to define a category of presheaves that map from the category of separoids to the category {\bf {Sets}}.
\end{proof}

With this category of separoids defined, we can formulate a static UIG over separoids as follows: 

\begin{definition}
    Two objects $c$ and $d$  in a category ${\cal C}$ of separoids are isomorphic if there is a separoid homomorphism $f: c \rightarrow d$ and a separoid homomorphism $g: d \rightarrow c$ such that $g \circ f \simeq {\bf id}_c$ and $f \circ g \simeq {\bf id}_d$. 
\end{definition}

\subsection*{Causal Imitation Games} 

\begin{figure}[t]
\begin{minipage}{0.7\textwidth}
\includegraphics[scale=0.4]{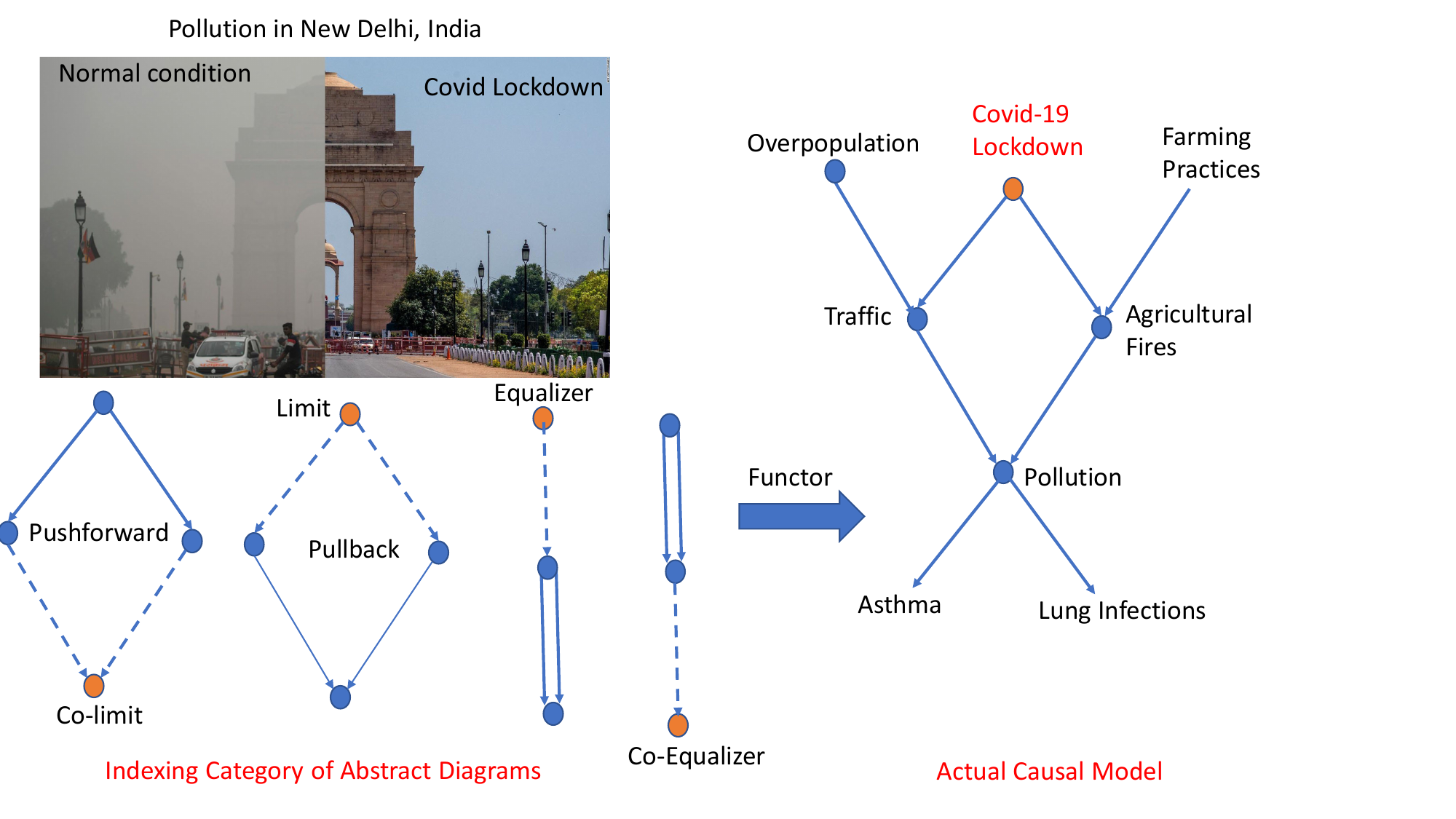}
\end{minipage}
\caption{A causal model of air pollution reduction due to the natural intervention imposed by a Covid-19 pandemic lockdown in New Delhi, India.}   
\label{a-causal-category}
\end{figure}

Causal imitation games have been extensively studied (see Table~\ref{causalcats}) for over a hundred years, ever since Darwin asked Galton to help him design a test that would separate out two species of plants that he had grown under different conditions \cite{darwin:1876}. Galton proposed a rank-ordered statistic, where the plants are all ranked by height, and the imitation game can be solved (i.e., variant of type 1 can be  discriminated from variant of type 2) if the specific ordering can be shown to be different from what one would expect in a random ordering. Since that time, for 150 years, researchers in every area of science have developed techniques for deciding if two groups of objects are the same or different, under the condition where one group can be subjected to some experiment. A thorough review of these techniques under the framework of {\em potential outcomes} is given in \cite{rubin-book}. The name derives from the fact that causal inference invariably involves counterfactuals due to unobservable confounders: the test group of objects is administered a treatment, whereas the control group is not, and the data from the experiment does not reveal the counterfactual outcomes. In computer science, causal inference has been studied on directed acyclic graphs (DAGs) \cite{pearl-book}, and more recently, on categories \cite{fong:ms,fritz:jmlr,DBLP:journals/entropy/Mahadevan23}. We review some of the main concepts here, with a special emphasis on our recent framework of universal causality \cite{DBLP:journals/entropy/Mahadevan23}. 

Humans have sought to understand causality since ancient times. 2500 years ago, Plato remarked that {\em "Everything that becomes or changes must do so owing to some cause; for nothing can come to be without a cause"}. Even a cursory study of the past literature on causal inference would immediately overwhelm the reader, since a veritable ``Tower of Babel" repertoire of languages and representations have been proposed to understand causality  (see  Table~\ref{causalcats}).  Category theory can be viewed as a ``Rosetta Stone" \cite{Baez_2010} to translate across diverse representations used in causal inference.  In a recent paper \cite{DBLP:journals/entropy/Mahadevan23}, we showed that the Yoneda Lemma provides deep insight into the nature of causal reasoning, through both theoretical results as well as providing an elegant way to capture universal properties of causality.  Specifically, we identified the {\em Causal Reproducing Property}, a special case of the Yoneda Lemma, that shows that the contravariant functor ${\cal C}^{op}(-, x)$ is the principal carrier of causal information in a category (which by default includes all the work on categorical modeling over directed graphs \cite{pearl-book}). 

\begin{table}[h] 
\caption{Category theory applies across a diverse range of causal imitation games.}
\vskip 0.1in
\begin{minipage}{0.7\textwidth}
 \begin{small}\hfill 
  \begin{tabular}{|c|c|c|c|} \hline 
Representation & Objects  & Morphisms & References  \\ \hline 
Graphs & Variables $X,Y,\ldots$ & Paths & \cite{pearl-book} \\ \hline 
Topological Spaces & Open sets $\{\{a\}, \{ b \}, \{a,b\}\}$  & Fences & \cite{Bradley_2022} \\ \hline  
Information fields & Measurable Spaces & Measurable mappings & \cite{sm:udm} \\ \hline
Resource Models & Monoidal resources & Profunctors & \cite{fong2018seven} \\ \hline
Concurrent Systems & Program variables & Bisimulation morphisms & \cite{JOYAL1996164} \\ \hline
Dynamical systems & States & Processes & \cite{sm:udm} \\ \hline
Counterfactuals & Propositions & Proofs & \cite{DBLP:journals/jphil/Lewis73} \\ \hline 
Network Economy & Consumers/Producers & Trade & \cite{nagurney:vibook} \\ \hline 
Discourse Sheaves & Users & Communication & \cite{discourse-sheaves} \\ \hline
\end{tabular}
\end{small}
\end{minipage}
\label{causalcats}
\end{table}

As a concrete example, Judea Pearl developed an approach to causal inference based on directed acyclic graph (DAG) models, where a variable $X$ can exert a causal influence on variable $Y$ if a directed path exists between $X$ and $Y$. Pearl's framework suggests an immediate categorical treatment, as the paths in a directed graph in fact correspond directly to the morphisms in the so-called {\em free category} associated with the graph.  We can define causal models in terms of functors that act on a category of algebraic structures, mapping them to a category of probabilistic representations defined as morphisms over a Kleisli category of a Giry monad \cite{giry1982}. In effect, a Bayesian network causal model is in fact a functor, because it is more than a graph, it is a mapping of a graph onto a structured probability distribution. A Bayesian network is more appropriately thought of as a functor that maps between two categories, a syntactic category of graphs as algebraic structures and a semantic category of probability distributions. Thinking of Bayesian networks as functors opens the door to including more sophisticated ideas. We can introduce universal constructions from category theory, including the {\em limit} or {\em co-limit} of an abstract causal diagram, or more generally, the Kan extension. We can formalize Plato's insight using the Yoneda Lemma, since every variable in a causal model is influenced by the set of all variables that can act on it through some causal pathway. This notion can be made formal through the Yoneda Lemma by constructing the presheaf contravariant functor from the category defining a causal model into a category of probability distributions or sets. 

Figure~\ref{a-causal-category} illustrates the main idea of using category theory to formalize causal inference using a real-world example showing how a natural intervention experiment caused by the Covid pandemic caused a significant reduction in the air quality in New Delhi, India, one of the world's most polluted cities. Causality studies the question of how an object (e.g., air pollution in New Delhi, India) {\em changes} due to an intervention on another object (e.g., Covid-19 lockdown). Causal influences are transmitted along paths in a causal model. In the directed acyclic graph (DAG) model shown, causal influence corresponds to a directed path from a variable $X$ into another variable $Y$ \cite{pearl-book}.  Covid-19 lockdown caused both a reduction in traffic as well as less burning of crops, both of which dramatically reduced air pollution in New Delhi. Universal Causality (UC) formulates the problem of causal inference in an abstract category ${\cal C}$ of causal diagrams, of which DAGs are one example, where objects can represent  arbitrary entities that interact in diverse ways (see the range of possibilities in Table~\ref{causalcats}).

One fundamental theoretical insight provided by category theory is the {\em universality of diagrams} in causal inference. More formally, any causal inference can be defined as an object in the contravariant functor category {\bf Set}$^{{\cal C}^{op}}$ of presheaves, which is representable as the co-limit of a small indexing diagram that serves as its universal element. The notion of a diagram is more abstract than previous diagrammatic representations in causal inference (see Figure~\ref{ucdiagram}). For example, in causal inference using DAGs, a collider node $B$ defines the structure $A \rightarrow B \leftarrow C$. In our framework, an abstract causal diagram functorially maps from an indexing category of diagrams, such as $\bullet \rightarrow \bullet \leftarrow \bullet$ into the actual causal diagram $A \rightarrow B \leftarrow C$. Abstract causal diagrams themselves form a category of functors.  The fundamental notion of a {\em limit} or {\em co-limit} of a causal diagram are abstract versions of more specialized {\em universal constructions}, such as ``pullbacks", ``pushouts", ``co-equalizers" and ``equalizers", commonly used in category theory, but hitherto not been studied in causal inference. For example, the limit of the abstract causal diagram mapping $\bullet \rightarrow \bullet \leftarrow \bullet$ into the actual causal model $A \rightarrow B \leftarrow C$ is defined as a variable, say $Z$, that acts as a ``common" cause between $A$ and $C$ in a ``universal" manner, in that any other common cause of $A$ and $C$ must factor through $Z$. In fact, a rich repertoire of other construction tools have been developed in category theory, including Galois extensions, adjunctions, decorated cospans, operads, and props, all of which enable building richer causal representations than have been previously explored in the causal inference literature. 

\begin{figure}[t]
\begin{minipage}{0.95\textwidth}
\includegraphics[scale=0.4]{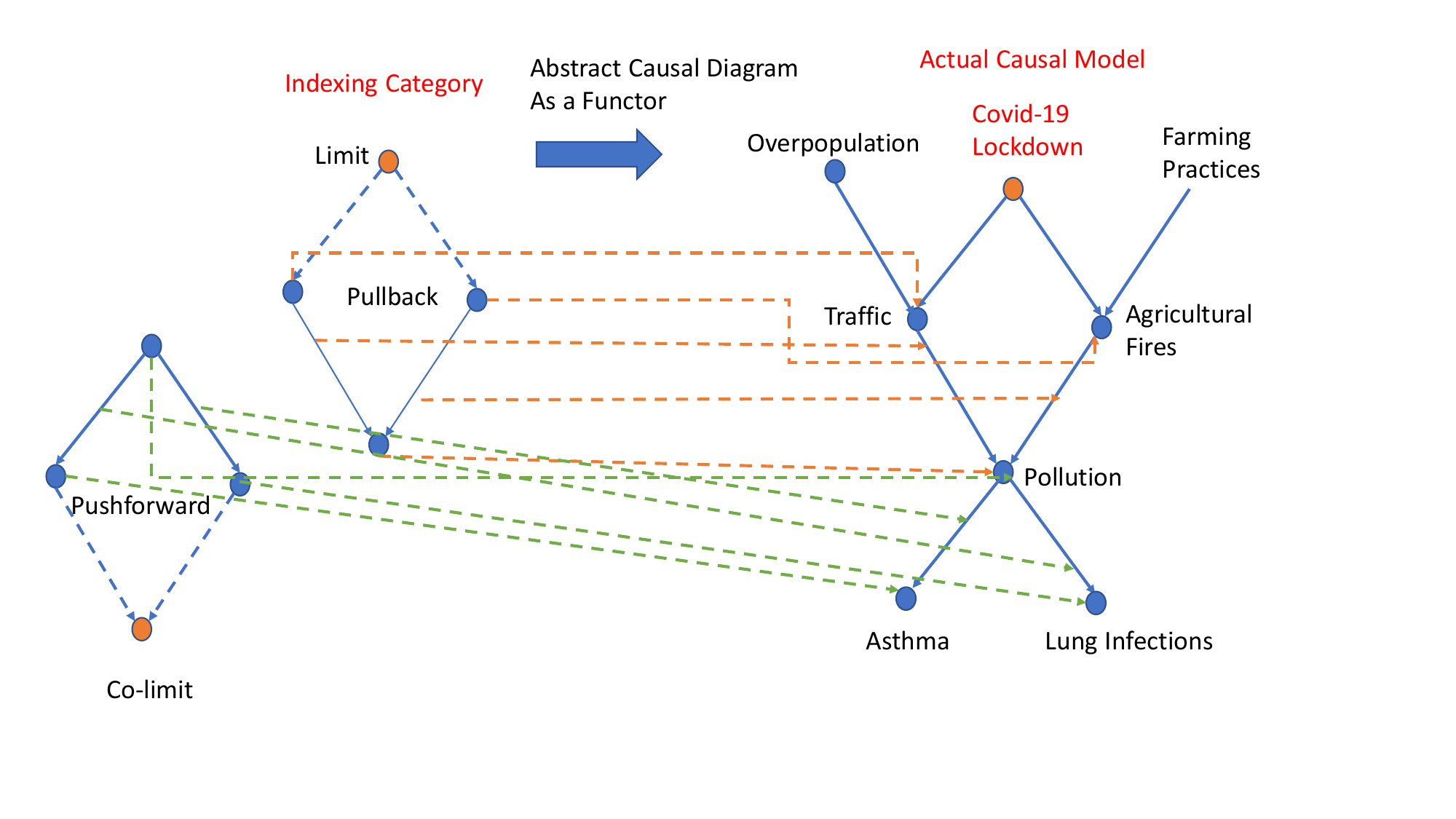}
\end{minipage} \vskip -0.3in
\caption{The notion of diagram in category theory is more abstract than typical diagrammatic representations in causal inference. A diagram in category theory is actually a functor mapping between an indexing category of diagrams to the actual causal model. Thus, the abstract diagram on the left maps functorially into the actual causal model by mapping each object $\bullet$ into a causal variable in the model, and each morphism into an edge in the DAG. .  Diagrams rely on universal constructions, such as pullbacks, pushouts, co-kernels and kernels, which are all special cases of {\em limit} or {\em co-limits}. For example, the limit of the abstract causal diagram in this example is a common cause $Z$ of {\bf Traffic} and {\bf Agricultural Fires}, such that $Z$ satisfies a universal property, namely every other common cause $W$ must factor through it. Although {\bf Covid-19} is certainly a common cause of {\bf Traffic} and {\bf Agricultural Fires} both being significantly lower than normal, in general, it might not be the limit, as it is possible there might be many other common causes (e.g., a weather event). Dually, the co-limit of {\bf Asthma} and {\bf Lung infection} is some common effect $E$ such that every other effect of these conditions must factor through $E$. Co-limits and limits generalize notions such as disjoint unions and products in sets, and joins and meets in partial orders.}
\label{ucdiagram}
\end{figure}

We want to briefly discuss a central result that we call the {\em Causal Reproducing Property}, as it is analogous to a key result in machine learning in the literature on kernel methods.  Reproducing Kernel Hilbert Spaces (RKHS's)  transformed the study of machine learning, precisely because they are the unique subcategory in the category of all Hilbert spaces that have representers of evaluation defined by a kernel matrix $K(x,y)$ \cite{kernelbook}. The reproducing property in an RKHS is defined as $\langle K(x, -), K(-, y) \rangle = K(x,y)$. An analogous but far more general reproducing property holds in the categorical causality framework, based on the Yoneda Lemma. 

\begin{theorem}\cite{DBLP:journals/entropy/Mahadevan23}
The {\bf causal reproducing property} (CRP) states that the set of all causal influences between any two objects $X$ and $Y$ can be defined from its presheaf functor objects, namely {\bf Hom}$_{\cal C}(X,Y) \simeq$ {\bf Nat}({\bf Hom}$_{\cal C}(-, X)$,{\bf Hom}$_{\cal C}(-, Y))$.  
\end{theorem}

{\bf Proof:} The proof of this theorem is a direct consequence of the Yoneda Lemma, which states that for every presheaf functor object $F$ in  $\hat{{\cal C}}$ of a category ${\cal C}$, {\bf Nat}({\bf Hom}$_{\cal C}(-, X), F) \simeq F X$. That is, elements of the set $F X$ are in $1-1$ bijections with natural transformations from the presheaf {\bf Hom}$_{\cal C}(-, X)$ to $F$. For the special case where the functor object $F = $ {\bf Hom}$_{\cal C}(-, Y)$, we get the result immediately that  {\bf Hom}$_{\cal C}(X,Y) \simeq$ {\bf Nat}({\bf Hom}$_{\cal C}(-, X)$,{\bf Hom}$_{\cal C}(-, Y))$. $\bullet$

The significance of the Causal Reproducing Property is that presheaves act as ``representers" of causal information, precisely analogous to how kernel matrices act as representers in an RKHS.  Remarkably, we show below that an analogous result holds in a completely different setting, that of generalized metric spaces that is again of central importance in machine learning. Thus, the framework for causal imitation games can be defined in terms of the causal reproducing property. 

\begin{definition}
    Two objects $c$ and $d$ in a category of causal models ${\cal C}$ are isomorphic if there is a natural transformation $F: {\cal C}(-, c) \rightarrow {\cal C}(-,d)$ and a natural transformation $G: {\cal C}(-, d) \rightarrow {\cal C}(-, c)$ such that $G \circ F$ defines a natural isomorphism between the presheaves of $c$ and $d$ that act as representers of causal information. 
\end{definition}

\subsection{Dynamic UIGs using Experimentation: Learning Diversity Automata}

In this section, we give another example of solving dynamic UIGs using active experimentation, which illustrates an interesting {\em diversity} representation of automata to model {\em environments} that are often studied in AI, including games such as Rubik's cube. The goal is to understand how a robot can construct a model of an unknown environment whose states it cannot completely observe by conducting experiments with it. This approach of using tests to reveal the underlying state of an environment relates to work in coalgebras on using traces to define the underlying semantics, which provides an interesting example to explain the coalgebraic approach. We call this diversity coalgebras to reflect the marriage of two viewpoints. 

This approach of using diversity representations for modeling finite state automata was originally proposed in the category theory literature by Arbib, Bainbridge and others, but popularized in the ML literature by Rivest and Schapire. The use of diversity representations provides an interesting example of universal coalgebras. Rivest and Schapire specify environments as (Moore) finite state automata ${\cal E} = (Q, B, P, q_0, \delta, \gamma)$ is defined by them as a finite nonempty set $Q$ of states, a finite nonempty set of input symbols or ``actions" $B$, a finite nonempty set of {\em predicate symbols} or ``sensations" $P$, an {\em initial state} $q_0 \in Q$, a state {\em transition function} $\delta: Q \times B \rightarrow Q$, and finally an output function $\gamma: Q \times P \rightarrow \{ {\bf True}, {\bf False} \}$. What is interesting about  the {\em diversity representation} is that it represents an alternate way to specify environments, not in terms of specifying states (which might be unobservable), but in terms of a collection of {\em tests}. The standard notion of accepting states in a finite state machine can be viewed as a special case of a single accept predicate. Rivest and Schapire show  a diversity representation can be constructed using an {\em update graph} whose vertices are {\em tests}, defined as sequences of input symbols followed by one or more predicate symbols that determine if a particular test succeeds in that state. We explore the problem of constructing the diversity automata from experiments in a later Section. Right now, we want to focus on defining a {\em diversity coalgebra} to show how to design more compact representations of finite state machines than the usual state-based framework. 

Let us define by $A = B^*$ the set of all strings of input symbols, and extend the transition function suitably as $\delta: A \times Q \rightarrow Q$. A {\em test} is defined as an element of the concatenated set $AP$, namely an action $a \in A$ followed by a predicate $p \in P$. The set of all possible tests is denoted by $T$. A test $t = a p$ {\em succeeds} in a state $q$ if $q t = q (a p) = (q a) p = {\bf True}$. Otherwise, we declare the test as {\em failed}. An environment is {\em reduced} if every pair of states can be distinguished by executing some test: 

\[ (\forall q \in Q) (\forall r \in Q) (q \neq r) \Rightarrow (\exists t \in T) qt \neq qr \]

The goal is to be able to predict the results of doing any test, by which a robot can be assumed to have a perfect model of its environment even if it is unable to perceive the true underlying state of the environment. We will see interesting examples of environments, such as the Rubik's cube, where the diversity representation is far more compact than the traditional state-based approach. A key notion is the use of {\em equivalences} among tests. Two tests $t_1$ and $t_2$ are considered {\em equivalent}, or $t_1 \equiv t_2$ if 

\[ (\forall q \in Q) (q t_1 = q t_2) \]

This equivalence relation partitions the state space into equivalence classes, so we can denote the equivalence class of a test $t$ by $[t]$. The {\em diversity} of an environment, defined as $D({\cal E})$ is defined as the number of equivalence classes of tests. 

\[ D({\cal E}) = \{ [t] | t \in T \} \]

A simple result to prove is that the diversity can be exponentially smaller than the number of states, but it can also be substantially larger. Thus, whether the diversity approach is beneficial depends entirely on a specific environment. It turns out that a large number of natural environments seem to be highly compact in terms of their diversity. Figure~\ref{regenv} illustrates a simple $n$-bit register world environment, where the state space grows exponentially in $n$, but the diversity representation can be exponentially smaller. 

\begin{theorem}
    For any reduced finite-state automaton ${\cal E} = (Q, B, P, q_0, \delta, \gamma)$, the diversity measure is in the following range: 

    \[ \log_2 (|Q|) \leq D({\cal E}) \leq 2^{|Q|} \]
\end{theorem}

The proof is straightforward: a state is uniquely identified by the set of equivalence classes of tests that are true at that state, as we assume the automaton is reduced, and hence any two states can be distinguished by some test. The lower bound simply states that the number of states $|Q| \leq 2^{\log_2 D({\cal E})}$. The upper bound holds because the equivalence class that a test defines is defined by the set of possible states where that equivalence class succeeds. 

\begin{figure}[t]
\centering
\includegraphics[scale=.45]{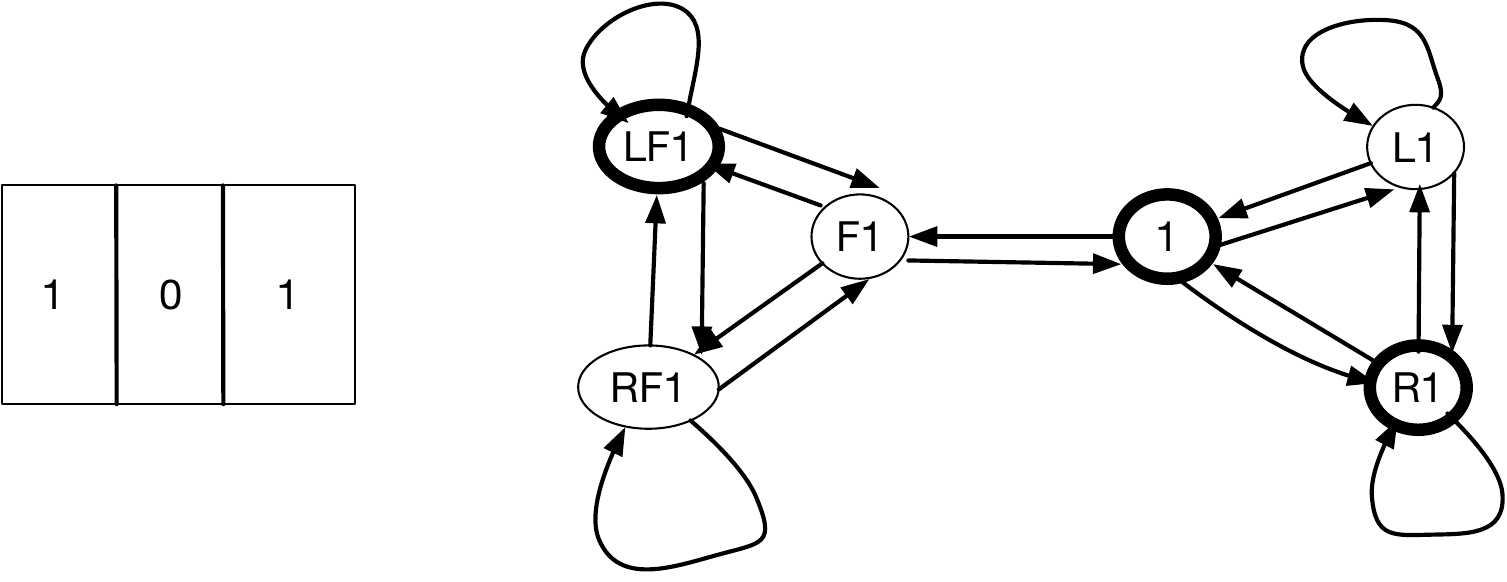}
\caption{The state space for an $n$-bit register environment is $2^n$, but its diversity size is only $2n$ because there is one test for checking whether bit $i$ is a $1$ and another test for checking if bit $i$ is a $0$. In the $3$-bit register environment shown here, $1$ refers to a test that returns {\bf True} if the leftmost bit is a $1$. $L$, $R$, and $F$ refer to input actions that rotate left or rotate right the entire register with wraparound, or flip the leftmost bit, respectively. The state $101$ is represented by the tests $LF1$, $1$, and $R1$ all returning {\bf True}. }
\label{regenv} 
\end{figure}

As shown in Figure~\ref{adjfsa}, we can define two categories of representations of environments, one based on a traditional state-based finite automata and the other based on the Rivest and Schapire diversity representation. We can define an adjoint functor that maps between the two categories, where the left adjoint maps a given state based environment (such as the one illustrated in Figure~\ref{regenv}) into its equivalent diversity-based representation, and the right adjoint maps a diversity-based automata into its state-based representation. Note that these functors must act functorially, which means that they map not only objects (i.e., state-based or diversity-based environment automata), but also the corresponding morphisms. To define these functors more formally, let us make more precise the two categories. 

\begin{figure}[t]
\centering
\includegraphics[scale=.4]{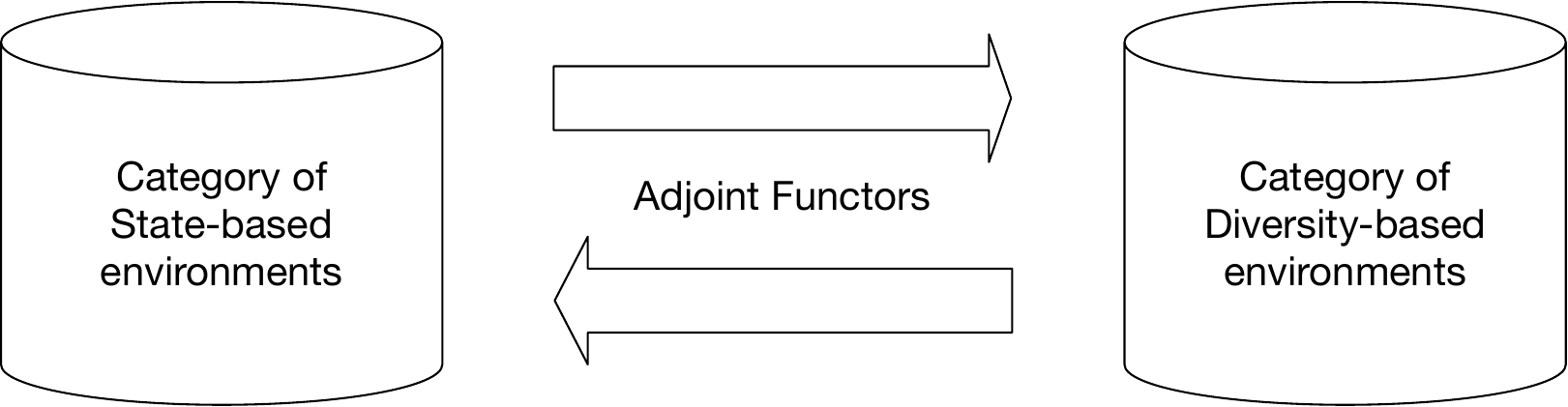}
\caption{Adjoint functors map between the categories of state-based finite state environments and diversity-based environments. }
\label{adjfsa} 
\end{figure}

\begin{definition}
    The category {\cal C}$_{S}$ of {\em state-based}  environments is defined by as a collection of objects, where each object is a state-based automata ${\cal E} = (Q, B, P, q_0, \delta, \gamma)$ as defined previously. Given a pair of environments ${\cal E}_1, {\cal E}_2$ in ${\cal C}_S$, each morphism in ${\cal C}_S({\cal E}_1, {\cal E}_2)$ is defined by a set of functions $f_Q: Q_1 \rightarrow Q_2$ that maps each state $q$ of ${\cal E}_1$ into the corresponding state $f_Q(q)$ of ${\cal E}_2$, $f_B: B_1 \rightarrow B_2$ maps each input action symbol $a$  of ${\cal E}_1$ into the corresponding action symbol $f_B(a)$ of ${\cal E}_2$, $f_P: P_1 \rightarrow P_2$ maps each predicate symbol $p$ of ${\cal E}_1$ into the corresponding symbol $f_P(p)$ of ${\cal E}_2$, $f_\delta: \delta_1 \rightarrow \delta_2$ that maps the transition function $\delta_1$ of ${\cal E}_1$ into the corresponding transition function $\delta_2$ of ${\cal E}_2$ such that $\delta_2(f_Q(q),f_B(a)) = f_Q(\delta_1(q,a))$, and finally $f_\gamma: \gamma_1 \rightarrow \gamma_2$ maps the output function of ${\cal E}_1$ into the corresponding output function of ${\cal E}_2$ such that $\gamma_2(f_Q(q),f_P(p)) = f_Q(\gamma(q,p))$. 

\end{definition}

We will define below more precisely the notion of {\em bisimulation} that specifies when a state-based environment ${\cal E}_2$ can faithfully simulate another state-based environment ${\cal E}_1$ (and equivalently when one diversity-based environment ${\cal D}_1$ can be simulated by another environment ${\cal D}_2$). Here, we want to briefly describe how the adjoint functors can be implemented as well, so that a diversity-based environment ${\cal D}$ can simulate a state-based environment. Note from the example of the $n$-bit register in Figure~\ref{regenv}, the functor mapping a state-based environment to a diversity based environment might be ``lossy" depending on whether an adequate set of tests are used. For example, if all the diversity-based environment cares about is whether the leftmost bit is a $1$ or $0$, it might be sufficient to use tests that are adequate to predict that, but not sufficient to predict the entire state-based automaton's transitions. However, given sufficient number of tests, which as was shown above might be exponentially larger (or smaller) than the number of states, the adjoint functors can be made faithful in both directions.   Given a state-based representation of an environment ${\cal E}_1$, the following result is easy to show. 

\begin{theorem}
    To simulate a state-based environment ${\cal E}$, which means to be able to predict the results of running every test $t \in T$ in ${\cal E}$, it is sufficient to build a diversity-based environment ${\cal D}$, and for each equivalence class of tests $[t]$ represented as a vertex in the diversity-based environment ${\cal D}$, the value $q t$ at the current state $q$ of ${\cal E}$. 
\end{theorem} 

To prove this, note that if the state-based automaton moves from state $q$ to the state $qb$ for some input $b$, the diversity-based automaton needs to compute $(q b)t = q (bt)$ for each equivalence class of tests $[t]$. The test $bt$ belongs to the unique equivalence class $[s]$ for which an edge labeled $b$ is directed from $[s]$ to $[t]$ in the diversity-based automata representation.  We can use this insight to define the category of diversity-based environments. 

\begin{definition}
    The category {\cal C}$_{D}$ of {\em diversity-based}  environments is defined by as a collection of objects, where each object is a diversity-based automata ${\cal D} = (T, B, \omega, \alpha)$, where $T$ is a set of {\em tests} comprised of an action sequence $a \in A$ followed by a predicate symbol $p \in P$, $B$ is a set of input ``action" symbols, $\omega: T \times B \rightarrow T$ is the test transition function, and $\alpha: T \rightarrow \{ {\bf True}, {\bf False} \}$ maps each equivalence class of tests into its value ${\bf True}$ or ${\bf False}$ in the current state. Given a pair of environments ${\cal D}_1, {\cal D}_2$ in ${\cal C}_D$, each morphism in ${\cal C}_D({\cal D}_1, {\cal D}_2)$ is defined by a set of functions $f_T: T_1 \rightarrow T_2$ that maps each test $t$ of ${\cal D}_1$ into the corresponding test $f_T(t)$ of ${\cal D}_2$, $f_B: B_1 \rightarrow B_2$ maps each input action symbol $a$  of ${\cal D}_1$ into the corresponding action symbol $f_B(a)$ of ${\cal D}_2$, $f_\omega: \omega_1 \rightarrow \omega_2$ maps the test transition function of ${\cal D}_1$ into the corresponding test transition function ${\cal D}_2$ such that $\omega_2(f_T(t), f_B(a)) = f_T(\omega_1(t,a))$, and finally $f_\alpha: \alpha_1 \rightarrow \alpha_2$ maps the output function for ${\cal D}_1$ into that for ${\cal D}_2$ such that $\alpha_2(f_T(t)) = f_\alpha(\alpha_1(t))$
\end{definition}

Finally, we can define the category of coalgebras for diversity-based environments using the above definitions. 

\begin{definition}
The category ${\cal C}_{DA}$ of coalgebras of diversity-based environments is defined as a collection of objects, where each object is defined as a coalgebra $({\cal D}, \beta_{DA})$ where ${\cal D} = (T, B, \omega, \alpha)$ is a diversity-based environment,  $F_{DA}$ is an endofunctor on diversity-based environments defined by $F_{DA}: T \rightarrow \{{\bf True}, {\bf False} \}^T \times T^B$ that simulates a state-based environment on any input symbol,  and each morphism in ${\cal C}_{DA}$ is defined as $F_{DA}: ({\cal D}_1, \beta^1_{DA}) \rightarrow ({\cal D}_2, \beta^2_{DA})$ such that the following diagram commutes for any morphism $f_{DA}$ over diversity-based environments that maps the environment ${\cal D}_1$ into ${\cal D}_2$:

\begin{center}
\begin{tikzcd}
  {\cal D}_1 \arrow[r, "f"] \arrow[d, "F_{DA}"]
    & {\cal D}_2 \arrow[d, "F_{DA}" ] \\
  F_{DA}({\cal D}_1) \arrow[r,  "F_{DA}(f)"]
& F_{DA}({\cal D}_2)
\end{tikzcd}
\end{center}
\end{definition}

We want to introduce a {\em topological} perspective here, which will become important in later Sections. Essentially, modeling an environment in terms of its diversity is tantamount to defining a topology on the state space of a finite-state machine, where each (open) set represents the set of states where a particular equivalence class of tests succeeds. An automaton is reduced if it defines a {\em Hausdorff} topology, that is, for any pair of elements $s$ and $r$, there is an open set that contains one but not the other. By interpreting the diversity representation as defining a Hausdorff topology, we can bring to bear sophisticated methods to analyze finite state automata using some interesting techniques from category theory combined with topology.

\subsection{Dynamic Games over Universal Coalgebras: from induction to coinduction}

Intuitively, algebras represent generic ways to combine entities to form new entities, whereas coalgebras define systems that capture (hidden) state, and enable representing circularly defined infinite data streams. Algebras give rise to {\em constructors}, whereas coalgebras give rise to {\em destructors} and {\em accessors} to obtain partial views of hidden state.  Table~\ref{univ-algebra-and-coalgebra-table} summarizes some of the key differences between universal algebras and universal coalgebras, which we will now describe in more detail below. Every notion in algebras, such as congruence, has a corresponding notion in coalgebras, in particular {\em bisimulation} will prove to be a central concept in the use of coalgebras to model dynamical systems in AI and ML. A key reason for using coalgebras in AI and ML is that they are extremely useful in modeling infinite data streams, as well as dynamical systems that capture a wide range of practical real-world problems that involve hidden state. 

\begin{table}[!t]
\caption{Universal Algebras vs. Universal Coalgebras}
\label{univ-algebra-and-coalgebra-table}       
%
%
\begin{center}
\begin{tabular}{ | l | r|}  \hline 
 {\bf Algebra} & {\bf Coalgebra} \\ \hline 
 $\Sigma$-algebra & Coalgebra = system \\  \hline 
Algebra Homomorphism \ \ \ & \ \ \ System homomorphism \\ \hline 
Congruence & Bisimulation \\ \hline 
Subalgebra & Subsystem \\ \hline 
Initial algebra & Initial system \\ \hline 
Final algebra & Final system \\ \hline 
Minimal algebra & Minimal subsystem \\  \hline 
\end{tabular}
\end{center}
\end{table}

Most of us are familiar with the concept of algebra from our high school education, but for our purposes, we need a more abstract category-theoretic description of universal algebras. Intuitively, algebras specify how entities are combined to form new entities, such as the famous equation discovered by Albert Einstein linking energy $e$, the mass of an object $m$, and the speed of light $c^2$: 

\[ e = m c^2\]

In an algebra, we combine quantities by multiplication, exponentiation, addition and so on. More abstractly, we can define some space $X$ of entities, which in general could correspond to a given domain of discourse. In relativity theory, $X$ might represent $e, m$ and $c$, as well as any other variables that are necessary. In a finite state machine, $X$ may represent the set of states. An algebra can be viewed as a mapping $F(X) \rightarrow X$, which should be interpreted as meaning that the function $F$ combines the objects in the space $X$ to produce another object $F(X)$ in $X$. In contrast, coalgebras are specified by the mapping $X \rightarrow F(X)$, where the major difference is that the right-hand term $F(X)$ may indeed contain $X$ again, which leads to a circular definition. This circularity is in fact intentional, as it is the route towards modeling infinite data streams. 

In other words, algebras can be viewed as sets equipped with a set of basic operations. For example, in the classic model of inductive inference studied by Gold, Solomonofff and others, ML is theoretically modeled as a problem in which the learner is given a sequence of examples of a function defined over the natural numbers $\mathbb{N}$. We can define the natural numbers as the algebra 

\[ \{\mathbb{N}, 0, {\bf succ} \}, \ \ \, \mathbb{N} = \{0, 1, 2, \ldots \}, \ \ \, 0 \in \mathbb{N}, \ \ \ {\bf succ}: \mathbb{N} \rightarrow \mathbb{N} \]

defined as the set $\mathbb{N}$ of the natural numbers, the constant $0$, and the successor function {\bf succ}$(n) = n+1, \ n \geq 0$. More succinctly, we can define the algebra of the natural numbers as the pair of functions:

\[ [{\bf zero}, {\bf succ}]: ({\bf 1} + \mathbb{N}) \rightarrow \mathbb{N} \]

where {\bf 1} $= \{ * \}$ is the single point set, {\bf 1}$+ \mathbb{N}$ is to be interpreted as a {coproduct}, or disjoint set union, of {\bf 1} and $\mathbb{N}$, and the two functions that are packaged into one unit are each defined as follows:

\begin{eqnarray}
    {\bf zero}: {\bf 1} &\rightarrow& \mathbb{N}, \ \ \ {\bf succ}: \mathbb{N} \rightarrow \mathbb{N} \\
    {\bf zero}(*) &=& 0, \ \ \ {\bf succ}(n) = n+1 
\end{eqnarray}

We now give a more abstract category-theoretic definition of an algebra defined by some functor $F$, which will prove invaluable in this paper for modeling coinductive inference over a diverse range of settings in ML. 

\begin{definition}
Let $F: {\cal C} \rightarrow {\cal C}$ be an endofunctor (mapping a category to itself). An {\bf $F$-algebra} is defined as a pair $(A, \alpha)$ consisting of an object $A$ and an arrow $\alpha: F(A) \rightarrow A$. It is common to denote $A$ as the {\em carrier} of the algebra, $F$ as the {\em type} and $\alpha$ as the {\em structure map} of the algebra $(A, \alpha)$. 
\end{definition}

In any category ${\cal C}$, an {\em initial object} $c$ is one such that for any other object $d$, there is a unique arrow $f: c \rightarrow d$. For example, in the category {\bf Set}, the empty set $\emptyset$ is the initial object. Similarly, in any category ${\cal C}$, the {\em final object} $d$ is one such that for any object $c$, there is a unique arrow $f: c \rightarrow d$. Again, for the category {\bf Set}, the final object is defined by the single point set $d = \{ * \}$, which we have simply denoted by the symbol "*". We can also define a category of $F$-algebras by defining the {\em homomorphism} or arrow between two algebras: 

\begin{definition}
Let $F: {\cal C} \rightarrow {\cal C}$ be an endofunctor. A {\em homomorphism} of $F$-algebras $(A, \alpha)$ and $(B, \beta)$ is an arrow $f: A \rightarrow B$ in the category ${\cal C}$ such that the following diagram commutes:
\begin{center}
\begin{tikzcd}
  F(A) \arrow[r, "F(f)"] \arrow[d, "\alpha"]
    & F(B) \arrow[d, "\beta" ] \\
  A \arrow[r,  "f"]
& B
\end{tikzcd}
\end{center}
\end{definition}

Notice that in the above definition, it was crucial for $F$ to be defined as a functor, since it maps not only the objects $A$ and $B$ to $F(A)$ and $F(B)$, but it also maps the arrow $f$ to $F(f)$. A major theme of this paper is to view AI and ML systems {\em functorially} by defining functors that act in this manner. The use of commutative diagrams such as the one above will also occur repeatedly in this paper, and is often referred to as ``diagram chasing" in the category theory literature. 

We now turn to describing coalgebras, a much less familiar construct that will play a central role in the proposed ML framework of coinductive inference. Coalgebras capture hidden state, and enable modeling infinite data streams. Recall that in the previous Section, we explored non-well-founded sets, such as the set $\Omega = \{ \Omega \}$, which gives rise to a circularly defined object. As another example, consider the infinite data stream comprised of a sequence of objects, indexed by the natural numbers: 

\[ X = (X_0, X_1, \ldots, X_n, \ldots ) \]

We can define this infinite data stream as a coalgebra, comprised of an accessor function {\bf head} that returns the head of the list, and a destructor function that gives the {\bf tail}  of the list, as we will show in detail below. 

To take another example, consider a deterministic finite state machine model defined as the tuple $M = (X, A, \delta)$, where $X$ is the set of possible states that the machine might be in, $A$ is a set of input symbols that cause the machine to transition from one state to another, and $\delta: X \times A \rightarrow X$ specifies the transition function. To give a coalgebraic definition of a finite state machine, we note that we can define a functor $F: X \rightarrow {\cal P}(A \times X)$ that maps any given state $x \in X$ to the subset of possible future states $y$ that the machine might transition to for any given input symbol $a \in A$. 

We can now formally define $F$-coalgebras analogous to the definition of $F$-algebras given above. 

\begin{definition}
Let $F: {\cal C} \rightarrow {\cal C}$ be an endofunctor on the category ${\cal C}$. An {\bf $F$-coalgebra} is defined as a pair $(A, \alpha)$ comprised of an object $A$ and an arrow $\alpha: A \rightarrow F(A)$. 
\end{definition}

The fundamental difference between an algebra and a coalgebra is that the structure map is reversed! This might seem to be a minor distinction, but it makes a tremendous difference in the power of coalgebras to model state and capture dynamical systems. Let us use this definition to capture infinite data streams, as follows. 

\[ {\bf Str}: {\bf Set} \rightarrow {\bf Set}, \ \ \ \ \ \ {\bf Str}(X) = \mathbb{N} \times X\]

Here, {\bf Str} is defined as a functor on the category {\bf Set}, which generates a sequence of elements. Let $N^\omega$ denote the set of all infinite data streams comprised of natural numbers:

\[ N^\omega = \{ \sigma | \sigma: \mathbb{N} \rightarrow \mathbb{N} \} \]

To define the accessor function {\bf head} and destructor function {\bf tail} alluded to above, we proceed as follows: 

\begin{eqnarray}
{\bf head}&:& \mathbb{N}^\omega \rightarrow \mathbb{N} \ \ \  \ \ \ \ {\bf tail}: \mathbb{N}^\omega \rightarrow\mathbb{N}^\omega \\
{\bf head}(\sigma) &=& \sigma(0) \ \ \ \ \ \ \ {\bf tail}(\sigma) = (\sigma(1), \sigma(2), \ldots )
\end{eqnarray}

Another standard example that is often used to illustrate coalgebras, and provides a foundation for many AI and  ML applications, is that of a {\em labelled transition system}. 

\begin{definition}
    A {\bf labelled transition system} (LTS) $(S, \rightarrow_S, A)$ is defined by a set $S$ of states, a transition relation $\rightarrow_S \subseteq S \times A \times S$, and a set $A$ of labels (or equivalently, ``inputs" or ``actions"). We can define the transition from state $s$ to $s'$ under input $a$ by the transition diagram $s \xrightarrow[]{a} s'$, which is equivalent to writing $\langle s, a,  s' \rangle \in \rightarrow_S$. The ${\cal F}$-coalgebra for an LTS is defined by the functor 

    \[ {\cal F}(X) = {\cal P}(A \times X) = \{V | V \subseteq A \times X\} \]
\end{definition}

{\em Kripke} models are widely used in the study of formal models of knowledge \cite{fagin}. We can define a Kripke model as a coalgebra that defines the ``transition dynamics" of a Kripke model under the accessibility relationships $K_i$ for each agent. 

\begin{definition}
    Given a Kripke model ${\cal M} = (S, \pi, K_1, \ldots, K_n)$ defining a Kripke model, the corresponding {\bf Kripke coalgebra} can be defined by the endofunctor 

    \[ {\cal F}_{{\cal M}}(S) = {\cal P}(\{1, \ldots, N \} \times S )\]

    where each labelled edge in the Kripke model denotes the accessibility relation between possible worlds in $S$ for a particular agent $i \in \{1, \ldots, n \}$.
\end{definition}

Just as before, we can also define a category of $F$-coalgebras over any category ${\cal C}$, where each object is a coalgebra, and the morphism between two coalgebras is defined as follows, where $f: A \rightarrow B$ is any morphism in the category ${\cal C}$. 

\begin{definition}
Let $F: {\cal C} \rightarrow {\cal C}$ be an endofunctor. A {\em homomorphism} of $F$-coalgebras $(A, \alpha)$ and $(B, \beta)$ is an arrow $f: A \rightarrow B$ in the category ${\cal C}$ such that the following diagram commutes:

\begin{center}
\begin{tikzcd}
  A \arrow[r, "f"] \arrow[d, "\alpha"]
    & B \arrow[d, "\beta" ] \\
  F(A) \arrow[r,  "F(f)"]
& F(B)
\end{tikzcd}
\end{center}
\end{definition}

For example, consider two labelled transition systems $(S, A, \rightarrow_S)$ and $(T, A, \rightarrow_T)$ over the same input set $A$, which are defined by the coalgebras $(S, \alpha_S)$ and $(T, \alpha_T)$, respectively. An $F$-homomorphism $f: (S, \alpha_S) \rightarrow (T, \alpha_T)$ is a function $f: S \rightarrow T$ such that $F(f) \circ \alpha_S  = \alpha_T \circ f$. Intuitively, the meaning of a homomorphism between two labeled transition systems means that: 

\begin{itemize}
    \item For all $s' \in S$, for any transition $s \xrightarrow[]{a}_S s'$ in the first system $(S, \alpha_S)$, there must be a corresponding transition in the second system $f(s) \xrightarrow[]{a}_T f(s;)$ in the second system. 

    \item Conversely, for all $t \in T$, for any transition $t \xrightarrow[]{a}_T t'$ in the second system, there exists two states $s, s' \in S$ such that $f(s) = t, f(t) = t'$ such that $s \xrightarrow[]{a}_S s'$ in the first system. 
\end{itemize}

If we have an $F$-homomorphism $f: S \rightarrow T$ with an inverse $f^{-1}: T \rightarrow S$ that is also a $F$-homomorphism, then the two systems $S \simeq T$ are isomorphic. If the mapping $f$ is {\em injective}, we have a  {\em monomorphism}. Finally, if the mapping $f$ is a surjection, we have an {\em epimorphism}.

The analog of congruence in universal algebras is {\em bisimulation} in universal coalgebras. Intuitively, bisimulation allows us to construct a more ``abstract" representation of a dynamical system that is still faithful to the original system. We will explore many applications of the concept of bisimulation to AI and ML systems in this paper. We introduce the concept in its general setting first, and then in the next section, we will delve into concrete examples of bisimulations. 

\begin{definition}
Let $(S, \alpha_S)$ and $(T, \alpha_T)$ be two systems specified as coalgebras acting on the same category ${\cal C}$. Formally, a $F$-{\bf bisimulation} for coalgebras defined on a set-valued functor $F: {\bf Set} \rightarrow {\bf Set}$ is a relation $R \subset S \times T$ of the Cartesian product of $S$ and $T$ is a mapping $\alpha_R: R \rightarrow F(R)$ such that the projections of $R$ to $S$ and $T$ form valid $F$-homomorphisms.

\begin{center}
\begin{tikzcd}
  R \arrow[r, "\pi_1"] \arrow[d, "\alpha_R"]
    & S \arrow[d, "\alpha_S" ] \\
  F(R) \arrow[r,  "F(\pi_1)"]
& F(S)
\end{tikzcd}
\end{center}

\begin{center}
\begin{tikzcd}
  R \arrow[r, "\pi_2"] \arrow[d, "\alpha_R"]
    & T \arrow[d, "\alpha_T" ] \\
  F(R) \arrow[r,  "F(\pi_2)"]
& F(T)
\end{tikzcd}
\end{center}

Here, $\pi_1$ and $\pi_2$ are projections of the relation $R$ onto $S$ and $T$, respectively. Note the relationships in the two commutative diagrams should hold simultaneously, so that we get 

\begin{eqnarray*}
    F(\pi_1) \circ \alpha_R &=& \alpha_S \circ \pi_1 \\
    F(\pi_2) \circ \alpha_R &=& \alpha_T \circ \pi_2 
\end{eqnarray*}

Intuitively, these properties imply that we can ``run" the joint system $R$ for one step, and then project onto the component systems, which gives us the same effect as if we first project the joint system onto each component system, and then run the component systems. More concretely, for two labeled transition systems that were considered above as an example of an $F$-homomorphism, an $F$-bisimulation between $(S, \alpha_S)$ and $(T, \alpha_T)$ means that  there exists a relation $R \subset S \times T$ that satisfies for all $\langle s, t \rangle \in R$

\begin{itemize}
    \item For all $s' \in S$, for any transition $s \xrightarrow[]{a}_S s'$ in the first system $(S, \alpha_S)$, there must be a corresponding transition in the second system $f(s) \xrightarrow[]{a}_T f(s;)$ in the second system, so that $\langle s', t' \rangle \in R$

    \item Conversely, for all $t \in T$, for any transition $t \xrightarrow[]{a}_T t'$ in the second system, there exists two states $s, s' \in S$ such that $f(s) = t, f(t) = t'$ such that $s \xrightarrow[]{a}_S s'$ in the first system, and $\langle s', t' \rangle \in R$.
\end{itemize}

\end{definition}

There are a number of basic properties about bisimulations, which we will not prove, but are useful to summarize here: 

\begin{itemize}
    \item If $(R, \alpha_R)$ is a bisimulation between systems $S$ and $T$, the inverse $R^{-1}$ of $R$ is a bisimulation between systems $T$ and $S$. 

    \item Two homomorphisms $f: T \rightarrow S$ and $g:T \rightarrow U$ with a common domain $T$ define a {\em span}. The {\em image} of the span $\langle f, g \rangle(T) = \{ \langle f(t), g(t) \rangle | t \in T \}$ of $f$ and $g$ is also a bisimulation between $S$ and $U$. 

    \item The composition $R \circ Q$ of two bisimulations $R \subseteq S \times T$ and $Q \subseteq T \times U$ is a bisimulation between $S$ and $U$. 

    \item The union $\cup_k R_k$ of a family of bisimulations between $S$ and $T$ is also a bisimulation. 

    \item The set of all bisimulations between systems $S$ and $T$ is a complete lattice, with least upper bounds and greatest lower bounds given by: 

    \[ \bigvee_k R_k = \bigcup_k R_k \]

    \[ \bigwedge_K R_k = \bigcup \{ R | R \ \mbox{is a bisimulation between} S \ \mbox{and} \ T \ \mbox{and} \ R \subseteq \cap_k R_k \} \]

    \item The kernel $K(f) = \{ \langle s, s' \rangle | f(s) = f(s') \}$ of a homomorphism $f: S \rightarrow T$ is a bisimulation equivalence.

    \end{itemize}

    As we will see in a subsequent Section, we can design a generic {\em coinductive inference} framework for systems by exploiting these properties of bisimulations. This framework will give us a broad paradigm for designing AI and ML systems that is analogous to inductive inference, which we will discuss in the next Section.

 We introduced the use of universal constructions in category theory in the previous sections. We will use these constructions now to build complex universal coalgebra systems out of simpler systems, focusing for now on $F$-coalgebras over the category of sets. All of the standard universal constructions in the category of sets straightforwardly carry over to universal coalgebras over sets.

\begin{itemize}
    \item The {\em coproduct} (or sum) of two $F$-coalgebras $(S, \alpha_S)$ and $(T, \alpha_T)$ is defined as follows. Let $i_S: S \rightarrow (S + T)$ and $i_T: T \rightarrow (S + T)$ be the injections (or monomorphisms) of the sets $S$ and $T$, respectively, into their coproduct (or disjoint union) $S + T$. From the universal property of coproducts, defined in the previous Section, it follows that there is a unique function $\gamma: (S + T) \rightarrow F(S+T)$ such that $i_S$ and $i_T$ are both homomorphisms: 

    \begin{center}
\begin{tikzcd}
  S \arrow[r, "i_S"] \arrow[d, "\alpha_S"]
    & S+T \arrow[d, "\gamma" ] \\
  F(S) \arrow[r,  "F(i_S)"]
& F(S+T)
\end{tikzcd}

\begin{tikzcd}
  T \arrow[r, "i_T"] \arrow[d, "\alpha_T"]
    & S+T \arrow[d, "\gamma" ] \\
  F(R) \arrow[r,  "F(i_T)"]
& F(T)
\end{tikzcd}
\end{center}

\item To construct the {\em co-equalizer} of two homomorphisms $f: (S \alpha_S) \rightarrow (T, \alpha_T)$ and $g: (S, \alpha_S) \rightarrow (T, \alpha_T)$, we need to define a system $(U, \alpha_U)$ and a homomorphism $h: (T, \alpha_T) \rightarrow (U, \alpha_T)$ such that $h \circ f = h \circ g$, and for every homomorphism $h': (T, \alpha_T) \rightarrow (U', \alpha_{U'})$ such that $h' \circ f = h' \circ g$, there exists a unique homomorphism $l: (U, \alpha_U) \rightarrow (U', \alpha_{U'})$ with the property that $h'$ factors uniquely as $h' = l \circ h$. Since $f$ and $g$ are by assumption set-valued functions $f: S \rightarrow T$ and $g: S \rightarrow T$, we know that there exists a coequalizer $h: T \rightarrow U$ in the category of sets (by the property that coequalizers exist in {\bf Sets}). If we define $F(h) \circ \alpha_T: T \rightarrow F(U)$, then since

\begin{eqnarray*}
    F(h) \circ \alpha_T \circ f &=& F(h) \circ F(f) \circ \alpha_S \\
    &=& F(h \circ f) \circ \alpha_S \\
    &=& F(h \circ g) \alpha_S \\
    &=& F(h) \circ F(g) \circ \alpha_S \\
    &=& F(h) \circ \alpha_T \circ g
\end{eqnarray*}

and given $h: T \rightarrow U$ is a coequalizer, there must exist a unique function $\alpha_U: U \rightarrow F(U)$ with the property we desire.

\item There is a more general way to establish the existence of coproducts, coequalizers, pullbacks, and in general colimits by defining a {\em forgetful functor} $U: Set_F \rightarrow Set$, which sends a coalgebra system to its underlying set $U(S, \alpha_S) = S$, and sends the $F$-homomorphism $f: (S, \alpha_S) \rightarrow (T, \alpha_T)$ to the underlying set-valued function $f: S \rightarrow T$. We can apply a general result on the creation of colimits as follows: 

\begin{theorem}
    The forgetful functor $U: Set_F \rightarrow F$ creates colimits, which means that any type of colimit in $Set_F$ exists, and can be obtained by constructing the colimit in Set and then defining for it in a unique way an $F$-transition dynamics. 
\end{theorem}

\item For pullbacks and limits, the situation is a bit more subtle. If $F: Set \rightarrow Set$ preserves pullbacks, then pullbacks in Set$_F$ can be constructed as follows. Let $f: (S, \alpha_S) \rightarrow (S, \alpha_T)$ and $g: (S, \alpha_S) \rightarrow (T, \alpha_T)$ be homomorphisms. Define the pullback of $f$ and $g$ in Set as follows, and extend that to get a pullback of $F(f)$: 

   \begin{center}
\begin{tikzcd}
  P \arrow[r, "\pi_1"] \arrow[d, "\pi_2"]
    & S \arrow[d, "f" ] \\
  U \arrow[r,  "g"]
& T
\end{tikzcd}

\begin{tikzcd}
  F(P) \arrow[r, "F(\pi_1)"] \arrow[d, "F(\pi_2)"]
    & F(S) \arrow[d, "F(f)" ] \\
  F(U) \arrow[r,  "F(g)"]
& F(T)
\end{tikzcd}

\end{center}

It can be shown that the pullback of two homomorphisms is a bisimulation even in the case that $F$ only preserves {\em weak} pullbacks (recall that in defining the pullback universal construction, it was necessary to assert a uniqueness condition, which is weakened to an existence condition in weak pullbacks). 

\end{itemize}

\subsection{Lambek's Theorem and Final Coalgebras: A Generalization of (Greatest) Fixed Points}

Let us illustrate the concept of final coalgebras defined by a functor that represents a monotone function over the category defined by a preorder $(S, \leq)$, where $S$ is a set and $\leq$ is a relation that is reflexive and transitive. That is, $a \leq a, \forall a \in S$, and if $a \leq b$, and $b \leq c$, for $a, b, c \in S$, then $a \leq c$. Note that we can consider $(S, \leq)$ as a category, where the objects are defined as the elements of $S$ and if $a \leq b$, then there is a unique arrow $a \rightarrow b$. 

Let us define a functor $F$ on a preordered set $(S, \leq)$ as any monotone mapping $F: S \rightarrow S$,  so that if $a \leq b$, then $F(a) \leq F(b)$. Now, we can define an $F$-algebra as any {\em pre-fixed point} $x \in S$ such that $F(x) \leq x$. Similarly, we can define any {\em post-fixed point} to be any $x \in S$ such that $x \leq F(x)$. Finally, we can define the {\em final $F$-coalgebra} to be the {\em greatest post-fixed point} $x \leq F(x)$, and analogously, the {\em initial $F$-algebra} to be the least pre-fixed point of $F$. 

In this section, we give a detailed overview of the concept of {\em final coalgebras} in the category of coalgebras parameterized by some endofunctor $F$. This fundamental notion plays a central role in the application of universal coalgebras to model a diverse range of AI and ML systems. Final coalgebras generalize the concept of (greatest) fixed points in many areas of application in AI, including causal inference, game theory and network economics, optimization, and reinforcement learning among others. The final coalgebra, simply put, is just the final object in the category of coalgebras. From the universal property of final objects, it follows that for any other object in the category, there must be a unique morphism to the final object. This simple property has significant consequences in applications of the coalgebraic formalism to AI and ML, as we will see throughout this paper. 

An $F$-system $(P, \pi)$ is termed {\bf final} if for another $F$-system $(S, \alpha_S)$, there exists a unique homomorphism $f_S: (S, \alpha_S) \rightarrow (P, \pi)$. That is, $(P, \pi)$ is the terminal object in the category of coalgebras $Set_F$ defined by some set-valued endofunctor $F$. Since the terminal object in a category is unique up to isomorphism, any two final systems must be isomorphic. 

\begin{definition}
    An $F$-coalgebra $(A, \alpha)$ is a {\em fixed point} for $F$, written as $A \simeq F(A)$ if $\alpha$ is an isomorphism between $A$ and $F(A)$. That is, not only does there exist an arrow $A \rightarrow F(A)$ by virtue of the coalgebra $\alpha$, but there also exists its inverse $\alpha^{-1}: F(A) \rightarrow A$ such that 

    \[ \alpha \circ \alpha^{-1} = \mbox{{\bf id}}_{F(A)} \ \ \mbox{and} \ \  \alpha^{-1} \circ \alpha = \mbox{{\bf id}}_A \]
\end{definition}

The following lemma was shown by Lambek, and implies that the transition structure of a final coalgebra is an isomorphism. 

\begin{theorem}
    {\bf Lambek:} A final $F$-coalgebra is a fixed point of the endofunctor $F$. 
\end{theorem}

{\bf Proof:} The proof is worth including in this paper, as it provides a classic example of the power of diagram chasing. Let $(A, \alpha)$ be a final $F$-coalgebra. Since $(F(A), F(\alpha)$ is also an $F$-coalgebra, there exists a unique morphism $f: F(A) \rightarrow A$ such that the following diagram commutes: 

\begin{tikzcd}
  F(A) \arrow[r, "f"] \arrow[d, "F(\alpha)"]
    & A \arrow[d, "\alpha" ] \\
  F(F(A)) \arrow[r,  "F(f)"]
& F(A)
\end{tikzcd}

However, by the property of finality, the only arrow from $(A, \alpha)$ into itself is the identity. We know the following diagram also commutes, by virtue of the definition of coalgebra homomorphism:

\begin{tikzcd}
  A \arrow[r, "\alpha"] \arrow[d, "\alpha"]
    & F(A) \arrow[d, "\alpha" ] \\
  F(A) \arrow[r,  "F(\alpha)"]
& F(F(A))
\end{tikzcd}

Combining the above two diagrams, it clearly follows that $f \circ \alpha$ is the identity on object $A$, and it also follows that $F(\alpha) \circ F(f)$ is the identity on $F(A)$. Therefore, it follows that: 

\[ \alpha \circ f  = F(f) \circ F(\alpha) = F(f \circ \alpha) = F(\mbox{{\bf id}}_A) = \mbox{{\bf id}}_{F(A)} \qed \]

By reversing all the arrows in the above two commutative diagrams, we get the easy duality that the initial object in an $F$-algebra is also a fixed point. 

\begin{theorem}
    {\bf Dual to Lambek}: The initial $F$-algebra $(A, \alpha)$, where $\alpha: F(A) \rightarrow A$, in the category of $F$-algebras is a fixed point of $F$. 
\end{theorem}

The proof of the duality follows in the same way, based on the universal property that there is a unique morphism from the initial object in a category to any other object. 

Lambek's lemma  has many implications, one of which is that final coalgebras generalize the concept of a (greatest) fixed point, which applies to many formulations of AI and ML problems, from causal inference, game theory and network economics, optimization problems, and reinforcement learning. Generally speaking, in optimization, we are looking to find a solution $x \in X$ in some space $X$ that minimizes a smooth real-valued function $f: X \rightarrow \mathbb{R}$. Since $f$ is smooth, a natural algorithm is to find its minimum by computing the gradient $\nabla f$ over the space $X$. The function achieves its minimum if $\nabla f = 0$, which can be written down as a fixed point equation. A very interesting application of universal coalgebras is to problems involving network economics and games, which can be formulated as {\em variational inequalities}.

\subsection*{Coinductive Inference over Non-Well-Founded Sets: Active Learning}

Analogous to Gold's paradigm of identification in the limit, we now define the novel paradigm of {\em coidentification in the limit}, using Figure~\ref{coindfig} as an illustrative guide. Just as in inductive inference, coinductive inference models the process of ML as a continual potentially infinite interaction between a teacher and a (machine) learner. Initially, the teacher selects a potential coalgebra from a prespecified category of coalgebras (such as finite state machines, MDPs in reinforcement learning, causal inference models etc.). The teacher then gives a {\em presentation} of a selected coalgebra to the learner. At each step, the learner outputs a hypothesized coalgebra to the teacher, representing its best guess of the teacher's coalgebra. Coidentification occurs when the learner's coalgebra is {\em isomorphic} to the teacher's coalgebra. We now give some details on the various components of the coinductive inference paradigm. The fundamental novel aspect of coinductive inference is that the structures enumerated are representations of non-well-founded sets, such as graph-based APG models, or category-theoretic universal coalgebras. Coidentification means bisimulation, which we will define more rigorously below. Analogous to inductive inference, coinductive inference involves the following components: 

\begin{enumerate} 

\item {\em Category of coalgebras:} In inductive inference, the teacher specifies a language from a pre-specified set of possible languages (e.g., all regular languages or all context-free languages). In coinductive inference, the teacher selects a category of coalgebras, which can be viewed as a functorial representation of the dynamics of machines that generate the associated languages. In the example in Figure ~\ref{versionspaces}, the category of coalgebras is defined by the hypothesis space of the set of all axis-parallel rectangles. 

\item {\em Presentation of coalgebra:} In inductive inference, having selected a target language, the teacher then decides on a particular presentation of the language to the learner, in the form of a {\em text} or a more informative presentation. In coinductive inference, the teacher must decide on a presentation of the selected coalgebra. In the example from Figure~\ref{versionspaces}, the teacher's selection of a particular rectangle from which positive and negative examples are generated is modeled as a coalgebra. Crucially, a coalgebra presentation involves not just listing pairs $(x, f(x))$ of some unknown function to be learned, like inductive inference, but crucially, it involves a presentation of a(n) (endo) functor that maps the input category to some output category. A simple way to think about this difference is that it not only involves listing pairs $(x, F(x))$, where $x$ is an object in the input category and $F(x)$ is an object in the output category, but it also invovles selecting some morphism $f: x \rightarrow y$ in the input category, and presenting its effect on the output category $F(f): F(x) \rightarrow F(y)$. 

\item {\em Output of the learner's hypothesis:} In inductive inference, the learner outputs a guess or hypothesis of the target language using some naming convention that defines the language, such as a grammar or a Turing machine. In coinductive inference, the learner outputs a candidate coalgebra that represents its best guess as to the target coalgebra selected by the teacher. As we will see below, a fundamental notion called {\em bisimulation} is used to compare coalgebras. Bisimulation is the equivalent in coalgebras what congruence represents in algebras. Two coalgebras are bisimiar if one can ``simulate" the other, intuitively every transition that can be made in one can be made in the other. For the example in Figure~\ref{versionspaces}, the learner has to guess a coalgebra such that the unique morphism from it to the final coalgebra (the teacher's coalgebra) is an injection, meaning the learner's coalgebra must be the smallest axis-parallel rectangle that correctly classifies every possible training example. In the modified PAC-framework of coinduction, the learner can guess a coalgebra that misclassifies examples, but its probability of misclassification must be bounded by $\epsilon$.

\item {\em Definition of coinduction:} In inductive inference, identification in the limit occurs precisely when at some finite point in time, the learner names the correct target language, and subsequently never changes its guess. In coinductive inference, in contrast, coidentification occurs when the learner's hypothesized coalgebra defines a monomorphism to the teacher's target coalgebra, meaning the learner's coalgebra is {\em isomorphic} to the teacher's (final) target coalgebra. 

\end{enumerate} 

Several crucial differences between inductive and coinductive inference are worth repeating. First and foremost, coinductive inference is about {\em teaching functors}, not functions. Functors map an input category into an output category, meaning not just a transformation of input objects into output objects, but also a transformation of the input morphisms into output morphisms. Second, coinductive inference includes non-well-founded sets, which are disallowed by inductive inference which takes place in the ZFC formulation of well-founded sets. Coinductive inference can be modeled simply as a process of enumerating APGs to find the one that is consistent with the teacher. In contrast, inductive inference is based on enumeration of recursively enumerable sets. In coinductive inference, it is never assumed that there is a finite point in time when the learner has converged to the correct answer. It is assumed that learning continues indefinitely, but it is assumed that there will eventually be a convergence to the final coalgebra (as a way of computing the greatest fixed point, as explained in the next section). Another fundamental difference comes from the category-theoretic nature of coinductive inference: instead of requiring equality, what is required is to find a hypothesis coalgebra that is isomorphic to the teacher's coalgebra. The notion of a minimum coalgebra is defined with respect to the unique morphism from the learner's hypothesized coalgebra to the teacher's target coalgebra.

\subsection{Dynamic Imitation Games in  Metric Spaces}

To make the somewhat abstract discussion of coinductive inference above a bit more concrete, we now briefly describe a few examples of coinductive inference algorithms that will be later discussed in more depth in the remainder of this paper. We begin with the concept of {\em metric coinduction} based on some work by Kozen \cite{kozen}. We will see in the next Section how to generalize this approach using the powerful Yoneda Lemma to generalized metric spaces. The basic idea is simple to describe, and is based on viewing algorithms as forming contraction mappings in  a metric space. The novelty here for many readers is understanding how this well-studied notion of contractive mappings is related to coinduction and coalgebras. 

Consider a complete metric space $(V, d)$ where $d: V \times V \rightarrow (0, 1)$ is a symmetric distance function that satisfies the triangle inequality, and all Cauchy sequences in $V$ converge (We will see later that the property of completeness itself follows from the Yoneda Lemma!). A function $H: V \rightarrow V$ is {\em contractive} if there exists $0 \leq c < 1$ such that for all $u, v \in V$, 

\[ d(H(u), H(v)) \leq c \cdot d(u, v) \]

In effect, applying the algorithm represented by the continuous mapping $H$ causes the distances between $u$ and $v$ to shrink, and repeated application eventually guarantees convergence to a fixed point. The novelty here is to interpret the fixed point as a final coalgebra. We will later see that the concept of a (greatest) fixed point is generalized by the concept of final coalgebras.

\begin{definition}
{\bf Metric Coinduction Principle} \cite{kozen}: If $\phi$ is a closed nonempty subset of a complete metric space $V$, and if $H$ is an eventually contractive map on $V$ that preserves $\phi$, then the unique fixed point $u^*$ of $H$ is in $\phi$. In other words, we can write: 

\begin{eqnarray*}
    \exists u \phi(u), \ \ \ \forall \phi(u) &\Rightarrow& \phi(H(u)) \\
    &\phi(u^*)& 
\end{eqnarray*}
\end{definition}

It should not be surprising to those familiar with contraction mapping style arguments that a large number of applications in AI and ML, including game theory, reinforcement learning and stochastic approximation involve proofs of convergence that exploit properties of contraction mappings. What might be less familiar is how to think of this in terms of the concept of {\em coinductive inference}. To explain this perspective briefly, let us introduce the relevant category theoretic terminology. 

Let us define a category $C$ whose objects are nonempty closed subsets of $V$, and whose arrows are reverse set inclusions. That is, there is a unique arrow $\phi_1 \rightarrow \phi_2$ if $\phi_1 \supseteq \phi_2$. Then, we can define an {\em endofunctor} $\bar{H}$ as the closure mapping ${\bar H}(\phi) = \mbox{cl}(H(\phi))$, where $\mbox{cl}$ denotes the closure in the metric topology of $V$. Note that $\bar{H}$ is an endofunctor on $C$ because $\bar{H}(\phi_1) \supseteq  \bar{H}(\phi_2)$ whenever $\phi_1 \supseteq \phi_2$. 

\begin{definition}
    An {\bf $\bar{H}$-coalgebra} is defined as the pair $(\phi, \phi \supseteq \bar{H}(\phi))$ (or equivalently, we can write $\phi \subseteq H(\phi)$. The {\bf final coalgebra} is the isomorphism $u^* \simeq \bar{H}(u^*)$ where $u^*$ is the unique fixed point of the mapping $H$. The metric coinduction rule can be restated more formally in this case as: 

    \[ \phi \supseteq H(\phi) \ \ \Rightarrow \ \ \phi \supseteq H(u^*) \]
\end{definition}

To appreciate the power of abstraction provided by universal coalgebras, we briefly state a foundational result proved by Lambek that showed that final coalgebras are isomorphisms in the category of coalgebras. This result has deep significance, as we will see later, and provides an elegant way to prove contraction style arguments in very general settings.

\subsection{Reinforcement Learning as Coinduction  over  Coalgebras}

In this section, we briefly discuss how to play dynamic imitation games with reinforcement learning (RL) \cite{bertsekas:alphazero,bertsekas:rlbook,DBLP:books/lib/SuttonB98},  which is a model of sequential decision making where an agent learns to choose actions by interacting with some environment that is modeled as a Markov decision process (MDP). This area has been extensively studied over several decades, and a full discussion is beyond the scope of this paper. However, the major difference in perspective between our paper and previous approaches is that we view RL in terms of {\em coinductive inference over universal coalgebras}, a point of view that is not yet mainstream in the RL literature. It is relatively straightforward to show that MDPs can be defined as coalgebras, as \cite{feys:hal-02044650} have elegantly shown.  In addition, it follows that MDPs can be compared in terms of MDP bisimulation metrics as a special case of the general bisimulation relationship defined on arbitrary coalgebras \cite{rutten2000universal}. A significant amount of literature in the RL community has explored bisimulation relationships between MDPs \cite{DBLP:conf/aaai/CastroP10,DBLP:conf/aaai/RuanCPP15}. A categorical treatment of bisimulation in MDPs and universal decision models was explored by us previously \cite{sm:udm}.  More specifically, the process of solving an MDP uses solution methods, such as policy iteration and value iteration, which can be shown to be special cases of the metric coinduction method that we described above \cite{feys:hal-02044650}. We only briefly mention here that this previous work can be extended to the case of using stochastic approximation  methods, instead of classical policy and value iteration. 

The convergence of RL algorithms is typically carried out using stochastic approximation \cite{rm,borkar}. At a high level, the approach is based on showing that the solution to the well-known Bellman equation 

\[ V*(s) = \max_{a \in A} \{ R(s,a) + \gamma \sum_{s'} P(s' | s,a) V^*(s') \} \]

can be solved using a Robbins-Monroe type stochastic approximation algorithm \cite{rm}, such as temporal-difference (TD)-learning \cite{DBLP:books/lib/SuttonB98}. As we previously pointed, the dynamic programming principle embodied in the Bellman equation is a special case of the general principle of sheaves that we described in Section~\ref{sheavestopoi}. The Bellman optimality principle essentially states that the shortest path of any restriction of the overall path must itself be a shortest path (otherwise a shorter path could replace it thereby reducing the length of the overall path). Thus, the sheaf condition plays an essential role in the design of both DP and RL algorithms, a perspective that is novel to the best of our knowledge. 

There is an extensive literature in RL over the past few decades analyzing the convergence properties of RL algorithms, which we will not attempt to summarize here. Abstractly, we can view an MDP as a universal coalgebra, and the process of solving a Bellman equation as finding a final coalgebra. \cite{feys:hal-02044650} gives a detailed analysis of this coalgebraic view of MDPs, including the classical dynamic programming (DP) approach to solving MDPs \cite{DBLP:books/lib/Bertsekas05}. Their work can be extended to the RL setting by suitably adapting the framework of metric coinduction \cite{kozen} to the stochastic approximation setting. We leave this extension to a subsequent paper. It is possible to define regularized RL that combines a {\em smooth} loss function with a non-smooth loss function, such as sparse regularization using $L_1$ norms \cite{mirror-prox}. We developed this approach in a framework called {\em proximal RL} \cite{mahadevan2014proximal}, which leads to provably convergent off-policy TD algorithms. A coalgebraic analysis of off-policy TD algorithms is left to a subsequent paper. We have previously shown how to define categories for RL, where objects can be MDPs or predictive state representations (PSRs) (see \cite{sm:udm} for additional details). \cite{feys:hal-02044650} additionally show how MDPs can be modeled as coalgebras. In fact, an entire family of probabilistic coalgebras can be defined for RL as described by \cite{SOKOLOVA20115095}.

We want to generalize the diversity based representation of environment described above to the stochastic setting, and for that purpose, we now introduce a specific class of stochastic coalgebras defined by Markov decision processes (MDPs) and partially observable MDPs (POMDPs). The diversity based representation of a POMDP is referred to as a predictive state representation (PSR) in the literature. 

An MDP is defined by a tuple $\langle S, A, \Psi, P, R \rangle$, where $S$ is a discrete set of states, $A$ is the discrete set of actions, $\Psi \subset S \times A$ is the set of admissible state-action pairs, $P: \Psi \times S \rightarrow [0,1]$ is the transition probability function specifying the one-step dynamics of the model, where $P(s,a,s')$ is the transition probability of moving from state $s$ to state $s'$ in one step under action $a$, and $R: \Psi \rightarrow \mathbb{R}$ is the expected reward function, where $R(s,a)$ is the expected reward for executing action $a$ in state $s$. MDP homomorphisms can be viewed as a principled way of abstracting the state (action) set of an MDP into a ``simpler" MDP that nonetheless preserves some important properties, usually referred to as the stochastic substitution property (SSP). This is a special instance of the general bisimulation property that holds among coalgebras, which we defined above. 

\begin{definition}
An MDP homomorphism \cite{DBLP:conf/ijcai/RavindranB03} from object  $M = \langle S, A, \Psi, P, R \rangle$ to $M' = \langle S', A', \Psi', P', R' \rangle$, denoted $h: M \twoheadrightarrow M'$, is defined by a tuple of surjections $\langle f, \{g_s | s \in S \} \rangle$, where $f: S \twoheadrightarrow S', g_s: A_s \twoheadrightarrow A'_{f(s)}$, where $h((s,a)) = \langle f(s), g_s(a) \rangle$, for $s \in S$, such that the stochastic substitution property and reward respecting properties below are respected: 
\begin{eqnarray} 
\label{mdp-hom}
P'(f(s), g_s(a), f(s')) = \sum_{s" \in [s']_f} P(s, a, s") \\
R'(f(s), g_s(a)) = R(s, a)
\end{eqnarray} 
\end{definition}

Given this definition, the following result is straightforward to prove. 

\begin{theorem}
The category ${\cal C}_{\mbox{MDP}}$ is defined as one where each object $c$ is defined by an MDP, and morphisms are given by MDP homomorphisms defined by Equation~\ref{mdp-hom}. 
\end{theorem}

{\bf Proof:} Note that the composition of two MDP homomorphisms $h: M_1 \rightarrow M_2$ and $h': M_2 \rightarrow M_3$ is once again an MDP homomorphism $h' \ h: M_1 \rightarrow M_3$. The identity homomorphism is easy to define, and MDP homomorphisms, being surjective mappings, obey associative properties. $\qed$

Given the category of MDPs defined above, we can straightforwardly define the category of coalgebras over MDPs, given any endofunctor $F_{\mbox{MDP}}$ that acts on objects in this category. We will explore later a novel formulation of reinforcement learning -- a stochastic coinductive inference approach -- to solving universal coalgebras defined over MDPs. 

We now define the category ${\cal C}_{\mbox{PSR}}$ of predictive state representations \cite{DBLP:journals/jmlr/ThonJ15}, based on the notion of homomorphism defined for PSRs proposed in  \cite{DBLP:conf/aaai/SoniS07}.  A PSR is (in the simplest case) a discrete controlled dynamical system, characterized by a finite set of actions $A$, and observations $O$. At each clock tick $t$, the agent takes an action $a_t$ and receives an observation $o_t \in O$. A {\em history} is defined as a sequence of actions and observations $h = a_1 o_1 \ldots a_k o_k$. A {\em test} is a possible sequence of future actions and observations $t = a_1 o_1 \ldots a_n o_n$. A test is successful if the observations $o_1 \ldots o_n$ are observed in that order, upon execution of actions $a_1 \ldots a_n$. The probability $P(t | h)$ is a prediction of that a test $t$ will succeed from history $h$. Note how the PSR straightforwardly generalizes the notion of a diversity-based representation that we explored above for the deterministic finite state environments to the stochastic setting. 

A state $\psi$ in a PSR is a vector of predictions of a suite of {\em core tests} $\{q_1, \ldots, q_k \}$. The prediction vector $\psi_h = \langle P(q_1 | h) \ldots P(q_k | h) \rangle$ is a sufficient statistic, in that it can be used to make predictions for any test. More precisely, for every test $t$, there is a $1 \times k$ projection vector $m_t$ such that $P(t | h) = \psi_h . m_t$ for all histories $h$. The entire predictive state of a PSR can be denoted $\Psi$. 

\begin{definition}
\label{psr-homo}
The category ${\cal C}_{\mbox{PSR}}$ defined by PSR objects, the morphism from object $\Psi$ to another  $\Psi'$ is defined by a tuple of surjections $\langle f, v_\psi(a)\rangle$, where $f: \Psi \rightarrow \Psi'$ and $v_\psi: A \rightarrow A'$ for all prediction vectors $\psi \in \Psi$ such that 
\begin{equation}
    P(\psi' | f(\psi), v_\psi(a)) = P(f^{-1}(\psi') | \psi, a) 
\end{equation}
for all $\psi' \in \Psi, \psi \in \Psi, a \in A$. 
\end{definition}

\begin{theorem}
The category ${\cal C}_{\mbox{PSR}}$ is defined by making each object $c$ represent a PSR, where the morphisms between two PSRs $h: c \rightarrow d$ is defined by the PSR homomorphism defined in \cite{DBLP:conf/aaai/SoniS07}. 
\end{theorem}

 {\bf Proof:} Once again, given the homomorphism definition in Definition~\ref{psr-homo}, the UDM category ${\cal P}_{\mbox{PSR}}$ is easy to define, given the surjectivity of the associated mappings $f$ and $v_\psi$. $\qed$
 
Finally, once again, we can easily define coalgebras over the category of PSRs, given any endofunctor $F_{\mbox{PSR}}$ that acts on PSR objects. We will also explore a novel stochastic coinductive framework for universal coalgebras over PSRs later in the book. 

Given a suitable category for RL, we can now define the RL imitation game as follows: 

 \begin{definition}
     The RL imitation game is defined as one where an object $c$ and $d$, representing either objects in the category of Markov decision processes or predictive state representations, can be made isomorphic by a suitable pair of MDP or PSR homomorphisms that map from $c$ to $d$ (or vice versa), in which case we view objects $c$ and $d$ as bisimilar \cite{bisim}. 
 \end{definition}

Much remains to be explored here, and we aim to study this framework in a later paper, focusing principally on the use of universal coalgebras in a generalized RL setting. Among the problems to be studied here is the use of coinductive inference techniques, such as metric coinduction \cite{kozen} to analyze the convergence of RL algorithms, and the study of more sophisticated probabilistic coalgebras to define novel classes of RL problems \cite{SOKOLOVA20115095}. 

\section{Evolutionary Universal Imitation Games}

\label{evolutionaryuig} 

\begin{quote}
Evolution is a blind watchmaker -- Richard Dawkins 
\end{quote} 

Finally, we turn to discuss  {\em evolutionary} UIGs, which represent perhaps the most complex and interesting case. Evolution is the story of life on earth: complex minds such as ours could never have been created without the forces of evolution acting over millions of years. No framework for UIGs will be complete without a discussion of evolutionary UIGs. What fundamentally distinguishes evolutionary UIGs from static or dynamic UIGs is not only are the participants non-stationary, but that the changes are {\em non-adaptive}! In other words, natural selection -- the driving force behind evolution -- is based on {\em random mutations} that occur in the natural population representing the participants. Those participants who happen to have been ``blessed" with the right mutations that confer on them an advantage in fitness tend to reproduce more, and that drives the evolutionary process forwards. As Richard Dawkins picteresque analogy suggests, the variation is not adapted, but natural! Birth or death is the essence of not just biology, but also any human endeavor, whether it be a scientific theory, a novel technology or a startup. In this section, we study evolutionary UIGs from the same perspective we have used in the previous two tyoes of UIGs, namely category theory, in particular universal properties in terms of initial and final objects, measurement probes defined by objects and morphisms of a category, and the use of universal coalgebras and coinduction to analyze the dynamical system generated by an evolutionary UIG. To begin with, we describe an adaptation of the classic PAC learning framework \cite{DBLP:journals/cacm/Valiant84} to study evolvability \cite{evolvability}. We then describe coalgebraic models of evolvability based on birth-death stochastic processes. Finally, we describe an adaptation of the classical variational inequality (VI) problem \cite{facchinei-pang:vi} to include evolutionary processes. VIs can be shown to generalize the classic game-theoretic paradigm pioneered by von Neumann and Morgenstern \cite{vonneumann1947} and Nash \cite{nash} to the setting of network economics \cite{nagurney:vibook}, where a society of agents is playing a game on a graph. 

A large number of models of evolutionary game theory have been studied \cite{Nowak_06}, such as modeling the spread of {\em mutants} like viruses in a population of interacting organisms, or the spread of ideas, such as generative AI, through a research community or the high-tech industry. The fundamental principle of evolution differs substantially from the model of dynamic UIGs in the previous section, which was based on machine learning models, such as inductive inference. In the case of evolution, organisms have a certain natural {\em fitness} value, which makes them singularly well-adapted or not to their environment. But changes are caused by natural selection due to random mutations that confer selective advantages, and not based on the particular order in which the organism sensed the world or was exposed to some dataset. Organisms that have a better fitness value tend to reproduce and produce more progeny, which has the effect of multiplying their descendants through a population. As \cite{nowak} shows, this framework is equally adept at modeling evolution in biology at multiple temporal and spatial scales, from cells to bacteria to complete ecosystems, as well as the spread of ideas through a community and the evolution of language. With this backdrop in mind, evolutionary UIGs actually span a very large number of potential applications. 

The principal difference between dynamic and evolutionary UIGs is that the latter are characterized as the search for equilibria. Each participant in an evolutionary UIG seeks to find a behavior policy that achieves a maximal fitness value, or payoff, which is conditioned on the behavior of others with whom it is interacting. Traditional game theory \cite{Maschler_Solan_Zamir_2013} assumes all players interact, but in large network economies as well as in biology, interactions are modulated over graphs that represent modes of interaction. In our framework, we will define interactions in terms of the arrows in a category.  Equilibrium solutions represent the behavior of the game in a steady state. However, in practice, often equilibria are often transient, and subject to change, as in the case of a virus like Covid-19 that spreads through a population.  The computation of Nash equilibria (NE) \cite{nash1950equilibrium} in non-cooperative games has been extensively studied for many decades \cite{nisan07,Maschler_Solan_Zamir_2013}.

\subsection{Evolvability vs. Learnability}

There has been previous work adapting the framework of PAC learning to evolutionary processes \cite{evolvability}, which builds on the extension of PAC learning to the case when the labels provided to a supervised inductive learner are corrupted by noise \cite{kearns94clt}. In that case, it is no longer possible to design a straightforward inductive inference method, because any example that is labeled positive might indeed be negative or vice versa. \cite{kearns94clt} discuss a framework where learning happens using some statistical summary of the training data, which is then adapted to the evolutionary setting by \cite{evolvability}. 

In Section~\ref{dynamicuig}, we explored the process of adaption based on processing individual training examples by which a ``learner" can imitate a ``teacher", using inductive or coinductive inference. We begin with drawing a sharp contrast between {\em evolvability} \cite{evolvability} and learnability \cite{GOLD1967447} (see Figure~\ref{evolvablefns}) as defined in a model proposed by \cite{evolvability}. The fundamental difference between PAC learnability \cite{DBLP:journals/cacm/Valiant84} and evolvability relate to whether the changes that are made in a participant playing an evolutionary UIG are sensitive to {\em individual} interactions, as is the case with dynamic UIGs described in Section~\ref{dynamicuig}, or are they more a holistic statistical function of the participants {\em entire experience} \cite{nowak}. Fundamentally, it can be argued that evolutionary processes have little to no knowledge of the complex mapping from an individual organism's genomic structure to its behavior.  Evolution operates at a macro scale, summarizing an organism's entire life into a scalar {\em fitness function}. It is common in evolutionary dynamics to model evolution as proceeding over a {\em fitness landscape} \cite{BEERENWINKEL2006409}. 

\begin{figure}[t]
\centering
\includegraphics[scale=.4]{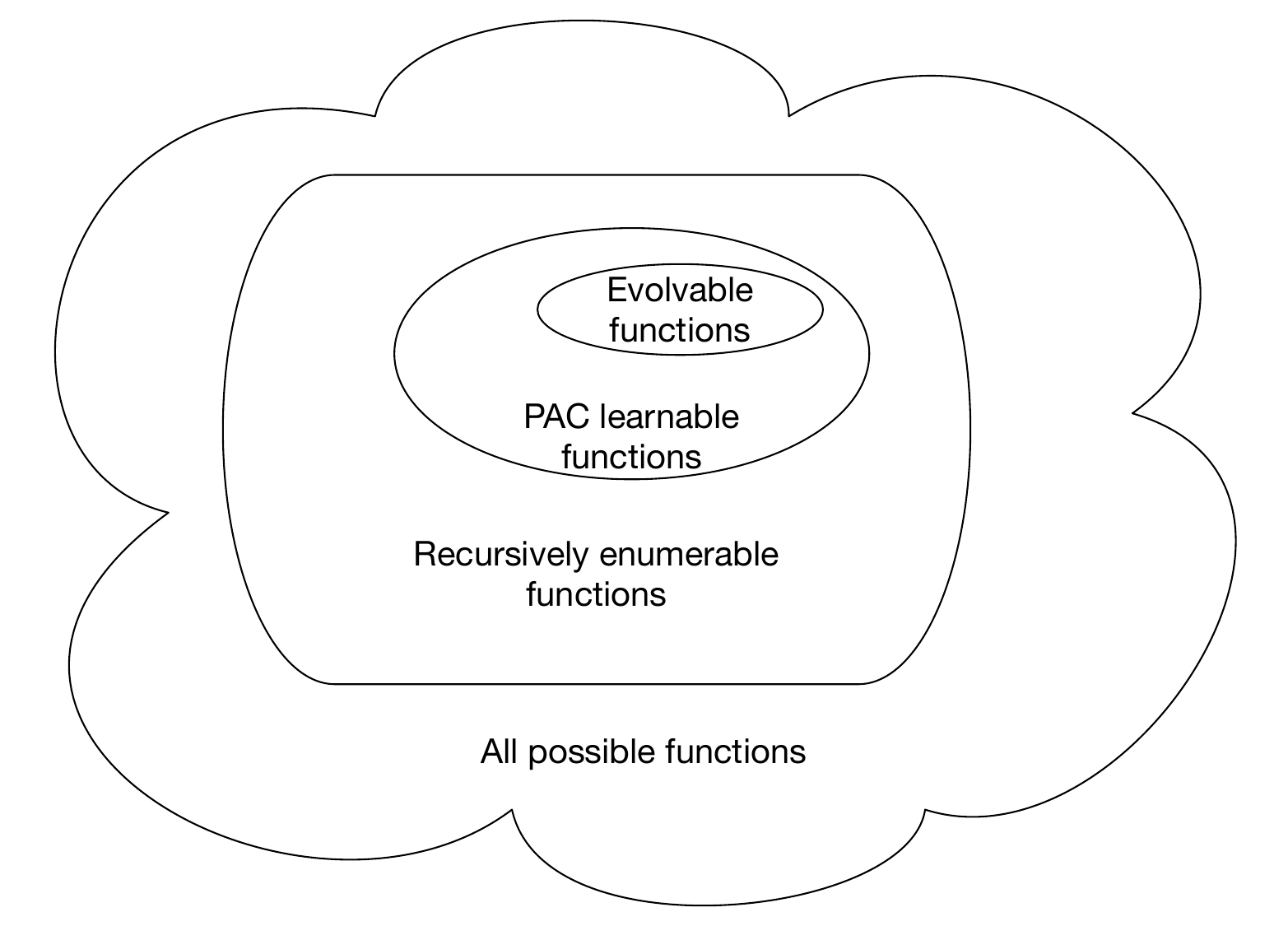}
\caption{In Valiant's model of evolvability \cite{evolvability}, functions that are evolvable form a strict subset of those that are PAC learnable.}
\label{evolvablefns} 
\end{figure}

The fundamental question addressed in Valiant's evolvability model relate to whether complex functions of many arguments can be evolved in a ``reasonable" amount of time with a finite population of variants playing an evolutionary UIG. The definition of what constitutes ``reasonable" is of course subjective, but computational learning theory models typically posit that models that use polynomial amount of resources in the parameters of interest, such as the size of the representation used to encode the function or the training experience, are reasonable, and distinguish them from those that consume a non-polynomial amount of resources. Leaving aside the tricky question of whether what is computable by a Turing machine or a digital computer necessarily translates into what is feasible in biology or physics (which operate according to the principles of quantum mechanics \cite{Coecke_Kissinger_2017}), let us consider this model to begin with to provide an initial analysis of the differences between evolvability and learnability. 

If we restrict our attention to monotone conjunctions, and define a target function $f$ as some subset of literatus $\{x_1, \ldots, x_n \}$.  $X_n$ denotes the $2^n$ possible arguments that the $n$ variables $x_1, \ldots, x_n$ can take. Let $D_n$ denote a probability distribution over $X_n$. We can interpret a ``target" function $f$ as the ``teacher" and any guess about $f$ as a ``participant" trying to ``imitate" the teacher. 

\begin{definition}
    The {\bf performance} of a participant in an evolutionary UIG $r: X_n \rightarrow \{-1, +1 \}$ with respect to a ``teacher participant $f: X_n \rightarrow \{-1, +1 \}$ for a given probability distribution $D_n$ over $X_n$ is defined as 

\[ \mbox{Perf}_f(r, D_n) = \sum_{x \in X_n} f(x) r(x) D_n(x) \]
\end{definition}

In simple terms, $\mbox{Perf}_f(r, D_n)$ is simply measuring the correlation between the teacher and the learner. The basic model posits that every time the participant has an experience in the form of a set of values $x \in X_n$, it will receive a ``benefit" of $+1$ if its ``circuit" $r$ agrees with the teacher's ``circuit" $f$ on that $x$, and receive a penalty $-1$ if it disagrees. Over a sequence of experiences over a ``lifetime", the total of all benefits and penalties are accumulated. Organisms that have a high benefit will be preferentially selected for ``reproduction" through imitation (or ``cloning") over those that have lower benefits. A large number of variations of this model have been studied in the evolutionary dynamics literature \cite{novak:book}. The function $\mbox{Perf}_f(r, D_n)$ is an example of a fitness landscape \cite{BEERENWINKEL2006409}. 

If we restrict ourselves to the very constrained class of monotone conjunctions, and the target teacher participant's behavior is described by some specific conjunction, say $f = x_2 x_5 x_20$, then the goal of the learner participant is to evolve its behavior to match that of the teacher. It can do so by essentially deleting or adding a single literal, which can be viewed as a toy example of actual evolution on  the genomic structure of an organism. In Valiant's framework, it can be shown that some classes of boolean functions are evolvable, such as monotone conjunctions, whereas others such as the parity function is not evolvable (from a uniform distribution on $D_n$). We will not discuss the specifics of this particular proposal modeling evolvability as it involves making many possibly ad-hoc assumptions on the nature of evolution, and the ultimate conclusion that evolvability is s subclass of learnabiilty might strike some as too limiting a notion. In biology, evolution seems capable of remarkable abilities, as it clearly has evolved rather complex creatures such as ourselves. Whether this is a function of purely a process that is even simpler than machine learning processes like gradient descent is a topic for debate that is beyond the scope of this paper.

\subsection{Evolutionary UIGs in categories}

To model the process of evolution as a search in a fitness landscape, we need to make assumptions on the structure of the space that the fitness landscape is defined over. In Valiant's model that was just discussed, the fitness landscape was boolean assignments over $n$ variables. Other approaches that have been studied include modeling fitness landscapes over distributed lattices \cite{BEERENWINKEL2006409}, graphs \cite{novak:book}, and in the vast literature on genetic programming, the fitness landscape is defined over the set of all programs (e.g., LISP programs in Koza's work \cite{Koza92}). A complete discussion of the merits of each such approach is beyond the scope of this paper. Rather, we want to focus our discussion specifically on the advantages of using category theory to formulate the problem of modeling fitness landscapes. We begin by considering the problem in some of the categories described above in Sections~\ref{staticuig} and Section~\ref{dynamicuig}. 

\subsection*{Evolutionary UIGs in Universal Coalgebras}

We saw above that universal coalgebras \cite{jacobs:book,rutten2000universal,SOKOLOVA20115095} provide a rich language for defining dynamical systems and stochastic processes. We can adapt the models of evolutionary dynamics proposed in the literature to the setting of coalgebras relatively easily. To make this more concrete, we begin with the simple example of evolutionary dynamics over Moran processes \cite{moran}. 

The Moran process  is a stochastic model that is widely used  to study evolution in finite populations \cite{nowak}. To illustrate this process as it applies to evolutionary UIGs, let us define a population of participants of a constant size $n$, who are comprised of two types of individuals. For example, in the setting of understanding how novel technologies spread through an economy, let us imagine that there are two types of businesses: those who have adopted generative AI in their products, indicated by blue, and those who have yet to adopt generative AI, indicated by red. 

The Moran process is a classic example of a stochastic process that is driven by a simple Markov chain, which can be simply defined as a triangular stochastic matrix \cite{nowak}. At each step in the process, a single participant is selected for reproduction. Under the somewhat unrealistic assumption that the population size remains constant, it is necessary therefore that some participant has to go extinct. We can imagine that at each step, one company that has adopted generative AI is allowed to reproduce, but another that failed to adopt generative AI goes out of business. Participants that are selected for reproduction and extinction are chosen randomly. The blue and red participants have different finesses that in turn change the rates at which they reproduce. If the fitness of type red is normalized to the value $1$, and the fitness of type blue is set to $r$, then the probability that type blue is chosen to reproduce is given by 

\[ p_i = \frac{r i}{(r i+N-i)} \] 

assuming there are $i$ individuals of type blue, and $N-i$ individuals of type red. This implies that the probability that participants of type red are chosen to reproduce is given by 

\[ p_i  = \frac{(N-i)}{(r i+N-i)}\] 

In Figure~\ref{moranfig}, blue individuals have a relative fitness of $r=1.2$,  and their sizes are proportional to their relative fitness. At the beginning, the population has fewer blue dots (indicating fewer companies have embraced generative AI), and at the end of the evolutionary process, all companies have incorporated generative AI. Those that did not have gone bankrupt, and are not shown in this illustration. 

It is relatively simple to design a universal coalgebra over a distribution functor ${\cal D}$ to capture a Moran process, defined coalgebraically as 

\[ X \rightarrow {\cal D}(X) \]

where ${\cal D}$ maps a set $X$ to a probability distribution over $X$ that defines each step of the Moran process. We can in fact relatively straightforwardly design a whole family of coalgebraic Moran processes, building on the library of stochastic coalgebras defined in \cite{SOKOLOVA20115095}. It is also immediate that the final coalgebra here $Z \rightarrow {\cal D}(Z)$ is defined by one of two recurrent states, where the population of participants is all of the same type (either all red or all blue dots in Figure~\ref{moranfig}), after which no further changes will be possible. 

\begin{figure}
    \centering
    \includegraphics[scale=.2]{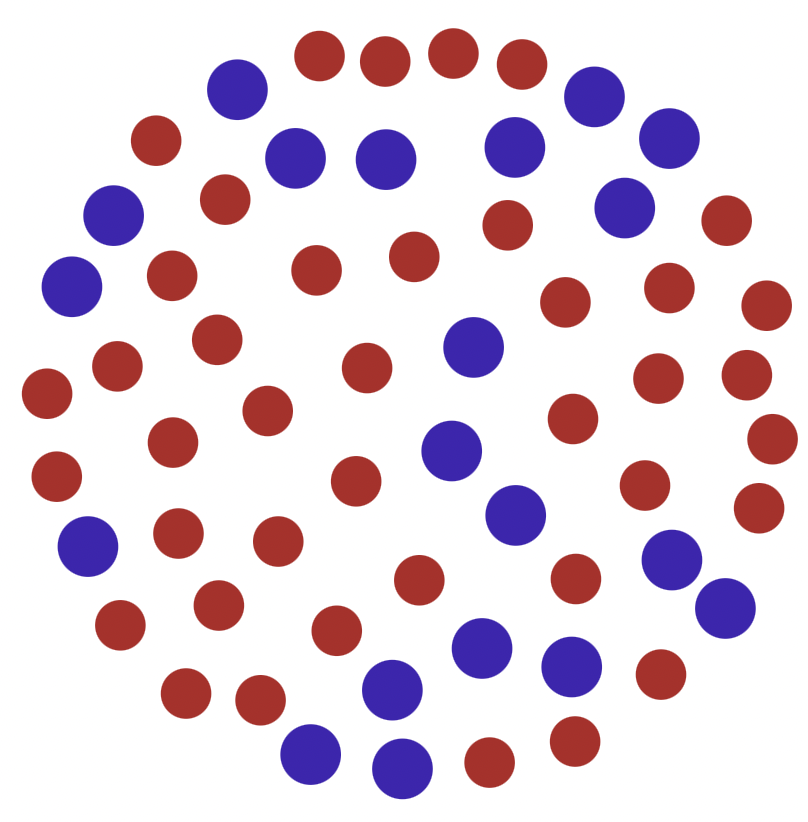}
     \includegraphics[scale=.2]{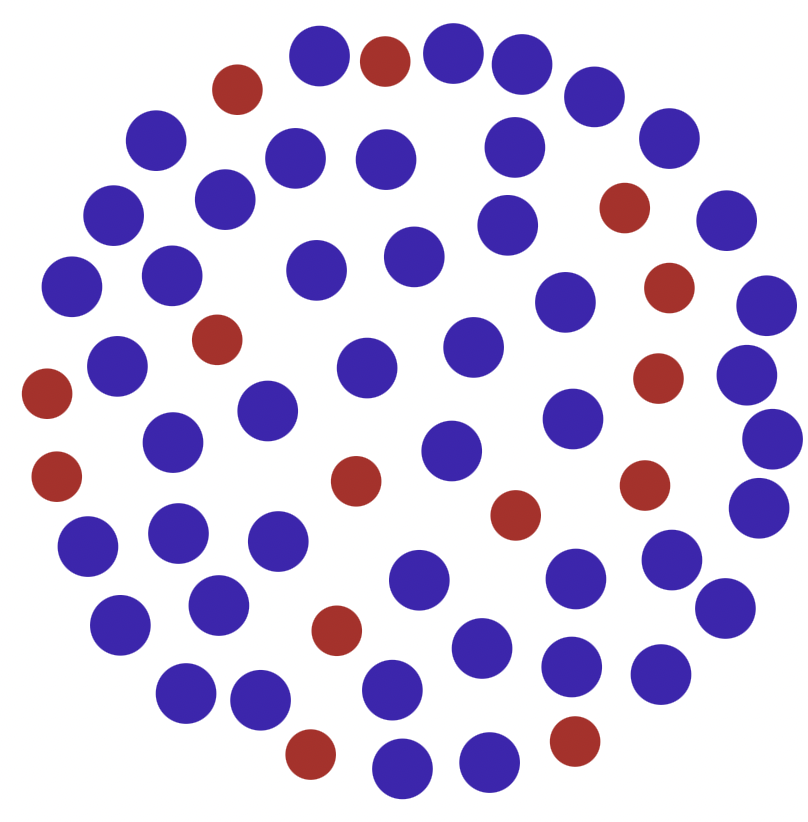}
      \includegraphics[scale=.2]{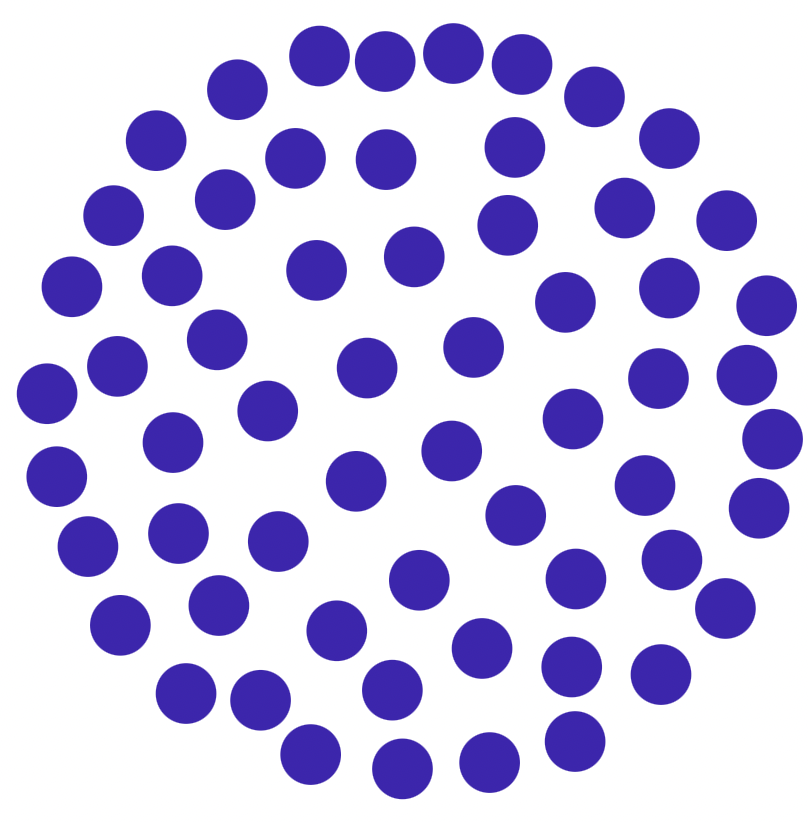}
    \caption{Three stages in a Moran process, from left to right,  that models evolutionary dynamics in a fixed population of finite size.}
    \label{moranfig}
\end{figure}

\subsection{Evolutionary UIGs in Finite Topological Spaces: A Bioinformatics Example}

A fundamental aspect of fitness landscape modeling in evolutionary UIGs is to assume that a participant's behavior is a function of its genomic structure that is largely based on the set of mutations that it has accumulated over time. That is, mutations are {\em irreversible} and we can define a category structure (such as a partial order) that specifies constraints on the order in which mutations have ocurred. We now summarize some previous work of ours \cite{sm:homotopy} that explored construction of topological causal models from genomic cancer data, where the topological structure essentially defines a type of category. In other words, when we use the word ``topological causal model" in this section, it should be interpreted as a finite category defined by a partial order. 

\begin{figure}
    \centering
    \begin{minipage}[t]{.3\linewidth}
\vspace{0pt}
\centering
   \begin{tabular}{ll}
    Tumor & Gene   \\ \hline
    Pa017C & KRAS   \\ \hline 
    Pa017C & TP53   \\ \hline
    Pa019C & KRAS   \\ \hline
    Pa022C & KRAS   \\ \hline
    Pa022C & SMAD4   \\ \hline
    Pa022C & TP53  \\ \hline
    Pa032X & CDKN2A    \\ \hline
    \end{tabular}%
\end{minipage}%
    \caption{Genetic mutations in pancreatic cancer.}
    \label{pancreaticcancer}
\end{figure}

 \begin{figure}[t]
\begin{minipage}[t]{.3\linewidth}
\vspace{0pt}
\centering
\includegraphics[scale=0.2]{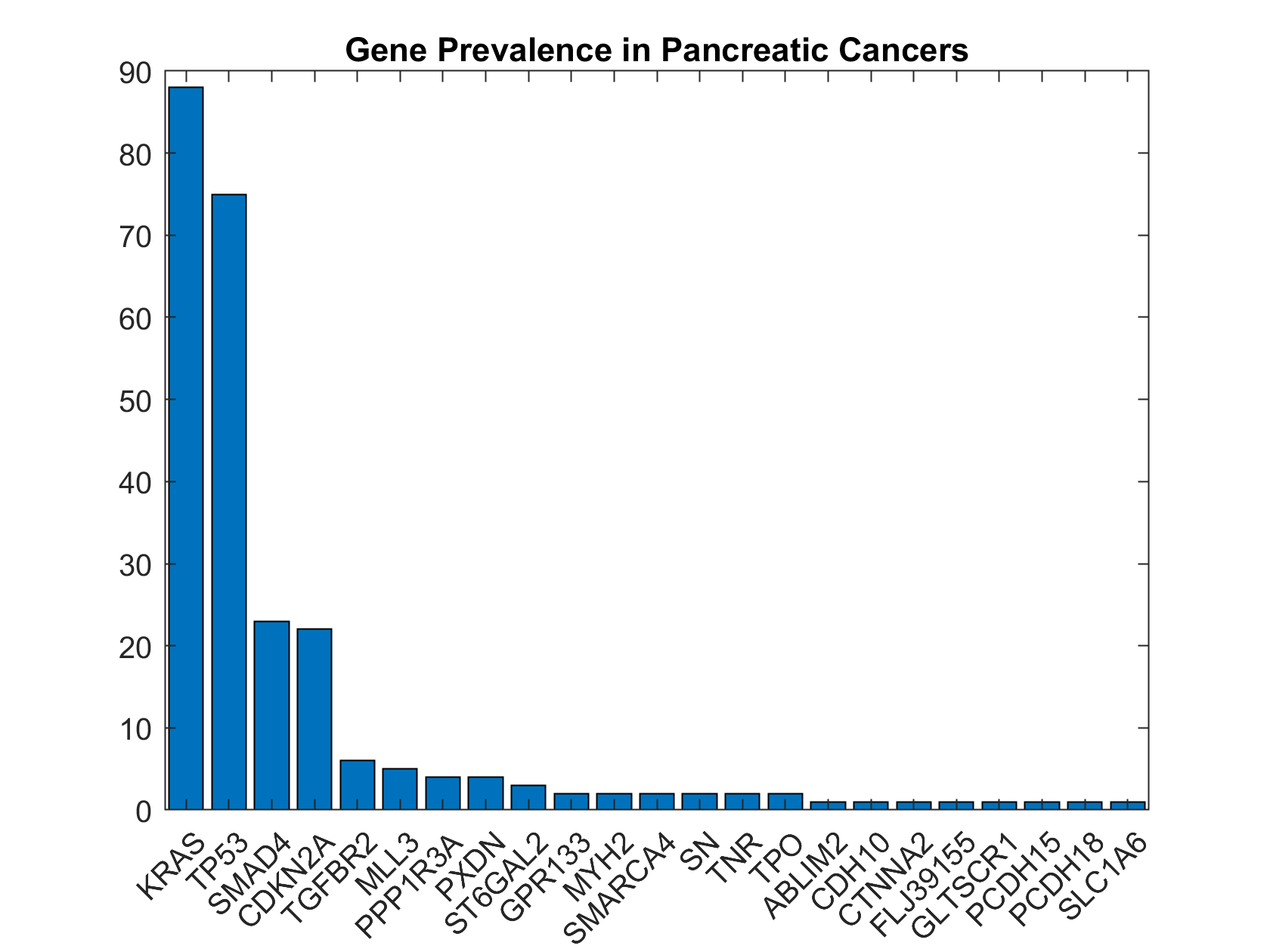} 
\end{minipage}
\begin{minipage}[t]{.3\linewidth}
\vspace{0pt}
\centering
\includegraphics[scale=0.2]{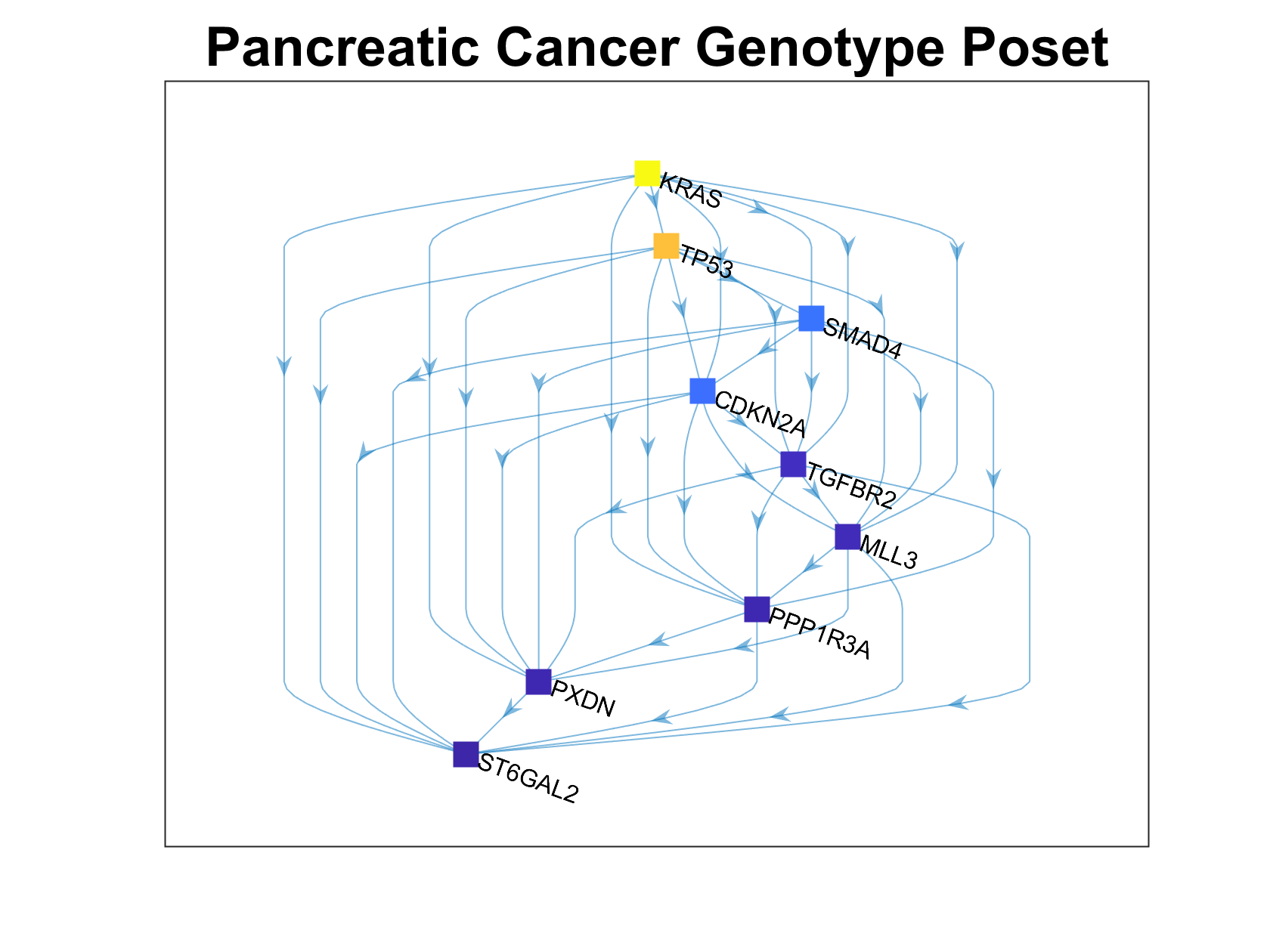} 
\end{minipage}
 \caption{Left: Genetic mutations in pancreatic cancer \citep{Jones1801}. Middle: histogram of genes sorted by mutation frequencies. Right: Poset learned from dataset.} 
\label{cancer-dataset}
\end{figure}%

Table~\ref{pancreaticcancer} shows a small fragment of a dataset for pancreatic cancer \citep{Jones1801}. Like many cancers, it is marked by a particular partial ordering of mutations in some specific genes, such as {\bf KRAS}, {\bf TP53}, and so on.  In order to understand how to model and treat this deadly disease, it is crucial to understand the inherent partial ordering in the mutations of such genes. Pancreatic cancer remains one of the most prevalent and deadly forms of cancer. Roughly half a million humans contract the disease each year, most of whom succumb to it within a few years. Figure~\ref{cancer-dataset} shows the roughly $20$ most common genes that undergo mutations during the progression of this disease. The most common gene, the KRAS gene,  provides instructions for making a protein called K-Ras that is part of a signaling pathway known as the RAS/MAPK pathway. The protein relays signals from outside the cell to the cell's nucleus. The second most common mutation occurs in the TP53 gene, which  makes the p53 protein that normally acts as the supervisor in the cell as the body tries to repair damaged DNA. Like many cancers, pancreatic cancers occur as the normal reproductive machinery of the body is taken over by the cancer. 

In the pancreatic cancer problem, for example, the topological space $X$ is comprised of the significant events that mark the progression of the disease, as shown in Table~\ref{cancer-dataset}. In particular, the table shows that specific genes are mutated at specific locations by the change of an amino acid, causing the gene to malfunction. We can model a tumor in terms of its {\em genotype}, namely the subset of $X$, the gene events, that characterize the tumor. For example, the table shows the tumor {\bf Pa022C} can be characterized by the genotype {\bf KRAS}, {\bf SMAD4}, and {\bf TP53}. In general, a finite space topology is just the elements of the space (e.g. genetic events), and the subspaces (e.g., genomes) that define the topology.  

We can now define the problem as one  of inferring topological causal models for cancer  \citep{10.1093/biomet/asp023,cbn,BEERENWINKEL2006409,gerstung}, which define a simple type of categorical structure.  The progression of many types of cancer are marked by mutations of key genes whose normal reproductive machinery is subverted by the cancer \citep{Jones1801}. Often, viruses such as HIV and COVID-19 are constantly mutating to combat the pressure of interventions such as drugs, and successful treatment requires understanding the partial ordering of mutations. A number of past approaches use topological separability constraints on the data, assuming observed genotypes separate events, which as we will show, is abstractly a separability constraint on the underlying topological space.

\begin{table}[htbp]
  \centering
  \caption{Core signaling pathways and processes genetically altered in most pancreatic cancers \citep{Jones1801}.}
    \begin{tabular}{|llll|} \hline
Regulatory pathway	& \% altered genes	& Tumors	& Representative altered genes \\ \hline
Apoptosis &	9	& 100\%	& CASP10, VCP, CAD, HIP1 \\ \hline 
DNA damage control	& 9 &	83\% &	ERCC4, ERCC6, EP300, TP53 \\ \hline
G1/S phase transition	& 19 &	100\% &	CDKN2A, FBXW7, CHD1, APC2 \\ \hline
Hedgehog signaling	& 19	& 100\%	& TBX5, SOX3, LRP2, GLI1, GLI3\\ \hline 
Homophilic cell adhesion &	30	& 79\%	& CDH1, CDH10, CDH2, CDH7, FAT\\ \hline 
Integrin signaling	& 24 & 	67\%	& ITGA4, ITGA9, ITGA11, LAMA1 \\ \hline
c-Jun N-terminal kinase &	9 & 96\%	& MAP4K3, TNF, ATF2, NFATC3 \\ \hline
KRAS signaling	& 5	& 100\%	& KRAS, MAP2K4, RASGRP3 \\ \hline
Regulation of invasion	& 46	& 92\%	& ADAM11, ADAM12, ADAM19\\ \hline 
Small GTPase–dependent & 33 &	79\%	& AGHGEF7, ARHGEF9, CDC42BPA\\ \hline 
TGF-$\beta$ signaling &	37	& 100\%	& TGFBR2, BMPR2, SMAD4, SMAD3 \\  \hline 
Wnt/Notch signaling	& 29 &	100\% &	MYC, PPP2R3A, WNT9A\\ \hline 
\end{tabular}
\label{pathways}
\end{table}

A key computational level in making model discovery tractable in evolutionary processes, such as pancreatic cancer, is that multiple sources of information are available that guide the discovery of the underlying poset model. In particular, for pancreatic cancer \citep{Jones1801}, in addition to the tumor genotype information show in Table~\ref{cancer-dataset}, it is also known that the disease follows certain pathways, as shown in Table~\ref{pathways}. This type of information from multiple sources gives the ability to construct multiple posets that reflect different event constraints \citep{BEERENWINKEL2006409}. Algorithm 1 is a generalization of past algorithms that  infer conjunctive Bayesian networks (CBN)  from a dataset of events (e.g., tumors or signaling pathways) and their associated genotypes (e.g., sets of genes) \citep{cbn,BEERENWINKEL2006409}  The pathway poset and DAG shown in Figure~\ref{cancer-dataset} were learned using Algorithm 1 using the pancreatic cancer dataset published in \citep{Jones1801}. 

\begin{algorithm}[t]
\caption{Application of Poset Discovery Algorithm to Bioinformatics.}
\SetAlgoLined
\KwIn{Dataset ${\cal D} = ({\cal E}, {\cal T})$ of a finite set of events ${\cal E}$ and their associated {\em genotypes} $U \in {\cal T}$, represented as a topological space of open sets $g \subset {\cal E}$. Here, it is assumed that each genome is an intervention target, whose size will affect the complexity of each causal experiment. A conditional independence oracle is also assumed.}
\KwOut{	Causally faithful poset ${\cal P}$ that is consistent with conditional independences in data.}
\Begin{
    Set the basic closed sets $F_e \leftarrow \emptyset$ \\ 
	\Repeat{convergence}{
		Select an open separating set $g \in {\cal T}$ and intervene on $g$. \\
	\For{$e \in g, f \notin g$}{
	 Use samples and the CI oracle to test if $(e \bigCI f)_{{\cal M}_{g}}$ on dataset ${\cal D}$. \\ 
	 If CI test fails, then set $F_e \leftarrow F_e \cup \{ f \}$ because $f$ is an ancestor of $e$.
	 }
	 }
	Define the relation $e \leq f$ if $f \in F_e$, for all $e, f \in {\cal E}$, and compute its transitive closure.  \\ 
	Return the poset ${\cal P} = ({\cal E}, \leq)$, where $\leq$ is the induced relation on the poset ${\cal P}$. 
}
\end{algorithm}

\begin{figure}[t]
\centering
\begin{minipage}{1\textwidth}
\centering
\includegraphics[angle=90,width=4in,height=5in]{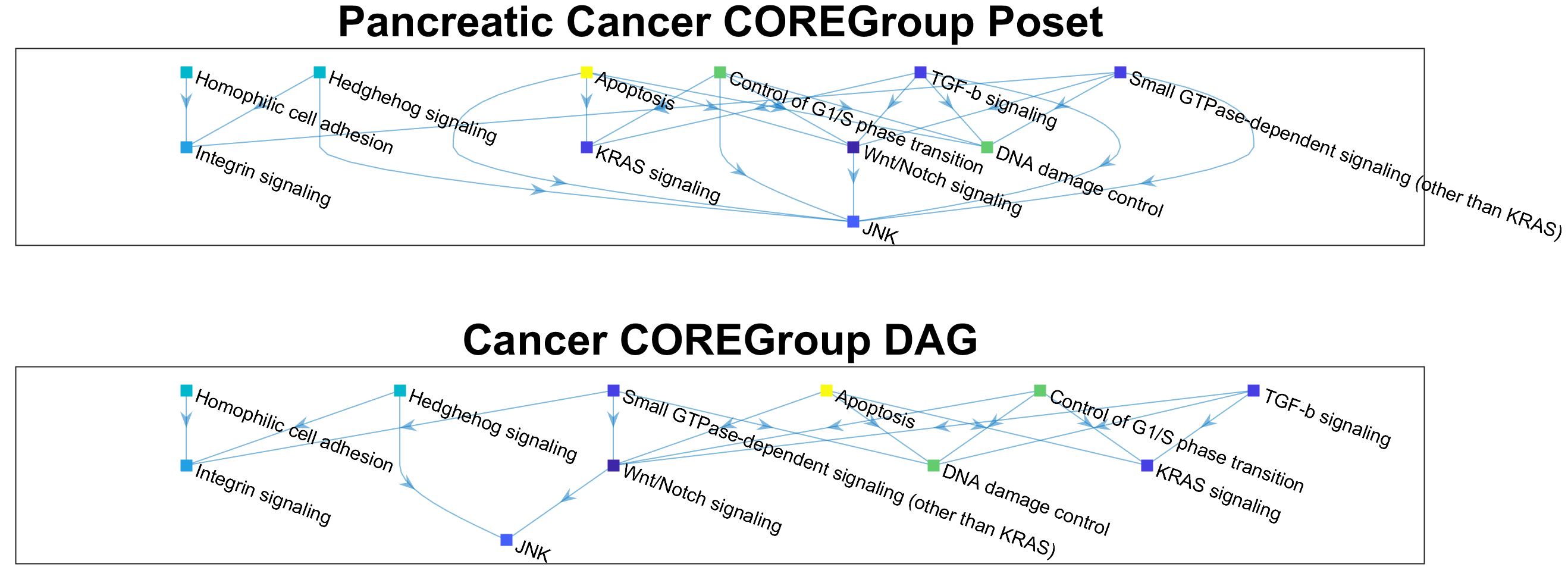} 
\end{minipage}
\caption{Poset and DAG model of pathways in pancreatic cancer learned from a real-world dataset \citep{Jones1801}, showing genetic mutations occur along distinct pathways.}
\label{covid-diagram3}
\end{figure}

\subsection{Game Theory and Variational Inequalities}

Having covered some of the basics in evolutionary dynamics, we now introduce the basics of game theoretic models. Our goal is to integrate these, and in particular, we want to focus on variational inequality models (VIs) that generalize classic game-theoretic models. Game theory was pioneered by von Neumann and Morgenstern \cite{vonneumann1947}, and later extended by Nash \cite{nash}. An excellent modern overview of game theory is given in \cite{Maschler_Solan_Zamir_2013}. Game theory has found countlessa applications in many fields, ranging from network economics \cite{nagurney:vibook} to algorithmic game theory in computer science \cite{nisan07}. Our aim in this section is to illustrate the use of category theory in developing a framework for evolutionary UIGs, principally combining ideas from traditional game theory, network economics, and evolutionary game theory \cite{nowak}. One way to relate this section to the previous is that humans, the target of an imitation game, are the result of millions of years of evolution. In building an AI system that is intended to match humans at imitation games, the role of evolutionary processes cannot be ignored, and in fact, may provide an elegant avenue to building more robust generative AI systems. We saw in the previous section that there are intrinsic limitations on the power of machine learning, which limit the ability of generative AI systems like LLMs at solving even relatively simple classes of tasks \cite{chomsy-neural}. Perhaps bringing in an evolutionary perspective will help transcend some of these limitations, as it has in biology. 

 A finite, $n$-person normal form game is a tuple $(N,A,U)$, where $N$ is a finite set of $n$ players indexed by $i$, $A = A_1 \times \cdots \times A_n$ is the joint action space formed from the set actions available to each player ($A_i$), and $U$ is a tuple of the players' utility functions (or payoffs) $u_i$ where $u_i: A \rightarrow \mathbb{R}$.  The difficulty in computing the equilibrium depends on the constraints placed on the game.  For instance, two-player, zero-sum games, ensure that player interests are diametrically opposed and can thus be formulated as linear programs (LPs) by the minmax theorem and solved in polynomial time \cite{shoham2008multiagent}.

By contrast, in two-player, \textit{general}-sum games, any increase in one player's utility does not necessarily result in a decrease in the other player's utility so a convenient LP formulation is not possible.  Finding Nash equilibria in two-player, general-sum games is thought to be time exponential in the size of the game in the worst case.  It has been shown that every game has at least one Nash equilibrium which delegates the problem to the class PPAD (polynomial parity argument, directed version) originally designated by Papadimitriou \cite{papadimitriou2001algorithms}.  Although this game type cannot be converted to an LP, it can be formulated as a linear complimentarily problem (LCP).  In crude terms, the LCP can be formed by introducing an additional constraint called a complementarity condition to the combination of constraints that would appear in each agent's LP had it only been a zero-sum game.  Unlike the LP, the LCP is only composed of constraints making it a pure constraint satisfaction problem (CSP).  The most popular game theoretic algorithm for solving these LCPs is the Lemke-Howson algorithm \cite{lemke1964equilibrium}.  This algorithm performs a series of pivot operations that swap out player strategies until all constraints are satisfied.  An alternate approach is to employ heuristics as in the case of the support-enumeration method (SEM) which repeatedly tests whether a Nash equilibrium exists given a pair of actions, or {\em support profile}.  The heuristic used is to favor testing smaller, more balanced support profiles in order to prune larger regions of the action space.

Finally, we encounter \textit{n}-player, general-sum games, in which the complementarity problem previously defined is now nonlinear (NCP).  One common approach is to approximate the solution of the NCP as solving a {\em sequence} of LCPs (SLCP).  This method is typically fast, however, it is not globally convergent.  Another technique is to solve an optimization problem in which the global minima equate to the Nash equilibria of the original problem.  The drawback is that there are local minima that do not correspond to Nash equilibria making global convergence difficult.  Several other algorithms for solving n-player, general-sum games exist as well including a generalized SEM but also simplical subdivision algorithms \cite{van1987simplicial}, the Govindan-Wilson method \cite{govindan2003global}, and other homotopy methods.

Some games exhibit a characteristic of {\em payoff independence} where a single player's payoff is dependent on only a subset of the other players in the game.  In this case, the reward table indexed by the joint action-space of all players is overly costly prompting a move from the normal form representation of the game to the more compact representation offered by graphical games.  This can often reduce the space of the representation from exponential to polynomial in the number of players.  When the graph is a tree, a commonly used method, NashProp, computes an $\epsilon$-Nash equilibrium with a back and forth sweep over the graph from the leaves to the root.

Still other representations and methods exist for extensive-form games, however, we will limit our focus here to normal-form games.  There is also a large body of work on the convergence to Nash equilibria in the multi-agent reinforcement learning (MARL) setting \cite{busoniu2008comprehensive, abdallah2008multiagent, singh2000nash} that we will not explore in this paper.  A discussion of the approaches above and more on the computation of equilibria in games can be found in \cite{shoham2008multiagent, von2002computing, mckelvey1996computation}.

There has been previous work on using category theory to model games, which has been referred to as {\em open games} due to its compositional nature \cite{ghani2018compositional,hedges2017morphisms}. The basic framework adopted here is to define games as a symmetric monoidal category, where the compositional structure permits the transmission of utilities backwards, and permits games to be composed. It is an elegant approach, but differs from our primary motivation of incorporating {\em evolutionary} birth-death processes into game theory. Fundamentally, evolution causes players in a game to be eliminated, which is an aspect of all real-life competitive games, from chess competitions to the Olympics, and competitive sports of any kinds features elimination tournaments. It is also a reality in the business world that companies that do not adapt to market changes go bankrupt and are no longer viable players in a network economy. To model such processes, we need to include additional components that are not in the open games formalism, although as we show below, they could be included in such categorical formulations of games. 

\subsection*{Stochastic Variational Inequalities}

We now  introduce classical variational inequalities (VIs) \cite{facchinei-pang:vi}, and outline a metric coinduction type algorithm for stochastic VIs based on \cite{iusem,DBLP:journals/mp/WangB15}. VIs generalize both classical game theory as well as optimization problems, in that any (convex) optimization problem or any Nash game can be formulated as a VI, but the converse is not true. More precisely, a variational inequality model $\cal{M} =$ VI($F,K$), where $F$ is a collection of modular vector-valued functions defined as $F_i$, where $F_i: K_i \subset \mathbb{R}^{n_i} \rightarrow \mathbb{R}^{n_i}$, with each $K_i$ being a convex domain such that $\prod_i K_i = K$. 

\begin{definition}
The category ${\cal C}_{\mbox{VI}}$ of  VIs is defined as one where each object is defined  as a 
finite-dimensional  variational inequality problem ${\cal M}$ = VI($F, K)$, where the vector-valued mapping $F: K \rightarrow \mathbb{R}^n$ is a given continuous function, $K$ is a given closed convex set, and $\langle .,.\rangle$ is the standard inner product in $\mathbb{R}^n$, and the morphisms from one object to another correspond to non-expansive functions.  Solving a  VI is defined as finding a vector $x^* = (x^*_1, \ldots, x^*_n) \in K \subset \mathbb{R}^n$ such that
\begin{equation*}
\langle F(x^*), (y - x^*) \rangle \geq 0, \ \forall y \in K
\end{equation*}
\end{definition}
\begin{figure}[t]
\begin{center}
\begin{minipage}[t]{0.45\textwidth}
\includegraphics[width=\textwidth,height=1.25in]{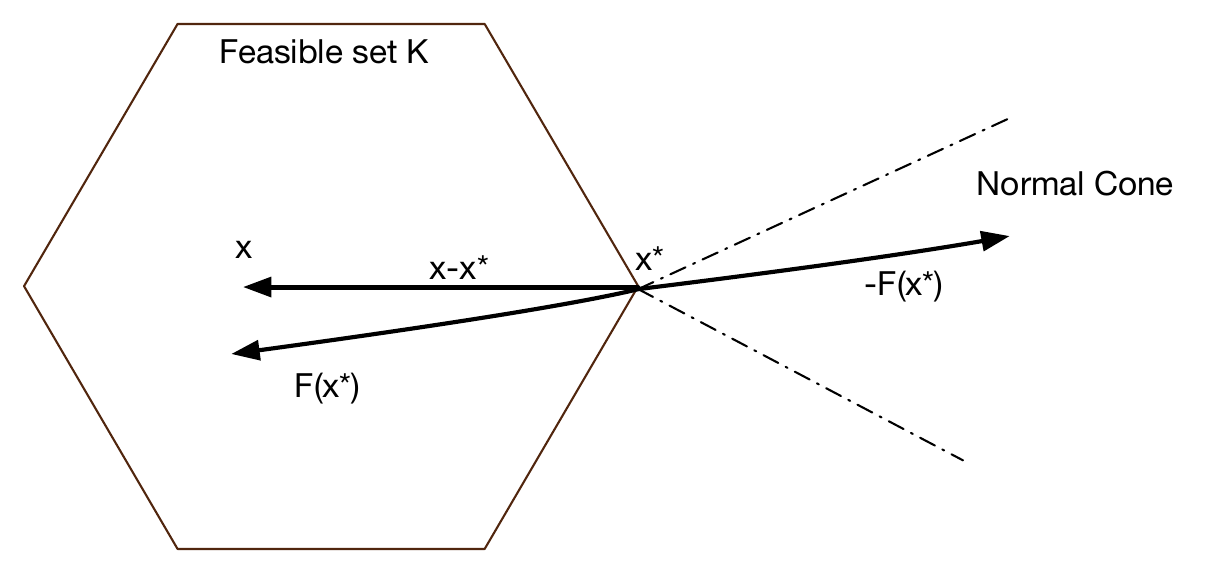}
\end{minipage}
\end{center}
\caption{This figure provides a geometric interpretation of a (deterministic) variational inequality $VI(F,K)$. The mapping $F$ defines a vector field over the feasible set $K$. At the solution point $x^*$, the vector field $F(x^*)$ defines a fixed point, that it, it is directed inwards at the boundary, and  $-F(x^*)$ is an element of the normal cone $C(x^*)$ of $K$ at $x^*$ where the normal cone  $C(x^*)$ at the vector $x^*$ of a convex set $K$ is defined as $C(x^*) = \{y \in \mathbb{R}^n | \langle y, x - x^* \rangle \leq 0, \forall x \in K \}$.}
\end{figure}

We can also define a category of {\em stochastic} VIs as follows. In the stochastic case, we start with a measurable space $({\cal M}, {\cal G})$, a measurable random operator $F: {\cal M} \times \mathbb{R}^n \rightarrow \mathbb{R}^n$, and a random variable $v: \Omega \rightarrow {\cal M}$ defined on a probability space $ (\Omega, {\cal F}, \mathbb{P})$, which enables defining the usual notions of expectation and distributions $P_v$ of $v$. 

\begin{definition}
The category ${\cal C}_{\mbox{SVI}}$ of  stochastic VIs is defined as one where each object is defined  as a  finite-dimensional  stochastic variational inequality problem ${\cal M}$ = SVI($F, K)$, where the vector-valued mapping $F: K \rightarrow \mathbb{R}^n$ is given by $T(x) = E[F(\xi, x)]$ for all $x \in \mathbb{R}^n$. 
\end{definition}

\subsection{Example of VI: Network Economics of Generative AI}

\begin{figure}[h]
\centering
\includegraphics[scale=0.8]{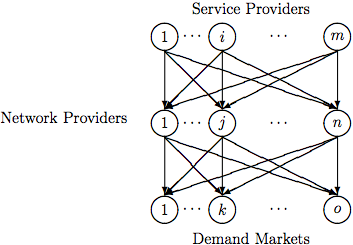}
\caption{Generative AI can be modeled as a network economy that comprises the {\em producers}, who create and sell the generative AI cloud computing services; the {\em transport} providers who control the ``plumbing" of the network used to transport the ``bits" of the generative AI products;  and the {\em users} forming the demand markets who must pay for the generative AI services and the network usage costs. For successful trade to occur in a generative AI network economy, an equilibrium condition must be met, which can be modeled as a variational inequality (VI) \cite{facchinei-pang:vi}.  this figure, a generic network economy is shown with producers forming the layer on top who compete with each other on prices, the middle layer of nodes corresponds to the network providers who compete with each other on price and service, and the bottom layer represents the demand market of users. The network's equilibrium is a complex game-theoretic dynamics, which we analyze in this paper as a coalgebra over categories. This example illustrates an instance of an evolutionary UIG.} 
\label{soifigure} 
\end{figure}

We can use the framework of category theory to also analyze complex network economies that are built on top of generative AI systems. Generative AI systems offer cloud computing services to users geographically dispersed around the world. They charge for these services based on factors such as the number of tokens processed in a large language model, or the number of images rendered using a diffusion model. These services incur significant costs for the vendors of generative AI models. A fundamental principle of network economics is that for successful trade to occur, there must be an equilibrium principle that is satisfied on the network between the {\em producers}, the {\em transporters} who control the ``plumbing" of the network itself, and the {\em users} who form the demand market for the generative AI services. In other words, the producers selling cloud computing generative AI services must be able to earn sufficient money by selling their products to continue to offer them. The users who must both pay for the generative AI services and for using the transport network choose the best combination of generative AI producer and network transporter that is within their budget. 

Formally, it can be shown that the network economics of generative AI is an example of a broad class of mathematical problems called {\em variational inequalities} (VIs) \cite{nagurney:vibook,facchinei-pang:vi}. A novel aspect of our paper is the use of category theory, in particular universal coalgebras, to model VIs, and the use of this framework to model the network economics of generative AI. The solution to a VI is unique, under the following conditions, where $K$ is compact and $F$ is continuous. 
\begin{definition}
$F(x)$ is {\em monotone} if $\langle F(x) - F(y), x - y \rangle \ge 0$, $\forall x, y \in K$.
\end{definition}

\begin{definition}
$F(x)$ is {\em strongly monotone} if $\langle F(x) - F(y), x  - y \rangle \geq \mu \| x - y \|^2_2, \mu > 0, \forall x,y \in K$.
\end{definition}
\begin{definition}
$F(x)$ is {\em Lipschitz} if $\| F(x) - F(y) \|_2 \leq L \|x - y \|_2, \forall x,y \in K$.
\end{definition}

We now describe how to model causal inference in a network economics problem, which will be useful in illustrating the abstract definitions from the previous section. The model in Figure~\ref{soifigure} is drawn from \cite{nagurney:vibook,nagurney:soi}. which were  deterministic, and included no analysis of causal interventions. This network economics model comprises of three tiers of agents: producer agents, who want to sell their goods, transport agents who ship merchandise from producers, and demand market agents interested in purchasing the products or services. The model applies both to electronic goods, such as video streaming, as well as physical goods, such as face masks and other PPEs. Note that the design of such an economic network requires specifying the information fields for every producer, transporter and consumer. For the sake of brevity, we assume that the definition of these information fields are implicit in the equations defined below, but a fuller discussion of this topic will be studied in a subsequent paper. 

The model assumes $m$ service providers, $n$ network providers, and $o$ demand markets.  Each firm's utility function is defined in terms of the nonnegative service quantity (Q), quality (q), and price ($\pi$) delivered from service provider $i$ by network provider $j$ to consumer $k$.  Production costs, demand functions, delivery costs, and delivery opportunity costs are designated by $f$, $\rho$, $c$, and $oc$ respectively.  Service provider $i$ attempts to maximize its utility function $U_i^1(Q,q^*,\pi^*)$ by adjusting $Q_{ijk}$ (eqn. \ref{U1}).  Likewise, network provider $j$ attempts to maximize its utility function $U_j^2(Q^*,q,\pi)$ by adjusting $q_{ijk}$ and $\pi_{ijk}$ (eqn. \ref{U2}).

\begin{subequations}
\begin{align}
\label{U1}
U_i^1(Q,q^*,\pi^*) &= \sum_{j=1}^n \sum_{k=1}^o \hat{\rho}_{ijk}(Q,q^*)Q_{ijk} - \hat{f}_i(Q)\\
&- \sum_{j=1}^n \sum_{k=1}^o \pi^*_{ijk}Q_{ijk}, \hspace{0.2cm} Q_{ijk} \ge 0 \nonumber
\end{align}
\begin{align}
\label{U2}
U_j^2(Q^*,q,\pi) = &\sum_{i=1}^m \sum_{k=1}^o \pi_{ijk}Q^*_{ijk}\\
- &\sum_{i=1}^m \sum_{k=1}^o (c_{ijk}(Q^*,q) + oc_{ijk}(\pi_{ijk})), \nonumber \\
&q_{ijk}, \pi_{ijk} \ge 0 \nonumber
\end{align}
\end{subequations}

We assume the governing equilibrium is Cournot-Bertrand-Nash and the utility functions are all concave and fully differentiable.  This establishes the equivalence between the equilibrium state we are searching for and the variational inequality to be solved where the $F$ mapping is a vector consisting of the negative gradients of the utility functions for each firm.  Since $F$ is essentially a concatenation of gradients arising from multiple independent, conflicting objective functions, it does not correspond to the gradient of any single objective function.  

\begin{subequations}
\begin{align}
\label{SOI-vi}
&\langle F(X^*),X-X^* \rangle \ge 0, \forall X \in \mathcal{K},\\
\text{where } \hspace{0.2cm} &X = (Q,q,\pi) \in \mathbb{R}^{3mno+} \nonumber \\
\text{and } \hspace{0.35cm} &F(X) = (F^1_{ijk}(X), F^2_{ijk}(X), F^3_{ijk}(X)) \nonumber
\end{align}
\begin{align}
F^1_{ijk}(X) &= \frac{\partial f_i (Q)}{\partial Q_{ijk}} + \pi_{ijk} - \rho_{ijk} - \sum_{h=1}^n \sum_{l=1}^o \frac{\partial \rho_{ihl} (Q,q) }{\partial Q_{ijk}} \times Q_{ihl} \label{F1} \\
F^2_{ijk}(X) &= \sum_{h=1}^m \sum_{l=1}^o \frac{\partial c_{hjl} (Q,q)} {\partial q_{ijk}} \label{F2} \\
F^3_{ijk}(X) &= -Q_{ijk} + \frac{\partial oc_{ijk}(\pi_{ijk})}{\partial \pi_{ijk}} \label{F3}
\end{align}
\end{subequations}

The variational inequality in Equations~\ref{SOI-vi} represents the result of combining the utility functions of each firm into standard form.  $F^1_{ijk}$ is derived by taking the negative gradient of $U_i^1$ with respect to $Q_{ijk}$.  $F^2_{ijk}$ is derived by taking the negative gradient of $U_j^2$ with respect to $q_{ijk}$.  And $F^3_{ijk}$ is derived by taking the negative gradient of $U_j^2$ with respect to $\pi_{ijk}$.

\subsubsection{Numerical Example} 

We extend the simplified numerical example in \cite{nagurney:soi} by adding stochasticity to illustrate our causal variational formalism. Let us assume that there are two service providers, one transport agent, and two demand markets. Define the production cost functions:
\[ f_1(Q) = q^2_{111} + Q_{111}  + \eta_{f_1}, f_2(Q) = 2 Q^2_{111} + Q_{211} + \eta_{f_2} \] 
where $\eta_{f_1}, \eta_{f_2}$ are random variables indicating errors in the model. 
Similarly, define the demand price functions as:
\begin{eqnarray*}
 \rho_{111}(Q,q) = -Q_{111} - 0.5 Q_{211} + 0.5 q_{111} + 100 + \eta_{\rho_{111}} \\
 \rho_{211}(Q,q) = -Q_{211} - 0.5 Q_{111} + 0.5 q_{211} + 200 + \eta_{\rho_{211}}
 \end{eqnarray*}
Finally, define the transportation cost functions as:
\begin{eqnarray*}
c_{111}(Q,q) = 0.5(q_{111} - 20)^2 + \eta_{c_{111}} \\ 
c_{211}(Q,q) = 0.5(q_{211} - 10)^2 + \eta_{c_{211}} 
\end{eqnarray*}
and the opportunity cost functions as:
\[ oc_{111}(\pi_{111}) = \pi_{111}^2 + \eta_{oc_{111}}, oc_{211}(\pi_{211}) =\pi_{211}^2 + \eta_{oc_{211}} \] 
Using the above equations, we can easily compute the component mappings $F_i$ as follows: 
\begin{eqnarray*}
F^1_{111}(X) = 4 Q_{111} + 0.5 Q_{211} - 0.5 q_{111} - 99 \\
F^1_{211}(X) = 6 Q_{211} + \pi_{211} - 0.5 Q_{111} - 0.5 q_{211} -199\\
F^2_{111}(X) = q_{111} - 20, \ F^2_{211}(X) = q_{211} - 10 \\
F^3_{111}(X) = -Q_{111} + 2 \pi_{111}, \ F^3_{211}(X) = -Q_{211} + 2 \pi_{211}
\end{eqnarray*}
For simplicity, we have not indicated the noise terms above, but assume each component mapping $F_i$ has an extra noise term $\eta_i$. For example, if we set the network service cost $q_{111}$ of network provider 1 serving the content producer 1 to destination market 1 to 0, then the production cost function under the intervention distribution is given by 
\[ E(f_1(Q) | \hat{q}_{111}) = Q_{111} + E(\eta_{f_1} | \hat{q}_{111})\] 
Finally, the Jacobian matrix associated with $F(X)$ is given by the partial derivatives of each $F_i$ mapping with respect to $(Q_{111}, Q_{211}, q_{111}, q_{211}, \pi_{111}, \pi_{211})$ is given as:
\[ - \nabla U(Q,q,\pi) =  \left( \begin{array}{cccccc} 4 & .5 & -.5 & 0 & 1 & 0 \\
0.5 & 6 & 0 & -.5 & 0 & 1 \\ 0 & 0 & 1 & 0 & 0 & 0 \\ 0 & 0 & 0 & 1 & 0 & 0 \\ -1 & 0 & 0 & 0 & 2 & 0 \\ 0 & -1 & 0 & 0 & 0 & 2\end{array} \right) \]
Note this Jacobian is non-symmetric, but positive definite, as it is diagonally dominant. Hence the induced vector field $F$ can be shown to be strongly monotone, and the induced VI has exactly one solution. 

Crucially, VI problems can only be converted into equivalent optimization problems when a very restrictive condition is met on the Jacobian of the mapping $F$, namely that it be symmetric. Most often, real-world applications of VIs do not induce symmetric Jacobians. 
\begin{theorem}
Assume $F(x)$ is continuously differentiable on $K$ and that the Jacobian matrix $\nabla F(x)$ of partial derivatives of $F_i(x)$ with respect to (w.r.t) each $x_j$ is symmetric and positive semidefinite. Then there exists a real-valued convex function $f: K \rightarrow \mathbb{R}$ satisfying $\nabla f(x) = F(x)$ with $x^*$, the solution of  VI(F,K), also being the mathematical programming problem of minimizing $f(x)$ subject to $x \in K$.
\end{theorem}
The algorithmic development of methods for solving VIs begins with noticing their connection to fixed point problems.
\begin{theorem}
The vector $x^*$ is the solution of VI(F,K) if and only if, for any $\alpha > 0$, $x^*$ is also a fixed point of the map  $x^* = P_K(x^* - \alpha F(x^*))$,
where $P_K$ is the projector onto convex set $K$.
\end{theorem}
In terms of the geometric picture of a VI, the solution of a VI occurs at a vector $x^*$ where the vector field $F(x^*)$ induced by $F$ on $K$ is normal to the boundary of $K$ and directed inwards, so that the projection of $x^* - \alpha F(x^*)$ is the vector $x^*$ itself. This property forms the basis for the projection class of methods that solve for the fixed point.

With this insight, we can define a category of coalgebras over VIs, where the endofunctor $F_{VI}$ is defined in terms of the vector field $F$ as defined above, and the solution to the VI corresponds to finding the final coalgebra representing the solution to the VI problem. Together, the VI and PDS frameworks provide a mathematically elegant approach to modeling and solving equilibrium problems in game theory \cite{fudenbergtheory,nisan2007algorithmic}.  A {\em Nash game} consists of $m$ players, where player $i$ chooses a strategy $x_i$ belonging to a closed convex set $X_i \subset \mathbb{R}^n$.  After executing the joint action, each player is penalized (or rewarded) by the amount $f_i(x_1,\ldots,x_m)$, where $f_i: \mathbb{R}^n \rightarrow \mathbb{R}$ is a continuously differentiable function.  A set of strategies $x^* = (x_1^*,\ldots,x_m^*) \in \Pi_{i=1}^M X_i$ is said to be in equilibrium if no player can reduce the incurred penalty (or increase the incurred reward) by unilaterally deviating from the chosen strategy.  If each $f_i$ is convex on the set $X_i$, then the set of strategies $x^*$ is in equilibrium if and only if $\langle \nabla_i f_i (x_i^*), (x_i - x_i^*) \rangle \ge 0$.  In other words, $x^*$ needs to be a solution of the VI $\langle F(x^*), (x-x^*) \rangle \ge 0$, where $F(x) = ( \nabla f_1(x), \ldots, \nabla f_m(x))$.

Two-person Nash games are closely related to {\em saddle point} optimization problems \cite{juditsky2011first,juditsky2011second,liu2012regularized} where we are given a function $f: X \times Y \rightarrow \mathbb{R}$, and the objective is to find a solution $(x^*,y^*) \in X \times Y$ such that \begin{equation} f(x^*,y) \le f(x^*,y^*) \le f(x,y^*), \forall x \in X, \forall y \in Y. \end{equation}  Here, $f$ is convex in $x$ for each fixed $y$, and concave in $y$ for each fixed $x$. 

The class of complementarity problems can also be reduced to solving a VI. When the feasible set $K$ is a cone, meaning that if $x \in K$, then $\alpha x \in K, \alpha \geq 0$, then the VI becomes a CP.
\begin{definition}
Given a cone $K \subset \mathbb{R}^n$ and mapping $F: K \rightarrow \mathbb{R}^n$, the complementarity problem CP(F,K) is to find an $x \in K$ such that $F(x) \in K^*$, the dual cone to $K$, and $\langle F(x), x \rangle \geq 0$. \footnote{Given a cone $K$, the dual cone $K^*$ is defined as $K^* = \{ y \in \mathbb{R}^n | \langle y, x \rangle \geq 0, \forall x \in K \}$.}
\end{definition}

The nonlinear complementarity problem (NCP) is to find $x^* \in \mathbb{R}^n_+$ (the non-negative orthant) such that $F(x^*) \geq 0$ and $\langle F(x^*), x^* \rangle = 0$. The solution to an NCP and the corresponding $VI(F, \mathbb{R}^n_+)$ are the same, showing that NCPs reduce to VIs. In an NCP, whenever the mapping function $F$ is affine, that is $F(x) = Mx + b$, where $M$ is an $n \times n$ matrix, the corresponding NCP is called a linear complementarity problem (LCP) \cite{murty:lcpbook}.

%
%
%

\subsection{A Metric Coinduction Algorithm for solving VIs}

 We now discuss algorithms for solving VI's. There are a wealth of existing methods for deterministic VI's \citep{facchinei-pang:vi,nagurney:vibook}).  First, we present a few of the standard algorithms used to compute solutions to variational inequalities and equilibria in projected dynamical systems. We finally describe an adaptation of a stochastic approximation method \cite{rm} to solve VIs, which can be viewed as a special type of metric coinduction method. 

The basic projection-based method (Algorithm 1) is as follows:
\begin{algorithm}
\caption{The Basic Projection Algorithm.}
{\bf INPUT:} Given VI(F,K), and a symmetric positive definite matrix $D$.
\begin{algorithmic}[1]
\STATE Set $k=0$ and $x_k \in K$.
\REPEAT
\STATE Set $x_{k+1} \leftarrow P_{K} (x_k - \alpha D^{-1} F(x_k))$.
\STATE Set $k \leftarrow k+1$.
\UNTIL{$x_k = P_{K} (x_k - \alpha D^{-1} F(x_k))$}.
\STATE Return $x_k$
\end{algorithmic}
\end{algorithm}
Here, $P_{K}$ is the orthogonal projector onto the convex set $K$. It can be shown that the basic projection algorithm solves any $VI(F,K)$ for which the mapping $F$ is strongly monotone and Lipschitz smooth.  A simple strategy is to set $D = I$ where $\alpha < \frac{L^2}{2 \mu}$, $L$ is the Lipschitz smoothness constant, and $\mu$ is the strong monotonicity constant.  Setting $D$ equal to a constant in this manner recovers what is known as Euler's method and is the most basic algorithm for solving both VIs and PDS.

The basic projection-based algorithm has two critical limitations.  First, it requires that the mapping $F$ be strongly monotone. If, for example, $F$ is the gradient map of a continuously differentiable function, strong monotonicity implies the function must be strongly convex.  Second, setting the parameter $\alpha$ requires knowing the Lipschitz smoothness $L$ and the strong monotonicity parameter $\mu$.

\citet{korpelevich} extended the projection algorithm with the well-known ``extragradient" method, which requires two projections, but is able to solve VI's for which the mapping $F$ is only monotone. The key idea behind the extragradient method is to use the F mapping evaluated at the result of the projection method, $x_{k+1}$, instead of $x_k$.  In other words, step 3 of Algorithm 2 is replaced with
\begin{align*}
\text{Set } &y_{k} \leftarrow P_{K} (x_k - \alpha  F(x_k))\\
\text{Set } &x_{k+1} \leftarrow P_{K} (x_k - \alpha F(y_k)).
\end{align*}

\begin{algorithm}[h]
\caption{The Extragradient Algorithm.}

{\bf INPUT:} Given VI(F,K), and a scalar $\alpha$.

\begin{algorithmic}[1]

\STATE Set $k=0$ and $x_k \in K$.

\REPEAT

\STATE  Set $y_{k} \leftarrow P_{K} (x_k - \alpha  F(x_k))$.

\STATE \label{projstep} Set $x_{k+1} \leftarrow P_{K} (x_k - \alpha F(y_k))$.

\STATE Set $k \leftarrow k+1$.

\UNTIL{$x_k = P_{K} (x_k - \alpha F(x_k))$}.

\STATE Return $x_k$

\end{algorithmic}
\end{algorithm}

\citet{GempM015} proposed a more sophisticated extragradient algorithm, combining Runge-Kutta methods from ODE's with a modified dual-space mirror-prox method  to solve large network games modeled as VI's, such as the network economy shown in Figure~\ref{soifigure}, but required multiple projections corresponding to the order of the Runge-Kutta approximation. If projections are expensive, particularly in large network economy models, these algorithms may be less attractive than incremental stochastic projection methods, which we turn to next. 

 \subsection*{Convergence Analysis of VI Algorithms} 
 
 At the heart of convergence analysis of any VI method is bounding the iterates of the algorithm. In the below derivation, $x^*$ represents the final solution to a VI, and $x_{k+1}$, $x_k$ are successive iterates: 
\begin{small}
\begin{eqnarray*}
\| x_{k+1} - x^* \|^2 &=& \| P_K[x_k - \alpha_k F_w(x_k)] - P_K[x^* - \alpha_k F_w(x^*)] \|^2 \\
&\leq& \| (x_k - \alpha_k F_w(x_k)) - (x^* - \alpha_k F_w(x^*)) \|^2 \\
&=& \| (x_k - x^*) - \alpha_k (F_w(x_k) - F_w(x^*)) \|^2 \\
&=& \|x_k - x^* \|^2 - 2 \alpha_k \langle (F_w(x_k) - F_w(x^*)), x_k - x^* \rangle \\ &+& \alpha_k^2 \|F_w(x_k) - F_w(x^*)\|^2 \\
&\leq& (1 - 2 \mu \alpha_k + \alpha_k^2 L^2) \| x_k - x^*\|^2
\end{eqnarray*} 
\end{small}
Here, the first inequality follows from the nonexpansive property of projections, and the last inequality follows from strong monotonicity and Lipschitz property of the $F_w$ mapping. Bounding the term $\langle (F_w(x_k) - F_w(x^*)), x_k - x^* \rangle$ is central to the design of any VI method.

\subsection*{A Stochastic Approximation Method for Stochastic VIs} 

We now describe an incremental two-step projection method for solving (stochastic) VI's, which is loosely based on \citet{DBLP:journals/mp/WangB15,iusem}. The general update equation can be written as: 
\begin{equation}
\label{2stepalg}
z_k = x_k - \alpha_k F_w(x_k, v_k), \ \ \ x_{k+1} = z_k - \beta_k (z_k - P_{w_k} z_k)
\end{equation}
where $\{v_k\}$ and $\{w_k\}$ are sequences of random variables, generated by sampling a VI model, and $\{\alpha_k\}$ and $\{\beta_k\}$ are sequences of positive scalar step sizes. Note that an interesting feature of this algorithm is that the sequence of iterates $x_k$ is not guaranteed to remain within the feasible space $K$ at each iterate. Indeed, $P_{w_k}$ represents the projection onto a randomly sampled constraint $w_k$. 

The analysis of convergence of this algorithm is somewhat intricate, and we will give the broad highlights as it applies to stochastic VI's. Define the set of random variables ${\cal F}_k = \{v_0, \ldots, v_{k-1}, w_0, \ldots, w_{k-1}, z_0, \ldots, z_{k-1}, x_0, \ldots, x_k \}$. Similar to the convergence of the projection method, it is possible to show that the error of each iteration is {\em stochastically contractive}, in the following sense: 
\[ E[ \| x_{k+1} - x^* \|^2 | {\cal F}_k ] \leq (1 - 2 \mu_k \alpha_k + \delta_k) \| x_k - x^* \|^2 + \epsilon_k, \ \ \ \mbox{w. p. 1} \]
where $\delta_k, \epsilon_k$ are positive errors such that $\sum_{k=0}^\infty \delta_k < \infty$ and $\sum_{k=0}^\infty \epsilon_k < \infty$. Note that the assumption of stochastic contraction is essentially what makes this algorithm an example of metric coinduction \cite{kozen}. 
The convergence of this method rests on the following principal assumptions, stated below: 

\begin{assumption}
The mapping $F_w$ is strongly monotone, and the sampled mapping $F_w(.,v)$ is {\em stochastically Lipschitz continuous} with constant $L > 0$, namely:
\begin{equation}
 E[ \| F_w(x, v_k) - F_w(y, v_k) \|^2 | {\cal F}_k] \leq L^2 \| x - y \|^2. \forall x, y \in \mathbb{R}^n
 \end{equation}
\end{assumption}

\begin{assumption}
The  mapping $F_w$ is bounded, with constant $B > 0$ such that 
\begin{equation}
\| F_w(x^*) \| \leq B, \ \ \ E[\| F_w(x^*, v) \|^2 | {\cal F}_k] \leq B^2, \ \ \forall k \geq 0
\end{equation}
\end{assumption}

\begin{assumption}
The distance between each iterate $x_k$ and the feasible space $K$ reduces ``on average", namely: 
\begin{equation}
\| x - P_K \|^2 \geq \eta \max_{i \in M} \|x - P_{K_i} x\|^2
\end{equation}
where $\eta > 0$ and $M = \{1, \ldots, m \}$ is a finite set of indices such that $\prod K_i = K$. 
\end{assumption}
\begin{assumption}
\[ \sum_{k=0}^\infty \alpha_k = \infty, \sum_{k=0}^\infty \alpha_k^2 < \infty, \sum_{k=0}^\infty \frac{\alpha_k^2}{\gamma_k} < \infty \]
where $\gamma_k = \beta_k (2 - \beta_k)$. 
\end{assumption}
\begin{assumption}
Supermartingale convergence theorem: Let ${\cal G}_k$ denote the collection of nonnegative random variables $\{y_k\}, \{u_k \}, \{a_k\}, \{b_k\}$ from $i=0, \ldots, k$
\begin{equation}
E[y_{k+1} | {\cal G}_k] \leq (1 + a_k) y_k - u_k + b_k, \ \ \forall k \geq 0, \ \mbox{w.p. 1}
\end{equation}
and $\sum_{k=0}^\infty a_k < \infty$ and $\sum_{k=0}^\infty b_k < \infty$ w.p. 1. Then, $y_k$ converges to a nonnegative random variable, and $\sum_{k=0}^\infty u_k < \infty$.
\end{assumption}
\begin{assumption}
The random variables $w_k, k=0, \ldots$ are such that for $\rho \in (0, 1]$
\[ \inf_{k\geq 0} P(w_k = X_i | {\cal F}_k) \geq \frac{\rho}{m}, \ \ i=1, \ldots, m, \ \mbox{w.p.1} \]
namely, the individual constraints will be sampled sufficiently. Also, the sampling of the stochastic components $v_k, k=0, \ldots$ ensures that
\[ E[F_w(x_k, v_k) | {\cal F}_k] = F_w(x), \ \ \forall x \in \mathbb{R}^n, \ k \geq 0 \]
\end{assumption}

Given the above assumptions, it can be shown that two-step stochastic algorithm given in Equation~\ref{2stepalg} converges to the solution of a (stochastic) VI. 
\begin{theorem}
Given a finite-dimensional stochastic  variational inequality problem is defined by a model ${\cal M}$ = SVI($F, K)$,  where $F(x) = E[F(x, \eta]$, where $E[.]$ now denotes expectation, the two-step algorithm given by Equation~\ref{2stepalg} produces a sequence of iterates $x_k$ that converges almost surely to $x^*$, where 
\begin{equation*}
\langle F_w(x^*), (y - x^*) \rangle \geq 0, \ \forall y \in K
\end{equation*}
\end{theorem}
{\bf Proof:} The proof of this theorem largely follows the derivation given in \citep{iusem,DBLP:journals/mp/WangB15}. $\qed$

To reiterate, these stochastic approximation methods for solving (stochastic) VIs are essentially instances of the more general metric coinduction framework \cite{kozen}, and the key to understanding them from a category-theoretic framework is to formulate them in the context of finding final coalgebras given by the fixed point equation that solves a given (stochastic) VI. 

\subsection*{Evolutionary Imitation Games with VIs} 

Finally, we can adapt the VI algorithms described above to the evolutionary setting by inserting in a ``mutation" component where at each step of the evolutionary process, some participant in a network economy goes ``extinct" and is replaced with another participant who is more ``fit". Following the similar strategy used in the discussion of Moran processes above, we can analyze the convergence of such a network economy where participants are randomly selected for extinction and replacement. A full analysis of this problem is beyond the scope of this paper, and is a topic for future work, but we give the high-level idea here. 

In a standard VI, we are given a vector field $F: K \rightarrow \mathbb{R}^n$, which represents a network economy such as the one illustrated in Figure~\ref{soifigure}, and the goal is to find the equilibrium of this network economy game. To define an evolutionary network economy game, we now add the ``birth-death" stochastic process or the mutation process defined above with Moran processes or with bioinformatics example of genetic mutations, or with Valiant's evolvability model. That is, we assume that each step of the evolutionary UIG, we have a set of participants playing the network economy game, and we solve that using (for example) the stochastic approximation algorithm described above. Given a solution to the current network economy game, we apply a mutation operator by eliminating one participant (say as in a Moran process), and then reproduce another participant of higher fitness to keep the population size fixed. This evolutionary step might correspond to a business going bankrupt or being bought over by another higher fitness business. The network economy game is then played again, and the process repeats. The convergence of such evolutionary UIGs is a matter of choosing the right type of incremental mutation operator such that the overall process converges. The principle of metric coinduction provides a broad category-theoretic framework to analyze usch evolutionary UIGs. 

\section{Universal Imitation Games on Quantum Computers} 

\begin{quote}
    ``Nature isn't classical, dammit, and if you want to make a simulation of nature, you'd better make it quantum mechanical, and by golly it's a wonderful problem, because it doesn't look so easy.” -- Richard Feynman
\end{quote}

\label{uigquantum}

Thus far in this paper, we have made an implicit assumption that the solutions to imitation games would be implemented on a classical computer, in particular  Turing's machine. However, the underlying mathematics we have used in analyzing imitation games generalizes elegantly to imitation games that can  and will undoubtedly be played in the near future on quantum computers \cite{DBLP:conf/lics/AbramskyC04,Coecke_Kissinger_2017}. In this final brief section, we want to give a high level description of quantum UIGs, and what they may imply in terms of strategies for playing imitation games. This possibility is not science fiction. \cite{DBLP:phd/ethos/Toumi22} gives a detailed description of how to model natural language processing on quantum computers using category theory in his PhD thesis. He also describes the DiScoPy Python programming environment for converting natural language sentences into symmetric monoidal categories. It is quite feasible, perhaps not yet practical, to design algorithms to play quantum imitation games using packages like DisCoPy \cite{de_Felice_2021}, although in the years ahead, it will become increasingly feasible, and perhaps the staggering energy costs of current generative AI software on conventional computers will leave us no choice but to turn to quantum computers to play imitation games. 

\subsection{Compact Closed Categories}

To begin with, we need to define a model of quantum computation, and along the same lines of study that we have undertaken so far. That is, the first question we have to answer is: what is the underlying category that quantum computing is based on? Figure~\ref{smorgasboard} illustrates what can be referred to as a smorgasbord of categories, a very small sampling of the possible ways in which categories can be defined. We give a brief description of {\em compact closed categories}, which has been proposed as a category for quantum protocols, such as entanglement \cite{DBLP:conf/lics/AbramskyC04}. 

\begin{definition}
    A {\bf symmetric monoidal category} is defined as a category ${\cal C}$ with a bifunctor $\otimes: {\cal C} \times {\cal C} \rightarrow {\cal C}$, a {\em unit object} ${\bf 1}$, and natural isomorphisms 

    \[ \lambda_A: A \simeq 1 \otimes A \ \ \ \rho_A: A \simeq A \otimes 1 \]

    \[ \alpha_{A, B, C}: A \otimes (B \otimes C) \simeq (A \otimes B) \otimes C \]

    \[ \sigma_{A, B}: A \otimes B \simeq B \otimes A \]

which needs to satisfy some additional coherence conditions as defined in \cite{maclane:71}. 
\end{definition}

\begin{figure}
    \centering
    \includegraphics[scale=0.5]{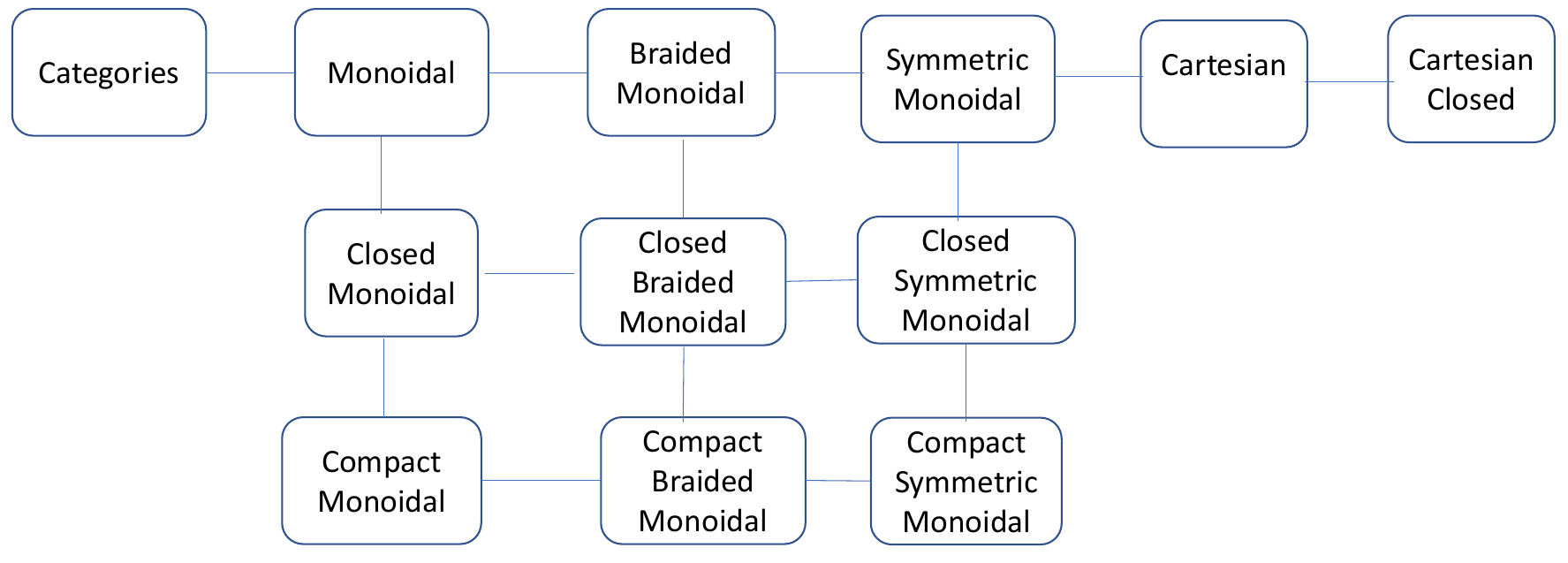}
    \caption{Quantum computing is modeled in {\em compact closed} symmetric monoidal categories.}
    \label{smorgasboard}
\end{figure}

\begin{definition}
    A category {\cal C} is {\em *-autonomous} if it is a symmetric monoidal category with {\em dual objects} represented by a full and faithful functor 

    \[ ()^*: {\cal C}^{op} \rightarrow {\cal C} \]
 satisfying the bijection 

 \[ {\cal C}(A \otimes B, C^*) \simeq {\cal C}(A, (B \otimes C)^*) \]
 that is natural over $A$, $B$, and $C$. A {\bf compact closed category} is an *-autonomous category with a self-dual tensor operator, satisfying the natural isomorphisms 

 \[ u_{A, B}: (A \otimes B)^* \simeq A^* \otimes B^*  \ \ \ u_{\bf 1}: {\bf 1}^* \simeq {\bf 1} \]
\end{definition}

An alternative way to think of compact closed categories is to view them as a symmetric monoidal category where every object $A$ has a {\em dual object}, given by $A^*$, such that it defines a {\em unit} 

\[ \eta_A: {\bf 1} \rightarrow A^* \otimes A \]

and a {\em counit} 

\[ \epsilon_A: A \otimes A^* \rightarrow {\bf 1} \]

satisfying the following commutative diagram: 

\[\begin{tikzcd}
	A &&& {A \otimes {\bf 1}} &&& {A \otimes (A^* \otimes A)} \\
	\\
	A &&& {I \otimes A} &&& {(A \otimes A^*) \otimes A}
	\arrow["{\rho_A}", from=1-1, to=1-4]
	\arrow["{{\bf 1}_A \otimes \eta_A}", from=1-4, to=1-7]
	\arrow["{\alpha_{A,A^*, A}}", from=1-7, to=3-7]
	\arrow["{\epsilon_A \otimes {\bf 1}_A}", from=3-7, to=3-4]
	\arrow["{\lambda_A^{-1}}", from=3-4, to=3-1]
	\arrow[from=1-1, to=3-1]
\end{tikzcd}\]

\begin{example}
    The monoidal category $({\bf Rel}, \times)$ of sets, relations and Cartesian products is a compact closed category. 
\end{example}

\begin{example}
    The monoidal category $({\bf Vec_K}, \otimes)$ of finite-dimensional vector spaces over a field $K$ is a compact closed category. 
\end{example}

\subsection{Quantum Teleportation and Entanglement} 

We want to briefly sketch out why compact closed categories provide a good categorical representation for quantum UIGs. Our discussion below is based on \cite{DBLP:conf/lics/AbramskyC04}. Quantum computing \cite{Coecke_Kissinger_2017} is based on modeling the state space of a system as a (finite-dimensional) Hilbert space ${\cal H}$, which is a finite-dimensional complex vector space with an inner product written in ``bra-ket" notation as $\langle \phi | \psi \rangle$ that is conjugate linear in the first argument and linear in the second. The fact that all computations in an quantum computer are inherently linear is a great advantage, however it is incredibly challenging in almost every other respect, not least of which is the physical challenges associated with preparing an input to a quantum computer, and ensuring that quantum bits (qubits) do not collapse into classical bits while computation is taking place. A {\em quantum state} is defined as a one-dimensional subspace ${\cal A}$ of ${\cal H}$, and represented as a vector $\psi \in {\cal A}$ of unit norm. Formally, qubits are $2$-dimensional Hilbert spaces, on the basis defined by $\{ | 0 \rangle, | 1 \rangle \}$. 

One of the strangest phenomena that occurs in quantum mechanics is {\em entanglement}, which arises in compound systems that can be viewed as tensor products of component subsystems ${\cal H_1}$ and ${\cal H_2}$: 

\[ {\cal H_1} \otimes {\cal H_2}  = \sum_{i=1}^n \alpha \cdot \phi_i \otimes \psi_i \]

where the two subsystems may exhibit correlations that cannot be decomposed into pairs of vectors in each individual subsystem. As a concrete example, illustrated in Figure~\ref{twoslit},  consider the compound system ${\cal H}_1 \otimes {\cal H}_2$ as represented by two slits through which electrons are shot through \cite{Feynman:1494701}. If one of the two slits is covered, then the probability of the electron being detected at a detector placed at the end opposite from where the electrons are being shot behaves as you would expect. The probability of electrons being detected is highest directly opposite the slit and gradually decays as you move away. However, this behavior which can be easily understood for each subsystem ${\cal H}_1$ and ${\cal H}_2$ individually completely changes when both slits are combined in the composite system ${\cal H}_1 \otimes {\cal H}_2$. Now, there is interference between the two slits and the behavior of the electron changes from being that of a particle to a wave. 

\begin{figure}[t]
    \centering
    \includegraphics[scale=0.5]{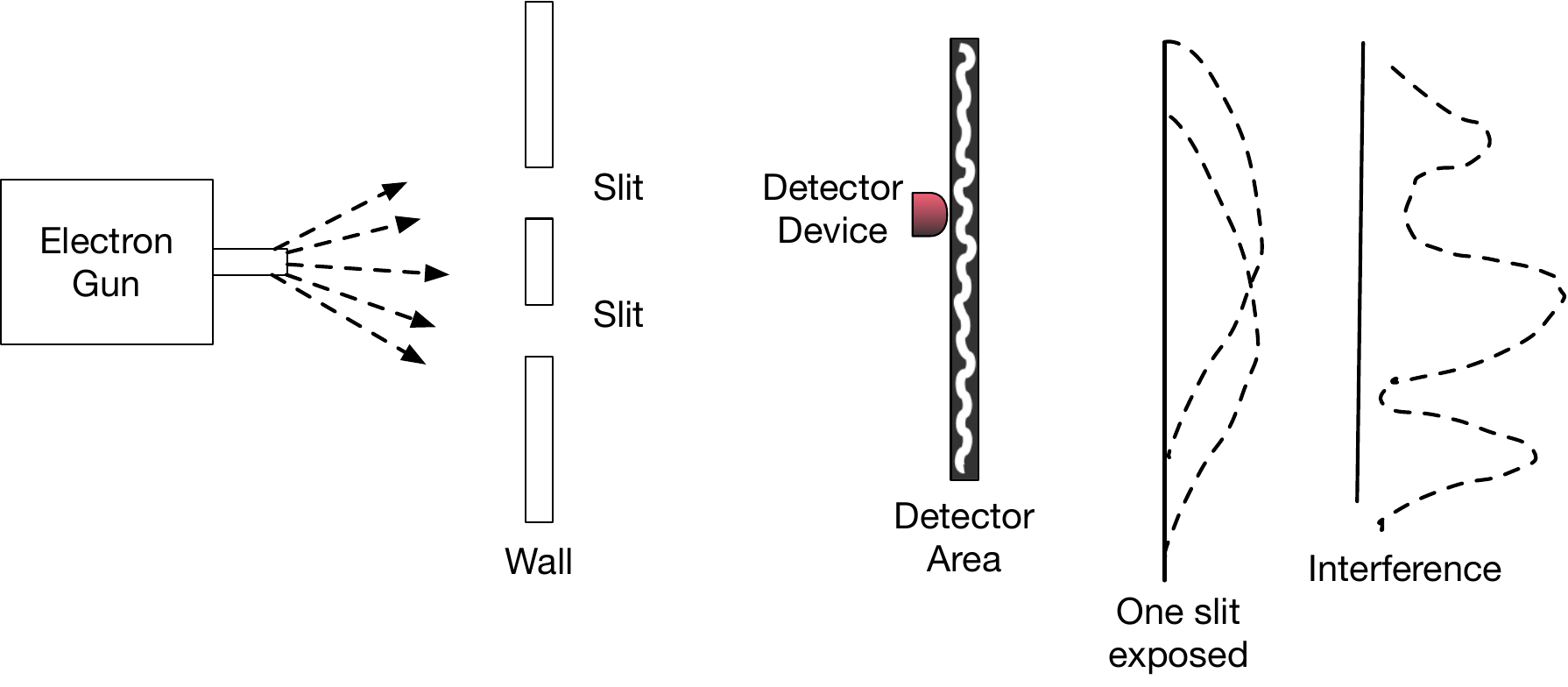}
    \caption{The fundamental essence of what makes quantum computing different from classical computing rests in the property of quantum entanglement, which results from the interference that occurs when quantum systems are composed. In this famous two-slit thought experiment \cite{Feynman:1494701}, the behavior of electrons is completely predictable when the beam of electrons is allowed to pass through just one slit, but when both slits are exposed, the behavior of the composite system is no longer a decomposable function of the individual systems. If we model each individual system in terms of a probability distribution where the electron might be detected, the behavior of the combined system violates the total law of probability, which states that the probability $P(E_1 + E_2) = P(E_1) + P(E_2)$, where $P(E_i)$ is the probability of detecting an electron through slit $i$. This  paradoxical behavior necessitated the use of complex numbers in modeling quantum behavior, and thus probabilistic events in quantum mechanics are represented by complex amplitudes.  }
    \label{twoslit}
\end{figure}

All computations in quantum computers are inherently {\em reversible}, due to the fact that data transformations are represented as unitary transformations. Given a linear map $f: {\cal H}_1\rightarrow {\cal H}_2$ between two Hilbert spaces, the {\em adjoint} is defined as the linear map $f^\dagger: {\cal H}_2 \rightarrow {\cal H}_1$ such that for all $\phi \in {\cal H}_2$ and $\psi \in {\cal H}_1$, 

\[ \langle \phi, f(\psi) \rangle_{{\cal H}_2} = \langle f^\dagger (\phi) | \psi \rangle_{{\cal H}_1}\]

A unitary transformation is a linear isomorphism

\[ U: {\cal H}_1 \rightarrow {\cal H}_2 \]

satisfying the property that $U^{-1} = U^\dagger: {\cal H}_2 \rightarrow {\cal H}_1$. By tbe bilinearity property, these transformations must preserve inner products

\[ \langle U(\phi) | U(\psi) \rangle_{{\cal H}_2} = \langle (U^\dagger U) (\phi) | \psi \rangle_{{\cal H}_1} = \langle \phi | \psi \rangle_{{\cal H}_1} \]

Self-adjoint operators mapping one Hilbert space into another are linear transformations $M: {\cal H} \rightarrow {\cal H}$ such that $M = M^\dagger$. Measurements in quantum computing systems are represented by self-adjoint operators, and can be divided into two phases: 

\begin{enumerate}
    \item An observer receives the outcome of a measurement  outcome as a value $x$ in the spectrum $\sigma(M)$ of the corresponding self-adjoint operator $M$. 

    \item As a result of the measurement, the state of the system undergoes a change, which can be modeled as the action of a projector $P$ whose effect can be understood by a spectral decomposition of the measurement operator

    \[ M = x_1 P_1 + \ldots + x_n P_n \]
\end{enumerate}

The {\em Born rule} states that the probability of $x_i \in \sigma(M)$ does not depend on the value of $x_i$, but rather on $P_i$ and the state of the system $\psi$: 

\[ \mbox{Prob}(P_i, \psi) = \langle \psi | P_i(\psi) \rangle \]

To understand the quantum teleportation protocol, let us consider three qubits $a, b$, and $c$, where $a$ is in state $| \phi \rangle$, and qubits $b$ and $c$ are entangled and form an ``EPR-pair" with their joint state being $|00 \rangle + |11 \rangle$. If $a$ and $b$ are positioned at the source $A$, and $c$ is positioned at the target $B$, we perform a {\em Bell-base} measurement on $a$ and $b$, which means each $P_i$ projects onto one of the one-dimensional subspaces spanned by a vector in the Bell basis: 

\begin{eqnarray*}
    b_1 \coloneqq \frac{1}{\sqrt{2}} \cdot (| 00 \rangle + | 11 \rangle) \ \ \   b_2 \coloneqq \frac{1}{\sqrt{2}} \cdot (| 01 \rangle + | 10 \rangle)  \\
      b_3 \coloneqq \frac{1}{\sqrt{2}} \cdot (| 00 \rangle - | 11 \rangle)  \ \ \   b_4 \coloneqq \frac{1}{\sqrt{2}} \cdot (| 01 \rangle - | 10 \rangle) 
\end{eqnarray*}

We can perform a unitary transformation based on the outcome of the measurement, which physically involves transmission of two classical bits, from the location of $a$ and $b$ to the location of $c$. Quantum teleportation occurs when the final state of $c$ is $| \phi \rangle$ as well. Since a continuous variable has been transmitted, but the actual communication involved the passage of two classical bits, it must be that besides the classical flow of information, there must have been quantum flow of information as well. \cite{DBLP:conf/lics/AbramskyC04} give a detailed analysis of how this type of quantum information flow can be modeled abstractly in a compact closed category, to which we refer the reader for further details. We turn instead to a brief discussion of the issues involved in designing UIG approaches for quantum computing. 

\subsection{Quantum Coalgebras: Generative AI on Quantum Computers} 

As we discussed previously in Section~\ref{staticuig} and Section~\ref{dynamicuig}, generative AI models, from the basic (non)deterministic finite state machine to probabilistic systems to  large language models based on neural \cite{DBLP:conf/nips/VaswaniSPUJGKP17} or dynamical system concepts \cite{DBLP:conf/iclr/GuJTRR23} can all be viewed as universal coalgebras over some category. We can now consider how to define {\em quantum coalgebras} to generalize generative AI over classical computers to quantum computers. 

\cite{MOORE2000275} define quantum versions of finite-state and push-down automata, as well as regular and context-free grammars. We will follow their treatment below, and then relate it to quantum UIGs. 

\begin{definition}
    A {\bf real-time quantum automaton} is defined as: 
    \begin{enumerate}
        \item A Hilbert space ${\cal H}$
        \item an initial state vector $|s_{\mbox{init}}| \in {\cal H}$ with $|s_{\mbox{init}}|^2  = 1$
        \item A subspace $H_{\mbox{accept}} \subset H$ and an operator $P_{\mbox{accept}}$ that projects onto it. 
        \item An input alphabet $A$ 
        \item A unitary transition matrix $U_a$ for each symbol $a \in A$. 
    \end{enumerate}
\end{definition}

Given any string $w \in A^*$, we can define the quantum language accepted by $Q$ as the function 

\[ f_Q(w) = |s_{\mbox{init}} U_w P_{\mbox{accept}}|^2 \]

which defines a function from strings in $A^*$ to probabilities in $(0, 1]$. The basic steps in a quantum automaton are to start in the state $\langle s_{\mbox{init}}$, apply the unitary matrices $U_{w_i}$ for each symbol in $w$ in the given order, and measure the probability of the final outcome that the resulting state is in $H_{\mbox{accept}}$ by applying the projection operator $P_{\mbox{accept}}$, and measuring its resulting norm (which gives us the probability). 

To relate it to the discussion of quantum entanglement in the previous section, note that if we have a quantum system that is initially prepared to be a superposition of initial states, and we present it the input string, the matrix product $U_w$ in effect sums over all possible paths that the machine can take (viewed as a stochastic automaton). Each path has a complex amplitude that is given by the product of all the amplitudes of the transitions at each step. The biggest difference from stochastic automata used in generative AI and reinforcement learning is that quantum entanglement will cause destructive interference exactly as illustrated in Figure~\ref{twoslit}. Paths that have opposite phases in the complex plane can cancel each other out, resulting in a total probability less than the sum of each path, since the amplitude $|a + b|^2 \leq |a|^2 + |b|^2$. 

\begin{definition}
    A {\bf quantum finite state automaton} (QFA) is a quantum automaton where ${\cal H}$, $s_{\mbox{init}}$, and the $U_a$ have finite dimensionality $n$. A {\bf quantum regular language}(QRL) is a quantum language recognized by a QFA. 
\end{definition}

\cite{MOORE2000275} go on to define quantum analogues of context-free grammars, push-down automata, etc. and explore their properties, to which the reader is referred to for a more in-depth discussion. By suitably defining the category, e.g. Hilbert spaces, and an endofunctor on the category, it is clear that we can define a coalgebra for a quantum finite state automaton (as we can as well for the other generative models in \cite{MOORE2000275}). There has been initial work on quantum generative AI systems, such as quantum transformers \cite{cherrat2022quantum} as well as the work on quantum NLP mentioned earlier \cite{DBLP:phd/ethos/Toumi22}. Fundamentally, to solve a UIG using a quantum computer, we have to define the notion of (weak) isomorphism in each of the cases that we discussed before. We briefly discuss each of these cases. 

\subsection{Quantum Universal Imitation Games}

We proposed using $\yo(x)$ Yoneda embeddings for a participant in a static UIG and determine if two participants are isomorphic based on their Yoneda embeddings. We can generalize the same approach to the quantum computing world by determining quantum Yoneda embeddings in a compact closed category. Following \cite{KELLY1980193}, we can define compact closed categories as a symmetric monoidal category where every object $A$ has a {\em dual object}, given by $A^*$, such that it defines a {\em unit} 

\[ \eta_A: {\bf 1} \rightarrow A^* \otimes A \]

and a {\em counit} 

\[ \epsilon_A: A \otimes A^* \rightarrow {\bf 1} \]

satisfying certain coherence  equations given in \cite{KELLY1980193}. 

\begin{definition}
Two objects $X$ and $Y$ in a quantum UIG defined as a compact closed category ${\cal C}$ are deemed {\bf {isomorphic}}, or $X \cong Y$ if and only if there is an invertible morphism $f: X \rightarrow Y$, namely $f$ is both {\em left invertible} using a morphism $g: Y \rightarrow X$ so that $g \circ f = $ {{\bf id}}$_X$, and $f$ is {\em right invertible} using a morphism $h$ where $f \circ h = $ {{\bf id}}$_Y$. 
\end{definition}

What is fundamentally different in the quantum computing world is that computation is now taking place in a compact closed category where objects have duals. As \cite{KELLY1980193} define it, a compact closed category is simply a symmetric monoidal category with symmetry isomorphisms $A \otimes B \simeq B \otimes A$ in which every object $A$ has a left adjoint, and they work out in detail exactly what form this left adjoint takes. An interesting problem for further work is to understand the full consequences of operating in such compact closed categories. We discuss a few sample questions to be explored in further research: 

\begin{itemize}
    \item {\cal Causal Inference in Compact Closed Categories}: Consider the problem of determining whether two participants in a quantum UIG are isomorphic by doing a series of causal interventions. Although there has been substantial work on causal inference in traditional categories, such as graphs, vector spaces, Hilbert spaces, etc. (see Table~\ref{causalcats}), we are not aware of any work in causal inference in compact closed categories where objects have duals. This aspect of quantum computing raises many interesting questions that need to be explored. It is worth mentioning as well that causal inference has been explored in the quantum computing world \cite{coecke2016categorical,Javidian_2022}. 

    \item {\cal Quantum Machine Learning}: There has been substantial work on quantum machine learning \cite{DBLP:books/sp/Pastorello23}, including a celebrated algorithm for solving systems of linear equations $A x = b$ \cite{Harrow_2009}. These approaches look very promising due their potentially exponentially lower computational costs, although the broader question of whether there are problems that are provably solvable faster by quantum computers remains an open research question \cite{quantumcomplexity}.   To that end, \cite{bshouty} propose an extension of the classic PAC learning framework to the quantum computing world, where they show that the class of all disjunctive normal form (DNF) concepts are efficiently learnable with respect to the uniform distribution by a quantum algorithm using a quantum oracle. 

    \item {\cal Quantum Language Models}: As we mentioned above, \cite{DBLP:phd/ethos/Toumi22} explores the use of quantum computing in natural language processing, and others \cite{Bradley_2022} have explored improved approaches for representing large language models using quantum computers. 
\end{itemize}

\section{Summary} 

In this paper, we proposed a broad framework for universal imitation games (UIGs), extending  Alan Turing's original proposal from 1950 where  participants are to be classified {\tt Human or Machine} solely from  natural language interactions. In our framework, we allow interactions to be arbitrary, and not limited to natural language questions and answers. We also bring to bear mathematics largely developed since Turing -- category theory -- that involves defining a collection of objects and a collection of composable arrows between each pair of objects that represent ``measurement probes"  for solving UIGs. We built on on two celebrated results by Yoneda.  The first, called the Yoneda Lemma, shows that objects in categories can be identified up to isomorphism solely with measurement probes defined by composable arrows. Yoneda embeddings are universal representers of objects in categories. A simple yet general solution to the static UIG problem, where the participants are not changing during the interactions,  is to determine if the Yoneda embeddings are (weakly) isomorphic. In this approach, we are not considering computational constraints, but mainly focusing on formulating the problem in a mathematically general way. The second foundational result of Yoneda from 1960 that we use is an abstract integral calculus defined by coends and ends, which ``integrate" out the objects of a category. We illustrated how these unify disparate notions in AI and ML, from distance based approaches to probabilistic generative models.   When participants adapt during interactions, we investigated two special cases: in {\em dynamic UIGs}, ``learners" imitate ``teachers".  We contrasted the initial object  framework of {\em passive learning from observation} over well-founded sets using inductive inference -- extensively studied by Gold, Solomonoff, Valiant, and Vapnik -- with the final object framework of {\em coinductive inference} over non-well-founded sets and universal coalgebras, which formalizes  learning from {\em active experimentation} using causal inference or reinforcement learning. We defined a category-theoretic notion of minimum description length or Occam's razor based on final objects in coalgebras.   Finally, we explored {\em evolutionary UIGs}, where a society of participants is playing a large-scale imitation game. Participants in evolutionary UIGs can go extinct from ``birth-death" evolutionary processes that model how novel technologies or ``mutants" disrupt previously stable equilibria. We ended the paper with a brief discussion of how to play UIGs on quantum computers, a future that looks increasibly likely given the staggering costs of playing imitation games on classical computers. Many directions remain to be pursued in UIGs, and throughout the paper, we sketched out a number of possible avenues that need further study. This paper represents an initial exploration of a fascinating problem, one that will have enormous implications for the future of human society as millions of humans begin to engage in an imitation game with generative AI products. 

\newpage 


\end{document}